\documentclass{article}
\usepackage[utf8]{inputenc}
\usepackage[T1]{fontenc}
\usepackage{nicefrac}
\usepackage{algorithmic}
\usepackage[ruled,vlined]{algorithm2e}
\usepackage{amsmath,amsfonts,amssymb}
\usepackage{amsthm}
\usepackage{booktabs}
\usepackage{array}
\usepackage{graphicx}
\usepackage{xcolor}
\usepackage{tikz}
\usepackage{hyperref}
\usepackage{microtype}
\usepackage{natbib}
\usepackage{subcaption}
\usepackage{placeins}
\usepackage{float}
\usepackage{enumitem}
\usepackage{bm}
\newcolumntype{C}{>{\centering\arraybackslash}m{0.24\linewidth}}
\newcolumntype{E}{>{\centering\arraybackslash}m{0.33\linewidth}}

\newcommand{\axisspineplot}[2][]{%
	\begin{tikzpicture}
		\node[anchor=south west, inner sep=0pt, outer sep=0pt] (img) at (0,0) {\includegraphics[#1]{#2}};
		\begin{scope}[x={(img.south east)}, y={(img.north west)}]
			\draw[line width=0.35pt, black!35] (0.075,0.129) rectangle (0.986,0.976);
			\draw[line width=0.9pt, black] (0.075,0.129) -- (0.075,0.976);
			\draw[line width=0.9pt, black] (0.075,0.129) -- (0.986,0.129);
		\end{scope}
	\end{tikzpicture}%
}
\newcommand{\bfsqualpanel}[1]{%
	\IfFileExists{#1}{\includegraphics[width=\linewidth]{#1}}{%
		\fbox{\begin{minipage}[c][0.045\textheight][c]{0.94\linewidth}
				\centering\scriptsize Pending panel
	\end{minipage}}}%
}
\newcommand{\bfsqualheader}{%
	\multicolumn{1}{c}{\small\textbf{Ground truth}} & \multicolumn{1}{c}{\small\textbf{Base}} & \multicolumn{1}{c}{\small\textbf{Chebyshev}} & \multicolumn{1}{c}{\small\textbf{GLEAM}}\\[0.15em]
}
\newcommand{\bfsqualblock}[3]{%
	\bfsqualpanel{#2_ground_truth.png} & \bfsqualpanel{#2_base.png} & \bfsqualpanel{#2_chebyshev.png} & \bfsqualpanel{#2_gleam.png}\\[-0.35em]
	{\small\textbf{\(#1\)}} & \bfsqualpanel{#2_base_error.png} & \bfsqualpanel{#2_chebyshev_error.png} & \bfsqualpanel{#2_gleam_error.png}\\[#3]
}
\newcommand{\eaglequalheader}{%
	\multicolumn{1}{c}{\small\textbf{Ground truth}} & \multicolumn{1}{c}{\small\textbf{Base}} & \multicolumn{1}{c}{\small\textbf{Chebyshev}}\\[0.15em]
}
\newcommand{\eaglequalblock}[3]{%
	\includegraphics[width=\linewidth,keepaspectratio]{#2_ground_truth.png} & \includegraphics[width=\linewidth,keepaspectratio]{#2_base.png} & \includegraphics[width=\linewidth,keepaspectratio]{#2_chebyshev.png}\\[0.08em]
	{\small\textbf{\(#1\)}} & \includegraphics[width=\linewidth,keepaspectratio]{#2_base_error.png} & \includegraphics[width=\linewidth,keepaspectratio]{#2_chebyshev_error.png}\\[#3]
}

\usepackage{PRIMEarxiv}

\newcommand{\R}{\mathbb{R}}
\newcommand{\norm}[1]{\left\lVert #1 \right\rVert}
\newcommand{\paren}[1]{\left( #1 \right)}
\newcommand{\diag}{\operatorname{diag}}
\newtheorem{proposition}{Proposition}

\newcommand{\del}[1]{}
\newcommand{\sep}{\enspace\textbullet\enspace}

\newcommand{\PrintCredit}{}
\newenvironment{coi}{}{}
\newenvironment{ack}{}{}

\title{Scale-Aware Learning of Chaotic Dynamics on Unstructured Meshes via Binned Spectral Losses}
\author{
	Kanad Sen$^{1}$, Romit Maulik$^{1}$\\[0.5em]
	\normalfont $^{1}$ School of Mechanical Engineering, Purdue University, West Lafayette, USA
}
\date{}

\begin{document}
	\raggedbottom
	
	\maketitle
	
	\begin{abstract}
		Surrogate modeling for high-dimensional nonlinear dynamical systems that exhibit chaos requires mechanisms that preserve not only pointwise accuracy but also the scale-dependent structure of physical fields. Bandwise spectral power losses, such as the binned spectral loss function, provide such supervision on structured grids, where Fourier modes define a standard frequency decomposition. On irregular meshes, however, no canonical Fourier basis exists, and spectral representations must be constructed from graph operators induced by mesh connectivity and geometry. In this study, we extend the binned spectral power loss for application to unstructured-mesh surrogate modeling of nonlinear dynamical systems. This is obtained by replacing Fourier bands with graph-Laplacian frequency bands, and we provide scalable Chebyshev and multilevel approximations for improving long-horizon rollout fidelity. In its full-spectrum form, our approach uses graph Laplacian eigenspaces to provide a graph analogue of Fourier band-power matching, but incurs the high cost of spectral decomposition. As a scalable approximation, we replace exact band projectors with sparse Chebyshev polynomial graph filters, avoiding explicit eigendecomposition. When utilizing multilevel graph architectures, we introduce Graph Laplacian Energy Alignment for Meshes (GLEAM), which applies retained-subspace scale-aware supervision across graph hierarchies so that coarse and fine representations are regularized during autoregressive rollout. Our results show that the proposed spectral losses improve long-horizon rollout fidelity and preserve statistical invariants for the forecasting of turbulent flows on unstructured meshes, compared to deterministic baselines.
	\end{abstract}
	
		
		
		
	
	
	\section{Introduction}\label{Xsec1-1}
	
	High-fidelity numerical simulation remains one of the central tools for studying complex nonlinear dynamical systems governed by partial differential equations. In realistic computational fluid dynamics (CFD), however, resolving complex geometry, boundary layers, multiscale vortical structures, and long-time transients makes repeated simulation expensive for design, control, uncertainty quantification, and inverse problems. Surrogate modeling addresses this bottleneck by learning an emulator of either the numerical solver or the underlying solution map, so that new parameter settings or future states can be evaluated at a fraction of the original cost. Modern PDE surrogates follow several routes: neural operators learn field-to-field maps that approximate solution operators, while graph- and mesh-based simulators learn updates directly on discretized physical states \citep{li2020gno,li2021fno,kovachki2023neuraloperator,pfaff2020meshgraphnets}. This paper focuses on the graph-based setting, where the forecast is applied repeatedly on mesh-resolved states. The central difficulty for time-dependent flows is that low one-step error is not sufficient: when the model is used autoregressively, small local mistakes can change the distribution of energy across scales and grow into long-horizon drift of invariant statistics.
	
	This issue is especially important on unstructured CFD meshes. Production CFD commonly relies on unstructured or hybrid meshes to represent curved boundaries, boundary-layer refinement, complex topologies, local adaptivity, and moving or sample-dependent geometries \citep{mavriplis1997unstructured}. A mesh can be interpreted as a graph whose nodes carry physical variables and whose edges encode adjacency, stencil connectivity, or geometric neighborhoods. This makes graph neural networks a well-suited modeling choice: message passing respects mesh topology, node counts may vary across samples, and local interactions can be learned without embedding the field onto a global Cartesian tensor product grid. Recent surveys now identify unstructured-grid learning as a growing area across computational physics, with graph and mesh-based models becoming common for CFD, mechanics, and geophysical applications \citep{cheng2025unstructuredReview}.
	
	Classical stability analysis on unstructured meshes is already a mature topic for numerical solvers. Finite-element stabilization, finite-volume reconstruction analysis, discontinuous Galerkin entropy or energy arguments, CFL restrictions, and matrix or eigenvalue diagnostics have all been used to understand when discretizations remain stable on irregular grids \citep{brooks1982supg,zangeneh2019reconstruction}. Similar questions arise in geophysical models, where the stability and accuracy of unstructured horizontal discretizations are studied directly for ocean and climate solvers \citep{lapolli2024accuracy}. These works address the stability of the discretization, reconstruction, and time integrator. Learned surrogates introduce a different gap: even if the data-generating solver is stable, a neural autoregressive model can leave the training distribution, amplify its own errors, or dissipate physically important scales during rollout.
	
	Several strategies have been proposed to improve learned rollout stability. MeshGraphNets use noise-augmented training to make graph simulators robust to off-manifold inputs \citep{pfaff2020meshgraphnets}; message-passing neural PDE solvers use temporal bundling and pushforward-style training to expose the model to its own rolled-out states \citep{brandstetter2022message}; adversarial noise injection has been explored for learned turbulence rollouts \citep{su2022adversarial}; and recurrent neural operators explicitly train neural operators in closed loop to reduce the teacher-forcing mismatch between training and inference \citep{ye2025rno}. Other approaches modify the operator architecture or its time-stepping structure: stabilized neural operators reduce autoregressive error growth through architecture-level constraints \citep{mccabe2023stability}, spherical Fourier neural operators exploit the correct spectral geometry for global atmospheric dynamics \citep{bonev2023sfno}, and recent spectral-generator neural operators constrain spectral amplification in Fourier space for long-horizon PDE rollouts \citep{li2026sgno}. Dual-scale neural operators target long-term fluid forecasting by separating local detail evolution from global trend propagation \citep{dong2026dso}. These methods are valuable, but most are tied to a particular architecture, a structured or special spectral geometry, a closed-loop training protocol, or Euclidean Fourier-domain assumptions. Some graph methods do operate on irregular meshes, including mesh simulators for fluids, ocean surrogates, and mesh transformers for turbulent flow \citep{shi2022gnnsurrogate,janny2023eagle}; however, to our knowledge, they do not use bandwise graph-spectral energy agreement as a modular supervision objective for the unstructured CFD settings considered here.
	
	A key reason such a loss is needed is spectral bias. Standard neural networks often learn smooth, low-frequency components more readily than high-frequency components \citep{rahaman2019spectralbias}. In fluid prediction, large coherent structures dominate pointwise losses such as mean squared error, while small-scale fluctuations, shear-layer instabilities, and high-wavenumber content can be underweighted despite being crucial for long-horizon evolution. The problem is amplified on graphs: many graph convolution and message-passing mechanisms have a low-pass or smoothing interpretation, and repeated propagation can suppress high graph-frequency content or lead to oversmoothing \citep{li2018deeper,nt2019lowpass,balcilar2021spectral}. Thus, graph surrogates on CFD meshes face a double scale problem: the physics is multiscale, and the graph-learning machinery itself may favor smoother graph signals unless training explicitly penalizes scale-wise energy drift.
	
	Structured-grid methods have a standard way to expose this failure mode: Fourier analysis. For periodic or regularly sampled domains, Fourier modes separate a field into wavenumber components, making it possible to compare the energy carried by coarse and fine spatial scales. Binned spectral power (BSP) loss uses this idea by matching predicted and target spectral power across Fourier bands, improving chaotic-system prediction by correcting scale-wise energy imbalances that pointwise losses do not explicitly measure \citep{chakraborty2026bsp}. Its limitation is that the original formulation assumes a canonical Fourier basis and a meaningful Euclidean wavenumber grid. On an unstructured CFD mesh, neither assumption is generally available: nodes are irregularly distributed, connectivity depends on the discretization, samples may have different meshes, and there is no global translation symmetry from which Fourier modes arise. Although there exists  methods like PyNUFFT \citep{lin2017pynufft} , which accelerates non-uniform fast Fourier transforms for unstructured data , it still leverages interpolation onto structured grids thereby potentially inducing artifacts during the spectral-domain loss calculation.
	
	Our contribution is to move this band-power idea from Euclidean Fourier space to graph spectral space. In spectral graph theory and graph signal processing, graph Laplacian eigenvectors form an orthonormal basis for graph signals, and the corresponding eigenvalues order modes by graph Dirichlet energy \citep{chung1997spectral,shuman2013signal}. Low eigenvalues represent slowly varying graph signals, while larger eigenvalues represent increasingly oscillatory variation over the mesh. This provides a principled graph analogue of Fourier scale separation. Graph wavelets and polynomial graph filters build on the same view by defining filters as functions of the Laplacian, allowing frequency-selective operations without relying on a Cartesian grid \citep{hammond2011wavelets,defferrard2016convolutional,shuman2018chebyshev}.
	
	This paper develops a graph-spectral loss framework for unstructured-mesh surrogate modeling. We extend binned spectral power supervision to unstructured meshes by replacing Fourier bands with graph-Laplacian frequency bands and comparing the predicted and target band energies. The exact graph formulation provides one mathematical analogue of Fourier band-power matching, but requires spectral decomposition and dense graph Fourier projection. We therefore organize the method as a cost--fidelity hierarchy: exact Graph BSP is the reference objective which serves for theoretical purposes only, Chebyshev BSP approximates band projectors using sparse polynomial Laplacian filters, and Graph Laplacian Energy Alignment for Meshes (GLEAM) extends retained-subspace scale-aware supervision to multilevel graph representations. The advantage over existing rollout-stability methods is that the proposed objectives are designed to reduce spectral drift and improve scale-resolved fidelity on unstructured meshes where no regular Fourier grid exists. Chebyshev BSP can be added without changing the main forecasting backbone. GLEAM preserves the same fine-level prediction interface, while adding auxiliary multilevel normalization and prediction heads so that spectral supervision can be applied on the graph hierarchy during training. In both cases, the forecasting model still learns the dynamics, while the graph-spectral objective encourages mesh-resolved scales to remain accurate during long rollouts.
	
	\textbf{Contributions :}
	\begin{itemize}
		\item We generalize structured-grid band-power matching to unstructured meshes by defining graph-spectral band-power losses based on Laplacian eigenspaces and graph-frequency binning.
		\item We organize the proposed losses into a cost--fidelity hierarchy: exact Graph BSP as the full-spectrum reference, Chebyshev BSP as a scalable sparse approximation, and GLEAM as a hierarchy-aware graph-spectral supervision strategy.
		\item We formulate Chebyshev BSP as an auxiliary loss that leaves the forecasting backbone unchanged, and GLEAM as a hierarchy-aware variant that preserves the fine-level rollout interface while adding auxiliary multilevel heads for training-time supervision.
		\item We evaluate the resulting scale-aware supervision on unstructured-mesh flow datasets and show that graph-spectral energy alignment improves long-horizon rollout fidelity, distributional fidelity, and preservation of coherent flow structure.
	\end{itemize}
	\section{Mathematical Formulation of Spectral Losses for Unstructured Grids}
	The original binned spectral power formulation is defined on structured grids, where translation invariance permits a canonical Fourier decomposition and a standard notion of spectral power as a function of wavenumber \citep{chakraborty2026bsp}. On an unstructured mesh, however, there is no global coordinate system in which Fourier modes are canonically defined. We therefore replace the Fourier basis by a graph-spectral basis induced by the mesh itself. In spectral graph theory and graph signal processing, Laplacian eigenvectors play the role of generalized harmonics on irregular domains, while the associated eigenvalues provide an ordering from smooth, large-scale modes to oscillatory, small-scale modes \citep{chung1997spectral,shuman2013signal}. Related ideas have already been used in graph-spectral objectives and regularizers, including Laplacian smoothness penalties, graph Laplacian embedding losses, spectral clustering losses, and sparse spectral-domain constraints \citep{tong2020graphspectral,cheng2018structured,zhang2021spectral,humbert2021learning}. Our contribution is to adapt that viewpoint to a band-power matching objective tailored to unstructured CFD meshes.
	
	The formulation below begins with a graph-spectral extension of binned spectral power loss, so that the objective remains defined in terms of \emph{bandwise spectral power} rather than pointwise field mismatch \citep{chakraborty2026bsp}. We first define exact Graph BSP, then introduce two scalable forms: Chebyshev BSP for eigenfree band-energy matching on the fine graph, and GLEAM for hierarchy-aware supervision in multilevel graph backbones.
	
	{
		\paragraph{Notation convention.}
		Throughout the paper, a hat denotes a model prediction, as in \(\widehat{x}\), \(\widehat{\bm u}\), and \(\widehat{\bm F}_p\). Graph Fourier coefficients are denoted by a superscript \(\mathcal{F}\). We use \(\mathcal{E}\) for graph edge sets, while \(E_{\cdot}\) is reserved for modal or bandwise spectral-energy quantities.
		\begin{center}
			\begin{tabular}{@{}p{0.28\linewidth}p{0.66\linewidth}@{}}
				\toprule
				Symbol & Meaning \\
				\midrule
				\(G=(V,\mathcal{E},W)\) & weighted mesh graph with node set \(V\), edge set \(\mathcal{E}\), and weights \(W\) \\
				\(\bm u,\bm v\) & generic prediction and target fields in the graph-spectral loss formulation \\
				\(\widehat{\bm u},\widehat{x},\widehat{\bm F}_p\) & predicted rollout fields or diagnostics; hats indicate model predictions \\
				\(\bm u^{\mathcal{F}},\bm v^{\mathcal{F}}\) & graph Fourier coefficient matrices \(\Phi^\top\bm u\) and \(\Phi^\top\bm v\) \\
				\(E_u(k,c), E_u^{\mathrm{bin}}(m,c)\) & modal and bandwise spectral-energy quantities, not graph edge sets \\
				\bottomrule
			\end{tabular}
		\end{center}
	}
	
	\subsection{Graph representation of an unstructured mesh}
	Consider one mesh sample with node set \(V=\{1,\dots,N\}\), edge set \(\mathcal{E}\subseteq V\times V\), and node coordinates \(\bm{p}_i\in\R^d\). We denote the corresponding weighted undirected graph by
	\begin{equation}
		G=(V,\mathcal{E},W), \qquad |V|=N,
	\end{equation}
	where \(W=[w_{ij}]\) is a symmetric edge-weight matrix. Each node corresponds to a mesh point, control-volume center, finite-element degree of freedom, or any equivalent spatial discretization unit. Let
	\begin{equation}
		\bm{u}\in\R^{N\times C},
		\qquad
		\bm{v}\in\R^{N\times C},
	\end{equation}
	denote the predicted and target fields, respectively, with \(C\) channels.

	The graph connectivity may be inherited directly from the computational mesh, from a cell-neighbor stencil, or from a geometric \(k\)-nearest-neighbor construction when only point clouds are available. We assume throughout that the graph is connected; if multiple connected components arise, the formulation can be applied componentwise. A practical choice of edge weights is
	\begin{equation}
		w_{ij}=
		\begin{cases}
			\dfrac{1}{\norm{\bm{p}_i-\bm{p}_j}_2+\varepsilon_w}, & (i,j)\in \mathcal{E},\\[0.8em]
			0, & (i,j)\notin \mathcal{E},
		\end{cases}
	\end{equation}
	where \(\varepsilon_w>0\) prevents singular weights for nearly coincident nodes. More generally, any positive symmetric kernel that decays with distance or reflects mesh adjacency may be used; the key requirement is that the resulting Laplacian encode the geometry and local coupling structure of the mesh.
	
	Let \(A\in\R^{N\times N}\) denote the weighted adjacency matrix and let the degree matrix \(D\) be defined by
	\begin{equation}
		D_{ii}=\sum_{j=1}^N A_{ij}.
	\end{equation}
	We use the symmetric normalized Laplacian
	\begin{equation}
		L=I-D^{-1/2}AD^{-1/2}.
		\label{eq:laplacian}
	\end{equation}
	This choice is convenient because its spectrum is real, nonnegative, and bounded, while its eigenvectors remain orthonormal in the Euclidean inner product. The eigendecomposition of \(L\) is
	\begin{equation}
		L=\Phi\Lambda\Phi^\top,
		\qquad
		\Lambda=\diag(\lambda_1,\dots,\lambda_N),
		\label{eq:eig}
	\end{equation}
	with
	\[
	0=\lambda_1\le \lambda_2\le \cdots\le \lambda_N,
	\qquad
	\Phi=[\phi_1,\dots,\phi_N],
	\qquad
	\Phi^\top\Phi=I.
	\]
	The eigenvectors \(\phi_k\) are the graph analogue of Fourier modes, while the eigenvalues \(\lambda_k\) act as graph frequencies in the graph-signal-processing sense \citep{chung1997spectral,shuman2013signal}. Low eigenvalues correspond to smooth variations over the mesh, whereas larger eigenvalues correspond to increasingly oscillatory structure localized on shorter graph scales.
	
	\subsection{Graph Fourier representation and spectral energy}
	For each channel \(c\in\{1,\dots,C\}\), define the graph Fourier coefficients of the prediction and target using the superscript-\(\mathcal{F}\) notation
	\begin{equation}
		u^{\mathcal{F}}_{k,c}=\phi_k^\top\bm{u}_{:,c},
		\qquad
		v^{\mathcal{F}}_{k,c}=\phi_k^\top\bm{v}_{:,c}.
		\label{eq:spectral_coeffs}
	\end{equation}
	Equivalently, if \(\bm{u}^{\mathcal{F}}=\Phi^\top\bm{u}\) and \(\bm{v}^{\mathcal{F}}=\Phi^\top\bm{v}\), then the \(k\)-th row of these transformed fields collects the coefficients associated with the \(k\)-th graph mode across all channels. Because \(\Phi\) is orthonormal, Parseval's identity holds:
	\begin{equation}
		\norm{\bm{u}}_F^2=\norm{\bm{u}^{\mathcal{F}}}_F^2,
		\qquad
		\norm{\bm{v}}_F^2=\norm{\bm{v}^{\mathcal{F}}}_F^2.
		\label{eq:parseval_graph}
	\end{equation}
	This identity ensures that energy measured in physical space can be decomposed consistently into graph-spectral contributions.
	
	For each mode and channel, we define the modal energy
	\begin{equation}
		E_u(k,c)=\frac{1}{2}\left(u^{\mathcal{F}}_{k,c}\right)^2,
		\qquad
		E_v(k,c)=\frac{1}{2}\left(v^{\mathcal{F}}_{k,c}\right)^2.
		\label{eq:modal_energy}
	\end{equation}
	The factor \(1/2\) is inherited from standard power-spectrum conventions and is immaterial up to a constant scaling, but it makes the graph construction analogous to the structured-grid binned spectral power formulation.
	
	\subsection{Bandwise graph-spectral decomposition}
	Exact matching of every individual graph mode is usually neither necessary nor desirable. For graph Laplacians with repeated or clustered eigenvalues, the individual eigenvectors are not uniquely determined; under perturbations, the stable object is the associated invariant subspace rather than a particular eigenvector basis \citep{davis1970rotation,stewart1990matrix}. This issue is especially relevant on irregular meshes, where mild perturbations of node locations or edge weights can rotate eigenvectors within nearly degenerate eigenspaces even when the underlying scale content is essentially unchanged \citep{meyer2012perturbation,huang2024stability}. Following the central philosophy of binned spectral power, we therefore compare the aggregate spectral energy within graph-frequency bands rather than matching individual graph-Fourier coefficients, consistent with graph-signal-processing constructions based on spectral filters or windows \citep{shuman2013signal,hammond2011wavelets}.
	
	Let the nontrivial graph frequencies be partitioned into \(M\) disjoint bins
	\begin{equation}
		\mathcal{B}_1,\dots,\mathcal{B}_M,
		\qquad
		\bigcup_{m=1}^M\mathcal{B}_m \subseteq \{2,\dots,N\},
		\qquad
		\mathcal{B}_m\cap\mathcal{B}_{m'}=\varnothing \ \text{for } m\neq m'.
	\end{equation}
	The trivial mode \(k=1\) is typically excluded because it corresponds to the constant graph signal and does not encode spatial variation. We bin in the Laplacian eigenvalue coordinate because \(\lambda_k\) is the intrinsic scalar graph-frequency variable: for an eigenvector \(\phi_k\), the Rayleigh quotient gives \(\phi_k^\top L\phi_k=\lambda_k\), so \(\lambda_k\) measures graph Dirichlet energy and orders modes by spatial variation on the graph \citep{chung1997spectral,shuman2013signal}. This is also the coordinate in which graph spectral filters and wavelets are defined, through functions \(g(L)=\Phi g(\Lambda)\Phi^\top\) or kernels \(g(t\lambda)\) \citep{hammond2011wavelets,shuman2013signal,defferrard2016convolutional}. Radial Fourier shells are standard on structured grids because each mode has a Euclidean wave vector; on an unstructured graph there is no canonical wave-vector direction or radius beyond the Laplacian spectrum itself. We therefore use linear eigenvalue bins as the default operator-defined choice: they partition the normalized graph-frequency range uniformly, match the interval structure used by polynomial graph filters, and keep the exact, low-rank, and Chebyshev-filtered constructions on the same spectral intervals.
	\begin{equation}
		0=\eta_0<\eta_1<\dots<\eta_M=\lambda_N,
		\qquad
		\eta_m=\frac{m}{M}\lambda_N,
		\label{eq:exact_bin_edges}
	\end{equation}
	The graph-spectral bins are then defined by
	\begin{equation}
		\mathcal{B}_m=\{k\in\{2,\dots,N\}\,:\,\eta_{m-1}\le \lambda_k<\eta_m\},
		\quad m<M,
		\qquad
		\mathcal{B}_M=\{k\in\{2,\dots,N\}\,:\,\eta_{M-1}\le \lambda_k\le\eta_M\}.
		\label{eq:exact_bin_sets}
	\end{equation}
	Empty bins are omitted from the corresponding averages in practice.
	
	The orthogonal projector onto band \(m\) is
	\begin{equation}
		\Pi_m=\sum_{k\in\mathcal{B}_m}\phi_k\phi_k^\top.
		\label{eq:band_projector}
	\end{equation}
	Then the bandwise averaged energies can be written either modewise or projectorwise:
	\begin{equation}
		E_u^{\mathrm{bin}}(m,c)=\frac{1}{|\mathcal{B}_m|}\sum_{k\in\mathcal{B}_m} E_u(k,c)
		=\frac{1}{2|\mathcal{B}_m|}\norm{\Pi_m\bm{u}_{:,c}}_2^2,
		\label{eq:bin_energy_u}
	\end{equation}
	\begin{equation}
		E_v^{\mathrm{bin}}(m,c)=\frac{1}{|\mathcal{B}_m|}\sum_{k\in\mathcal{B}_m} E_v(k,c)
		=\frac{1}{2|\mathcal{B}_m|}\norm{\Pi_m\bm{v}_{:,c}}_2^2.
		\label{eq:bin_energy_v}
	\end{equation}
	These quantities measure how much of the channel energy is allocated to a prescribed graph-frequency range.
	
	
	\subsection{Exact graph analogue of band-power loss}
	We now define the full-rank graph-spectral loss. The graph analogue of the ratio-based band-power objective is
	\begin{equation}
		\mathcal{L}_{\mathrm{band}}^{\mathrm{rel}}(\bm{u},\bm{v})
		=\frac{1}{MC}
		\sum_{m=1}^M\sum_{c=1}^C
		\left(
		1-\frac{E_u^{\mathrm{bin}}(m,c)+\varepsilon_s}{E_v^{\mathrm{bin}}(m,c)+\varepsilon_s}
		\right)^2,
		\label{eq:bsp_rel}
	\end{equation}
	where \(\varepsilon_s>0\) is a small stabilization constant. This objective penalizes relative under- or over-allocation of energy in each graph-frequency band.
	
	\section{Scalable Approximations to Graph Band-Power Loss}
	The exact graph-spectral loss defined in \eqref{eq:bsp_rel} is the full-spectrum analogue of the structured-grid band-power objective, but its computational cost becomes prohibitive once the mesh is large or the supervision is applied repeatedly within a multiscale graph backbone. If the eigensystem is not already available, exact Graph BSP requires an eigendecomposition of the graph Laplacian, with leading cost \(\mathcal{O}(N^3)\) for dense solvers, followed by graph Fourier projection and band aggregation with cost \(\mathcal{O}(N^2C)\). Even when the eigensystem is precomputed on a fixed graph, the online projection cost and \(\mathcal{O}(N^2)\) storage of the full basis can become significant; across a hierarchy, the analogous exact cost scales like \(\sum_{\ell=1}^{S}\mathcal{O}(N_\ell^3+N_\ell^2C)\) if eigensystems are computed at each level.
	
	These costs make the exact formulation most useful as a reference objective, and motivate approximations that preserve its band-power interpretation without requiring full graph Fourier transforms at every rollout step. We use two such approximations. The Chebyshev variant replaces explicit spectral projectors with polynomial graph filters, so each band is approximated by repeated sparse applications of the Laplacian rather than by storing and multiplying the full eigenbasis \citep{shuman2018chebyshev}. This keeps the supervision tied to graph-frequency bands while making it compatible with large sparse meshes. The GLEAM variant addresses a complementary scaling issue: instead of approximating every band over the full spectrum on the finest graph, it applies reduced spectral supervision across the graph hierarchy, emphasizing dominant coarse modes and their geometric organization. In short, Chebyshev BSP is an eigenfree surrogate for bandwise graph filtering, while GLEAM is a hierarchy-aware spectral supervision strategy for multilevel settings. A compact cost comparison is given in Section~\ref{sec:cost-comparison}, with full complexity details in Appendix~\ref{app:complexity}.
	
	\subsection{Chebyshev filter-bank approximation}
	\label{sec:chebyshev-bsp}
	Exact eigendecomposition is often prohibitive for large meshes. A scalable alternative is to replace explicit spectral projection by graph filter banks implemented as matrix polynomials. The exact Graph BSP loss uses the projector \(\Pi_m\) in \eqref{eq:band_projector} to isolate the eigenvectors whose eigenvalues lie in band \(m\). Equivalently, if
	\[
	b_m(\lambda)=\mathbf{1}\{\eta_{m-1}\le \lambda < \eta_m\},
	\]
	then the exact band operator can be written as \(\Pi_m=b_m(L)\). The Chebyshev variant keeps this same band-energy idea, but replaces the discontinuous indicator \(b_m\) and the eigendecomposition of \(L\) with a smooth spectral window and a polynomial graph filter.
	
	For the normalized graph Laplacian in \eqref{eq:laplacian}, the spectrum lies in \([0,2]\); the proof is given in Appendix~\ref{app:laplacian-bound}.
	
	Following the standard shifted-Chebyshev construction for graph filters \citep{shuman2013signal,shuman2018chebyshev}, we map the spectrum to \([-1,1]\) by
	\begin{equation}
		\tilde{L}=\frac{2}{\lambda_{\max}}L-I,
		\qquad \lambda_{\max}=2,
		\qquad \tilde{L}=L-I.
		\label{eq:cheb_rescaled_laplacian}
	\end{equation}
	This fixed choice of \(\lambda_{\max}=2\) avoids an additional graph-specific eigenvalue estimate while still mapping the full possible spectrum of \(L\) into the Chebyshev interval \([-1,1]\).
	
	To mimic the same BSP-style bands indexed by \(m=1,\dots,M\) in \eqref{eq:exact_bin_edges}--\eqref{eq:exact_bin_sets}, we define \(M\) target spectral windows \(\varphi_m(\lambda)\) over \([0,\lambda_{\max}]\). The role of \(\varphi_m\) is to act as a smooth analogue of the hard bin indicator \(b_m\). In the implementation used here, these are triangular, or Bartlett, windows, which have also been used as spectral-domain windows in vertex-frequency graph signal decompositions \citep{stankovic2021vertexfrequency}. They are centered at linearly spaced locations,
	\begin{equation}
		\mu_m=\frac{m-1}{M-1}\lambda_{\max},
		\qquad
		\Delta \lambda = \frac{\lambda_{\max}}{M-1},
	\end{equation}
	\begin{equation}
		\varphi_m(\lambda)=\max\paren{1-\frac{|\lambda-\mu_m|}{\Delta \lambda},\,0}.
		\label{eq:cheb_triangular_window}
	\end{equation}
	Thus, the Chebyshev filter \(H_m^{(K,Q)}\) is the scalable approximation to the exact projector \(\Pi_m\):
	\begin{equation}
		H_m^{(K,Q)}
		\approx
		\varphi_m(L)
		\approx
		\frac{1}{2}a_{m,0}I+\sum_{j=1}^{K}a_{m,j}T_j(\tilde{L}).
		\label{eq:cheb}
	\end{equation}
	Here \(K\) is the polynomial order, \(Q\) is the number of quadrature points used to estimate the coefficients, and \(T_j\) is the \(j\)-th Chebyshev polynomial. The coefficients \(a_{m,j}\) are computed from Chebyshev quadrature rather than chosen heuristically. With quadrature nodes
	\begin{equation}
		\vartheta_q=\frac{\pi(q+1/2)}{Q},
		\qquad
		x_q=\cos(\vartheta_q),
		\qquad
		\lambda_q=\frac{\lambda_{\max}}{2}(x_q+1),
	\end{equation}
	for \(q=0,\dots,Q-1\), we use
	\begin{equation}
		a_{m,0}=\frac{2}{Q}\sum_{q=0}^{Q-1} \varphi_m(\lambda_q),
		\qquad
		a_{m,j}=\frac{2}{Q}\sum_{q=0}^{Q-1} \varphi_m(\lambda_q)\cos(j\vartheta_q),
		\quad j\ge 1.
		\label{eq:cheb_coeffs}
	\end{equation}
	The filtered responses are evaluated by the standard recurrence
	\begin{equation}
		T_0(\tilde{L})\bm{x}=\bm{x},
		\qquad
		T_1(\tilde{L})\bm{x}=\tilde{L}\bm{x},
	\end{equation}
	\begin{equation}
		T_j(\tilde{L})\bm{x}=2\tilde{L}T_{j-1}(\tilde{L})\bm{x}-T_{j-2}(\tilde{L})\bm{x},
		\qquad j\ge 2.
		\label{eq:cheb_recurrence}
	\end{equation}
	
	For channel \(c\), define the filtered responses
	\begin{equation}
		\bm{z}_{m,c}^u=H_m^{(K,Q)}\bm{u}_{:,c},
		\qquad
		\bm{z}_{m,c}^v=H_m^{(K,Q)}\bm{v}_{:,c}.
	\end{equation}
	These responses are the Chebyshev analogues of \(\Pi_m\bm{u}_{:,c}\) and \(\Pi_m\bm{v}_{:,c}\) in the exact Graph BSP loss. The corresponding band-energy estimates are
	\begin{equation}
		E_u^{\mathrm{cheb}}(m,c)=\frac{1}{2d_m^{\mathrm{cheb}}}\norm{\bm{z}_{m,c}^u}_2^2,
		\qquad
		E_v^{\mathrm{cheb}}(m,c)=\frac{1}{2d_m^{\mathrm{cheb}}}\norm{\bm{z}_{m,c}^v}_2^2,
		\label{eq:filter_energy}
	\end{equation}
	where \(d_m^{\mathrm{cheb}}\) is a band normalizer. One choice is \(d_m^{\mathrm{cheb}}=\mathrm{tr}\!\left((H_m^{(K,Q)})^\top H_m^{(K,Q)}\right)\), which reduces to \(|\mathcal{B}_m|\) when \(H_m^{(K,Q)}\) is the exact projector \(\Pi_m\). In practice, \(d_m^{\mathrm{cheb}}\) can be estimated once for a fixed graph or replaced by \(N\) when only relative comparisons across prediction and target are needed.
	
	The Chebyshev BSP loss is then the same relative band-energy comparison as \eqref{eq:bsp_rel}, with \(E^{\mathrm{bin}}\) replaced by \(E^{\mathrm{cheb}}\):
	\begin{equation}
		\mathcal{L}_{\mathrm{cheb}}
		=\frac{1}{MC}
		\sum_{m=1}^M\sum_{c=1}^C
		\left(
		1-\frac{E_u^{\mathrm{cheb}}(m,c)+\varepsilon_s}{E_v^{\mathrm{cheb}}(m,c)+\varepsilon_s}
		\right)^2.
		\label{eq:cheb_loss}
	\end{equation}
	This makes the approximation path explicit:
	\[
	\Pi_m\bm{u}_{:,c}
	\quad\longrightarrow\quad
	H_m^{(K,Q)}\bm{u}_{:,c},
	\qquad
	E_u^{\mathrm{bin}}(m,c)
	\quad\longrightarrow\quad
	E_u^{\mathrm{cheb}}(m,c).
	\]
	Thus, Chebyshev BSP preserves the bandwise comparison principle of Graph BSP without requiring an explicit eigenbasis. Compared with a heuristic cosine filter bank, the quadrature-derived coefficients in \eqref{eq:cheb_coeffs} explicitly approximate chosen graph-frequency windows and therefore more closely follow the exact band-power construction.
	
	\subsection{GLEAM: low-rank hierarchy-aware spectral geometry}
	The motivation for GLEAM comes from two related observations. First, auxiliary supervision at intermediate representations can improve optimization by shortening gradient paths and regularizing hidden features \citep{lee2015deeply,elinas2022deep}. Second, multilevel and multiscale graph simulators use coarse graph levels to propagate long-range interactions more efficiently than flat message passing, a principle shared with classical multigrid methods and recent mesh-based physical-simulation GNNs \citep{brandt1977multigrid,briggs2000multigrid,fortunato2022multiscale,cao2023efficient}. GLEAM is designed for this setting. Applying a full Chebyshev filter bank at every hierarchy level would provide direct band control, but it repeats polynomial graph filtering over all supervised levels and bands. GLEAM instead supervises the retained low-frequency geometry that carries long-range organization through the hierarchy, using low-rank band-energy matching and Fiedler-guided pairwise contrast matching. This makes GLEAM a more natural and lower-cost route for multilevel spectral supervision, while Chebyshev BSP remains the more direct choice when detailed fine-level band agreement is required. The design is motivated by a common autoregressive failure mode in unstructured-flow surrogates: small local errors can accumulate into drift of coherent transport, recirculation, shear-layer organization, or pressure-loading patterns that remain visible in low- and mid-frequency graph components.
	
	Let
	\begin{equation}
		L\psi_k=\lambda_k\psi_k,
		\qquad k=2,\dots,r+1,
	\end{equation}
	denote the first \(r\) nontrivial eigenpairs of \(L\). Following low-rank spectral embeddings and effective-resistance approximations \citep{wang2020graspel}, define
	\begin{equation}
		U=
		\left[
		\frac{\psi_2}{\sqrt{\lambda_2+\tau}},
		\dots,
		\frac{\psi_{r+1}}{\sqrt{\lambda_{r+1}+\tau}}
		\right]\in\R^{N\times r},
		\label{eq:embedding}
	\end{equation}
	where \(\tau>0\) prevents pathological amplification when an eigenvalue is very small. The row \(U_i\in\R^r\) is the retained spectral coordinate of node \(i\), and it captures the coarse graph geometry used by GLEAM.
	This construction is chosen to match the graph-resistance embedding induced by the Laplacian pseudoinverse. In particular, squared distances between two rows of \(U\) take the form
	\[
	\norm{U_i-U_j}_2^2
	=
	\sum_{k=2}^{r+1}
	\frac{\paren{\psi_k(i)-\psi_k(j)}^2}{\lambda_k+\tau},
	\]
	which is a truncated and regularized version of the effective-resistance distance on the graph. This makes the retained coordinates sensitive to global connectivity and coherent low- to mid-frequency geometry, rather than only to local Euclidean proximity. For multilevel flow backbones, this is useful because the coarse graph levels are responsible for transporting long-range information such as recirculation structure, shear-layer organization, and pressure-loading patterns. GLEAM therefore uses the same spectral variables that summarize this graph-resistance geometry, and then matches predicted and target fields in that retained coordinate system. A detailed derivation of this connection, and its relation to the lifted coarse-projector view of GLEAM, is given in Appendix~\ref{app:gleam-resistance}.
	
	For each channel, project the prediction and target onto the retained basis,
	\begin{equation}
		\tilde{\bm{u}}_c=U^\top\bm{u}_{:,c},
		\qquad
		\tilde{\bm{v}}_c=U^\top\bm{v}_{:,c}.
	\end{equation}
	We reuse the eigenvalue-space binning rule from \eqref{eq:exact_bin_edges}--\eqref{eq:exact_bin_sets}, but restrict it to the retained modes. Thus \(\lambda_N\) is replaced by \(\lambda_{r+1}\), and writing \(\bar{\lambda}_k=\lambda_{k+1}\) for \(k=1,\dots,r\) gives retained bands \(\mathcal{B}^{(r)}_m\subseteq\{1,\dots,r\}\). The retained-band energies are
	\begin{equation}
		E_u^{\mathrm{low}}(m,c)=\sum_{k\in\mathcal{B}^{(r)}_m}\frac{1}{2}\tilde{u}_{k,c}^2,
		\qquad
		E_v^{\mathrm{low}}(m,c)=\sum_{k\in\mathcal{B}^{(r)}_m}\frac{1}{2}\tilde{v}_{k,c}^2.
		\label{eq:lr_band_energy}
	\end{equation}
	With band weights \(\omega_m\ge 0\) and \(\sum_{m=1}^{M}\omega_m=1\), the GLEAM band-energy term is
	\begin{equation}
		\mathcal{L}_{\mathrm{low}}
		=\sum_{m=1}^{M}\omega_m \frac{1}{C}\sum_{c=1}^{C}
		\left(
		1-\frac{E_u^{\mathrm{low}}(m,c)+\varepsilon_s}{E_v^{\mathrm{low}}(m,c)+\varepsilon_s}
		\right)^2.
		\label{eq:low}
	\end{equation}
	Uniform weights recover the retained-spectrum analogue of BSP. Since modes beyond rank \(r\) are not supervised, GLEAM is not a full-spectrum replacement for exact Graph BSP; it is a hierarchy-aware approximation that emphasizes coarse and mid-scale organization when the retained hierarchy captures the spectral subspaces relevant to the forecast.
	
	Retained energy alone does not determine how structures are arranged spatially. GLEAM therefore uses the Fiedler vector \(\psi_F=\psi_2\) to sample graph-separated pairs \(\mathcal{Q}_F\subseteq V\times V\). For each pair, define
	\begin{equation}
		z_{pq}^{u}=\norm{\bm{u}_{p,:}-\bm{u}_{q,:}}_2^2,
		\qquad
		z_{pq}^{v}=\norm{\bm{v}_{p,:}-\bm{v}_{q,:}}_2^2,
		\qquad (p,q)\in\mathcal{Q}_F.
	\end{equation}
	Each pair is weighted by its retained-embedding separation,
	\begin{equation}
		w_{pq}
		=
		\frac{\norm{U_p-U_q}_2^2}
		{\frac{1}{|\mathcal{Q}_F|}\sum_{(a,b)\in\mathcal{Q}_F}\norm{U_a-U_b}_2^2+\varepsilon_s},
		\qquad (p,q)\in\mathcal{Q}_F.
		\label{eq:fiedler_pair_weight}
	\end{equation}
	When only the Fiedler coordinate is used, one may instead take \(w_{pq}\propto |\psi_F(p)-\psi_F(q)|^2\). The pairwise contrast term is
	\begin{equation}
		\mathcal{L}_{\mathrm{pair}}
		=
		\frac{1}{|\mathcal{Q}_F|}
		\sum_{(p,q)\in\mathcal{Q}_F}
		w_{pq}
		\left(z_{pq}^{u}-z_{pq}^{v}\right)^2.
		\label{eq:fiedler_pair_loss}
	\end{equation}
	The complete GLEAM objective is
	\begin{equation}
		\mathcal{L}_{\mathrm{GLEAM}}
		=(1-\rho)\mathcal{L}_{\mathrm{low}}+\rho\mathcal{L}_{\mathrm{pair}},
		\qquad \rho\in[0,1].
		\label{eq:gleam_objective}
	\end{equation}
	\(\rho=0\) gives pure retained-spectrum matching, while \(\rho>0\) additionally preserves target field contrasts across spectrally separated graph regions. The next subsection applies the same objective across a graph hierarchy.
	
	\subsection{GLEAM multilevel hierarchy with coarse-to-fine acceleration}
	\label{app:gleam-multilevel}
	
	Many graph-based surrogate models operate on a hierarchy of coarsened graphs. Let the graph hierarchy be indexed by
	\[
	\ell = 1,\dots,S,
	\]
	where \(\ell=1\) denotes the finest level and \(\ell=S\) denotes the coarsest level. Let
	\[
	\bm{u}^{(\ell)},\bm{v}^{(\ell)}
	\]
	be the predicted and target fields transferred to level \(\ell\) using the prescribed hierarchy maps, with distance-weighted restriction or interpolation when a fine node contributes to multiple coarse nodes.
	
	At each supervised level, the GLEAM objective may be evaluated, yielding
	\[
	\mathcal{L}_{\mathrm{GLEAM}}^{(\ell)}.
	\]
	The corresponding multilevel objective is then defined as
	\begin{equation}
		\mathcal{L}_{\mathrm{multi}}
		=
		\sum_{\ell=1}^{S}\alpha_\ell \mathcal{L}_{\mathrm{GLEAM}}^{(\ell)},
		\qquad
		\alpha_\ell \ge 0,
		\qquad
		\sum_{\ell=1}^{S}\alpha_\ell = 1.
		\label{eq:gleam_multilevel}
	\end{equation}
	This formulation encourages spectral consistency not only at the finest mesh resolution, but also across the graph hierarchy used by the model. 
	
	\paragraph{Coarse-to-fine spectral acceleration.}
	\label{app:coarse-to-fine}
	Although the multilevel formulation is well matched to hierarchical graph backbones, repeatedly solving eigensystems at every level may still be computationally expensive. To reduce this cost, one may adopt a coarse-to-fine strategy in which the spectral basis is computed first on the coarsest graph and then transferred to finer levels.
	
	\paragraph{Coarsest-level eigensolve.}
	Let \(L^{(S)}\) denote the Laplacian at the coarsest level. Instead of performing separate eigendecompositions at all levels, one solves
	\begin{equation}
		L^{(S)} \Psi^{(S)} = \Psi^{(S)} \Lambda^{(S)}
		\label{eq:coarsest_eig}
	\end{equation}
	only once at the coarsest resolution. The resulting eigenvectors define a coarse spectral embedding
	\[
	U^{(S)}.
	\]
	
	\paragraph{Prolongation to finer levels.}
	The coarse embedding is then prolongated to finer levels with distance-weighted interpolation over the available parent-child or index mappings. Denoting the distance-weighted prolongation operator from level \(\ell+1\) to level \(\ell\) by \(T_{\mathrm{dw}}^{(\ell)}\), one may write
	\begin{equation}
		U^{(\ell)} = T_{\mathrm{dw}}^{(\ell)} U^{(\ell+1)},
		\qquad \ell = S-1,\dots,1.
		\label{eq:prolongation}
	\end{equation}
	For example, the interpolation weights may be normalized inverse-distance or radial weights over the coarse parents associated with each fine node.
	
	\paragraph{Neighbor-mean refinement.}
	After distance-weighted prolongation, the embedding at each finer level may be refined by an unweighted neighbor-mean smoothing step. Let \(\mathrm{NbrMean}(U)\) denote the matrix obtained by replacing each node embedding with the arithmetic mean of its neighboring embeddings. Then one refinement step is
	\begin{equation}
		U \leftarrow (1-\beta)U + \beta\,\mathrm{NbrMean}(U),
		\label{eq:smoothing}
	\end{equation}
	where \(\beta \in [0,1]\) controls the amount of smoothing. Repeating this operation for a prescribed number of refinement steps improves coherence of the prolongated embedding with the local graph geometry.
	
	\paragraph{Corrector-based refinement with exact-level anchoring.}
	A stronger refinement can be introduced through a local correction residual
	\begin{equation}
		R = \mathrm{NbrMean}(U) - U.
		\label{eq:corrector_residual}
	\end{equation}
	The local refinement step is written as
	\begin{equation}
		U_i^{+} = U_i + \Gamma_i(U,R)\,R_i,
		\label{eq:general_corrector}
	\end{equation}
	where \(U_i^{+}\) denotes the corrected embedding at node \(i\), \(R_i\) is the local correction residual, and \(\Gamma_i(U,R)\) is a correction coefficient that determines how strongly the residual is applied. Different corrector strategies correspond to different choices of \(\Gamma_i\):
	\begin{align}
		\text{No correction:}\qquad
		&\Gamma_i(U,R) = 0,
		\label{eq:no_corrector}
		\\[4pt]
		\text{Scalar correction:}\qquad
		&\Gamma_i(U,R) = \sigma(\alpha),
		\label{eq:scalar_corrector_clean}
	\end{align}
	Here, \(\alpha\) is a learnable scalar parameter and \(\sigma(\cdot)\) denotes the sigmoid function.
	
	When an exact eigensolve is available at some reference level \(\ell_\star\), the corrected prolongated embedding at that same level may be anchored to the exact embedding rather than introducing a separate finest-level penalty. Let \(\tilde U^{(\ell_\star)}\) denote the prolongated-and-corrected embedding at level \(\ell_\star\), and let \(\tilde U_{\mathrm{exact}}^{(\ell_\star)}\) denote the exact level-\(\ell_\star\) embedding after column-wise sign alignment. The anchor term is
	\begin{equation}
		\mathcal{L}_{\mathrm{anchor}}
		=
		\gamma \norm{\tilde U^{(\ell_\star)}-\tilde U_{\mathrm{exact}}^{(\ell_\star)}}_F^2,
		\label{eq:anchor_loss}
	\end{equation}
	where \(\gamma \ge 0\) controls the strength of the anchor. Thus the refinement and anchor act on the same active level: the corrector adjusts the prolongated basis locally, and the anchor keeps that corrected basis close to the exact reference embedding available at level \(\ell_\star\). In particular, \(\ell_\star\) need not be the finest graph; when only a lower-resolution exact solve is used, the anchored corrected embedding is subsequently prolongated further to the remaining finer levels.
	
	\section{Experiments}\label{sec:experiments}
	
	We evaluate the proposed graph-spectral losses on unstructured-mesh forecasting benchmarks with long-horizon autoregressive dynamics. The experiments are designed to test whether scale-resolved supervision improves rollout fidelity beyond a pointwise objective, and whether the resulting predictions better preserve coherent structures whose errors accumulate over time.
	
	The benchmark suite includes EAGLE \citep{janny2023eagle}, the backward-facing step case \citep{kim2025generalizable}, and \texttt{pOnWing} from the DGN4CFD benchmark suite \citep{li2020gno,lino2025dgn}. These problems span full-field and surface-only prediction, two- and three-dimensional geometries, fixed and varying meshes, and both short-context and long-horizon autoregressive evaluation. This setting allows us to compare exact graph-spectral diagnostics, Chebyshev band-power supervision, and GLEAM's retained-hierarchy supervision across different mesh topologies, engineering geometries, and forecasting backbones.
	
	\subsection{EAGLE}
	\label{sec:eagle-results}
	EAGLE \citep{janny2023eagle} provides an unsteady two-dimensional flow benchmark on graph-structured meshes with moving flow sources and varying scene geometries. The dataset contains approximately \(1.1\) million simulated meshes across \(600\) scenes and reports pressure and velocity fields, making it a useful test of graph-spectral supervision beyond a single fixed CFD mesh. Detailed dataset and training settings are provided in Appendix~\ref{app:training-details}.
	
	The benchmark is relevant because the original EAGLE study reports a long-horizon failure mode in difficult scenes: spatially extended turbulent flow develops over time, while autoregressive neural predictors accumulate error and smooth the airflow as the forecast horizon grows \citep{janny2023eagle}. Several test samples also show increasingly chaotic dynamics at later rollout steps, where vortical structures influence the subsequent flow state.
	
	We use EAGLE to test whether a pointwise rollout objective can match the dominant velocity magnitude while still losing the distribution of turbulent content. Chebyshev BSP is expected to penalize this loss of scale-resolved energy, so the evaluation considers both average rollout RMSE and graph-frequency allocation associated with vortices, shear layers, and turbulent filaments.
	
	To validate the Chebyshev BSP formulation, we first perform an ablation study on EAGLE by varying the number of graph-frequency bands \(M\), the Chebyshev polynomial order \(K\), and the quadrature count \(Q\) used in \(H_m^{(K,Q)}\). Tables~\ref{tab:eagle_chebyshev_ablation} and~\ref{tab:eagle_spectral_energy_rel_l1} report rollout RMSE and state spectral-energy relative \(L_1\) error at selected autoregressive horizons, allowing us to test whether the spectral approximation improves both pointwise accuracy and long-horizon spectral drift.
	
	\begin{table}[!htbp]
		\centering
		\caption{Ablation study of the Chebyshev BSP approximation parameters on the EAGLE dataset for rollout RMSE, reported as mean \(\pm\) standard deviation over the \(118\) test simulations.}
		\label{tab:eagle_chebyshev_ablation}
		\scriptsize
		\setlength{\tabcolsep}{2.5pt}
		\resizebox{\linewidth}{!}{%
			\begin{tabular}{ccccccccc}
				\toprule
				\(K\) & \(M\) & \(Q\) & \(t=1\) & \(t=50\) & \(t=100\) & \(t=200\) & \(t=300\) & \(t=400\) \\
				\midrule
				\multicolumn{3}{c}{Base} & \(1.665 \pm 0.294\) & \(7.218 \pm 2.359\) & \(7.906 \pm 2.672\) & \(8.950 \pm 3.297\) & \(8.727 \pm 3.393\) & \(8.928 \pm 3.524\) \\
				16 & 32 & 256 & \(1.612 \pm 0.273\) & \(5.690 \pm 2.142\) & \(6.415 \pm 2.653\) & \(7.066 \pm 3.092\) & \(7.078 \pm 3.394\) & \(7.203 \pm 3.343\) \\
				16 & 32 & 128 & \(\mathbf{1.595 \pm 0.288}\) & \(\mathbf{5.556 \pm 1.901}\) & \(6.356 \pm 2.632\) & \(7.079 \pm 3.223\) & \(6.919 \pm 3.321\) & \(7.322 \pm 3.840\) \\
				16 & 64 & 256 & \(1.606 \pm 0.286\) & \(5.843 \pm 2.408\) & \(6.617 \pm 2.817\) & \(7.669 \pm 3.817\) & \(7.759 \pm 3.831\) & \(8.546 \pm 5.440\) \\
				12 & 24 & 256 & \(1.605 \pm 0.274\) & \(5.802 \pm 2.214\) & \(\mathbf{6.271 \pm 2.767}\) & \(7.005 \pm 3.218\) & \(7.034 \pm 3.327\) & \(7.257 \pm 3.596\) \\
				8 & 16 & 128 & \(1.608 \pm 0.288\) & \(5.831 \pm 2.125\) & \(6.595 \pm 2.913\) & \(7.451 \pm 3.684\) & \(8.094 \pm 4.346\) & \(8.555 \pm 4.546\) \\
				6 & 16 & 128 & \(1.627 \pm 0.280\) & \(5.684 \pm 2.015\) & \(6.306 \pm 2.543\) & \(\mathbf{6.762 \pm 3.038}\) & \(\mathbf{6.851 \pm 3.247}\) & \(\mathbf{7.012 \pm 3.281}\) \\
				\bottomrule
			\end{tabular}
		}
	\end{table}
	
	\begin{table}[!htbp]
		\centering
		\caption{Ablation study of the Chebyshev BSP approximation parameters on the EAGLE dataset for state spectral-energy relative \(L_1\) error, reported as mean \(\pm\) standard deviation over the \(118\) test simulations.}
		\label{tab:eagle_spectral_energy_rel_l1}
		\scriptsize
		\setlength{\tabcolsep}{2.5pt}
		\resizebox{\linewidth}{!}{%
			\begin{tabular}{ccccccccc}
				\toprule
				\(K\) & \(M\) & \(Q\) & \(t=1\) & \(t=50\) & \(t=100\) & \(t=200\) & \(t=300\) & \(t=400\) \\
				\midrule
				\multicolumn{3}{c}{Base} & \(0.0312 \pm 0.0116\) & \(0.1623 \pm 0.0919\) & \(0.1768 \pm 0.1356\) & \(0.2323 \pm 0.1931\) & \(0.2453 \pm 0.1617\) & \(0.2454 \pm 0.1629\) \\
				16 & 32 & 256 & \(\mathbf{0.0302 \pm 0.0113}\) & \(0.1319 \pm 0.1043\) & \(0.1452 \pm 0.0920\) & \(0.1661 \pm 0.1105\) & \(\mathbf{0.1730 \pm 0.1358}\) & \(\mathbf{0.1871 \pm 0.1911}\) \\
				16 & 32 & 128 & \(0.0302 \pm 0.0118\) & \(\mathbf{0.1155 \pm 0.0621}\) & \(0.1520 \pm 0.1161\) & \(0.1913 \pm 0.1403\) & \(0.1892 \pm 0.1395\) & \(0.2074 \pm 0.1523\) \\
				16 & 64 & 256 & \(0.0318 \pm 0.0116\) & \(0.1276 \pm 0.1063\) & \(0.1614 \pm 0.1282\) & \(0.2453 \pm 0.2797\) & \(0.2461 \pm 0.2885\) & \(0.3482 \pm 0.7757\) \\
				12 & 24 & 256 & \(0.0318 \pm 0.0111\) & \(0.1507 \pm 0.1167\) & \(0.1624 \pm 0.1288\) & \(0.2068 \pm 0.1717\) & \(0.2198 \pm 0.1768\) & \(0.2306 \pm 0.1790\) \\
				8 & 16 & 128 & \(0.0332 \pm 0.0112\) & \(0.1368 \pm 0.0835\) & \(0.1573 \pm 0.1174\) & \(0.2180 \pm 0.2898\) & \(0.2797 \pm 0.4294\) & \(0.3124 \pm 0.5292\) \\
				6 & 16 & 128 & \(0.0361 \pm 0.0127\) & \(0.1249 \pm 0.0785\) & \(\mathbf{0.1409 \pm 0.0986}\) & \(\mathbf{0.1654 \pm 0.1318}\) & \(0.2012 \pm 0.2070\) & \(0.2149 \pm 0.1926\) \\
				\bottomrule
			\end{tabular}
		}
	\end{table}
	
	Tables~\ref{tab:eagle_chebyshev_ablation} and~\ref{tab:eagle_spectral_energy_rel_l1} show that adding a BSP loss improves rollout RMSE over the Base model for every tested Chebyshev variant. The state spectral-energy relative \(L_1\) metric measures the relative discrepancy between predicted and target spectral-energy distributions over the state channels, so lower values indicate better preservation of scale-resolved energy. The strongest Chebyshev settings also improve this spectral metric, indicating that the auxiliary constraint can improve both the targeted diagnostic and physical-space prediction under autoregressive rollout. In Table~\ref{tab:eagle_chebyshev_ablation}, the best short-horizon RMSE is obtained with \(K=16\), \(M=32\), \(Q=128\), while the lowest RMSE values at \(t=200\), \(t=300\), and \(t=400\) are obtained by the coarser \(K=6\), \(M=16\), \(Q=128\) setting. Table~\ref{tab:eagle_spectral_energy_rel_l1} shows the complementary spectral trend: \(K=6\), \(M=16\), \(Q=128\) gives the best mid-horizon spectral agreement at \(t=100\) and \(t=200\), whereas \(K=16\), \(M=32\), \(Q=256\) gives the lowest spectral error at the longest horizons.
	
	We therefore use \(K=16\), \(M=32\), and \(Q=256\) as the representative EAGLE Chebyshev BSP setting in the following diagnostics, because it gives the strongest late-horizon spectral preservation while retaining a clear RMSE improvement over the Base model at every reported horizon. The \(K=16\), \(M=64\), \(Q=256\) row shows that increasing the spectral resolution does not automatically improve the result: although it still improves RMSE over Base, it is not competitive with the best Chebyshev RMSE rows, and its late-horizon spectral relative \(L_1\) error and variance are worse than the selected \(M=32\) setting. A likely reason is that increasing \(M\) at fixed polynomial order creates narrower spectral windows. As defined in Eq.~\eqref{eq:cheb}, the Chebyshev graph filter is represented by a degree-\(K\) polynomial in the rescaled Laplacian. Smooth and broad spectral windows can be represented with modest \(K\), whereas sharper windows require higher polynomial order to approximate their transitions accurately; otherwise the approximation can smear, oscillate, or leak energy between nearby bands. On EAGLE, excessively fine partitioning therefore appears to make the loss more sensitive to approximation error and sample-to-sample spectral variability, while \(M=32\) provides a better balance between resolution, robustness, and cost.
	
	Figure~\ref{fig:eagle_vorticity_rmse_pdf_l1} illustrates the same behavior in vorticity space for the \(K=16,M=32,Q=256\) Chebyshev setting. The RMSE panel, shown on a logarithmic vertical scale, captures localized vorticity errors that grow during rollout. The Base curve rises more rapidly and has a wider min--max envelope, consistent with some samples entering the turbulent failure regime described by \citet{janny2023eagle}. The PDF \(L_1\) panel gives the complementary distributional view: Chebyshev BSP keeps the vorticity distribution closer to the target over long horizons, indicating that the model preserves more of the statistical spread of vortical content rather than only reducing pointwise error.
	
	\begin{figure}[!htbp]
		\centering
		\includegraphics[width=\linewidth]{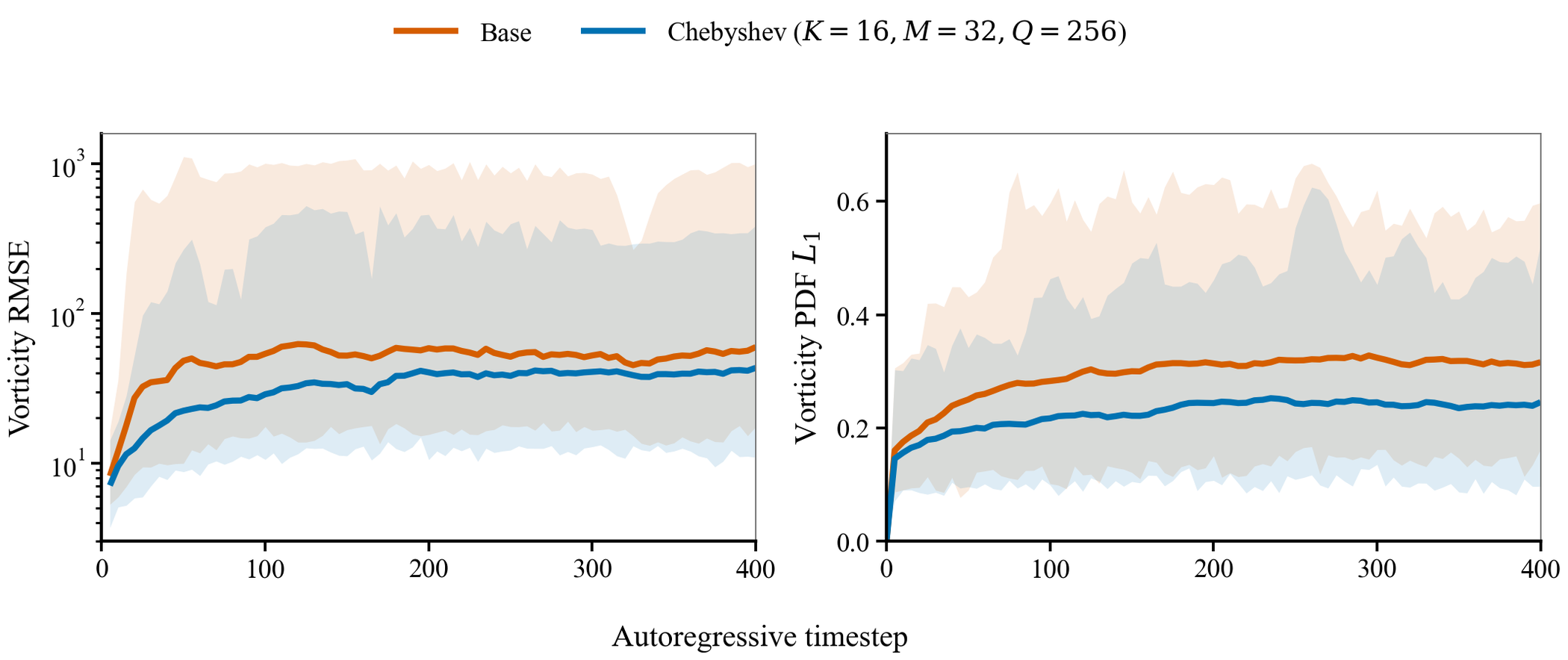}
		\caption{Vorticity rollout diagnostics for the Base model and Chebyshev BSP with \(K=16\), \(M=32\), and \(Q=256\) for Eagle dataset. The panels report vorticity RMSE and vorticity-PDF \(L_1\) error over autoregressive rollout; solid curves denote the sample mean, and shaded envelopes denote the min--max range across test samples.}
		\label{fig:eagle_vorticity_rmse_pdf_l1}
	\end{figure}
	
	Figure~\ref{fig:eagle_qualitative} shows two representative EAGLE velocity-magnitude cases at different rollout times, while Appendix Fig.~\ref{fig:eagle_appendix_qualitative} follows one additional sample over time. In the \(t=140\) case, both models recover the dominant large-scale motion, but the Base error is already concentrated around the shear layer and emerging vortices. In the separate \(t=310\) case, the flow is more turbulent and the forecast has accumulated more autoregressive error; the Base prediction is visibly smoother and loses fine spatial variation. The appendix sequence shows the same trend within a single sample across \(t=5,50,100,180,\) and \(260\): an initially organized curved shear layer develops into a richer vortical region with stronger small- and mid-scale variation. Chebyshev BSP better preserves the coherent vortical pattern and the distribution of high-velocity structures, while the Base model increasingly damps them. The Chebyshev model still diffuses some late-time details and does not eliminate local error, but its errors remain more spatially localized and the predicted flow retains more of the turbulent organization. This qualitative behavior is consistent with the quantitative diagnostics: graph-spectral supervision helps preserve scale-resolved statistics and vortical content that standard autoregressive training tends to smooth at late rollout times. We do not report GLEAM on EAGLE because the EAGLE forecasting backbone used here is not a multilevel graph model and does not expose the coarsened graph states and hierarchy maps required for retained-hierarchy supervision; applying GLEAM would therefore require changing the backbone rather than adding the same hierarchy-aware loss to an existing multilevel architecture.
	
	\begin{figure}[!htbp]
		\centering
		\setlength{\tabcolsep}{1pt}
		\renewcommand{\arraystretch}{1.0}
		\scriptsize
		\begin{tabular}{@{}EEE@{}}
			\eaglequalheader
			\eaglequalblock{t=140}{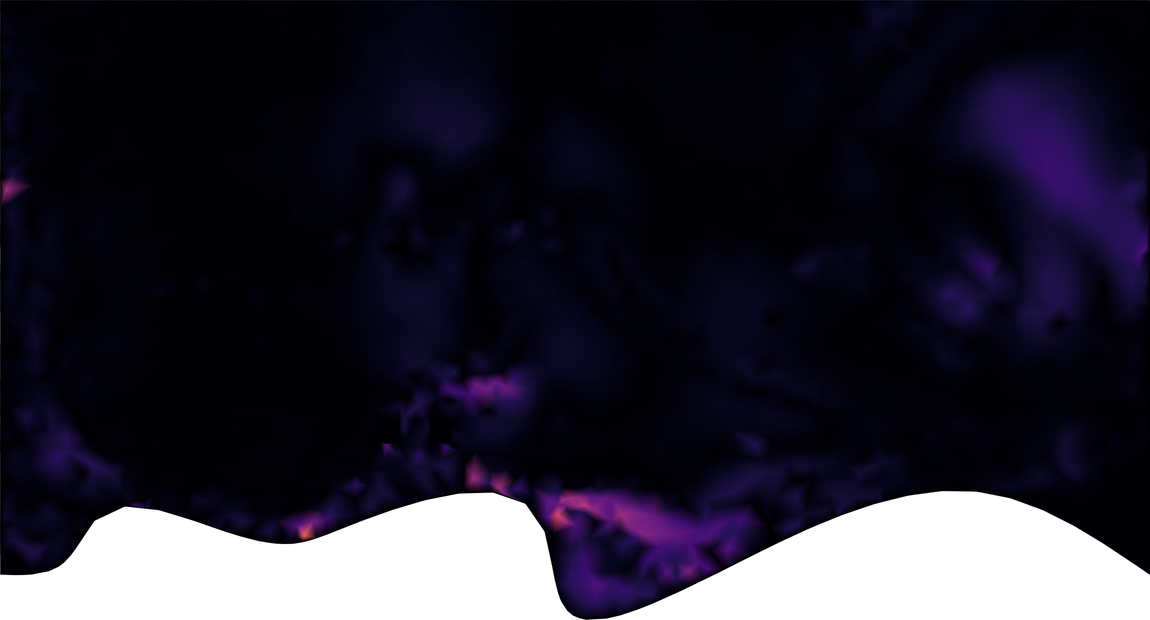}{1.2em}
			\eaglequalblock{t=310}{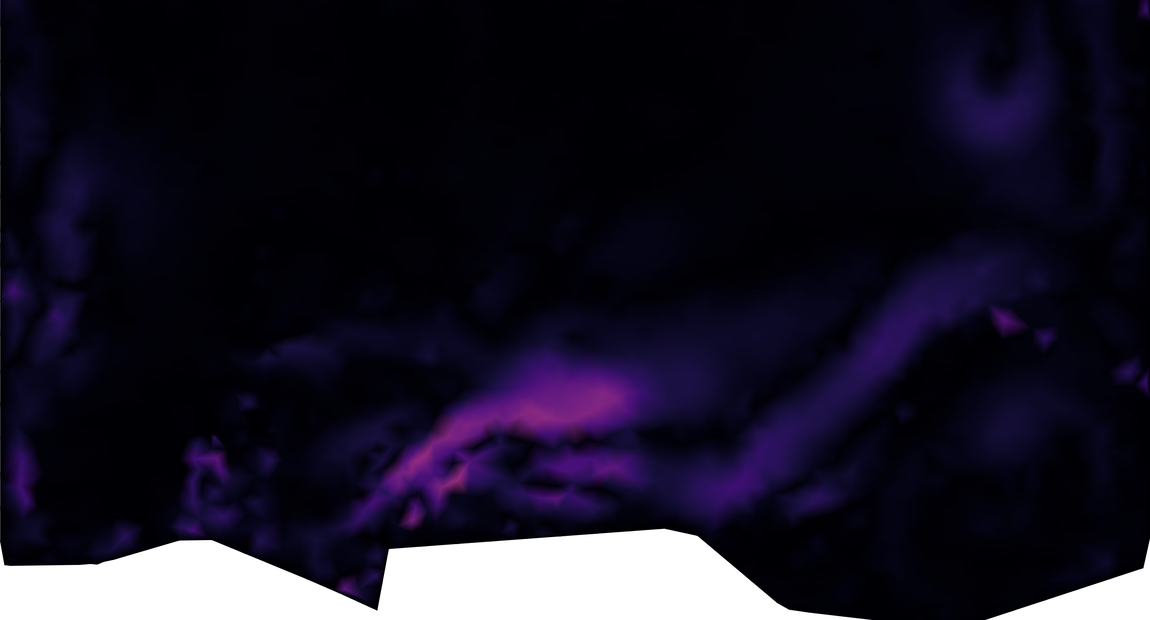}{0pt}
		\end{tabular}
		\vspace{0.2em}
		
		\includegraphics[width=0.78\linewidth]{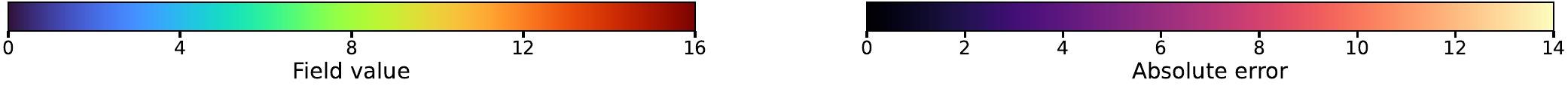}
		\caption{Qualitative EAGLE velocity-magnitude comparison at \(t=140\) and \(t=310\). The horizontal colorbars report the field-magnitude and absolute-error scales.}
		\label{fig:eagle_qualitative}
	\end{figure}
	
	\subsection{2D Turbulence over a Backward-Facing Step}
	\label{sec:bfs-results}
	We use the two-dimensional backward-facing step (BFS) case of \citet{kim2025generalizable} as a long-horizon rollout test for separated flow. After discarding the initial transient, the sequence is split contiguously into training, validation, and test segments. The models are trained with a short autoregressive curriculum that increases from one to four rollout steps and are evaluated autoregressively on the held-out test segment; dataset and training details are given in Appendix~\ref{app:training-details}. Here we compare the pointwise Base model, final-level Chebyshev spectral supervision, and hierarchical GLEAM supervision.
	
	Figure~\ref{fig:bfs2d_results} shows that both spectral variants improve the long-horizon BFS rollout relative to the Base model. The Chebyshev model receives spectral supervision only at the final graph level; among the tested weights, \(\lambda_{\mathrm{Cheb}}=0.1\) gives the best long-horizon rollout accuracy, while \(\lambda_{\mathrm{Cheb}}=1.0\) over-regularizes the forecast and worsens the final RMSE. GLEAM applies hierarchical supervision through level \(4\) together with a scalar corrector, giving a multilevel constraint on the forecast dynamics insofar as the coarsened hierarchy retains the relevant low- and mid-frequency error subspaces. In the saved hyperparameter sweep, the strongest GLEAM rollout uses \(\lambda_{\mathrm{GLEAM}}=10^{-4}\). Larger GLEAM weights are worse than this setting; although they remain below the Base model at late horizons, \(\lambda_{\mathrm{GLEAM}}=10^{-2}\) is worse than Base over part of the early-to-mid rollout. Comparison with methods that improve autoregressive forecasting is introduced in Appendix~\ref{app:bfs-similar-methods}, and a detailed study of the GLEAM hyperparameters is presented in Appendix~\ref{app:bfs-gleam-ablation}.
	
	\begin{figure}[!htbp]
		\centering
		\begin{subfigure}[t]{0.49\linewidth}
			\centering
			\axisspineplot[width=\linewidth]{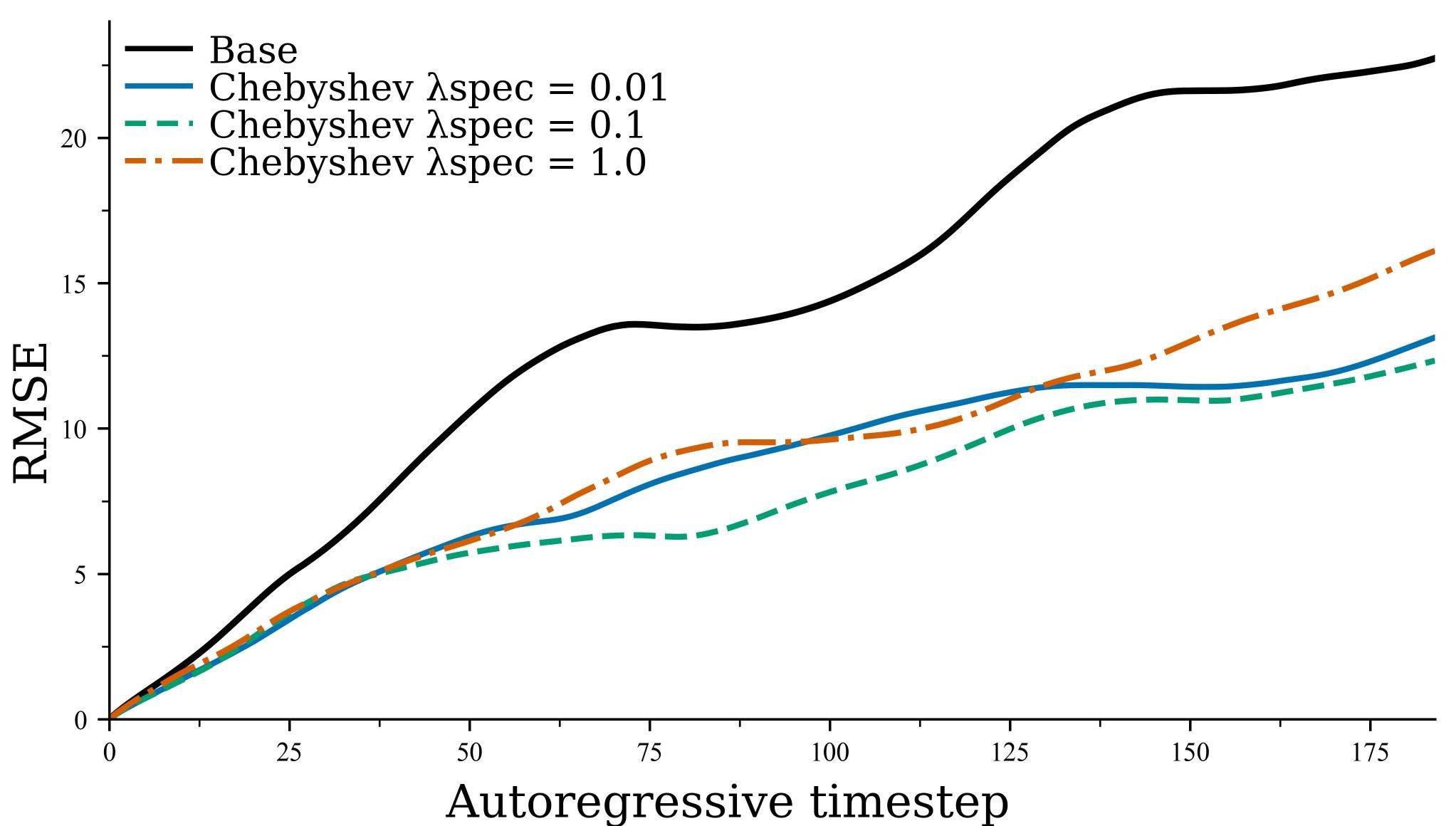}
			\caption{Chebyshev graph-spectral supervision.}
			\label{fig:bfs2d_chebyshev_results}
		\end{subfigure}
		\hfill
		\begin{subfigure}[t]{0.49\linewidth}
			\centering
			\axisspineplot[width=\linewidth]{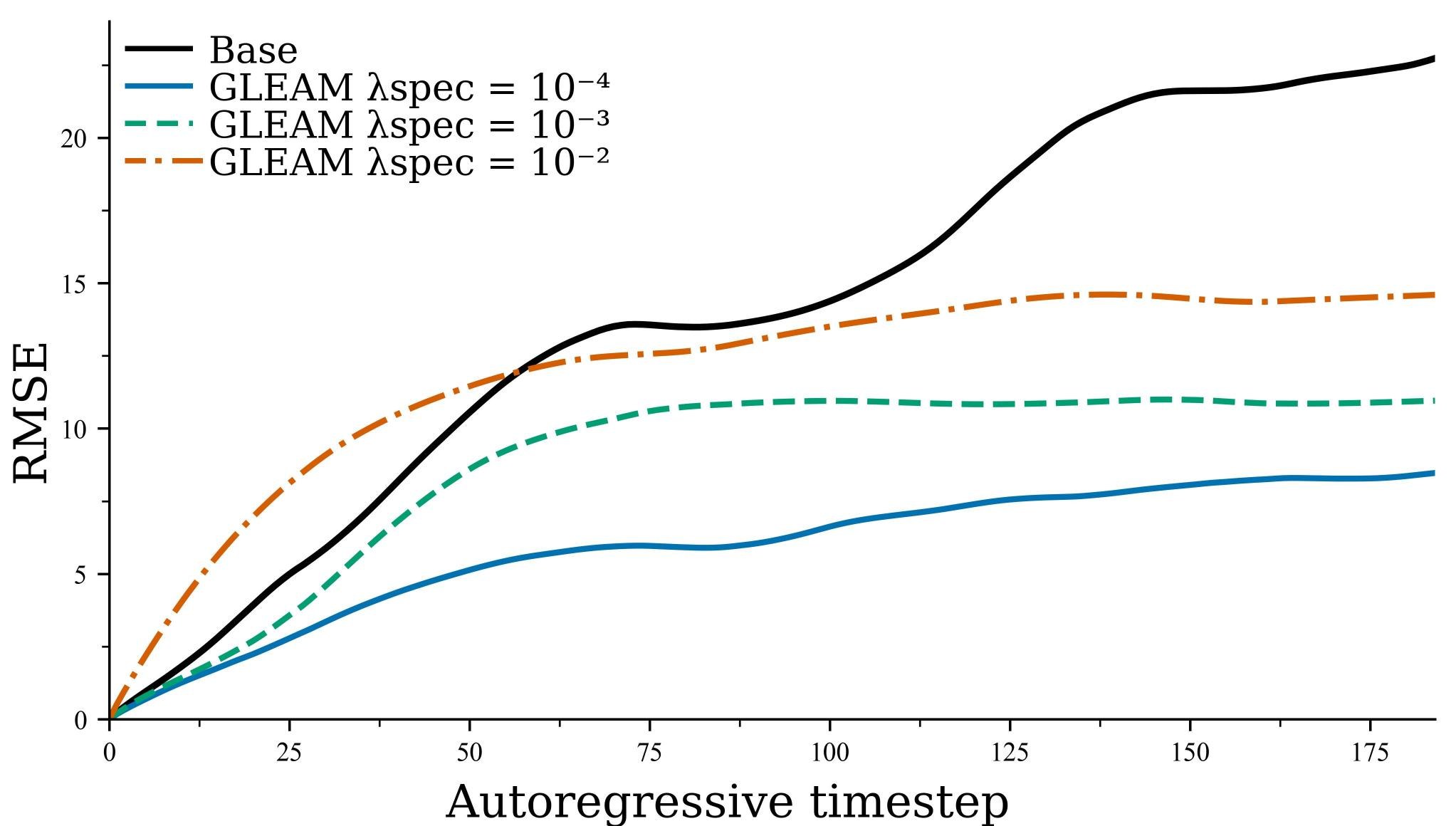}
			\caption{GLEAM low-rank graph-spectral supervision.}
			\label{fig:bfs2d_gleam_results}
		\end{subfigure}
		\caption{Autoregressive RMSE on the two-dimensional backward-facing step dataset for (a) Chebyshev final-level spectral supervision and (b) GLEAM hierarchical spectral supervision.}
		\label{fig:bfs2d_results}
	\end{figure}
	
	The RMSE curves in Fig.~\ref{fig:bfs2d_results} show an improvement in long-horizon rollout accuracy, but they do not by themselves identify whether the improvement occurs in the flow structures that drive autoregressive drift. For BFS, the relevant failure mode is not only global pointwise error growth: small errors near the separated shear layer and reverse-flow region can change the downstream recirculation pattern and shift energy in low graph-spectral bands. We therefore compute the mechanism-oriented diagnostics in Table~\ref{tab:bfs_mechanism_diagnostics} to test whether the trained-model rollouts are consistent with this transport-drift interpretation. These metrics measure spatially localized velocity error, reverse-flow extent and centroid drift, and exact low-band \(U_y\) graph-spectral drift; their definitions and calculation details are given in Appendix~\ref{app:bfs-mechanism-diagnostics}.
	
	\begin{table}[!htbp]
		\centering
		\caption{Mechanism-oriented BFS diagnostics computed from fresh autoregressive rollouts of the trained models. Metric definitions and calculation details are given in Appendix~\ref{app:bfs-mechanism-diagnostics}.}
		\label{tab:bfs_mechanism_diagnostics}
		\scriptsize
		\setlength{\tabcolsep}{3.5pt}
		\begin{tabular}{lcccccc}
			\toprule
			Model & Global RMSE & Reverse RMSE & High-gradient RMSE & Reverse-area err. & Centroid drift & Low-band RMSE \\
			\midrule
			Base & 11.36 & 18.26 & 16.40 & 0.1005 & 0.0577 & 6.624 \\
			Chebyshev, \(\lambda_{\mathrm{Cheb}}=0.1\) & 4.852 & 12.59 & 8.348 & 0.0088 & 0.0392 & 0.903 \\
			GLEAM, \(\lambda_{\mathrm{GLEAM}}=10^{-4}\) & 3.101 & 10.06 & 6.557 & 0.0050 & 0.0174 & 0.652 \\
			\bottomrule
		\end{tabular}
	\end{table}
	
	These BFS diagnostics are best viewed within the transport-centered behavior shared by the turbulent-flow benchmarks. Across EAGLE, BFS, and \texttt{pOnWing}, the measured errors are associated with vortical, shear, recirculation, and pressure-loading structures whose graph-scale energy changes during autoregressive rollout \citep{janny2023eagle,kim2025generalizable,lino2025dgn}. The graph-spectral losses target this failure mode by penalizing drift in the frequency bands where these transported structures remain represented. This interpretation is reflected across the experiments: EAGLE shows improved preservation of vortical content and vorticity statistics, Table~\ref{tab:bfs_mechanism_diagnostics} shows reduced recirculation-area, centroid, shear-region, and low-band \(U_y\) errors, and \texttt{pOnWing} shows smaller pressure-force, distributional, and graph-spectral deviations. In this setting, GLEAM's role is not fine-level spectral control alone, but stabilization of transported flow content across the retained graph scales.
	
	These trends also clarify the role of the spectral weight. Too little spectral supervision leaves the rollout close to the Base model, while too much supervision can force spectral agreement at the expense of local pointwise correction. The useful settings therefore occupy an intermediate regime: they regularize the energy distribution enough to reduce long-horizon drift without overwhelming the physical-space forecast loss. This tradeoff is especially visible in BFS because the separated flow repeatedly converts small near-wall errors into larger downstream structures.
	
	The rollout fields in Fig.~\ref{fig:bfs2d_qualitative_rollout} track the same BFS comparison at \(t=1,5,10,20,\) and \(50\). At early times all models remain close to the target, so the improvement is not due to a different initial state. As the separated shear layer develops, the Base model accumulates broader error near the step, lower wall, and downstream recirculation region. Chebyshev BSP and GLEAM keep the error more localized, with GLEAM better preserving the large-scale organization of the recirculation zone in this experiment. The visual comparison is consistent with the RMSE curves: autoregressive rollout fidelity improves when the loss constrains both pointwise values and graph-resolved energy across scales, provided the supervised hierarchy retains the relevant spectral subspaces. The additional BFS diagnostics in Appendix~\ref{app:bfs-uy-rollout-snapshots} support the same interpretation for the \(U_y\) component.
	
	\begin{figure}[!htbp]
		\centering
		\includegraphics[width=0.96\linewidth]{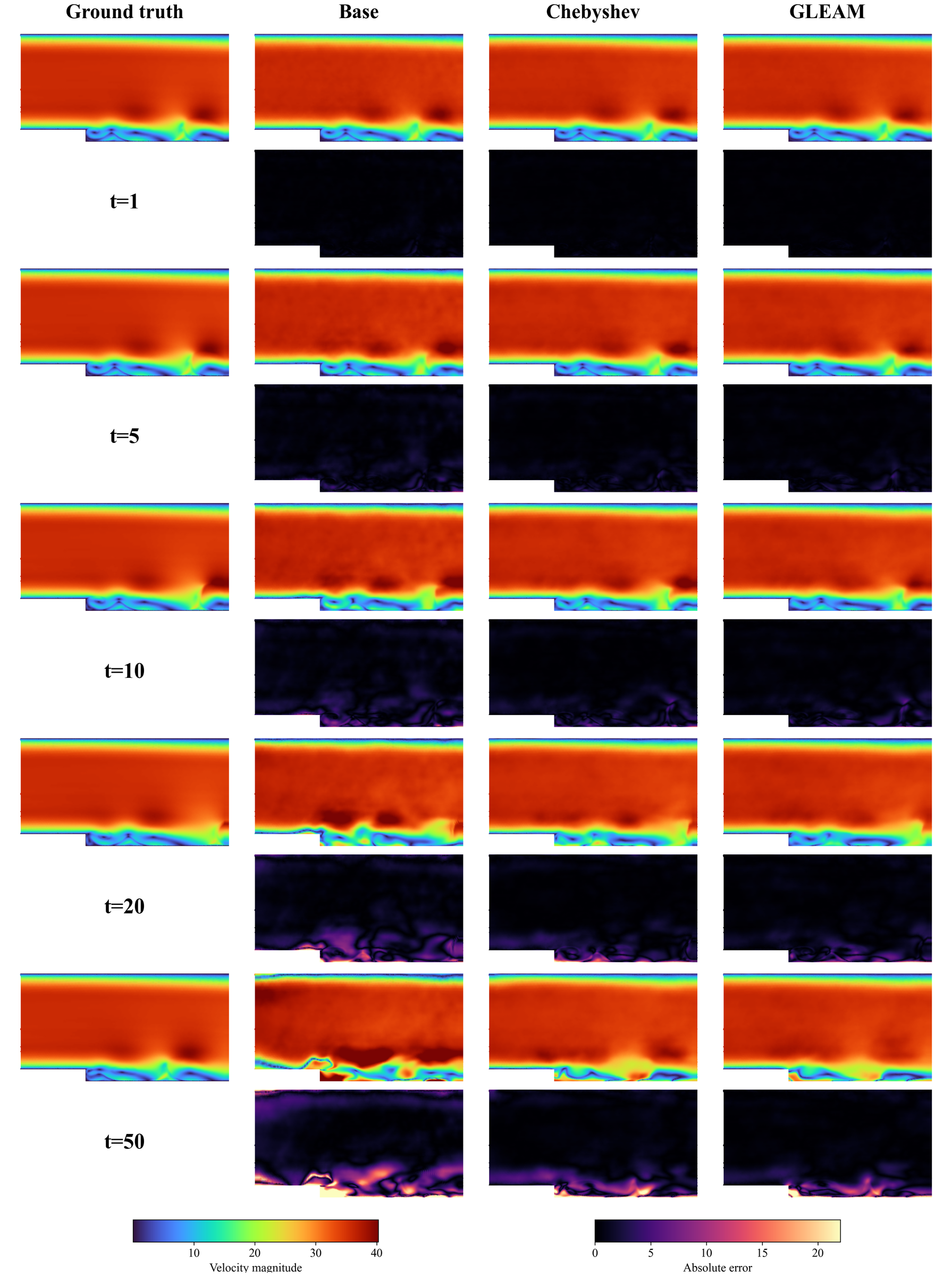}
		\caption{Backward-facing step rollout fields at \(t=1,5,10,20,\) and \(50\). Each block shows predicted velocity fields and absolute-error maps for the Base, Chebyshev BSP, and GLEAM models.}
		\label{fig:bfs2d_qualitative_rollout}
	\end{figure}
	\subsection{\texttt{pOnWing}}
	\label{sec:ponwing-results}
	We also evaluate \texttt{pOnWing}, a three-dimensional turbulent wing-flow case from the DGN4CFD benchmark suite \citep{li2020gno,lino2025dgn}. \texttt{pOnWing} extends the evaluation to pressure forecasting on a complex aerodynamic surface, where long-rollout drift affects both local pressure fields and integrated loading. We test the spectral objectives on both deterministic and probabilistic graph forecasting backbones; dataset geometry, rollout protocol, and hierarchy details are given in Appendix~\ref{app:training-details}.
	
	Figure~\ref{fig:ponwing_surface_centered_sim5} shows a representative surface-centered pressure rollout for \texttt{pOnWing}. The Base model begins close to the truth but loses the centered pressure pattern more strongly by \(t=50\), while Chebyshev BSP and GLEAM retain more of the spatial structure and keep the absolute error more localized. Additional surface-centered examples and pressure-PDF snapshots are provided in Appendix~\ref{app:ponwing-benchmark}. We also evaluate a force-level diagnostic that integrates the predicted pressure over the wing surface. At each autoregressive step, the unscaled pressure field is converted into a pressure-force vector \( \bm{F}_p(t) \), and the rollout is scored by the stabilized relative error \( \|\widehat{\bm{F}}_p(t)-\bm{F}_p(t)\|_2/\max(\|\bm{F}_p(t)\|_2,10^{-12}) \). This complements the field and PDF diagnostics: the snapshots test spatial localization, the PDFs test whether the distribution of pressure values is preserved, and the integrated force error tests whether the forecast preserves the net surface loading induced by those pressures.
	
	\begin{figure}[!htbp]
		\centering
		\setlength{\tabcolsep}{1pt}
		\renewcommand{\arraystretch}{0.72}
		\scriptsize
		\begin{tabular}{@{}>{\raggedleft\arraybackslash}p{0.035\linewidth}cccc@{}}
			& \textbf{\(t=1\)} & \textbf{\(t=5\)} & \textbf{\(t=10\)} & \textbf{\(t=50\)}\\[-0.1em]
			\textbf{(a)} &
			\includegraphics[width=0.23\linewidth]{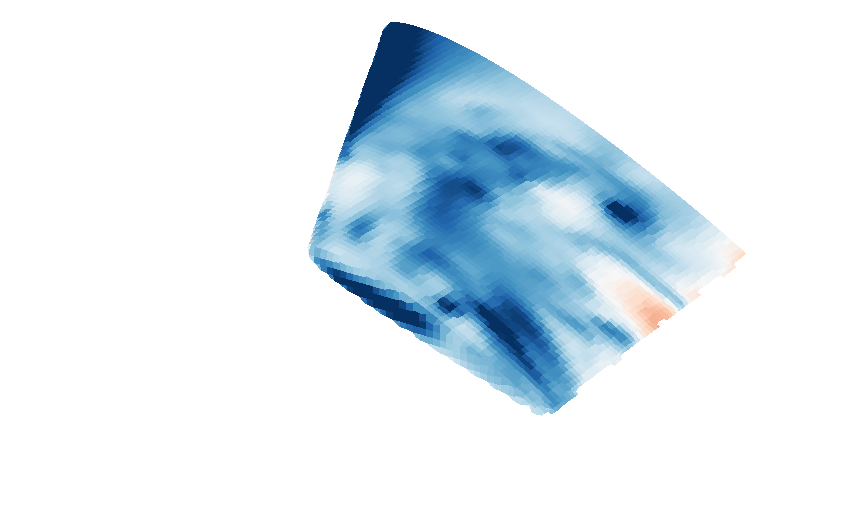} &
			\includegraphics[width=0.23\linewidth]{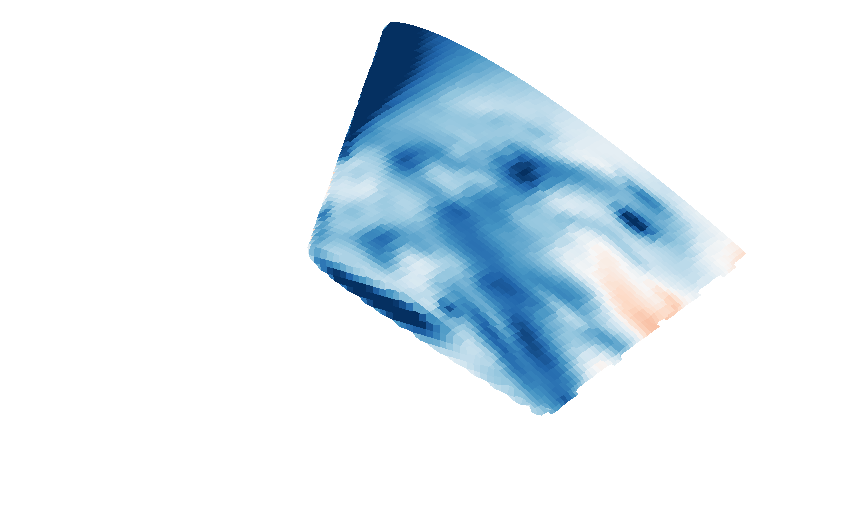} &
			\includegraphics[width=0.23\linewidth]{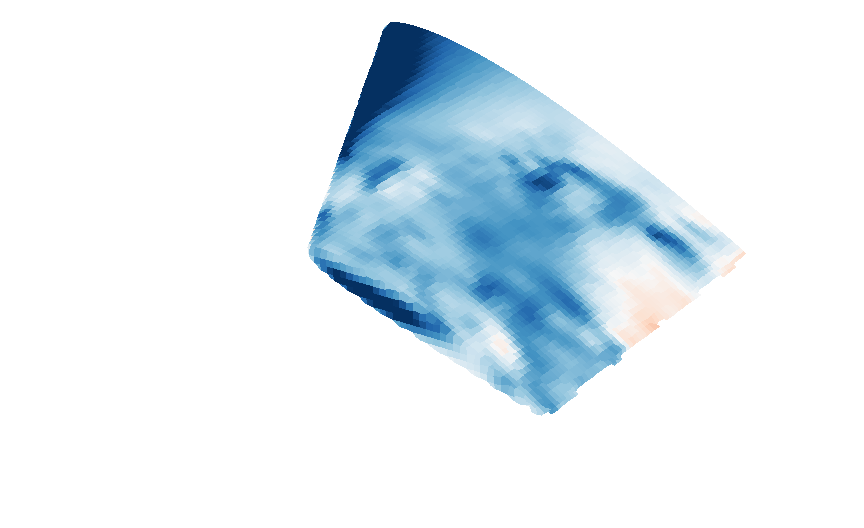} &
			\includegraphics[width=0.23\linewidth]{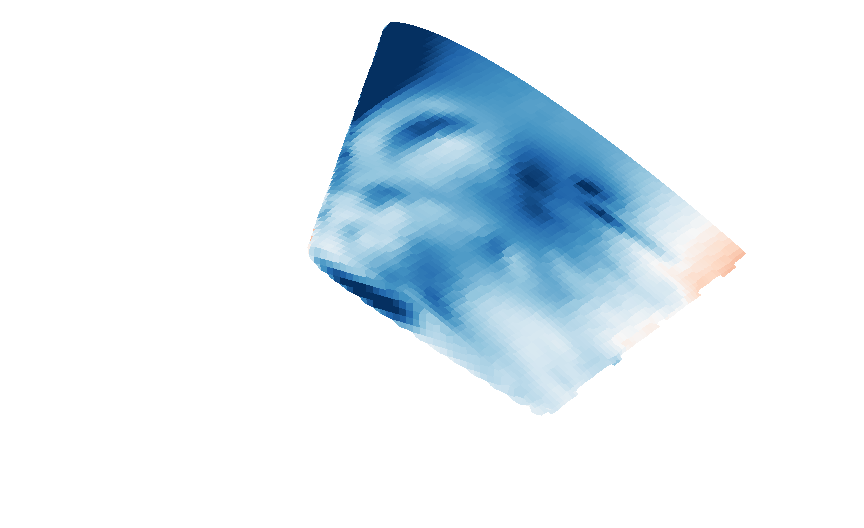}\\[-0.2em]
			\textbf{(b)} &
			\includegraphics[width=0.23\linewidth]{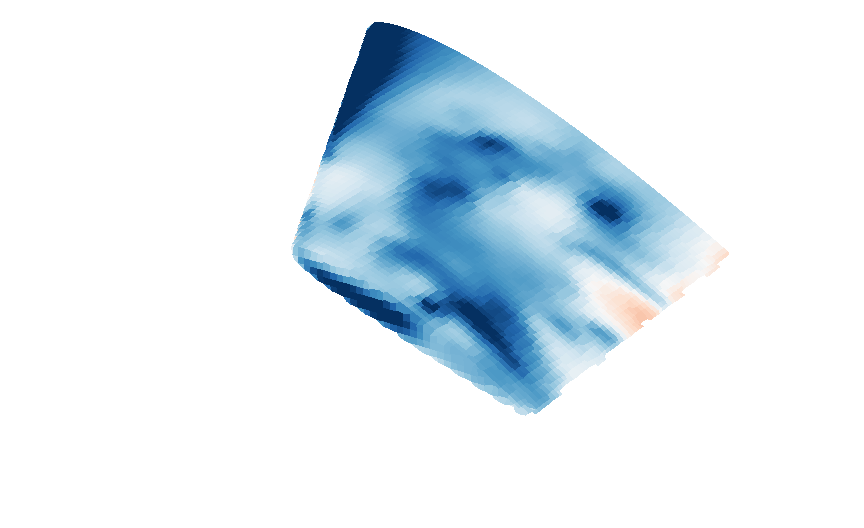} &
			\includegraphics[width=0.23\linewidth]{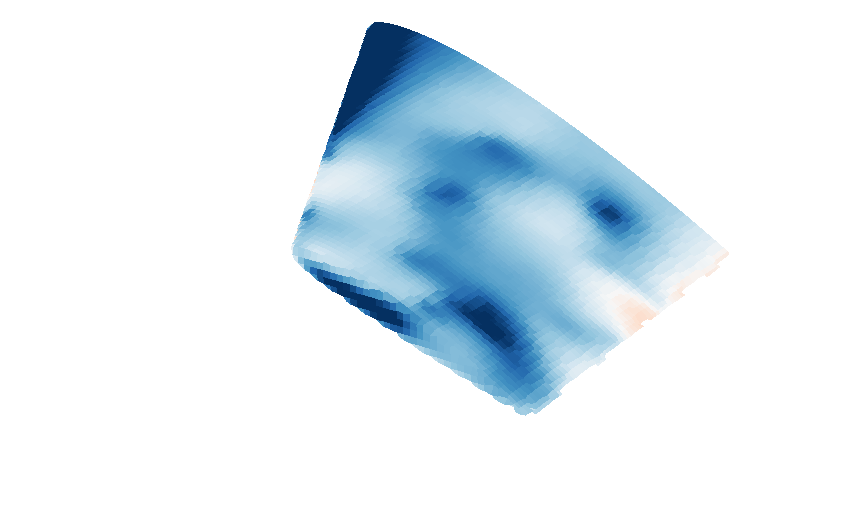} &
			\includegraphics[width=0.23\linewidth]{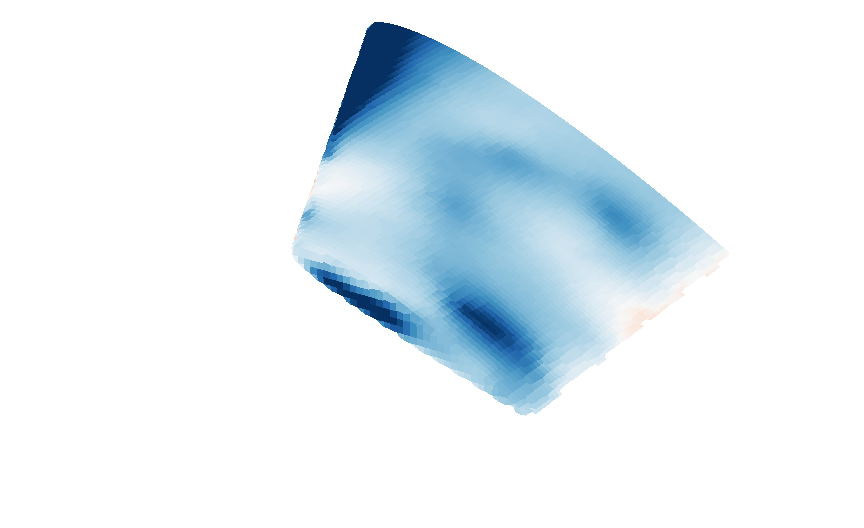} &
			\includegraphics[width=0.23\linewidth]{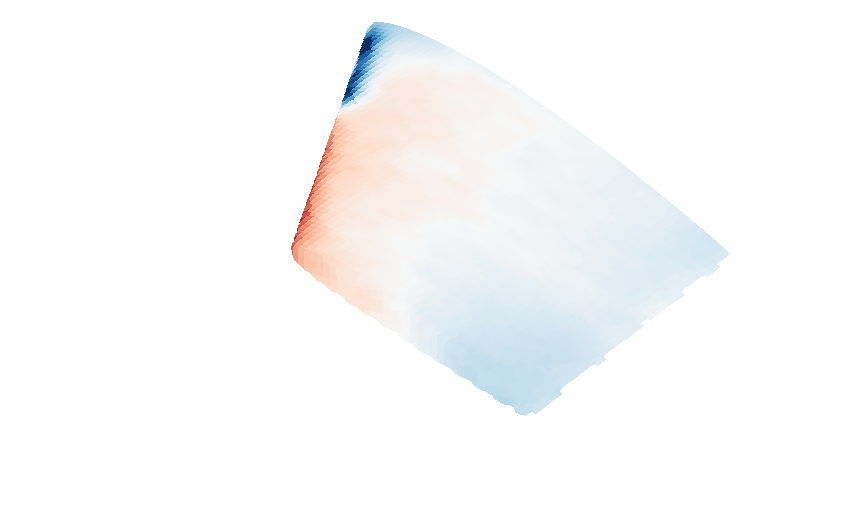}\\[-0.2em]
			\textbf{(c)} &
			\includegraphics[width=0.23\linewidth]{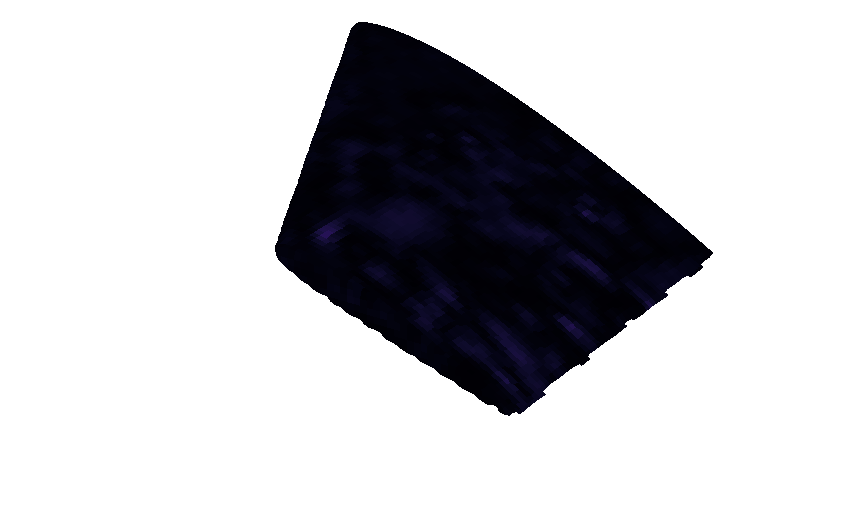} &
			\includegraphics[width=0.23\linewidth]{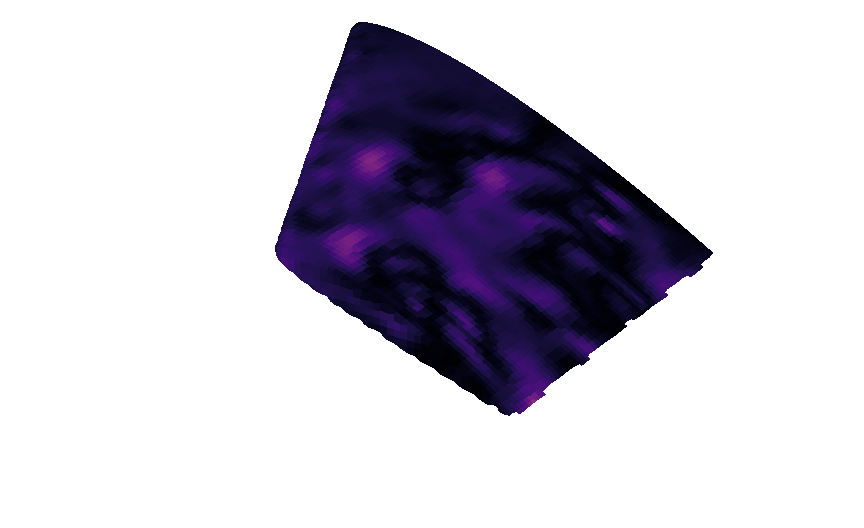} &
			\includegraphics[width=0.23\linewidth]{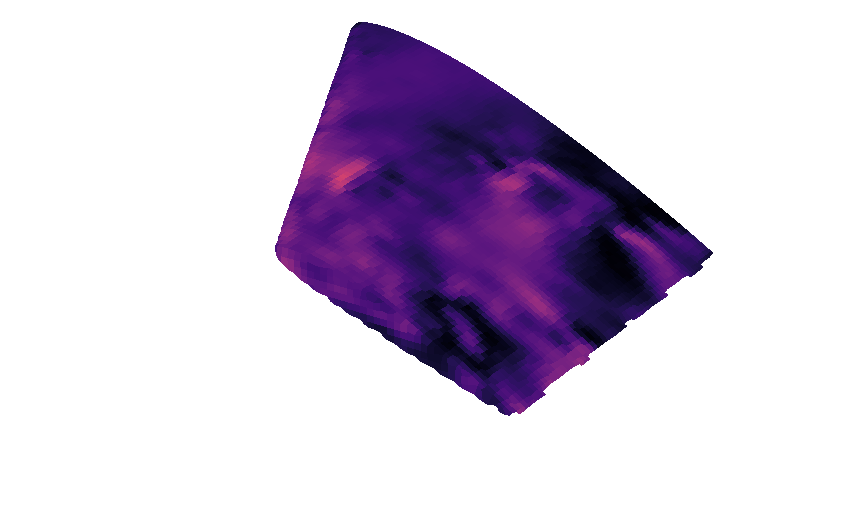} &
			\includegraphics[width=0.23\linewidth]{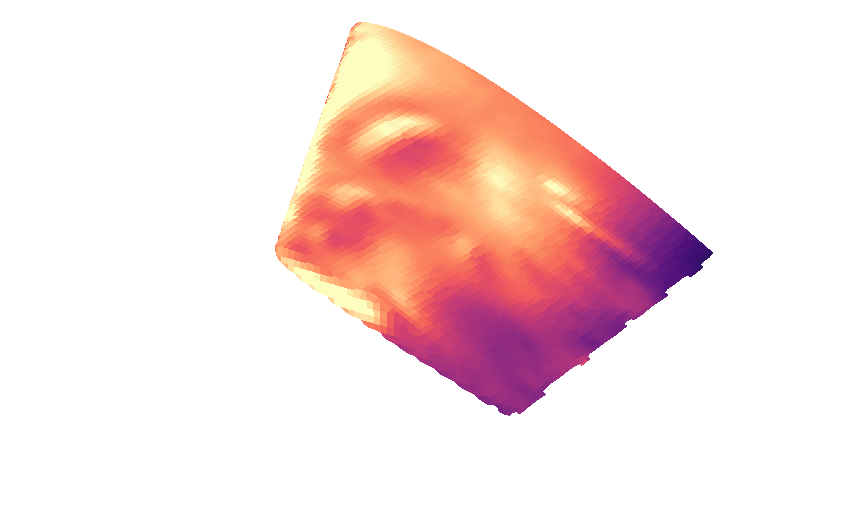}\\[-0.2em]
			\textbf{(d)} &
			\includegraphics[width=0.23\linewidth]{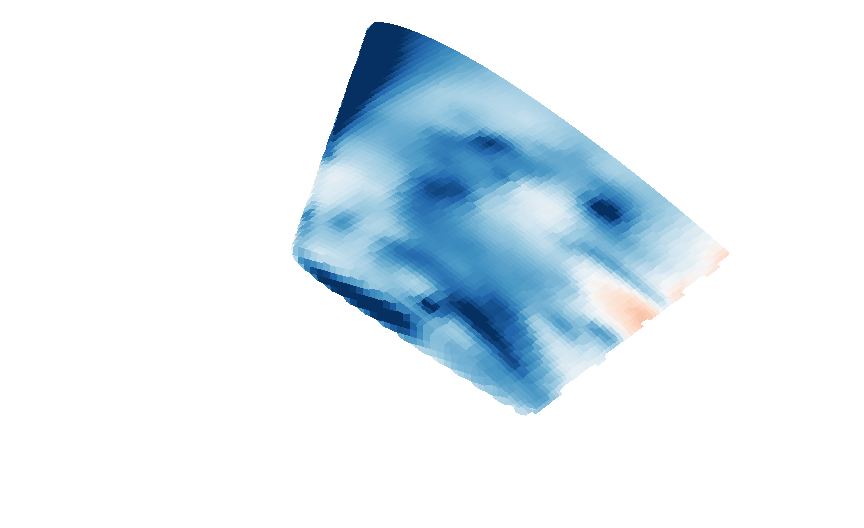} &
			\includegraphics[width=0.23\linewidth]{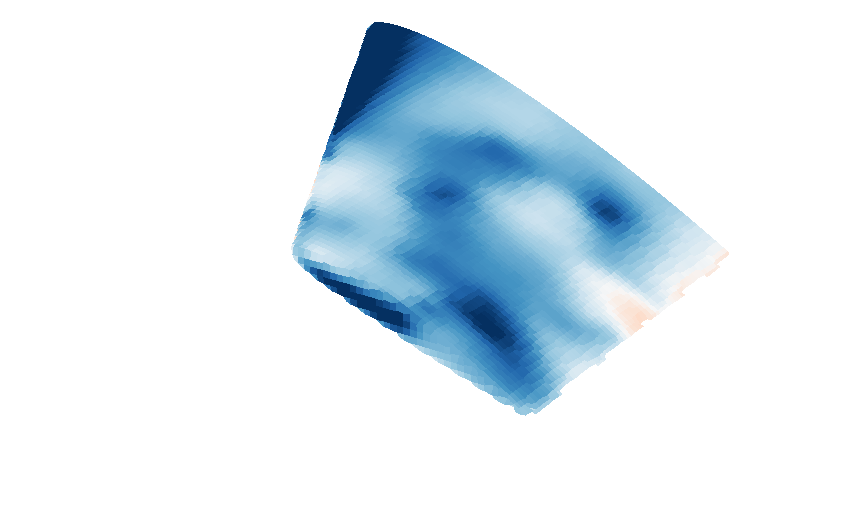} &
			\includegraphics[width=0.23\linewidth]{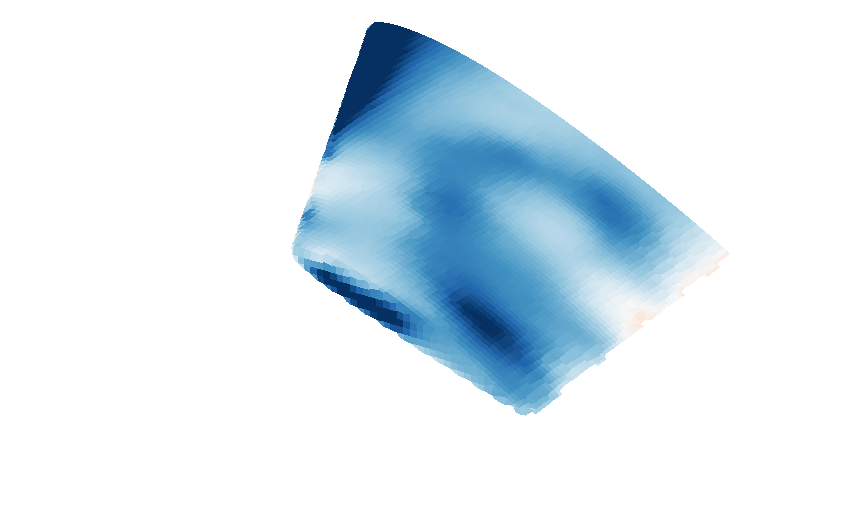} &
			\includegraphics[width=0.23\linewidth]{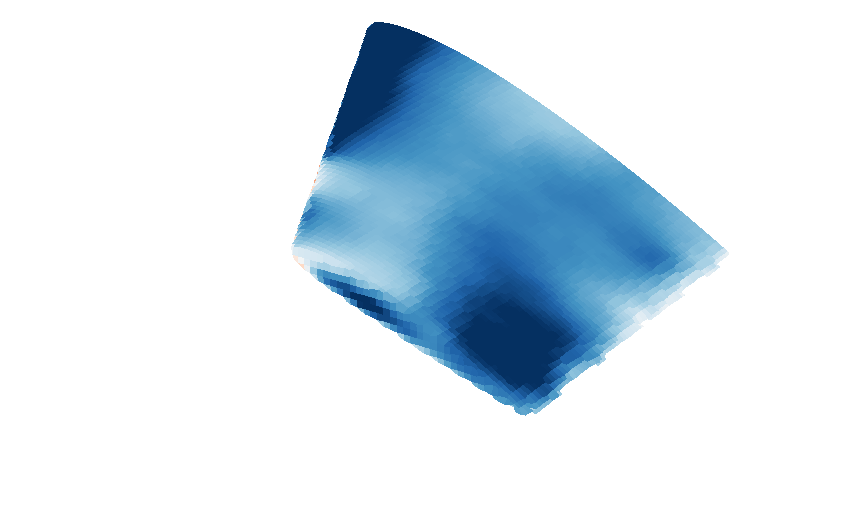}\\[-0.2em]
			\textbf{(e)} &
			\includegraphics[width=0.23\linewidth]{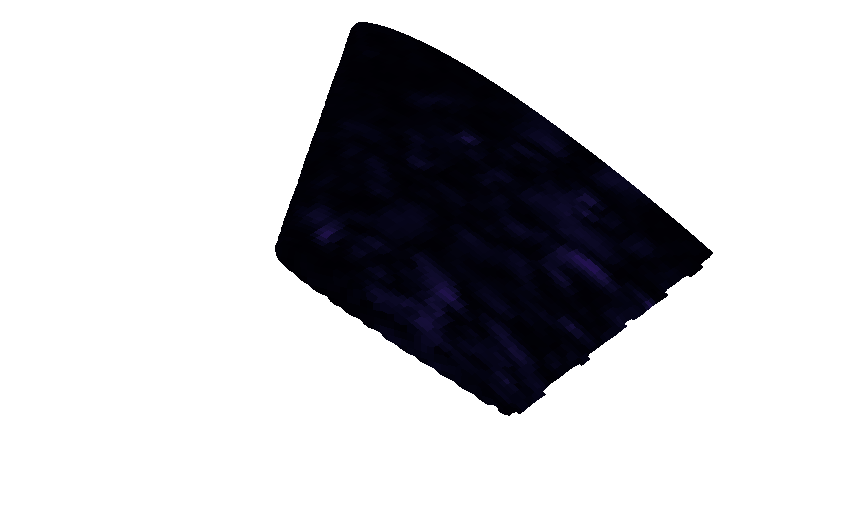} &
			\includegraphics[width=0.23\linewidth]{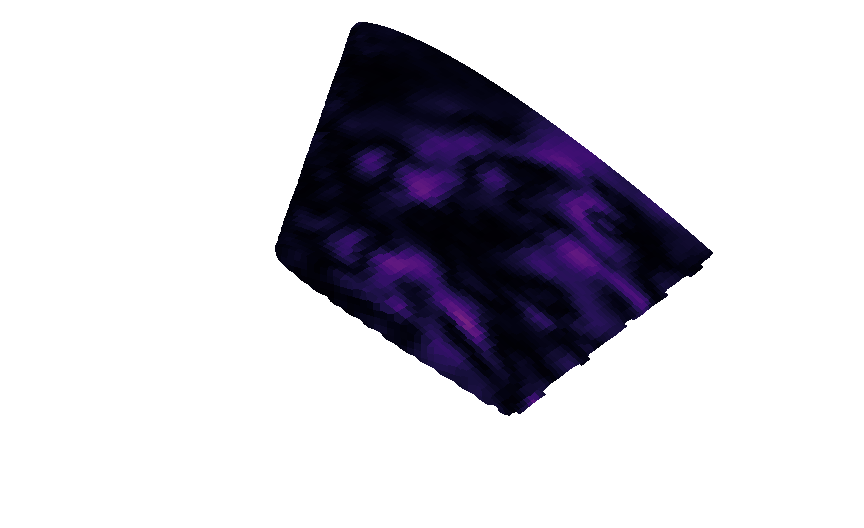} &
			\includegraphics[width=0.23\linewidth]{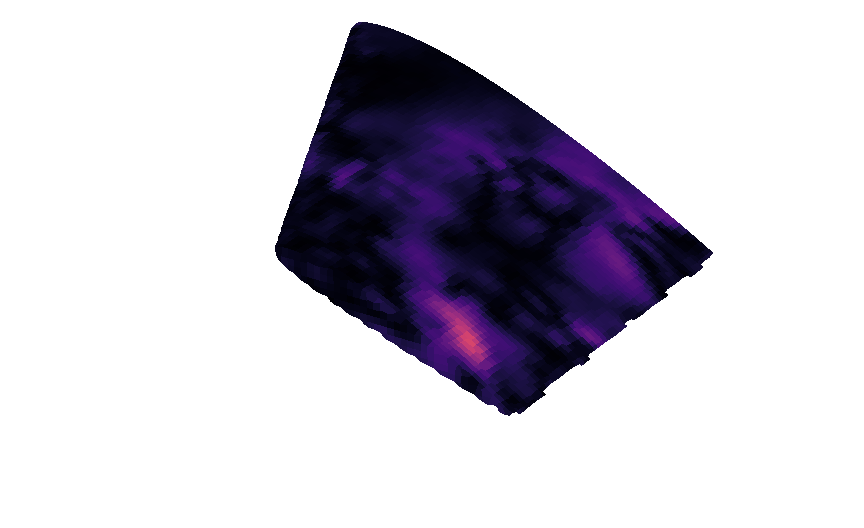} &
			\includegraphics[width=0.23\linewidth]{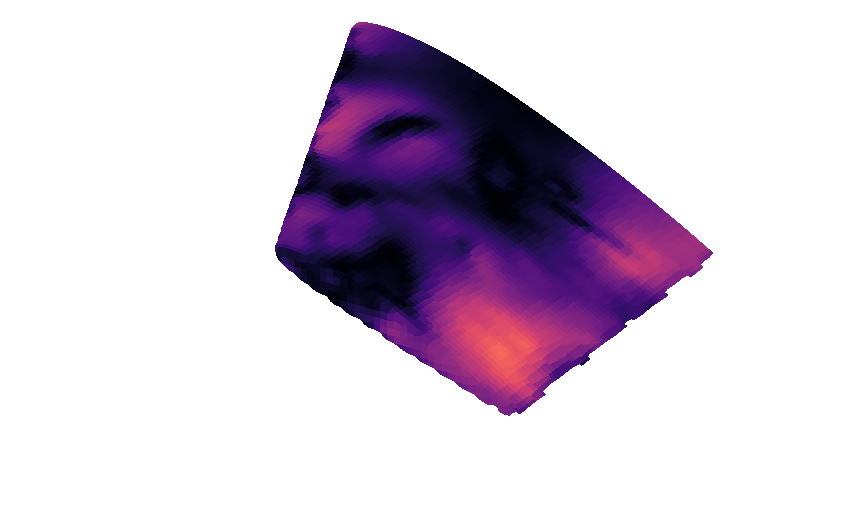}\\[-0.2em]
			\textbf{(f)} &
			\includegraphics[width=0.23\linewidth]{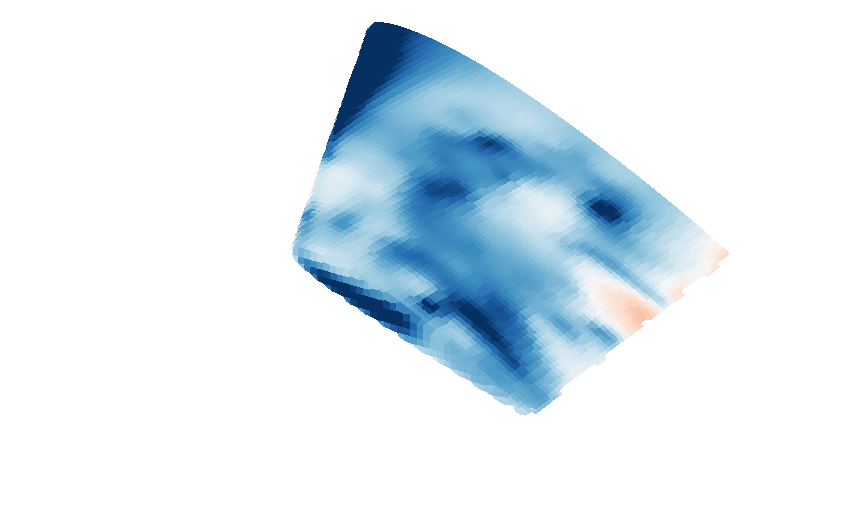} &
			\includegraphics[width=0.23\linewidth]{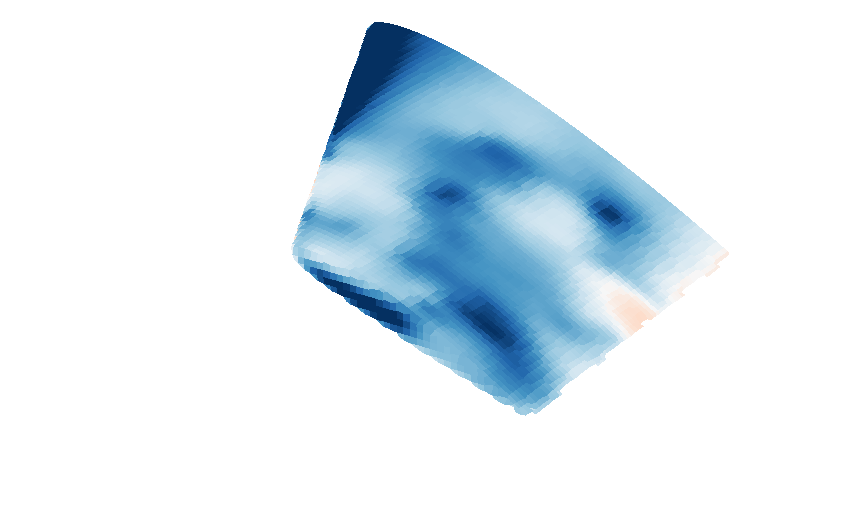} &
			\includegraphics[width=0.23\linewidth]{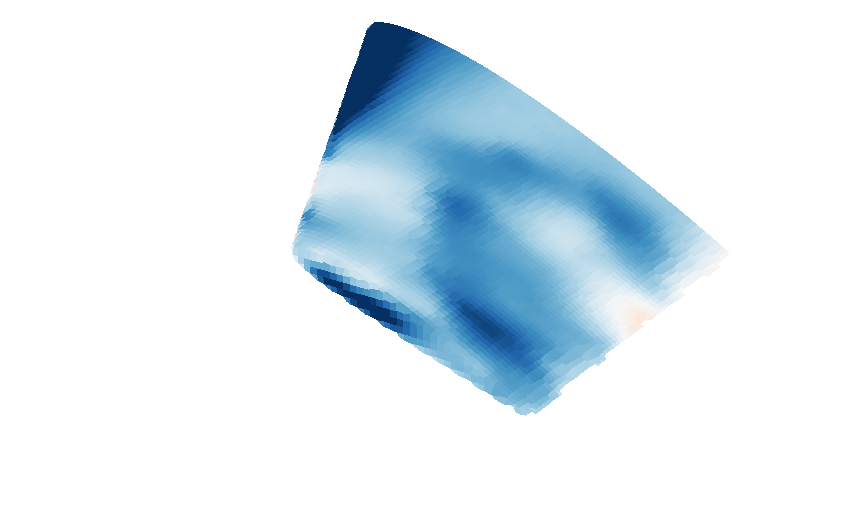} &
			\includegraphics[width=0.23\linewidth]{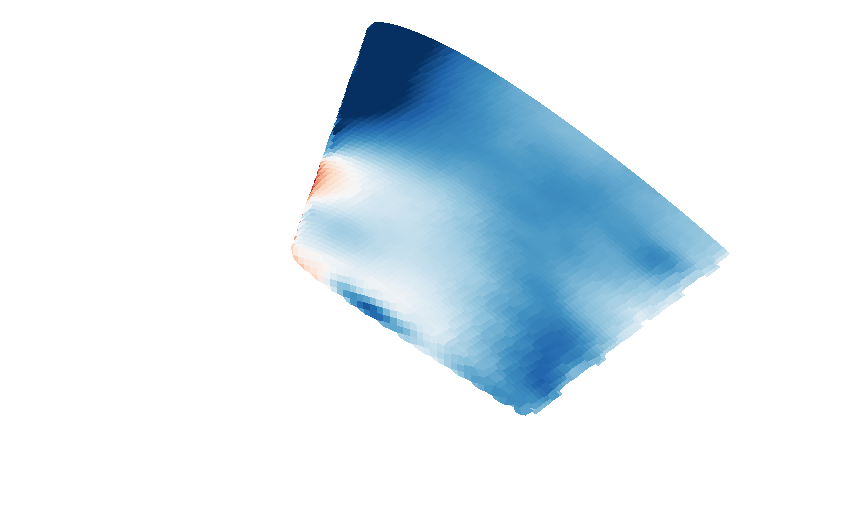}\\[-0.2em]
			\textbf{(g)} &
			\includegraphics[width=0.23\linewidth]{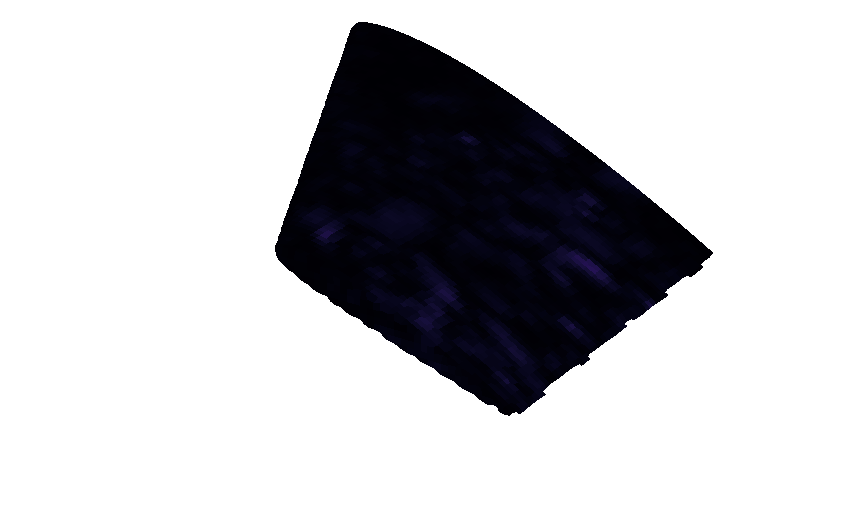} &
			\includegraphics[width=0.23\linewidth]{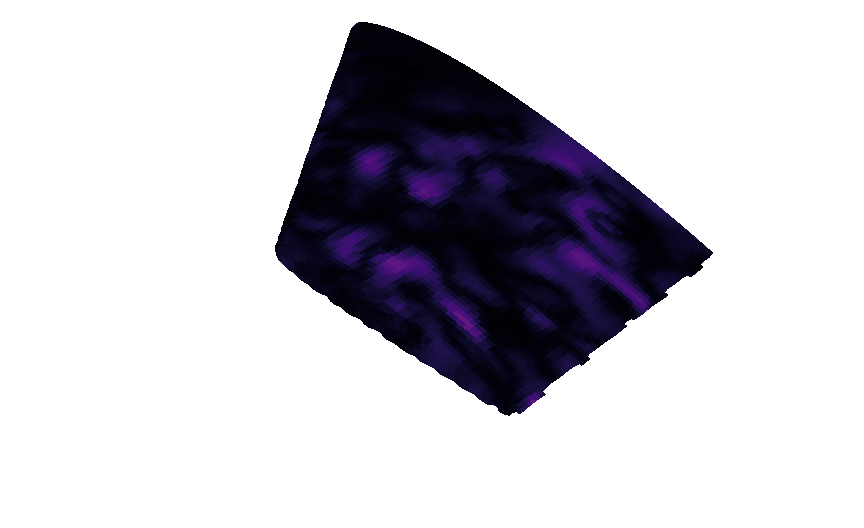} &
			\includegraphics[width=0.23\linewidth]{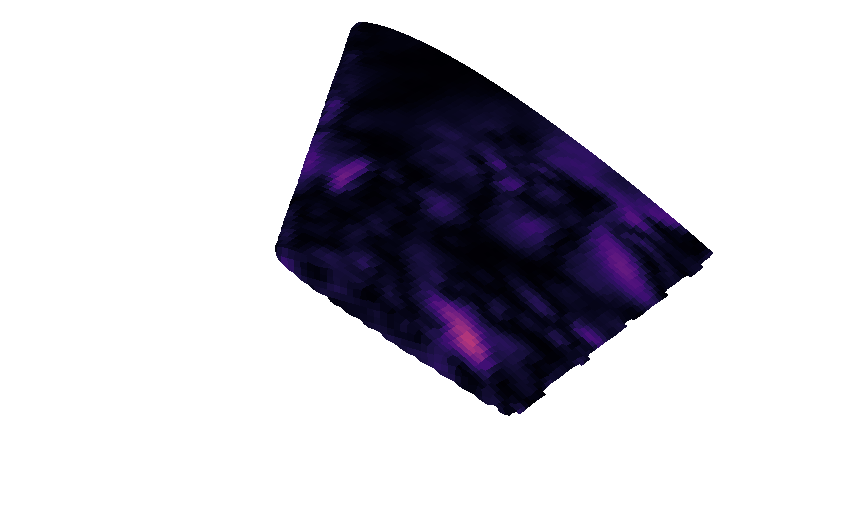} &
			\includegraphics[width=0.23\linewidth]{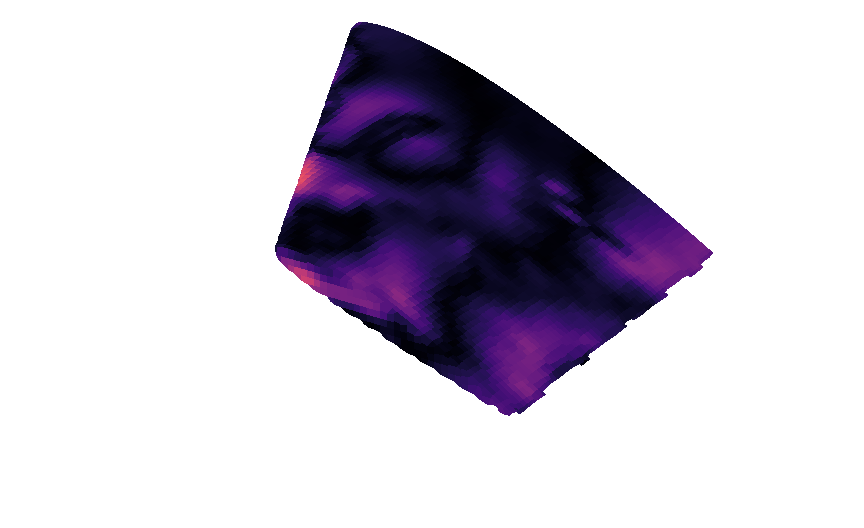}
		\end{tabular}
		\vspace{0.1em}
		\includegraphics[width=0.96\linewidth]{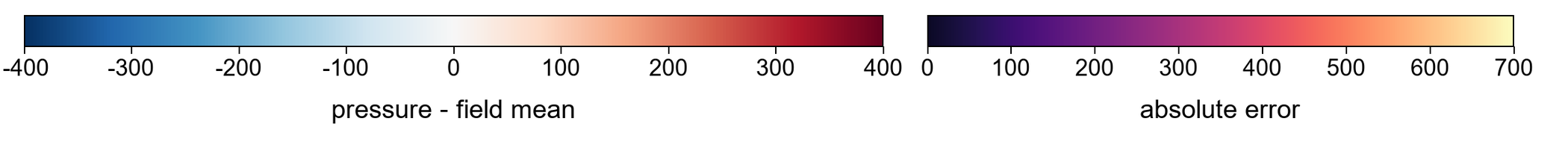}
		\caption{\texttt{pOnWing} surface-centered pressure comparison for simulation \(5\). Columns are rollout steps \(t=1,5,10,50\); rows show (a) centered ground truth, (b) Base prediction, (c) Base absolute error, (d) Chebyshev BSP prediction, (e) Chebyshev BSP absolute error, (f) GLEAM prediction, and (g) GLEAM absolute error.}
		\label{fig:ponwing_surface_centered_sim5}
	\end{figure}
	
	Figure~\ref{fig:ponwing_pressure_force_error} reports this force-level diagnostic over a \(1000\)-step autoregressive rollout on \(16\) held-out \texttt{pOnWing} simulations. The Base mean curve rapidly drifts to a relative pressure-force error near \(0.5\), whereas the GLEAM and Chebyshev BSP mean curves remain in the much lower \(0.12\)--\(0.15\) range over most of the rollout. The area-normalized diagnostic shows the same ordering, so the improvement is not only a consequence of the instantaneous force magnitude in the denominator. Chebyshev BSP is slightly better at the final horizon, while GLEAM is competitive and is lowest around the middle of the rollout. This supports the pressure-PDF evidence in Appendix~\ref{app:ponwing-benchmark}: spectral supervision improves not only the distribution of predicted pressure values but also the integrated loading induced by those pressures.
	
	\begin{figure}[!htbp]
		\centering
		\includegraphics[width=0.98\linewidth]{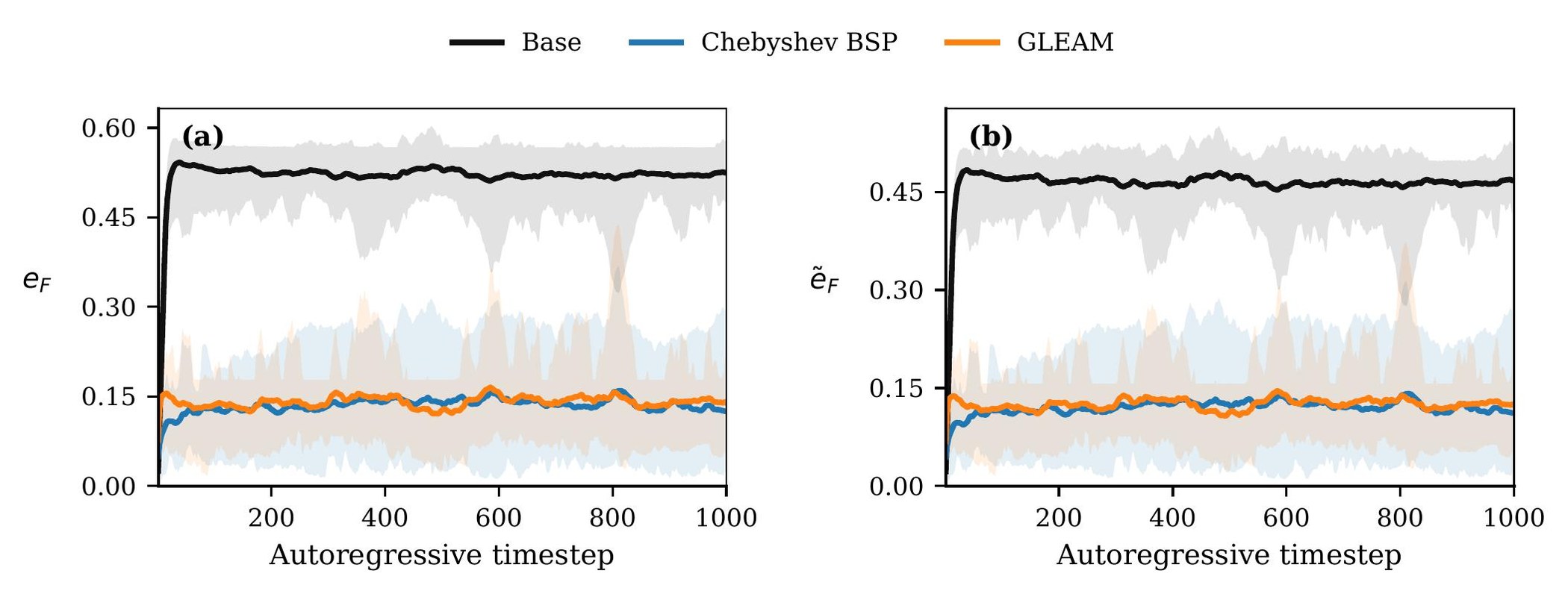}
		\caption{\texttt{pOnWing} pressure-force error over a \(1000\)-step autoregressive rollout for the trained Base, GLEAM, and Chebyshev BSP models. Curves show the mean over simulations \(0\)--\(15\), and shaded regions show the corresponding min--max range. Panel (a) reports \(e_F=\|\widehat{\bm{F}}_p(t)-\bm{F}_p(t)\|_2/\max(\|\bm{F}_p(t)\|_2,10^{-12})\); panel (b) reports \(\tilde e_F=\|\widehat{\bm{F}}_p(t)-\bm{F}_p(t)\|_2/\max(\sum_i |p_i(t)|A_i,10^{-12})\).}
		\label{fig:ponwing_pressure_force_error}
	\end{figure}
	
	Figure~\ref{fig:ponwing_fmgn_rmse_spectral_error} reports the corresponding FMGN aggregate over \(100\) autoregressive steps and \(16\) simulations. The physical-space RMSE and exact graph-spectral error show the same ordering: the Base model accumulates rapid rollout drift, Chebyshev BSP reduces both errors, and GLEAM gives the lowest mean curve across the evaluated horizon. The agreement between physical-space RMSE and exact graph-spectral error shows that the improvement appears not only in distributional pressure diagnostics but also in a graph-Laplacian band-energy metric on the FMGN backbone.
	
	\begin{figure}[!htbp]
		\centering
		\includegraphics[width=\linewidth]{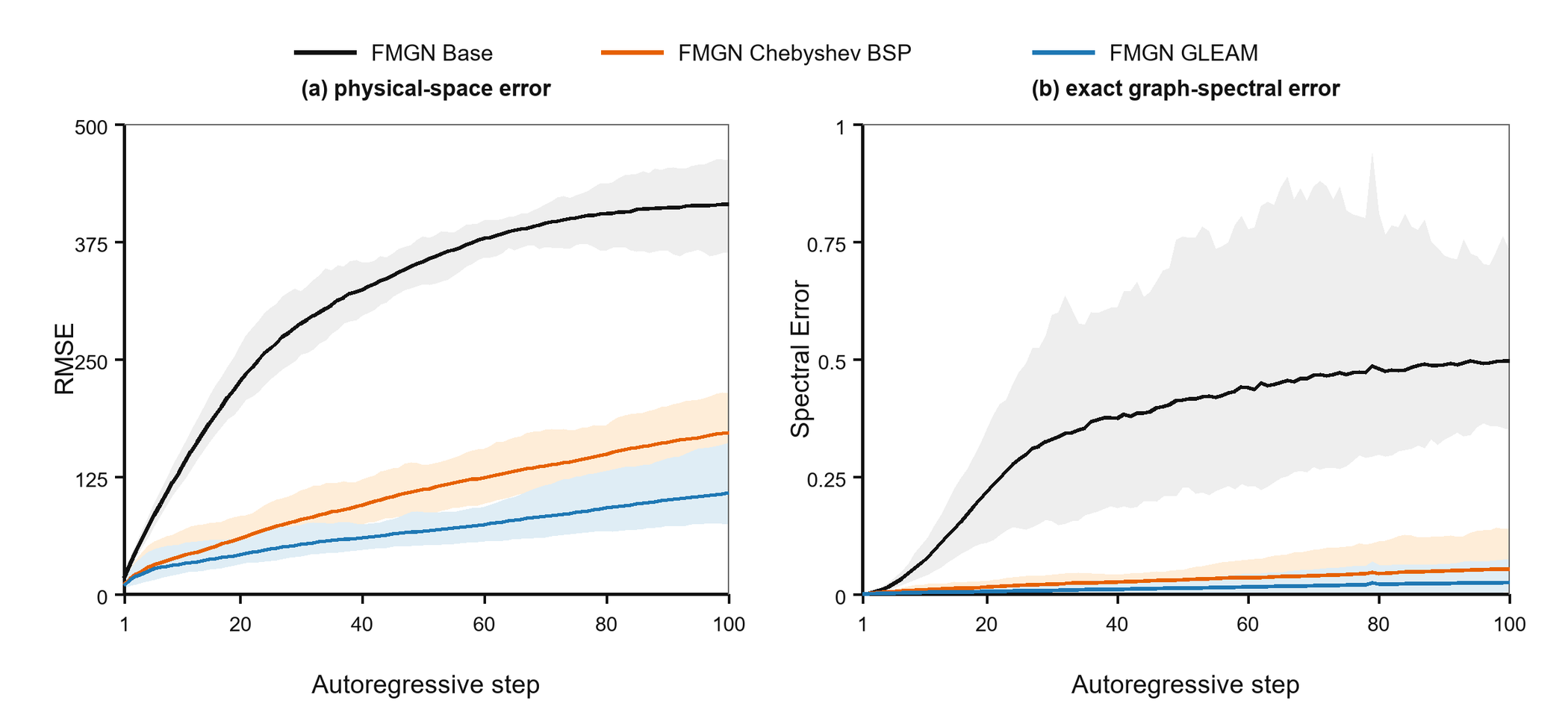}
		\caption{\texttt{pOnWing} FMGN aggregate rollout diagnostics over \(100\) autoregressive steps. Curves show the mean over \(16\) simulations, and shaded regions show the min--max range. Panel (a) reports physical-space RMSE, and panel (b) reports the non-log exact normalized-Laplacian linear BSP spectral error.}
		\label{fig:ponwing_fmgn_rmse_spectral_error}
	\end{figure}
	
	Additional VGAE exact graph-spectral and full-field qualitative checks are provided in Appendix~\ref{app:ponwing-benchmark}, Fig.~\ref{fig:ponwing_vgae_appendix_diagnostics}(a) and (b).
	
	\subsection{Graph-Spectral Rollout Diagnostics}
	\label{sec:spectral-diagnostics}
	The EAGLE results in Section~\ref{sec:eagle-results} already provide one form of long-horizon rollout-fidelity evidence. There, the vorticity RMSE and vorticity PDF \(L_1\) curves in Fig.~\ref{fig:eagle_vorticity_rmse_pdf_l1} show that Chebyshev BSP reduces the growth of localized vortical errors and better preserves the distribution of vorticity values, while Fig.~\ref{fig:eagle_qualitative} shows the same effect qualitatively as the later \(t=310\) field becomes more turbulent. Thus, EAGLE supports the claim that spectral supervision helps reduce autoregressive smoothing and loss of small- to mid-scale flow content.
	
	The BFS analysis below provides a more explicit graph-band view of the same spectral-drift mechanism. Instead of measuring vorticity statistics, it measures how much predicted \(U_y\) energy is misplaced across graph-Laplacian bands during rollout.
	
	For the BFS rollout, the RMSE curves in Fig.~\ref{fig:bfs2d_results} quantify long-horizon pointwise error growth, while Fig.~\ref{fig:bfs_uy_band_mismatch_density} checks whether the same models preserve graph-band energy. We compute the absolute \(U_y\) band-energy mismatch using dense normalized graph-Laplacian eigenpairs and \(24\) spectral bands. The log-mismatch density shows the distribution of band errors at representative rollout steps \(t=30\) and \(t=82\), with the full band-energy RMSE trajectory and corresponding \(U_y\) field snapshots reported together in Fig.~\ref{fig:bfs_uy_rollout_appendix}(a) and (b).
	
	\begin{figure}[!htbp]
		\centering
		\includegraphics[width=\linewidth]{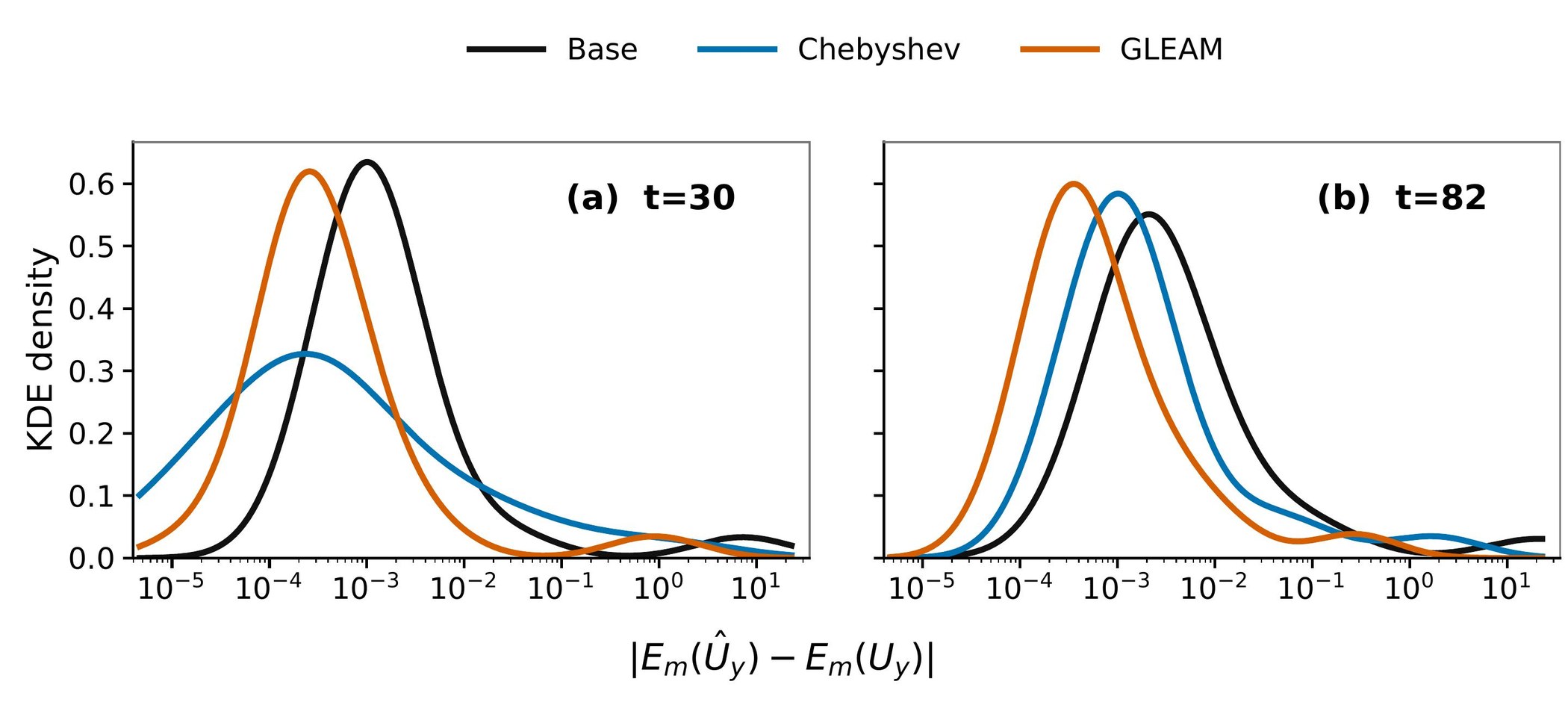}
		\caption{\(U_y\) graph-band energy mismatch distributions on the backward-facing step rollout at \(t=30\) and \(t=82\). Each curve is a smoothed KDE density over \(\log_{10}|E_m(\hat{U}_y)-E_m(U_y)|\) across \(24\) graph-Laplacian bands, plotted against the original mismatch values on a logarithmic axis.}
		\label{fig:bfs_uy_band_mismatch_density}
	\end{figure}
	
	This diagnostic is complementary to pointwise RMSE because it asks a different question: whether the prediction places the correct amount of kinetic structure at each graph-resolved spatial scale. A model can have a visually plausible field while still allocating too much energy to large graph-scale modes or suppressing the bands associated with the separated shear layer. Such spectral redistribution is especially important in autoregressive rollout, where small scale-allocation errors at early steps can feed back into the next predicted state and gradually alter the coherent structures near the step and downstream recirculation region.
	
	The two selected times probe different regimes of the same rollout. At \(t=30\), the models are still close enough to the initial condition that the mismatch mostly reflects how each loss shapes the early redistribution of band energy. At \(t=82\), accumulated forecast error has had more time to interact with the separated flow, so the tails of the mismatch distribution reveal whether a method reduces the tendency for a few graph bands to become strongly over- or under-energized. The lower GLEAM error in Fig.~\ref{fig:bfs_uy_rollout_appendix}(a) and the consistent leftward shift in Fig.~\ref{fig:bfs_uy_band_mismatch_density} suggest that, for this hierarchy, hierarchical spectral supervision improves rollout fidelity by regularizing retained scale allocation rather than only reducing pointwise error.
	
	\subsection{\texorpdfstring{Computational Cost Comparison}{Computational Cost Comparison}}
	\label{sec:cost-comparison}
	The three graph-spectral objectives differ primarily in their computational bottlenecks. Exact Graph BSP is the reference formulation, but it requires either a dense eigensystem or repeated sparse partial eigensolves; it is therefore best suited for analysis, small graphs, or fixed graphs where the eigenbasis can be cached. Chebyshev BSP removes eigendecomposition and replaces each band projector with sparse Laplacian recurrences. For \(M\) bands, polynomial order \(K\), \(C\) channels, and \(|\mathcal{E}|\) graph edges, a direct single-level implementation scales as \(\mathcal{O}(MKC|\mathcal{E}|)\) per loss evaluation after the Chebyshev coefficients have been formed. If the same full Chebyshev BSP loss is applied across a graph hierarchy, this becomes \(\sum_{\ell=1}^{S}\mathcal{O}(MKC|\mathcal{E}_\ell|)\), or \(\sum_{\ell=1}^{S}\mathcal{O}(KC|\mathcal{E}_\ell|+MKCN_\ell)\) when Chebyshev recurrences are shared across bands at each level. Thus, full hierarchical Chebyshev supervision remains eigenfree but can still be expensive because sparse graph filtering is repeated on every supervised level.
	
	GLEAM has a different computational profile. It computes or reuses a rank-\(r\) spectral embedding on selected graph levels and then evaluates low-rank band energies and pairwise geometric terms on the same hierarchy used by the multiscale backbone. Once the embedding is cached, the online cost over supervised levels is approximately \(\sum_{\ell\in\mathcal{H}_{\mathrm{sup}}}\mathcal{O}(N_\ell r C+rC+P_\ell C)\). Thus, Chebyshev BSP gives more fine-band control on the supervised graph, while GLEAM is lower-cost and better aligned with supervision across pooled coarse graphs when those graphs retain the relevant low- and mid-frequency subspaces. This is the practical distinction used in the experiments: Chebyshev BSP is the full-band approximation used for fine-level band control, whereas GLEAM is the hierarchy-aware route used for multilevel supervision. The full setup and per-step complexity expressions are provided in Appendix~\ref{app:complexity}.
	
	\begin{table}[!htbp]
		\centering
		\begin{minipage}[t]{0.56\linewidth}
			\vspace{0pt}
			\centering
			\caption{Dominant computational cost of the graph-spectral losses.}
			\label{tab:loss_cost_summary}
			\scriptsize
			\setlength{\tabcolsep}{2pt}
			\resizebox{\linewidth}{!}{%
				\begin{tabular}{@{}ll@{}}
					\toprule
					Loss & Dominant cost \\
					\midrule
					Exact Graph BSP & \(\mathcal{O}(N^3)+\mathcal{O}(N^2C)\) \\
					Chebyshev BSP & \(\mathcal{O}(MKC|\mathcal{E}|)\) \\
					Hierarchical Chebyshev BSP & \(\sum\nolimits_{\ell=1}^{S}\mathcal{O}(MKC|\mathcal{E}_\ell|)\) \\
					GLEAM & \(\sum\nolimits_{\ell\in\mathcal{H}_{\mathrm{sup}}}\mathcal{O}(N_\ell r C+rC+P_\ell C)\) \\
					\bottomrule
				\end{tabular}
			}
		\end{minipage}\hfill
		\begin{minipage}[t]{0.42\linewidth}
			\vspace{0pt}
			\centering
			\caption{Empirical training cost for the VGAE backbone.}
			\label{tab:vgae_training_cost}
			\scriptsize
			\setlength{\tabcolsep}{3pt}
			\begin{tabular*}{\linewidth}{@{\extracolsep{\fill}}lc@{}}
				\toprule
				Spectral variant & s/it \\
				\midrule
				Chebyshev BSP & \(82\) \\
				GLEAM no corrector & \(64.5\) \\
				GLEAM scalar & \(70.8\) \\
				\bottomrule
			\end{tabular*}
		\end{minipage}
	\end{table}
	
	Here the hierarchy has \(S\) levels; \(G_\ell\) has \(N_\ell\) nodes and \(|\mathcal{E}_\ell|\) edges, and \(P_\ell\) is the number of sampled pairwise terms at level \(\ell\). The set \(\mathcal{H}_{\mathrm{sup}}\) denotes the levels where GLEAM supervision is evaluated. Table~\ref{tab:vgae_training_cost} gives a wall-clock snapshot for the Graph VAE backbone; these measurements illustrate the practical effect of replacing full Chebyshev filtering with low-rank GLEAM variants under the same training setup.
	
	\section{Limitations and Future Work}
	The proposed losses introduce a multi-objective training problem: the forecasting backbone is optimized for local field accuracy while the auxiliary graph-spectral term constrains the distribution of energy across mesh-resolved scales. These objectives are complementary but not identical, and their relative weight can affect the balance between pointwise error reduction and long-horizon spectral fidelity. In the present experiments this balance is handled with fixed spectral weights and ablations over the main approximation parameters. A useful next step is to make this balancing more adaptive, for example by using continuation schedules or gradient-aware weighting so that spectral regularization is introduced in proportion to rollout drift rather than as a fixed penalty throughout training.
	
	The two scalable approximations also have different operating regimes. Chebyshev BSP gives eigenfree control of graph-frequency bands on the supervised graph, but its accuracy depends on the polynomial order, quadrature resolution, and sharpness of the target windows. GLEAM is intentionally different: it supervises a retained low-rank graph-resistance geometry across the hierarchy, making it well suited for coherent structures and long-range organization carried by multilevel graph backbones. Its limitation is therefore not computational but spectral: fine-scale discrepancies outside the retained subspace are less explicitly constrained than in full-band Chebyshev BSP. This is why GLEAM should be viewed as a hierarchy-aware complement to Chebyshev BSP rather than a universal replacement for fine-level band filtering.
	
	A useful future extension is to combine these strengths by pairing GLEAM with a small targeted final-level high-frequency Chebyshev correction. In that hybrid view, GLEAM would continue to regularize coherent coarse organization through the model hierarchy, while a small number of finest-graph bands would constrain unresolved high-frequency content. This would retain much of the computational advantage of GLEAM relative to full hierarchical Chebyshev BSP, but would address cases where late-time rollout errors are dominated by small vortical structures, sharp pressure variations, or other fine-scale features. Further physical diagnostics, such as intermittency measures, QR plots, or force/invariant-based statistics, can then be used to test which flow regimes benefit most from the coarse-plus-fine decomposition.
	
	\section{Conclusion}\label{sec:conclusion}
	This work extends binned spectral-power supervision from structured Fourier grids to unstructured-mesh surrogate modeling by using the graph Laplacian as the intrinsic frequency operator. The resulting Graph BSP objective compares predicted and target energy over graph-frequency bands, providing a way to penalize spectral drift on irregular meshes where Euclidean Fourier modes are not available. Exact Graph BSP provides the reference formulation, Chebyshev BSP replaces explicit eigendecomposition with sparse polynomial graph filters, and GLEAM adapts the same scale-aware principle to multilevel graph backbones through low-rank graph-resistance geometry and hierarchy-aware pairwise constraints. Together, these variants form a cost--fidelity hierarchy rather than a single architecture-specific modification.
	
	The experiments show that this graph-spectral supervision improves behavior that pointwise losses alone do not reliably control. On EAGLE, Chebyshev BSP improves both rollout RMSE and spectral-energy agreement while better preserving vortical structure in increasingly turbulent scenes. On the backward-facing step case, Chebyshev supervision reduces final-level spectral drift and GLEAM further benefits from applying the constraint across the graph hierarchy used by the forecasting backbone. On \texttt{pOnWing}, the spectral variants improve pressure-field structure, pressure-distribution diagnostics, and integrated pressure-force errors on a complex three-dimensional aerodynamic surface. These results support the central claim that matching scale-resolved graph energy is a practical mechanism for improving autoregressive mesh-based forecasting in the benchmarks studied here. The method does not replace physical modeling choices or careful architecture design; instead, it supplies a modular loss-level constraint that can be attached to existing graph surrogates to preserve mesh-resolved structure across long rollouts.
	
	\appendix
	\section*{Appendix}
	
	\section{Normalized Laplacian Spectral Bound}
	\label{app:laplacian-bound}
	
	\begin{proposition}[Normalized Laplacian spectral bound]
		\label{prop:cheb-normalized-laplacian-spectrum}
		Let \(G=(V,\mathcal{E},W)\) be an undirected graph with nonnegative symmetric weights and positive degrees \(d_i=\sum_j A_{ij}\). For the symmetric normalized Laplacian
		\[
		L=I-D^{-1/2}AD^{-1/2},
		\]
		all eigenvalues satisfy
		\[
		0\le \lambda_k\le 2.
		\]
		This is a standard property of the normalized graph Laplacian \citep{chung1997spectral,shuman2013signal}.
	\end{proposition}
	
	\begin{proof}
		For any vector \(x\in\R^N\), the quadratic form of \(L\) can be written as
		\[
		x^\top Lx
		=
		\frac{1}{2}\sum_{i,j}A_{ij}
		\left(
		\frac{x_i}{\sqrt{d_i}}-\frac{x_j}{\sqrt{d_j}}
		\right)^2.
		\]
		Hence \(x^\top Lx\ge 0\), so \(L\) is positive semidefinite and every eigenvalue is nonnegative. For the upper bound, use \((a-b)^2\le 2a^2+2b^2\) in the expression above:
		\[
		x^\top Lx
		\le
		\frac{1}{2}\sum_{i,j}A_{ij}
		\left(
		2\frac{x_i^2}{d_i}+2\frac{x_j^2}{d_j}
		\right)
		=
		2\sum_i x_i^2
		=2x^\top x.
		\]
		Therefore, for every nonzero \(x\),
		\[
		0\le \frac{x^\top Lx}{x^\top x}\le 2.
		\]
		The eigenvalues of the real symmetric matrix \(L\) are extrema of this Rayleigh quotient, so \(0\le \lambda_k\le 2\) for all \(k\).
	\end{proof}
	
	\section{Dataset and Training Details}
	\label{app:training-details}
	
	This appendix records the training setup used for the experiments without repeating the dataset descriptions in Section~\ref{sec:experiments}. The common comparison protocol is to keep the data split, rollout horizon, predicted variables, and fine-level prediction interface fixed within each benchmark. The Base model uses the pointwise training loss alone, the Chebyshev model adds a final-level band-power penalty without changing the main forecasting backbone, and GLEAM adds hierarchical spectral supervision on the graph levels used by the multiscale propagator. For GLEAM, auxiliary level-normalization and prediction-head modules are used to expose intermediate hierarchy-level fields for the training loss, while the autoregressive fine-level input--output interface remains the same at evaluation time. Unless stated otherwise, validation and test metrics are computed from autoregressive rollouts rather than isolated one-step predictions, since the main goal is long-horizon rollout fidelity.
	
	All computations were performed on the Aurora cluster. In the compute accounting used here, each Aurora node provides \(12\) GPU units, and one training rank is assigned to each GPU unit unless otherwise stated. Therefore, the dataset-specific training descriptions below report the number of Aurora nodes rather than repeatedly listing the per-node GPU count.

	\subsection{EAGLE}
	EAGLE is generated from two-dimensional unsteady airflow simulations of a moving UAV interacting with procedurally generated floor geometries \citep{janny2023eagle}. The original benchmark contains \(1{,}184\) simulations, each with \(990\) time steps corresponding to \(33\) seconds at \(30\) fps. The scenes are grouped into Step, Triangular, and Spline floor profiles, with \(197\), \(199\), and \(196\) geometries respectively; each geometry produces two simulations depending on whether the UAV crosses the scene to the left or to the right. We use the published split of \(948\) training simulations, \(118\) validation simulations, and \(118\) test simulations.
	
	The simulations were performed in Ansys Fluent using the Reynolds-averaged Navier--Stokes equations with a Reynolds-stress turbulence model \citep{janny2023eagle}. The raw simulations contain approximately \(162{,}760\) control points per mesh and about \(3.9\) TB of raw data; these are downsampled to an irregular triangular graph with \(3{,}388\) nodes on average and compressed to about \(270\) GB. Because the UAV moves through the scene, the mesh is dynamically adapted over time, but the future mesh geometry is provided to the forecasting model rather than being predicted. For reproducibility, all our experiments are recorded by separating the rollout protocol, optimizer, distributed batch construction, and spectral-loss settings. The backbone and data protocol are held fixed across the Base and Chebyshev BSP ablations in Tables~\ref{tab:eagle_chebyshev_ablation} and~\ref{tab:eagle_spectral_energy_rel_l1}; only the auxiliary graph-spectral objective and its approximation parameters are varied.
	
	\begin{table}[!htbp]
		\centering
		\caption{Training configuration summary for the datasets used in the experiments.}
		\label{tab:eagle-training-details}
		\scriptsize
		\setlength{\tabcolsep}{3pt}
		\begin{tabular}{@{}p{0.20\linewidth}p{0.24\linewidth}p{0.25\linewidth}p{0.21\linewidth}@{}}
			\toprule
			\textbf{Setting} &
			\begin{tabular}[c]{@{}l@{}}\textbf{EAGLE}\\[-0.2ex]{\scriptsize \citep{janny2023eagle}}\end{tabular} &
			\begin{tabular}[c]{@{}l@{}}\textbf{Backward-facing step}\\[-0.2ex]{\scriptsize \citep{kim2025generalizable}}\end{tabular} &
			\begin{tabular}[c]{@{}l@{}}\textbf{\texttt{pOnWing}}\\[-0.2ex]{\scriptsize \citep{lino2025dgn}}\end{tabular} \\
			\midrule
			Predicted fields & pressure and velocity & \(U_x,U_y\) & pressure \(p\) \\
			Model & EAGLE mesh transformer & U-Net-style graph neural network with GATv2 & VGAE; FMGN \\
			Epochs & \(1000\) & \(200\) & VGAE: \(600\); FMGN: \(1000\) \\
			Optimizer & Adam & AdamW & Adam \\
			Learning rate & \(10^{-4}\) & \(5\times 10^{-4}\) & VGAE: \(10^{-4}\); FMGN: \(10^{-4}\) \\
			Train horizon & \(6\) & \(1\)--\(4\)-step rollout curriculum & prefix length \(T=250\) \\
			Validation horizon & \(25\) & held-out contiguous rollout & test simulations to \(2501\) steps \\
			Compute layout & \(4\) Aurora nodes & \(1\) Aurora node & \(1\) Aurora node \\
			Per-rank batch size & \(1\) & \(4\) & VGAE: \(6\); FMGN: \(6\) \\
			Effective global batch & \(48\) & \(48\) & \(72\) \\
			Spectral weight & \(\lambda_{\mathrm{Cheb}}=0.01\) & \(\lambda_{\mathrm{Cheb}}=0.1\); \(\lambda_{\mathrm{GLEAM}}=10^{-4}\) & \(\lambda_{\mathrm{Cheb}}\): VGAE \(10^{-2}\), FMGN \(10^{-4}\); \(\lambda_{\mathrm{GLEAM}}\): VGAE \(10^{-3}\), FMGN \(10^{-4}\) \\
			Spectral fields & all channels & \(U_x,U_y\) & \(p\) \\
			\(\varepsilon_s\) & \(10^{-6}\) & \(10^{-6}\) & \(10^{-6}\) \\
			\bottomrule
		\end{tabular}
	\end{table}
	
	The EAGLE column in Table~\ref{tab:eagle-training-details} corresponds to the graph-filtered band-energy objective \(H_m^{(K,Q)}\) used in the EAGLE ablation tables. The run uses the EAGLE mesh-transformer backbone with Adam for \(1000\) epochs, learning rate \(10^{-4}\), train horizon \(6\), validation horizon \(25\), \(20\) clusters, \(4\) Aurora nodes, and an effective global batch size of \(48\). The Chebyshev spectral penalty is applied to all predicted state channels with \(\lambda_{\mathrm{Cheb}}=0.01\), so the reported spectral-energy relative \(L_1\) errors measure the effect of band-energy matching under a fixed rollout and distributed-training protocol.
	
	The ablation varies the approximation parameters in \(H_m^{(K,Q)}\) rather than changing the optimizer or rollout schedule. Thus \(K\) controls the Chebyshev filter order, \(M\) controls the number of band-power bins, and \(Q\) controls the quadrature resolution used to approximate the band energies. This convention makes Tables~\ref{tab:eagle_chebyshev_ablation} and~\ref{tab:eagle_spectral_energy_rel_l1} directly comparable: differences between rows reflect the spectral approximation and not a change in training budget.
	
	\subsection{Backward-Facing Step}
	For the backward-facing step experiments, we train graph neural forecasters on the cropped velocity field. The graph contains \(11{,}930\) spatial nodes and \(70{,}272\) directed edge entries, with the two velocity components \(U_x\) and \(U_y\) stored at each node. The first \(100\) snapshots are discarded as transient, leaving \(924\) usable snapshots from the original \(1024\)-step sequence. After accounting for the four-step target horizon, the temporal samples are split contiguously into \(644\) training, \(92\) validation, and \(184\) test samples, corresponding to a \(70/10/20\) split of the post-transient sequence.
	
	All BFS spectral variants use the same forecasting backbone: a U-Net-style graph neural network with \(256\) hidden channels, Gaussian Fourier positional features, GATv2 message passing with four attention heads, three graph coarsening scales with voxel sizes \(0.015\), \(0.03\), and \(0.06\), and direct prediction of the next normalized velocity field rather than a residual update. The model is trained as a one-step forecaster but optimized with a short autoregressive curriculum. The rollout horizon is increased from one to four steps during training: one-step rollouts for epochs \(1\)--\(80\), two-step rollouts for epochs \(81\)--\(120\), three-step rollouts for epochs \(121\)--\(160\), and four-step rollouts thereafter. This follows a pushforward-style strategy for exposing the forecaster to its own predicted states during training \citep{brandstetter2022message}. Predictions are detached between rollout steps, so each predicted step contributes to the loss without backpropagating through the full autoregressive chain.
	
	The Chebyshev BSP run optimizes
	\[
	\mathcal{L}_{\mathrm{cheb}}
	=
	\mathrm{MSE}(\widehat{x}_{t+1},x_{t+1})
	+\lambda_{\mathrm{Cheb}}\mathcal{L}_{\mathrm{BSP}}^{\mathrm{Cheb}},
	\]
	where the band-power term approximates graph spectral band energies using Chebyshev polynomial filtering rather than an explicit eigendecomposition. The final Chebyshev configuration uses \(M=24\) spectral bands, Chebyshev order \(K=12\), \(Q=128\) quadrature points, and spectral weight \(\lambda_{\mathrm{Cheb}}=0.1\). It is trained for \(200\) epochs with AdamW, learning rate \(5\times 10^{-4}\), weight decay \(10^{-4}\), cosine annealing to one percent of the initial learning rate after a five-epoch warmup, gradient clipping at norm \(1.0\), bfloat16 mixed precision, and exponential moving average of the model weights after the first \(10\) epochs.
	
	The GLEAM run uses the same pointwise forecasting objective but replaces the Chebyshev filter-bank penalty with hierarchical spectral supervision,
	\[
	\mathcal{L}_{\mathrm{GLEAM}}
	=
	\mathrm{MSE}(\widehat{x}_{t+1},x_{t+1})
	+\lambda_{\mathrm{GLEAM}}\mathcal{L}_{\mathrm{spec}}^{\mathrm{GLEAM}}.
	\]
	GLEAM evaluates \(\mathcal{L}_{\mathrm{spec}}^{\mathrm{GLEAM}}\) on the same graph hierarchy used by the multiscale forecasting backbone. Let \(\mathcal{G}^{(\ell)}=(\mathcal{V}^{(\ell)},\mathcal{E}^{(\ell)})\) denote hierarchy level \(\ell\), and let \(P^{(\ell)}\) be the pooling map from the fine graph to that level. For predicted and target velocity fields \(\widehat{\bm{u}},\bm{u}\in\mathbb{R}^{N\times C}\), the supervised fields at level \(\ell\) are
	\[
	\widehat{\bm{u}}^{(\ell)}=P^{(\ell)}\widehat{\bm{u}},
	\qquad
	\bm{u}^{(\ell)}=P^{(\ell)}\bm{u}.
	\]
	The loss is therefore applied on the same pooled nodes and coarse graphs that participate in message passing. If \(U_r^{(\ell)}\in\mathbb{R}^{|\mathcal{V}^{(\ell)}|\times r}\) denotes the regularized retained embedding formed as in \eqref{eq:embedding} on \(\mathcal{G}^{(\ell)}\), GLEAM projects the level fields as
	\[
	\widehat{\bm{a}}^{(\ell)}=(U_r^{(\ell)})^\top\widehat{\bm{u}}^{(\ell)},
	\qquad
	\bm{a}^{(\ell)}=(U_r^{(\ell)})^\top\bm{u}^{(\ell)}.
	\]
	For band \(\mathcal{B}_m\), the low-rank band energies are
	\[
	E_{\widehat{u}}^{(\ell)}(m,c)
	=
	\frac{1}{2}\sum_{k\in\mathcal{B}_m}\left(\widehat{a}_{k,c}^{(\ell)}\right)^2,
	\qquad
	E_{u}^{(\ell)}(m,c)
	=
	\frac{1}{2}\sum_{k\in\mathcal{B}_m}\left(a_{k,c}^{(\ell)}\right)^2.
	\]
	The band term compares these energies using the same stabilized relative mismatch as the main Graph BSP objective, while the Fiedler-pair term compares weighted field contrasts across selected node pairs \(\mathcal{P}^{(\ell)}\) from the coarse spectral embedding. The weights \(w_{ij}^{(\ell)}\) are the levelwise analogue of \eqref{eq:fiedler_pair_weight}:
	\[
	\mathcal{L}_{\mathrm{pair}}^{(\ell)}
	=
	\frac{1}{|\mathcal{P}^{(\ell)}|}
	\sum_{(i,j)\in\mathcal{P}^{(\ell)}}
	w_{ij}^{(\ell)}
	\left[
	\sum_{c=1}^{C}
	\left(\widehat{u}_{i,c}^{(\ell)}-\widehat{u}_{j,c}^{(\ell)}\right)^2
	-
	\sum_{c=1}^{C}
	\left(u_{i,c}^{(\ell)}-u_{j,c}^{(\ell)}\right)^2
	\right]^2 .
	\]
	Thus the hierarchy-level penalty has the form
	\[
	\mathcal{L}_{\mathrm{spec}}^{\mathrm{GLEAM}}
	=
	\sum_{\ell\in\mathcal{H}_{\mathrm{sup}}}\alpha_{\ell}
	\left(
	\left(1-\rho\right)\mathcal{L}_{\mathrm{band}}^{(\ell)}
	+\rho\,\mathcal{L}_{\mathrm{pair}}^{(\ell)}
	\right),
	\]
	so the auxiliary supervision regularizes both low-rank band energy and long-range coarse geometric variation along the same coarse-to-fine pathway used by the backbone.
	
	The final BFS GLEAM configuration uses a rank-\(32\) spectral embedding, \(\rho=0.2\), \(32\) Fiedler pairs, and \(8\) spectral bands. Band energies are compared with stabilized relative mismatch, and coarse embeddings are transferred across levels using distance-weighted prolongation followed by two unweighted neighbor-mean smoothing refinement steps with \(\beta=0.15\) and a scalar correction. The finest exact eigensolve is disabled in this run, so the loss remains a hierarchy-aware low-rank approximation rather than a full-spectrum penalty on the finest graph. The selected comparison uses the \(\lambda_{\mathrm{GLEAM}}=10^{-4}\) run, trained for \(200\) epochs with the same AdamW optimizer, learning-rate schedule, gradient clipping, bfloat16 mixed precision, rollout curriculum, and EMA-weighting scheme as the Chebyshev model.
	
	At test time, the trained models are evaluated autoregressively on the held-out contiguous test segment. The model starts from the first normalized test state and repeatedly feeds its own prediction back as the next input. Rollout outputs are denormalized before saving metrics and fields. Qualitative figures report both field values and absolute errors, while the graph-spectral rollout diagnostic uses the \(U_y\) channel with \(24\) normalized graph-Laplacian bands. Thus, RMSE measures pointwise forecast growth, whereas band-energy mismatch measures whether the rollout preserves the distribution of energy across graph-resolved spatial scales.
	
	\subsubsection{BFS mechanism diagnostics}
	\label{app:bfs-mechanism-diagnostics}
	{
		Table~\ref{tab:bfs_mechanism_diagnostics} uses mechanism-oriented diagnostics computed from the denormalized autoregressive BFS fields. Let \(\bm u_t\) and \(\widehat{\bm u}_t\) denote the ground-truth and predicted two-component velocity fields at rollout step \(t\), let \(\bm e_t=\widehat{\bm u}_t-\bm u_t\), and let \(a_i\) be the nodal area weight obtained by assigning one third of each triangular cell area to each of its vertices. For a node set \(D\), the area-weighted velocity RMSE is
		\[
		\mathrm{RMSE}_{D}(t)
		=
		\left[
		\frac{\sum_{i\in D}a_i\norm{\bm e_{t,i}}_2^2}
		{2\sum_{i\in D}a_i}
		\right]^{1/2},
		\]
		where the factor \(2\) is the number of velocity components. Global RMSE uses all nodes. Reverse RMSE restricts this same error to the reverse-flow node set
		\[
		D_{\mathrm{rev}}
		=
		\left\{i:\frac{1}{T}\sum_{t=1}^{T}\mathbf{1}\{u_{x,t,i}<0\}\ge 0.05\right\},
		\]
		which marks nodes that participate in recirculation during at least \(5\%\) of the rollout. High-gradient RMSE uses the top \(15\%\) of nodes by mean incident-edge velocity-gradient score,
		\[
		g_i=\operatorname{mean}_{(i,j)\in\mathcal{E}}
		\frac{1}{T}\sum_{t=1}^{T}
		\frac{\norm{\bm u_{t,i}-\bm u_{t,j}}_2}{\norm{\bm p_i-\bm p_j}_2}.
		\]
		The reverse-area error compares the predicted and target reverse-flow area,
		\[
		\frac{\left|\widehat A_t^{-}-A_t^{-}\right|}{A_\Omega},
		\qquad
		A_t^{-}=\sum_i a_i\mathbf{1}\{u_{x,t,i}<0\},
		\quad
		\widehat A_t^{-}=\sum_i a_i\mathbf{1}\{\widehat u_{x,t,i}<0\},
		\]
		where \(A_\Omega=\sum_i a_i\) is the total mesh area. The centroid drift is the Euclidean distance between predicted and target reverse-flow centroids,
		\[
		\norm{\widehat{\bm z}_t^{-}-\bm z_t^{-}}_2,
		\qquad
		\bm z_t^{-}
		=
		\frac{\sum_i a_i[-u_{x,t,i}]_{+}\bm p_i}
		{\sum_i a_i[-u_{x,t,i}]_{+}},
		\]
		with the predicted centroid \(\widehat{\bm z}_t^{-}\) defined analogously. Low-band RMSE is computed using the exact normalized-Laplacian eigensystem \(L=\Phi\Lambda\Phi^\top\) for evaluation only. For the \(U_y\) channel and spectral band \(B_m\),
		\[
		E_{m,t}^{U_y}
		=
		\frac{1}{2N}\sum_{k\in B_m}
		\left(\phi_k^\top \bm u_{y,t}\right)^2,
		\]
		and the reported low-band value is the rollout mean, excluding the initial step, of the RMSE over bands \(m=0,\ldots,7\) from the \(24\)-band partition.
		
		These calculations connect the BFS rollout to the broader transport hypothesis tested across the turbulent-flow benchmarks. Reverse-area error tests whether the model preserves the size of the separated recirculation region, while centroid drift tests whether that region is transported to the correct location. Reverse-region and high-gradient RMSE test whether improvements occur in the physically relevant regions where recirculation and shear-layer errors are expected to seed later autoregressive drift. Low-band \(U_y\) RMSE tests whether the same rollout preserves the coarse graph-spectral energy associated with large-scale vertical transport. The reductions in Table~\ref{tab:bfs_mechanism_diagnostics} therefore support the interpretation that GLEAM is not merely lowering a global pointwise score; it is also reducing the reverse-flow extent error, recirculation displacement, shear-region error, and low-frequency spectral drift predicted by the same transport-drift mechanism that motivates the EAGLE and \texttt{pOnWing} diagnostics. Because these are post-hoc diagnostics of trained models, they should be read as evidence consistent with the mechanism rather than as a causal proof of which hierarchy level produced the improvement.
	}
	
	\subsubsection{\texorpdfstring{BFS GLEAM Hyperparameter Ablation}{BFS GLEAM Hyperparameter Ablation}}
	\label{app:bfs-gleam-ablation}
	Tables~\ref{tab:bfs_gleam_hyperparameter_ablation} and~\ref{tab:bfs_gleam_hyperparameter_spectral_ablation} summarize the GLEAM hyperparameter ablation on the backward-facing step benchmark. Both tables include the base GLEAM run as a reference configuration, with \(M=8\) GLEAM bands, retained rank \(r=32\), \(P=32\) Fiedler-guided contrast pairs, pairwise weight \(\rho=0.2\), and a scalar coarse-to-fine correction. Table~\ref{tab:bfs_gleam_hyperparameter_ablation} reports held-out autoregressive velocity RMSE at local rollout horizons \(t=1,20,50,100,150,\) and \(180\). Table~\ref{tab:bfs_gleam_hyperparameter_spectral_ablation} reports the same autoregressive horizons using the exact normalized-Laplacian spectral-energy relative \(L_1\) diagnostic used in the BFS trained-model comparison, computed over \(24\) exact graph-frequency bands and both velocity channels.
	
	\begin{table}[!htbp]
		\centering
		\caption{BFS GLEAM hyperparameter ablation using trained models. The rollout columns report autoregressive velocity RMSE.}
		\label{tab:bfs_gleam_hyperparameter_ablation}
		\scriptsize
		\setlength{\tabcolsep}{2.2pt}
		\resizebox{\linewidth}{!}{%
			\begin{tabular}{llcccccc}
				\toprule
				Aspect tested & Setting & \(t=1\) & \(t=20\) & \(t=50\) & \(t=100\) & \(t=150\) & \(t=180\) \\
				\midrule
				Reference & Base GLEAM: \(M=8,r=32,P=32,\rho=0.2\), scalar corr. & 0.168 & 2.245 & 5.133 & \(\mathbf{6.624}\) & \(\mathbf{8.057}\) & 8.362 \\
				\midrule
				Band resolution & \(M=4\) & 0.157 & 2.295 & 6.303 & 10.812 & 12.924 & 13.868 \\
				Band resolution & \(M=16\) & \(\mathbf{0.154}\) & \(\mathbf{2.114}\) & 6.097 & 8.754 & 9.064 & 10.441 \\
				\midrule
				Retained spectral rank & \(r=16\) & 0.176 & 2.133 & 5.762 & 10.152 & 12.039 & 13.285 \\
				Retained spectral rank & \(r=64\) & 0.174 & 2.193 & \(\mathbf{5.106}\) & 7.127 & 8.254 & \(\mathbf{8.354}\) \\
				\midrule
				Fiedler-pair sampling & \(P=16\) & 0.159 & 2.274 & 5.977 & 11.013 & 12.836 & 13.081 \\
				Fiedler-pair sampling & \(P=64\) & 0.155 & 2.373 & 6.250 & 10.295 & 12.060 & 12.854 \\
				\midrule
				Pairwise contrast weight & \(\rho=0.1\) & 0.159 & 2.646 & 6.067 & 10.260 & 12.118 & 12.390 \\
				Pairwise contrast weight & \(\rho=0.5\) & 0.156 & 2.456 & 5.962 & 10.455 & 11.526 & 12.402 \\
				\midrule
				Corrector design & GNN corrector & 0.165 & 2.289 & 6.131 & 8.436 & 8.695 & 9.077 \\
				\bottomrule
			\end{tabular}
		}
	\end{table}
	
	\begin{table}[!htbp]
		\centering
		\caption{BFS GLEAM hyperparameter ablation measured by exact normalized-Laplacian spectral-energy relative \(L_1\). Values are computed from autoregressive rollouts using trained models, the full BFS normalized-Laplacian eigensystem, \(24\) graph-frequency bands, and both velocity channels.}
		\label{tab:bfs_gleam_hyperparameter_spectral_ablation}
		\scriptsize
		\setlength{\tabcolsep}{3.0pt}
		\resizebox{\linewidth}{!}{%
			\begin{tabular}{llcccccc}
				\toprule
				Aspect tested & Setting & \(t=1\) & \(t=20\) & \(t=50\) & \(t=100\) & \(t=150\) & \(t=180\) \\
				\midrule
				Reference & Base GLEAM: \(M=8,r=32,P=32,\rho=0.2\), scalar corr. & 0.0063 & 0.0511 & 0.1258 & 0.1829 & 0.1805 & 0.1701 \\
				\midrule
				Band resolution & \(M=4\) & 0.0051 & 0.0756 & 0.1952 & 0.3756 & 0.4777 & 0.4686 \\
				Band resolution & \(M=16\) & \(\mathbf{0.0019}\) & \(\mathbf{0.0278}\) & 0.1332 & \(\mathbf{0.1066}\) & \(\mathbf{0.0353}\) & \(\mathbf{0.0537}\) \\
				\midrule
				Retained spectral rank & \(r=16\) & 0.0045 & 0.0495 & \(\mathbf{0.1234}\) & 0.2383 & 0.2796 & 0.3168 \\
				Retained spectral rank & \(r=64\) & 0.0082 & 0.0637 & 0.1276 & 0.2171 & 0.2529 & 0.2714 \\
				\midrule
				Fiedler-pair sampling & \(P=16\) & 0.0031 & 0.0496 & 0.1403 & 0.3469 & 0.4653 & 0.4622 \\
				Fiedler-pair sampling & \(P=64\) & 0.0020 & 0.0389 & 0.1471 & 0.2910 & 0.3608 & 0.3463 \\
				\midrule
				Pairwise contrast weight & \(\rho=0.1\) & 0.0024 & 0.0484 & 0.1434 & 0.2745 & 0.3237 & 0.4054 \\
				Pairwise contrast weight & \(\rho=0.5\) & 0.0031 & 0.0431 & 0.1390 & 0.2857 & 0.3218 & 0.4138 \\
				\midrule
				Corrector design & GNN corrector & 0.0024 & 0.0372 & 0.1388 & 0.1438 & 0.1004 & 0.0570 \\
				\bottomrule
			\end{tabular}
		}
	\end{table}
	The RMSE ablation shows that the base GLEAM configuration is a strong and balanced reference rather than a pointwise optimum at every horizon. The base run gives the lowest RMSE at \(t=100\) and \(t=150\), while the \(r=64\) variant is best at \(t=50\) and is marginally best at \(t=180\). The \(M=16\) run improves RMSE relative to \(M=4\), but it does not dominate the mid- and late-rollout columns, indicating that increasing the number of retained GLEAM bands alone is not the determining factor for long-horizon physical-space accuracy. Increasing the retained rank from \(r=16\) to \(r=64\) gives the clearest late-rollout improvement among the one-parameter ablations, consistent with the need for a richer retained spectral subspace in BFS transport. Varying the number of Fiedler-guided contrast pairs produces smaller changes. The contrast-weight ablation is not monotone: \(\rho=0.5\) improves over \(\rho=0.1\) at \(t=150\), but the two settings are comparable at \(t=180\), and the base \(\rho=0.2\) setting remains more reliable in the middle of the rollout. The GNN corrector reduces late-rollout RMSE relative to several scalar-corrector ablations, but it does not outperform the base or \(r=64\) settings at the latest horizons.
	
	The exact-Laplacian spectral-energy diagnostic gives a complementary view of the same ablation family. Unlike the RMSE table, the spectral table strongly favors the \(M=16\) band-resolution run, which gives the lowest relative \(L_1\) error at every selected horizon except \(t=50\), where the \(r=16\) run is marginally lower. This suggests that increasing the number of GLEAM bands improves agreement with the full-eigensystem energy distribution more directly than it improves pointwise velocity RMSE. The GNN corrector is also spectrally effective at late horizons and is close to \(M=16\) at \(t=180\), although it remains noticeably higher than \(M=16\) at \(t=150\). In contrast, the \(r=64\) variant is most competitive for late RMSE while remaining noticeably worse in exact spectral relative \(L_1\), showing that improved physical-space rollout accuracy and improved all-band spectral-energy matching are related but distinct evaluation criteria. Overall, the balanced base configuration and the \(r=64\) retained-rank variant are the strongest RMSE choices in this ablation family, while \(M=16\) gives the clearest reduction in exact graph-spectral drift.
	
	\subsubsection{Additional BFS Rollout Spectral Snapshots}
	\label{app:bfs-uy-rollout-snapshots}
	Figure~\ref{fig:bfs_uy_rollout_appendix} gives the full \(U_y\) band-energy RMSE trajectory over the BFS rollout together with the spatial \(U_y\) fields and absolute-error maps corresponding to the graph-band mismatch diagnostics in Section~\ref{sec:spectral-diagnostics}. These snapshots are kept in the appendix because they support the interpretation of the spectral rollout diagnostics but are not needed for the main quantitative argument.
	
	\begin{figure}[!htbp]
		\centering
		\begin{subfigure}[c]{0.34\linewidth}
			\centering
			\includegraphics[width=\linewidth]{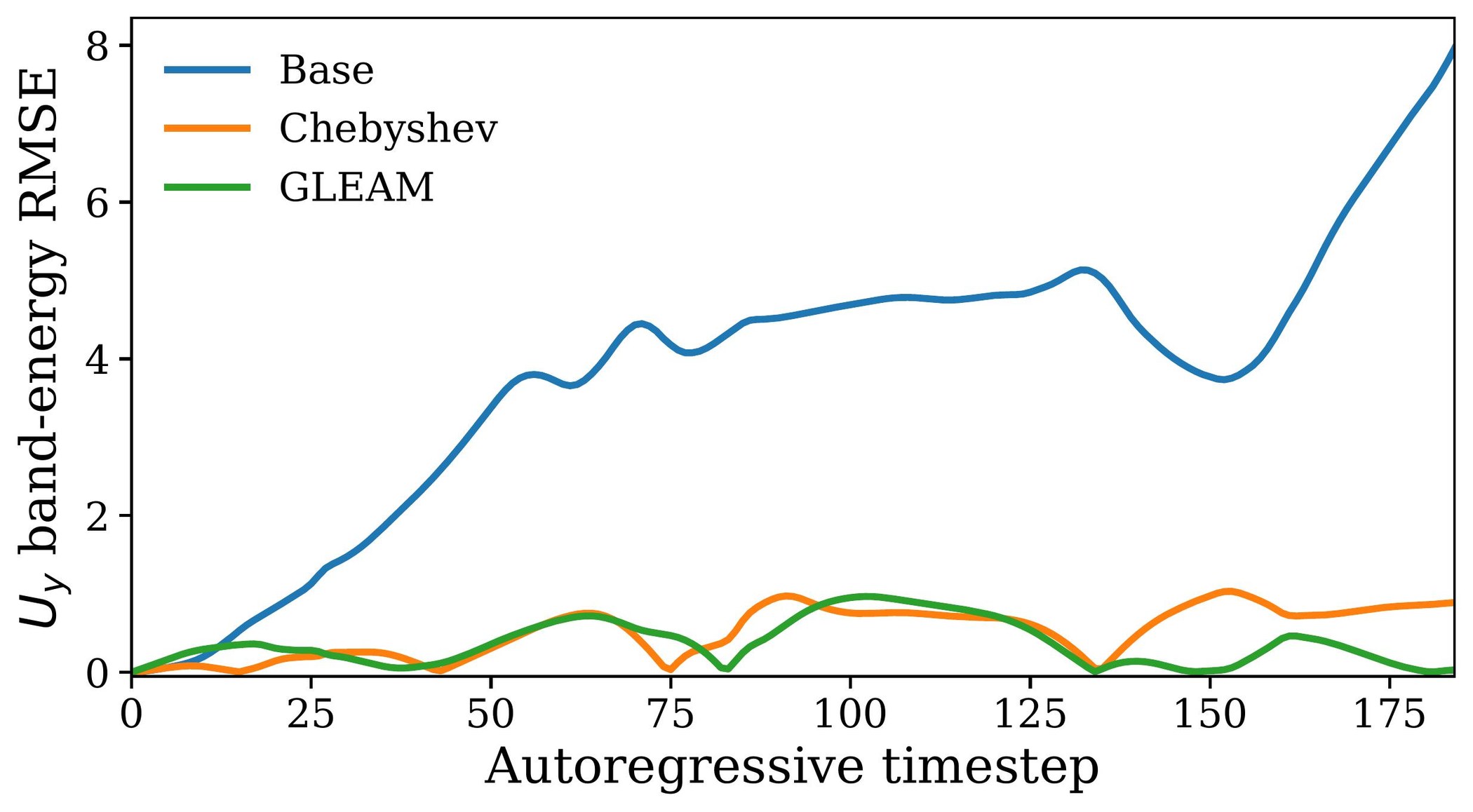}
			\caption{Band-energy RMSE.}
			\label{fig:bfs_uy_band_energy_rmse_rollout}
		\end{subfigure}\hfill
		\begin{subfigure}[c]{0.64\linewidth}
			\centering
			\setlength{\tabcolsep}{1pt}
			\renewcommand{\arraystretch}{0.70}
			\tiny
			\begin{tabular}{@{}CCCC@{}}
				\bfsqualheader
				\bfsqualblock{t=30}{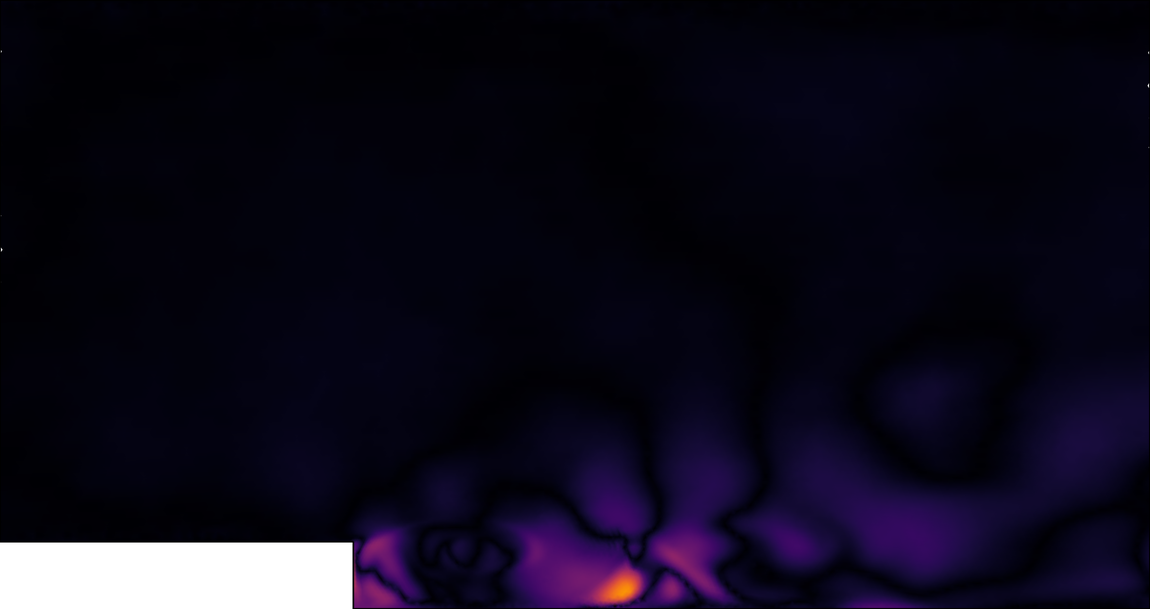}{0.08em}
				\bfsqualblock{t=82}{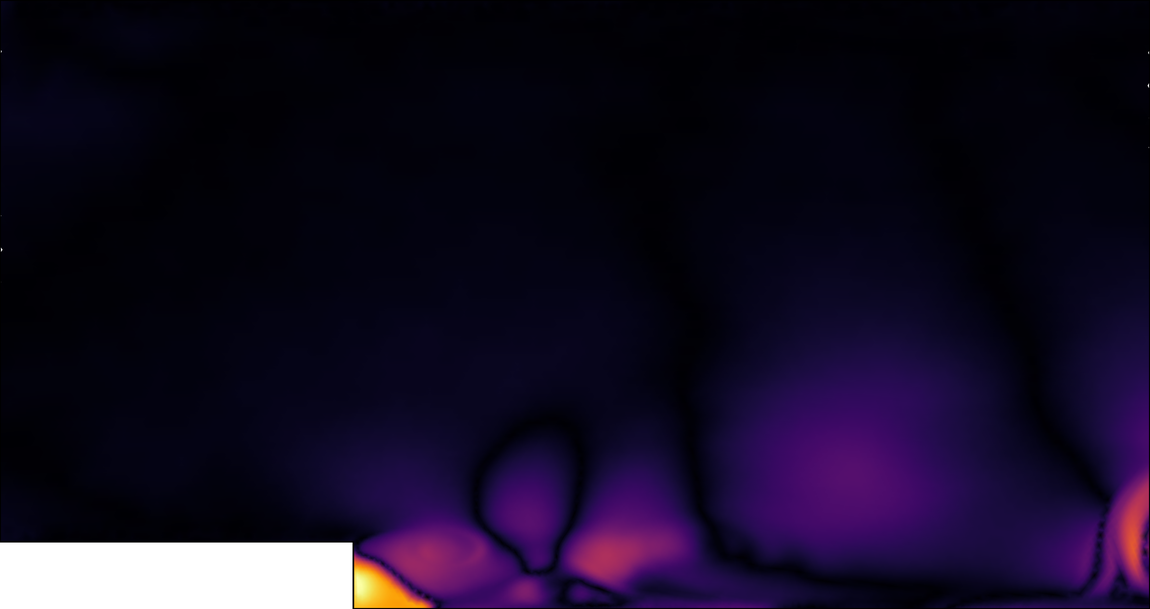}{0pt}
			\end{tabular}
			\vspace{0.1em}
			
			\includegraphics[width=0.86\linewidth]{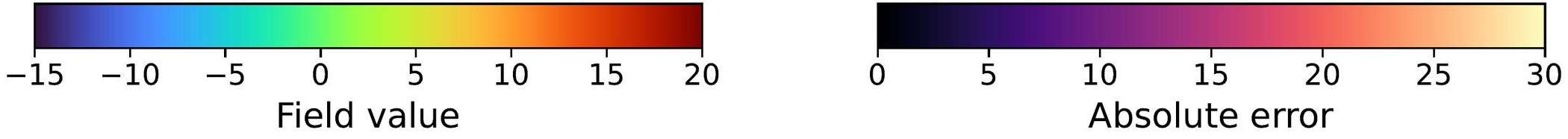}
			\caption{Field and error snapshots.}
			\label{fig:bfs_uy_qualitative_rollout_app}
		\end{subfigure}
		\caption{Additional \(U_y\) rollout diagnostics for the backward-facing step. Panel (a) reports band-energy RMSE across \(24\) graph-Laplacian bands, and panel (b) shows representative \(U_y\) fields and absolute-error maps at \(t=30\) and \(t=82\).}
		\label{fig:bfs_uy_rollout_appendix}
	\end{figure}
	
	\subsection{\texttt{pOnWing} Benchmark}
	\label{app:ponwing-benchmark}
	For \texttt{pOnWing}, the training protocol uses the same pressure-prediction target, temporal prefix, and graph hierarchy across the Base, Chebyshev, and GLEAM variants. The benchmark comes from the WING task in DGN4CFD, which predicts pressure on the surface of a three-dimensional wing in turbulent flow \citep{lino2025dgn}. The wing sections are NACA \(24XX\) airfoils, and the geometry varies across simulations through the relative thickness, taper ratio, sweep angle, and twist angle. The DGN4CFD simulations were generated with the PISO solver in OpenFOAM using a Spalart--Allmaras delayed-detached-eddy simulation model; the meshes were produced with \texttt{snappyHexMesh} \citep{lino2025dgn}.
	
	DGN4CFD reports the WING case as \(Re\sim 2\times 10^6\). This value is the chord-based Reynolds number computed from the free-stream velocity, reference chord, and kinematic viscosity,
	\[
	Re_c=\frac{U_\infty c_{\mathrm{root}}}{\nu}.
	\]
	Using the reported \(U_\infty=100\,\mathrm{km/h}=27.78\,\mathrm{m/s}\), \(c_{\mathrm{root}}=1\,\mathrm{m}\), and \(\nu=1.5\times10^{-5}\,\mathrm{m^2/s}\), we obtain
	\[
	Re_c=\frac{27.78\times 1}{1.5\times10^{-5}}
	\approx 1.85\times10^6,
	\]
	which explains the rounded value \(Re\sim 2\times10^6\). This Reynolds number is therefore a fixed simulation-condition scale for the WING benchmark rather than a varying input parameter in our \texttt{pOnWing} runs. In the public graph representation, each node stores its surface position and unit outer normal, edges encode relative positions and free-stream projections along local edge-aligned axes, and the prediction target is the nodal pressure \(p_i(t)\).
	
	The public DGN4CFD setup provides \(251\)-step training simulations and evaluates generalization on longer \(2501\)-step trajectories. During training, a prefix of length \(T=250\) is used so that the model sees the short-trajectory regime but is still evaluated under long autoregressive rollout. The graph hierarchy contains \(6\) levels with relative-position scalings \([0.015,0.03,0.06,0.12,0.2,0.4]\). Because GLEAM applies auxiliary supervision on this same multiscale structure, its spectral penalty is aligned with the hierarchy used by the propagator rather than with an unrelated post-processing graph. The \texttt{pOnWing} training data are taken from \texttt{pOnWingTrain.h5}, while autoregressive evaluation uses \texttt{pOnWingInDist.h5}.
	
	Two \texttt{pOnWing} backbone families are used in the final diagnostics: Graph VAE (VGAE) and FMGN. The VGAE variants are trained for \(600\) epochs and the FMGN variants for \(1000\) epochs with the Adam optimizer, teacher forcing, and input-noise augmentation with standard deviation \(0.02\), following prior noise-augmented graph-simulator training practice \citep{pfaff2020meshgraphnets}. Both backbone families use batch size \(6\), learning rate \(10^{-4}\), one Aurora node, and a common rollout horizon \(h=6\) for the Base, hierarchical Chebyshev BSP, and hierarchical GLEAM variants. The VGAE Chebyshev run uses \(\lambda_{\mathrm{Cheb}}=10^{-2}\), \(M=32\), \(K=16\), \(Q=256\), no logarithmic transform, and graph-mean centering; the corresponding GLEAM run uses \(\lambda_{\mathrm{GLEAM}}=10^{-3}\), the log form, \(r=16\), \(\rho=0.2\), \(32\) Fiedler pairs, \(8\) GLEAM bands, and scalar correction. The FMGN Chebyshev run uses the same hierarchical Chebyshev parameters with \(\lambda_{\mathrm{Cheb}}=10^{-4}\), while the FMGN GLEAM run uses \(\lambda_{\mathrm{GLEAM}}=10^{-4}\), the log form, \(r=64\), \(\rho=0.2\), \(\tau=10^{-6}\), \(64\) Fiedler pairs, \(32\) GLEAM bands, and scalar correction.
	
	For both backbones, Chebyshev BSP is evaluated hierarchically on the graph levels used by the backbone. GLEAM also uses the model hierarchy, but compares low-rank banded spectral energies and transfers coarse embeddings through distance-weighted prolongation followed by two unweighted neighbor-mean smoothing refinement steps with \(\beta=0.15\). The final GLEAM configuration uses the scalar corrector, and the finest exact eigensolve is disabled. The no-corrector Graph VAE curve in Fig.~\ref{fig:ponwing_vgae_appendix_diagnostics}(a) uses the same GLEAM configuration described above, but disables the scalar correction as an ablation.
	
	The \texttt{pOnWing} force diagnostic is computed after unscaling pressure back to physical units. For a surface node \(i\), let \(p_i(t)\) and \(\widehat{p}_i(t)\) denote the true and predicted physical pressures. Nodal areas are obtained by distributing each surface-cell area equally to its vertices. For a polygonal cell \(C\) with centroid \(\bm{c}_C\), vertices \(\bm{x}_j\), and cyclic indexing,
	\[
	A_C=\sum_{j\in C}\frac{1}{2}\left\|(\bm{x}_j-\bm{c}_C)\times(\bm{x}_{j+1}-\bm{c}_C)\right\|_2,
	\qquad
	A_i \leftarrow A_i+\frac{A_C}{|C|}.
	\]
	The node normal is approximated from the stored graph location vector,
	\[
	\bm{n}_i=\frac{\bm{\ell}_i}{\|\bm{\ell}_i\|_2},
	\]
	where \(\bm{\ell}_i\) is the node entry in \texttt{graph.loc}. The true and predicted pressure-force vectors are then
	\[
	\bm{F}_p(t)=-\sum_i p_i(t)\bm{n}_i A_i,
	\qquad
	\widehat{\bm{F}}_p(t)=-\sum_i \widehat{p}_i(t)\bm{n}_i A_i .
	\]
	The main reported row in the CSV uses
	\[
	e_{\mathrm{rel}}(t)=
	\frac{\|\widehat{\bm{F}}_p(t)-\bm{F}_p(t)\|_2}{\max(\|\bm{F}_p(t)\|_2,10^{-12})},
	\]
	and the area-normalized variant uses
	\[
	e_{\mathrm{area}}(t)=
	\frac{\|\widehat{\bm{F}}_p(t)-\bm{F}_p(t)\|_2}{\max(\sum_i |p_i(t)|A_i,10^{-12})}.
	\]
	For each rollout step, the tabulated statistics are the mean, minimum, maximum, standard deviation, and number of simulations over the held-out test set.
	
	\begin{figure}[!htbp]
		\centering
		\begin{subfigure}[c]{0.40\linewidth}
			\centering
			\includegraphics[width=\linewidth]{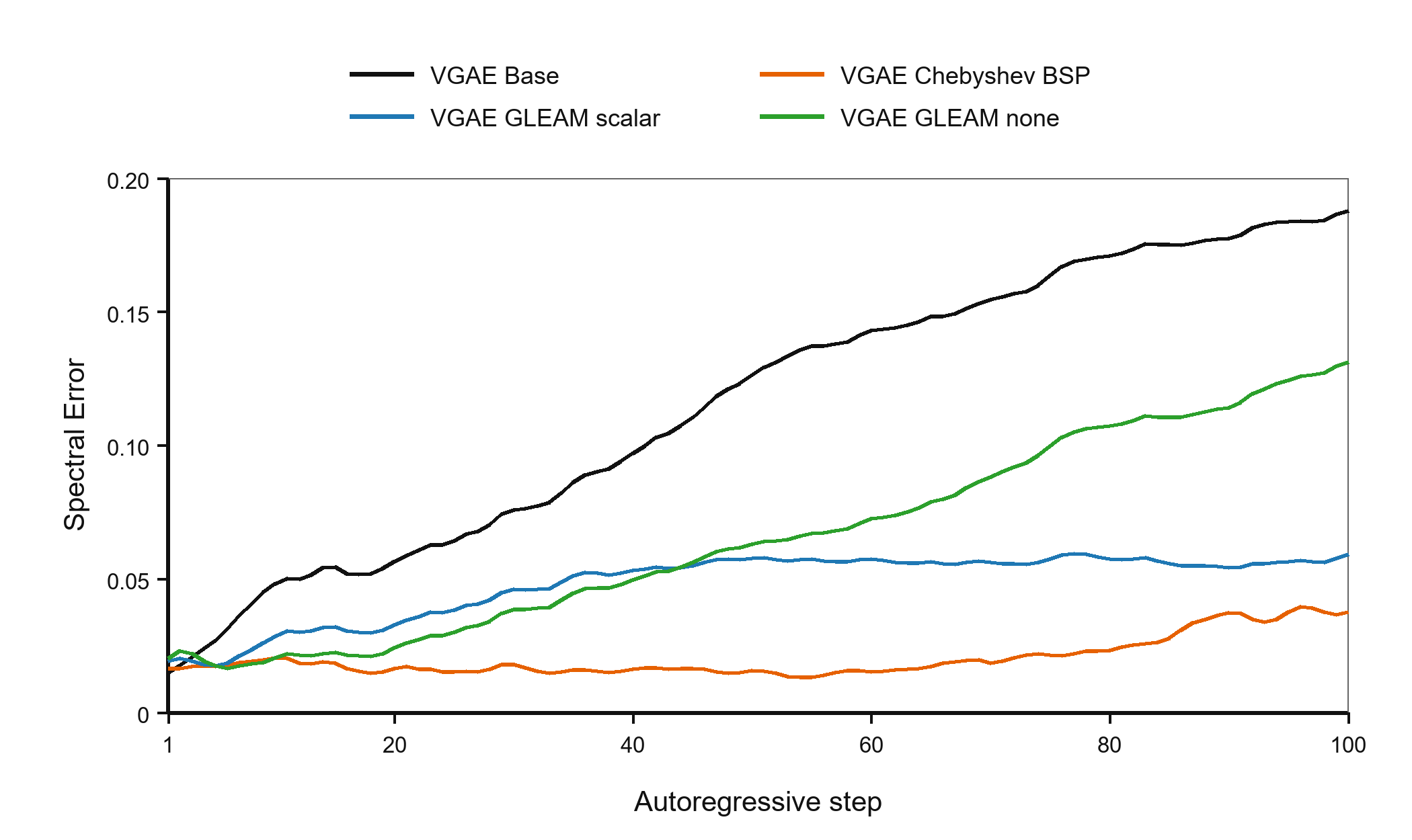}
			\caption{Spectral-error diagnostic.}
			\label{fig:ponwing_exact_laplacian_spectral_error}
		\end{subfigure}\hfill
		\begin{subfigure}[c]{0.58\linewidth}
			\centering
			\setlength{\tabcolsep}{1pt}
			\renewcommand{\arraystretch}{0.68}
			\tiny
			\begin{tabular}{@{}CCCC@{}}
				\bfsqualheader
				\bfsqualblock{\mathrm{sim}\ 5,\ t=50}{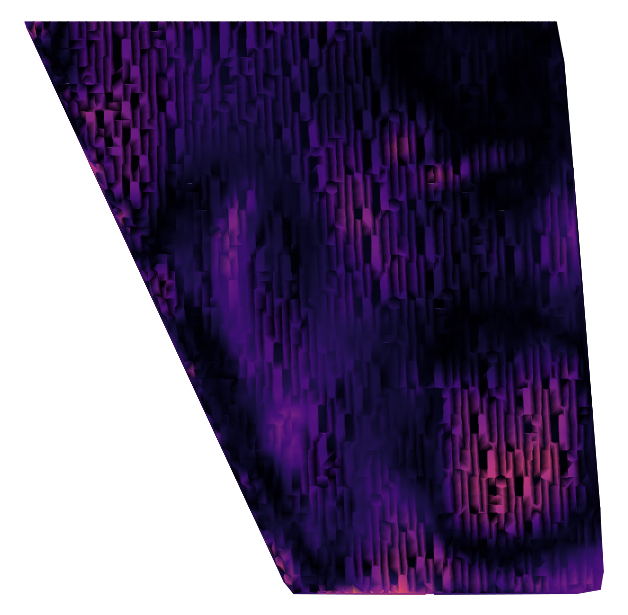}{0.08em}
				\bfsqualblock{\mathrm{sim}\ 11,\ t=50}{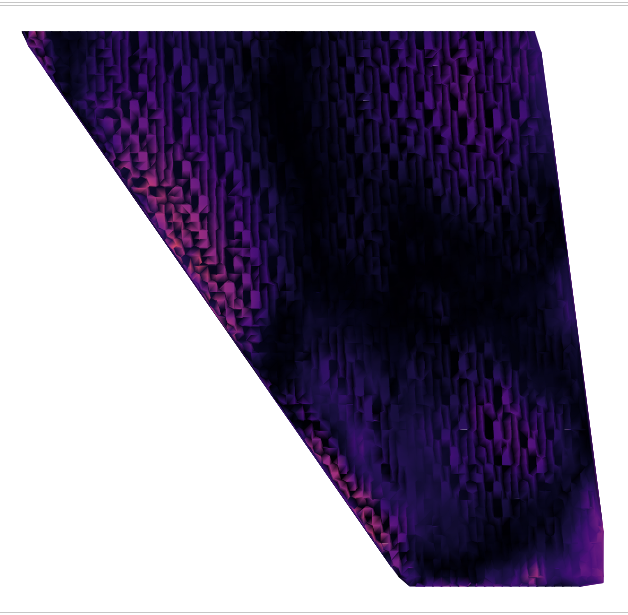}{0pt}
			\end{tabular}
			\vspace{0.1em}
			\includegraphics[width=0.88\linewidth]{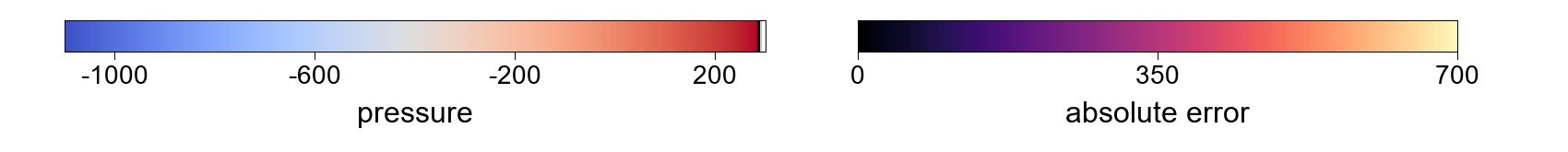}
			\caption{Full-field pressure comparison.}
			\label{fig:ponwing_vgae_full_field_t50}
		\end{subfigure}
		\caption{\texttt{pOnWing} VGAE diagnostics. Panel (a) reports the exact normalized-Laplacian spectral-error diagnostic over a representative \(100\)-step autoregressive rollout, including the GLEAM scalar and no-corrector variants. Panel (b) compares full-field pressure predictions at \(t=50\) for simulations \(5\) and \(11\), with shared field and absolute-error colorbars.}
		\label{fig:ponwing_vgae_appendix_diagnostics}
	\end{figure}
	
	The diagnostic in Fig.~\ref{fig:ponwing_vgae_appendix_diagnostics}(a) is computed with the full normalized-Laplacian eigensystem only for evaluation, not as the training loss. It therefore acts as an independent check on whether the cheaper spectral objectives reduce true graph-band drift. In this representative rollout, the Base model accumulates the largest spectral error, the scalar-corrected GLEAM variant remains lower, and Chebyshev BSP gives the smallest spectral drift over most of the horizon. The late-horizon gap between the scalar-corrected and no-corrector GLEAM curves suggests that the scalar correction helps stabilize the low-rank hierarchy after coarse-to-fine transfer.
	
	\begin{figure}[!htbp]
		\centering
		\setlength{\tabcolsep}{1pt}
		\renewcommand{\arraystretch}{0.72}
		\scriptsize
		\begin{tabular}{@{}>{\raggedleft\arraybackslash}p{0.035\linewidth}cccc@{}}
			& \textbf{\(t=1\)} & \textbf{\(t=5\)} & \textbf{\(t=10\)} & \textbf{\(t=50\)}\\[-0.1em]
			\textbf{(a)} &
			\includegraphics[width=0.23\linewidth]{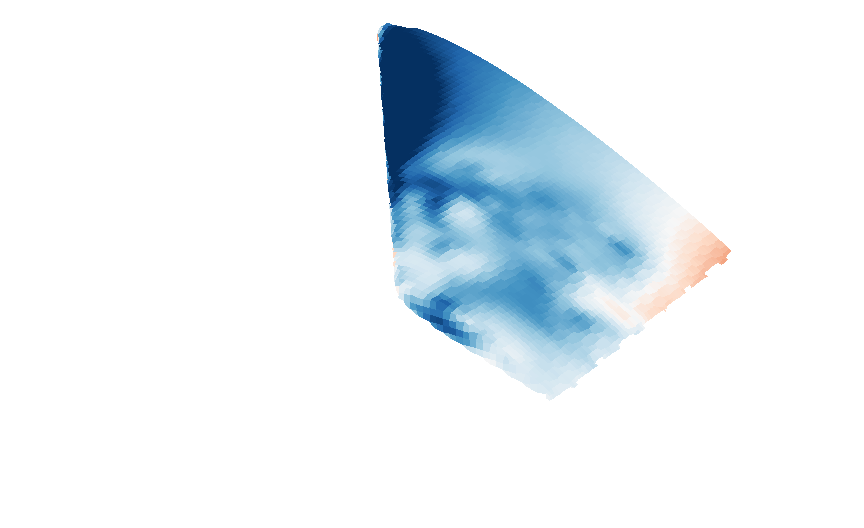} &
			\includegraphics[width=0.23\linewidth]{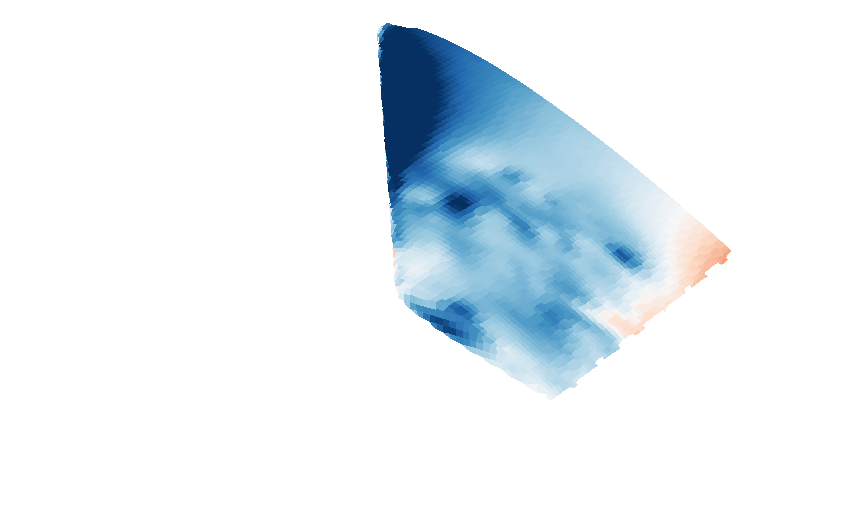} &
			\includegraphics[width=0.23\linewidth]{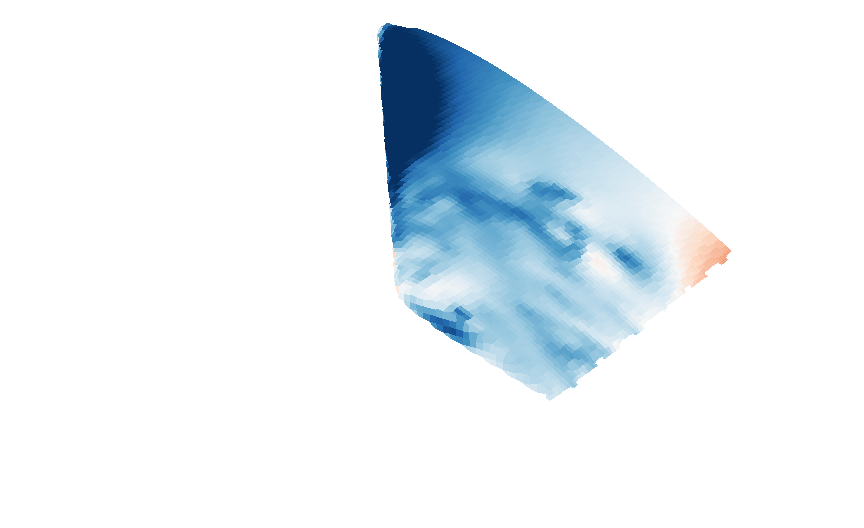} &
			\includegraphics[width=0.23\linewidth]{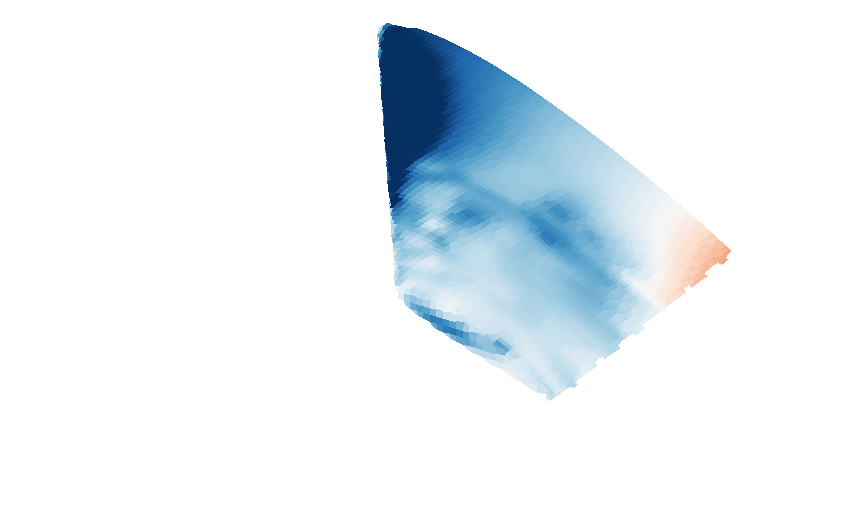}\\[-0.2em]
			\textbf{(b)} &
			\includegraphics[width=0.23\linewidth]{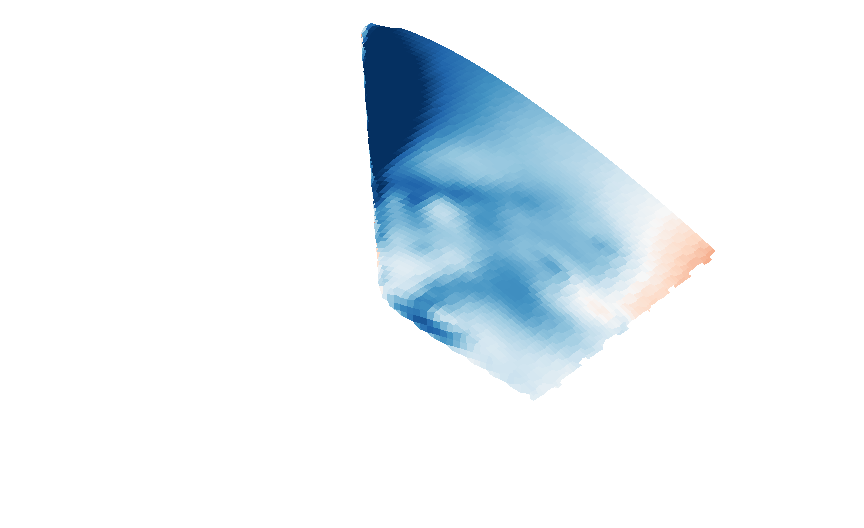} &
			\includegraphics[width=0.23\linewidth]{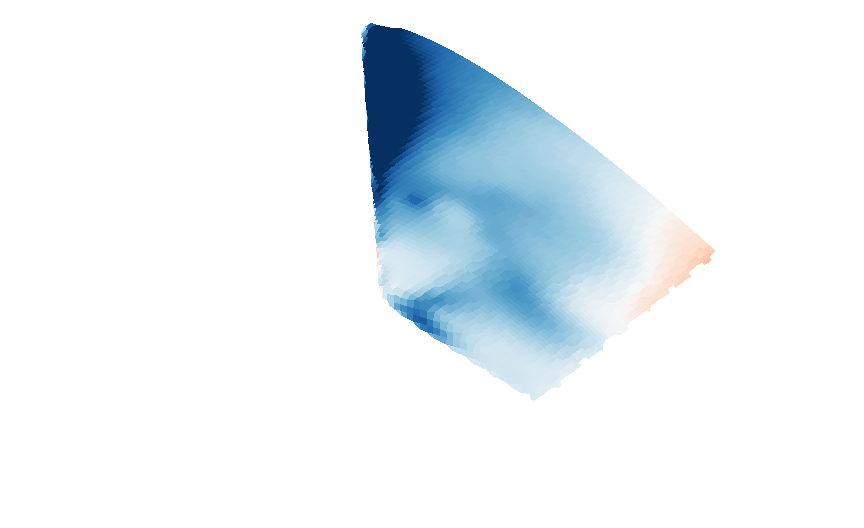} &
			\includegraphics[width=0.23\linewidth]{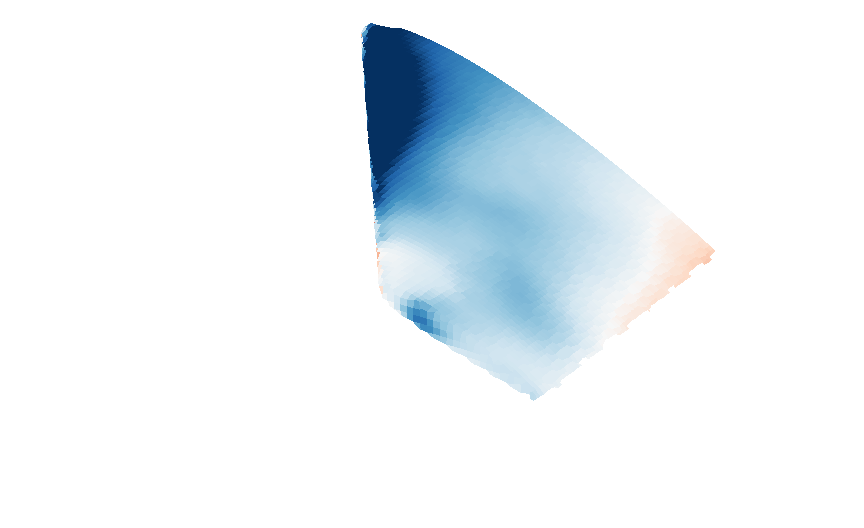} &
			\includegraphics[width=0.23\linewidth]{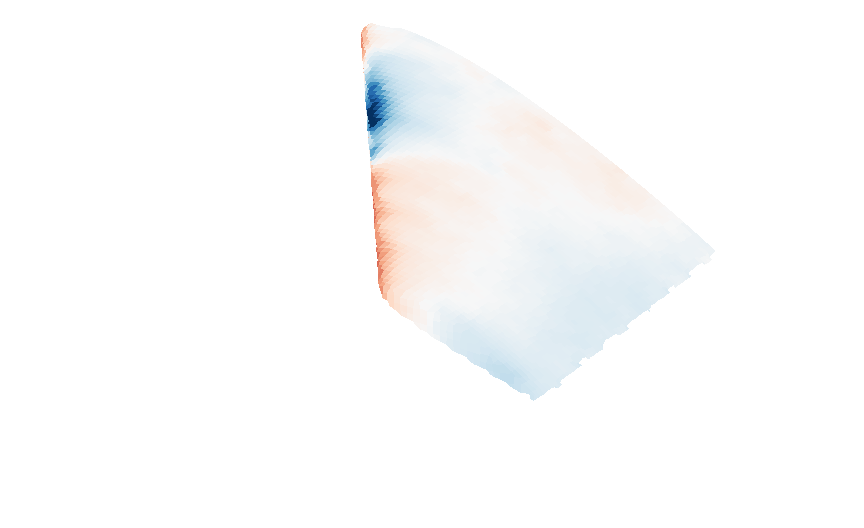}\\[-0.2em]
			\textbf{(c)} &
			\includegraphics[width=0.23\linewidth]{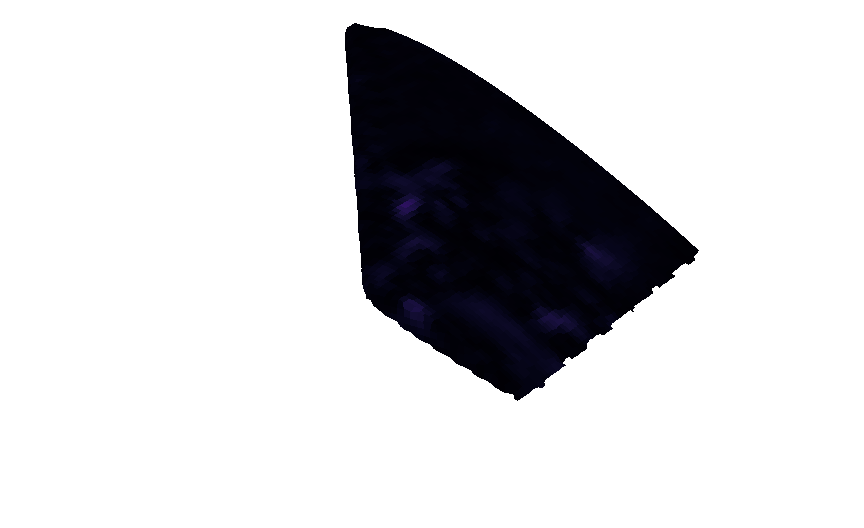} &
			\includegraphics[width=0.23\linewidth]{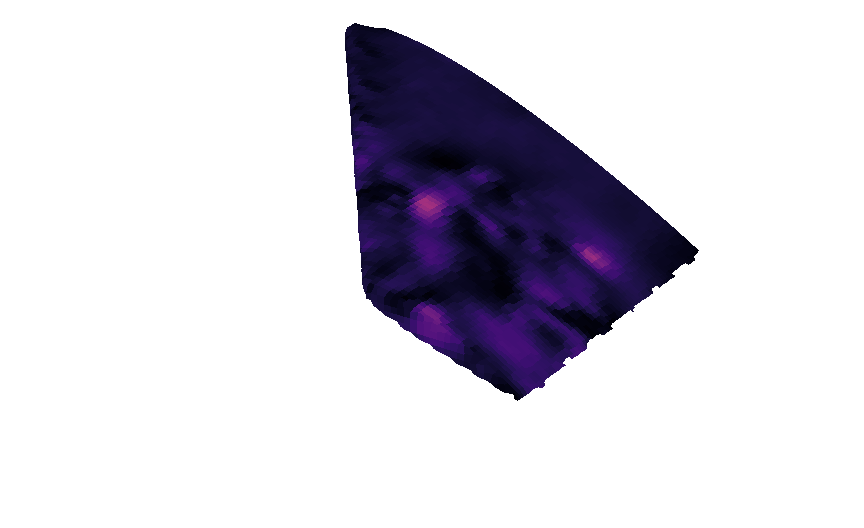} &
			\includegraphics[width=0.23\linewidth]{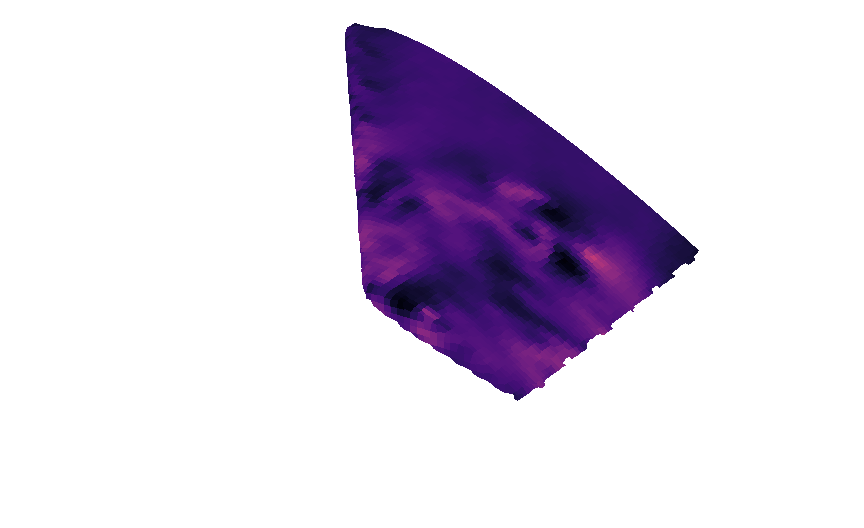} &
			\includegraphics[width=0.23\linewidth]{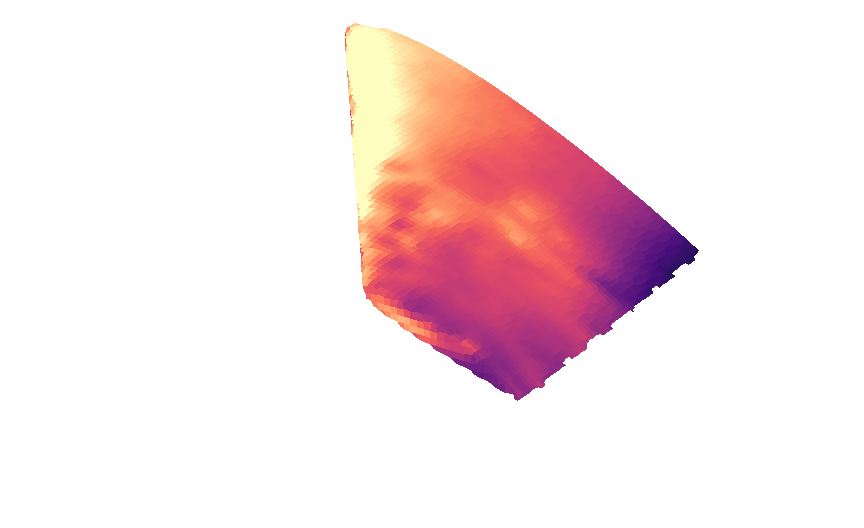}\\[-0.2em]
			\textbf{(d)} &
			\includegraphics[width=0.23\linewidth]{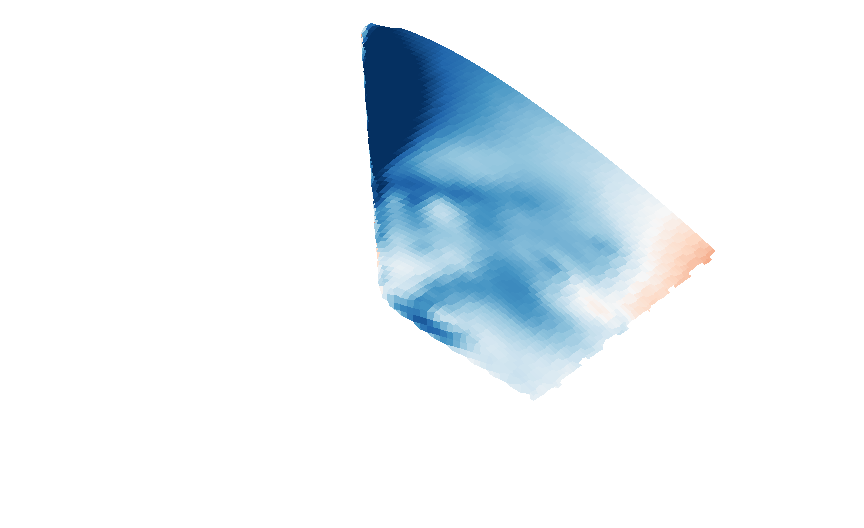} &
			\includegraphics[width=0.23\linewidth]{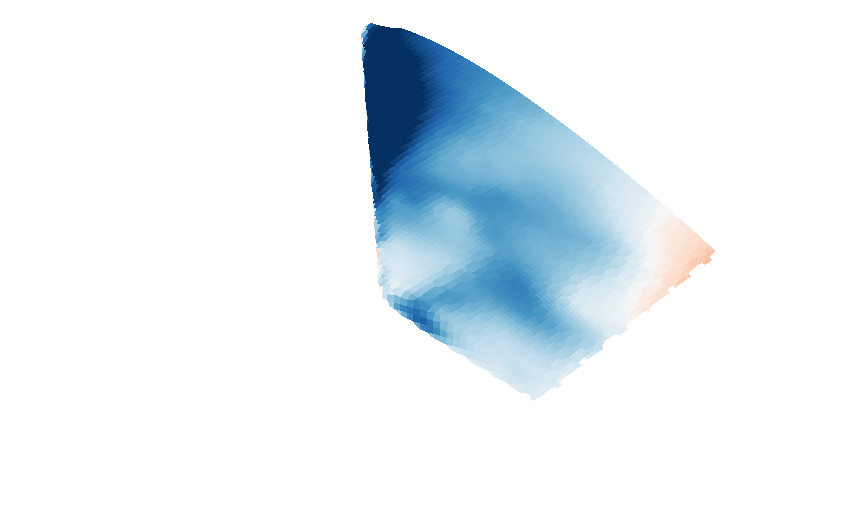} &
			\includegraphics[width=0.23\linewidth]{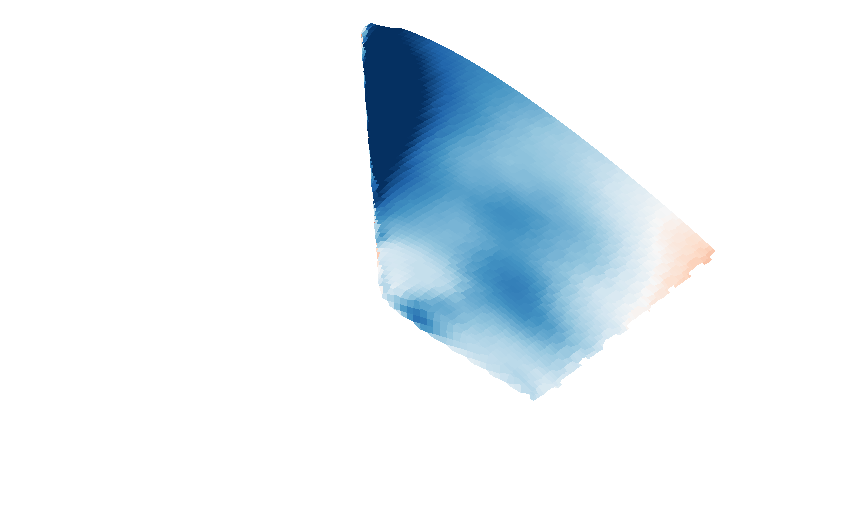} &
			\includegraphics[width=0.23\linewidth]{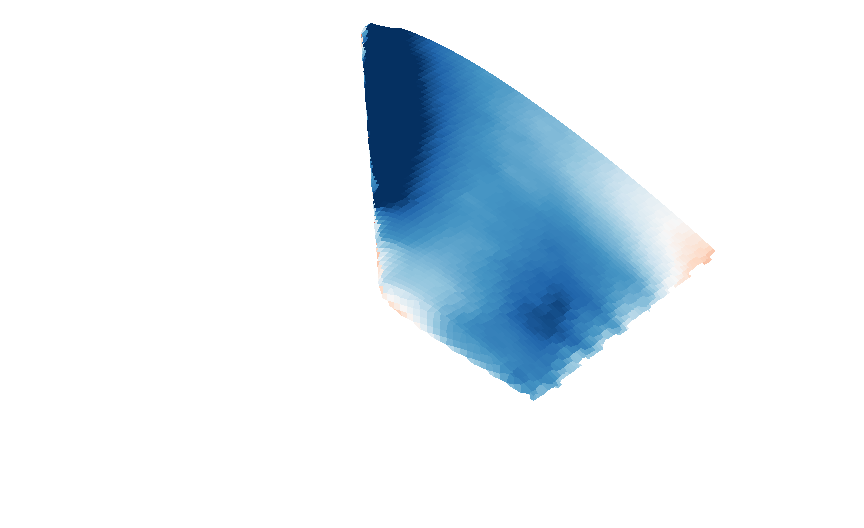}\\[-0.2em]
			\textbf{(e)} &
			\includegraphics[width=0.23\linewidth]{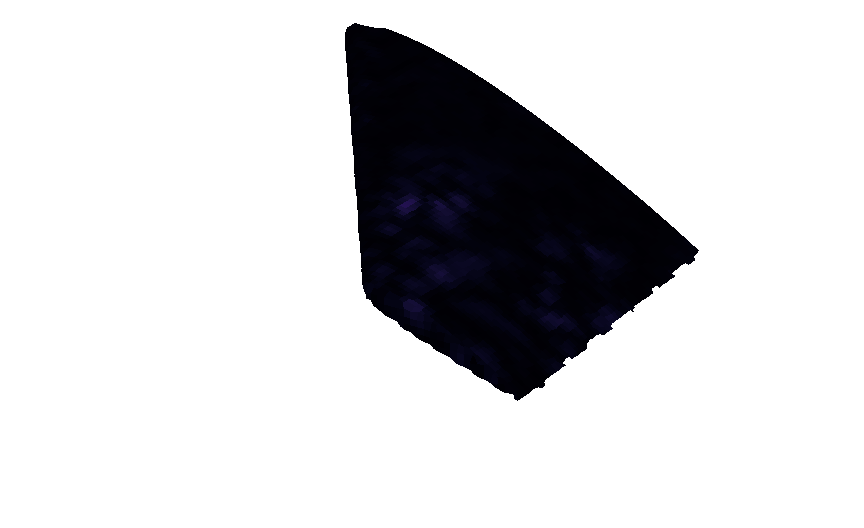} &
			\includegraphics[width=0.23\linewidth]{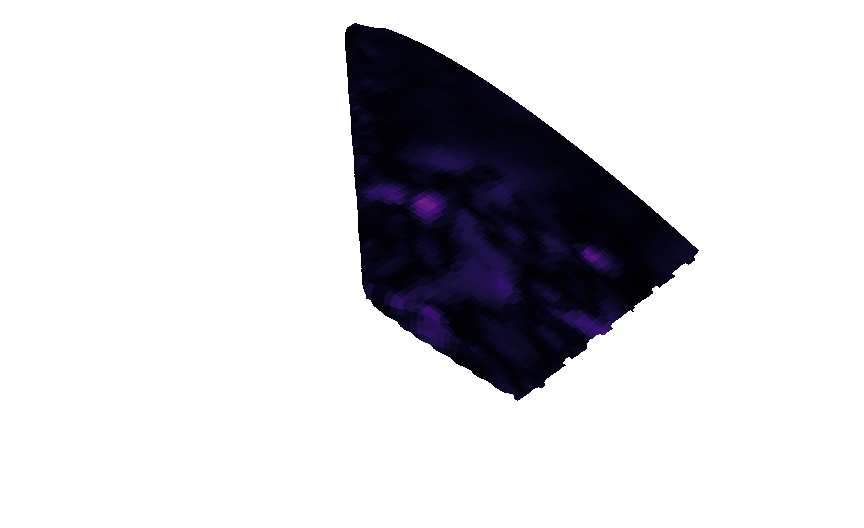} &
			\includegraphics[width=0.23\linewidth]{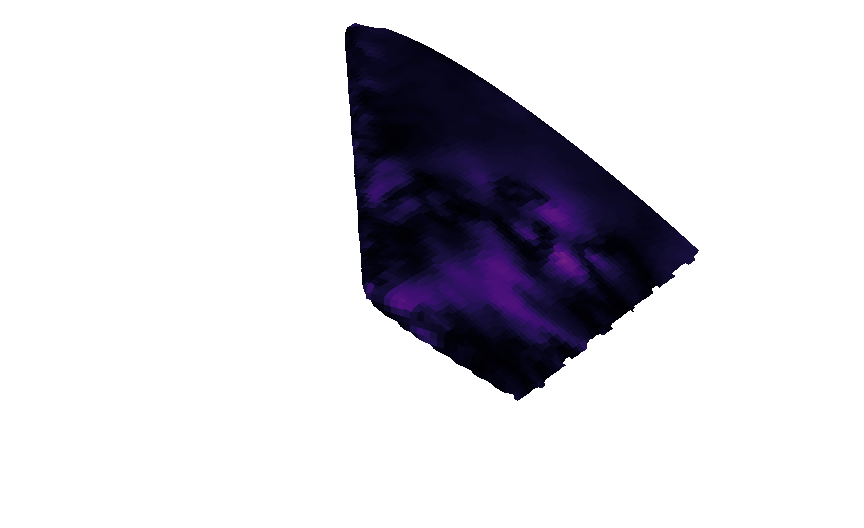} &
			\includegraphics[width=0.23\linewidth]{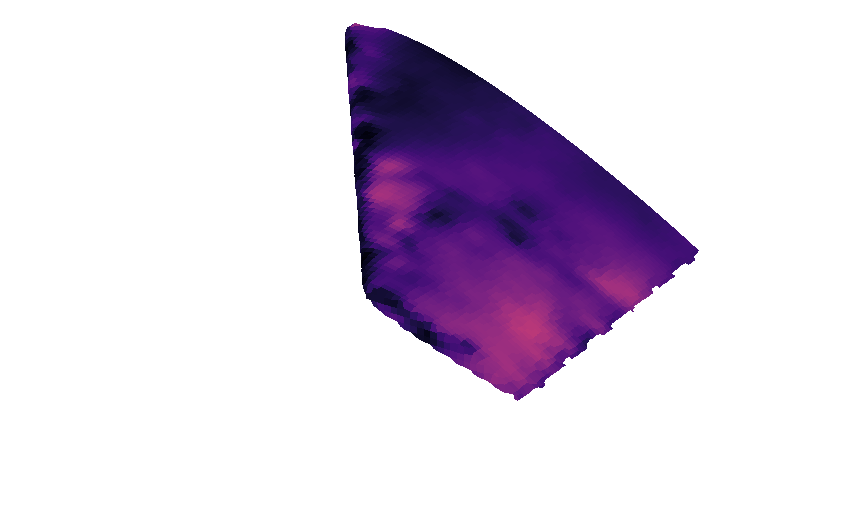}\\[-0.2em]
			\textbf{(f)} &
			\includegraphics[width=0.23\linewidth]{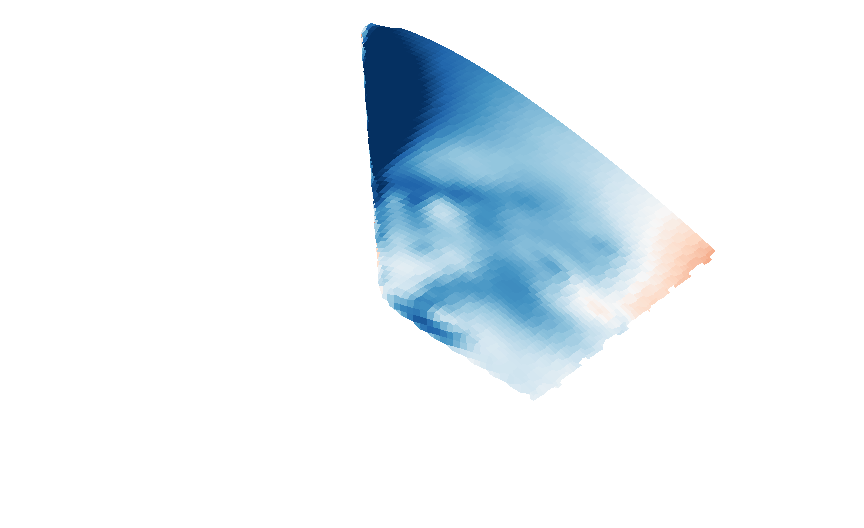} &
			\includegraphics[width=0.23\linewidth]{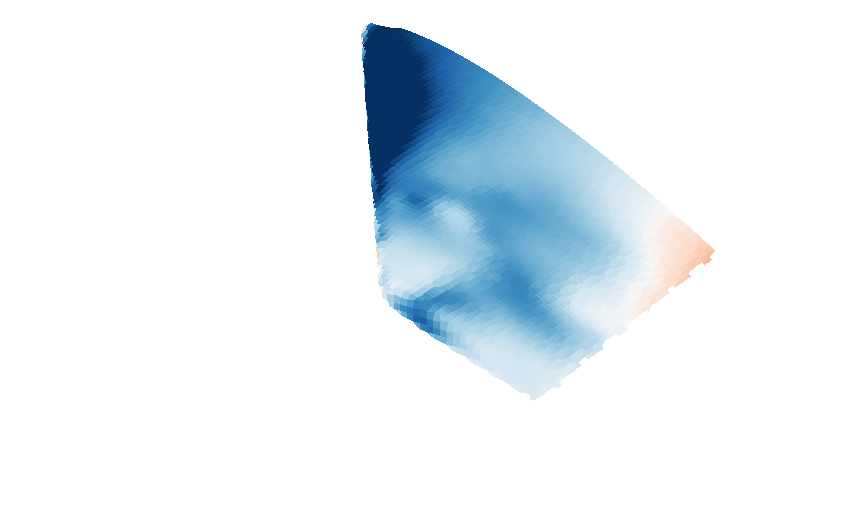} &
			\includegraphics[width=0.23\linewidth]{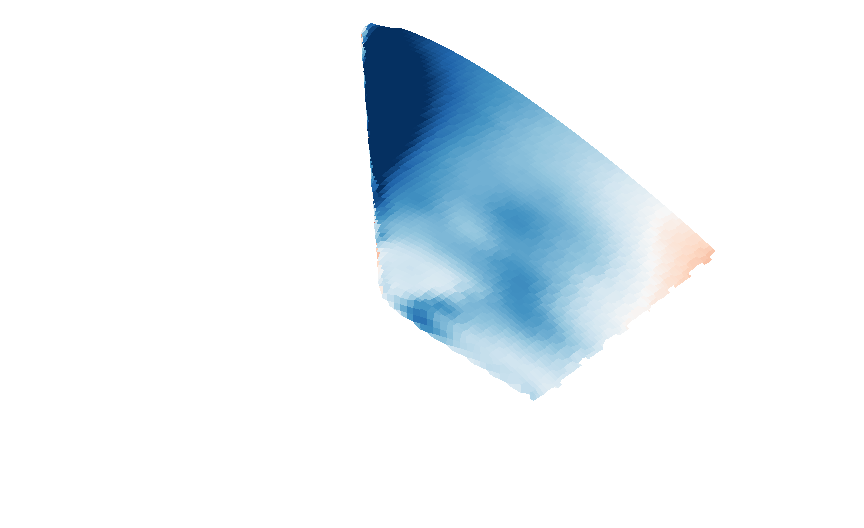} &
			\includegraphics[width=0.23\linewidth]{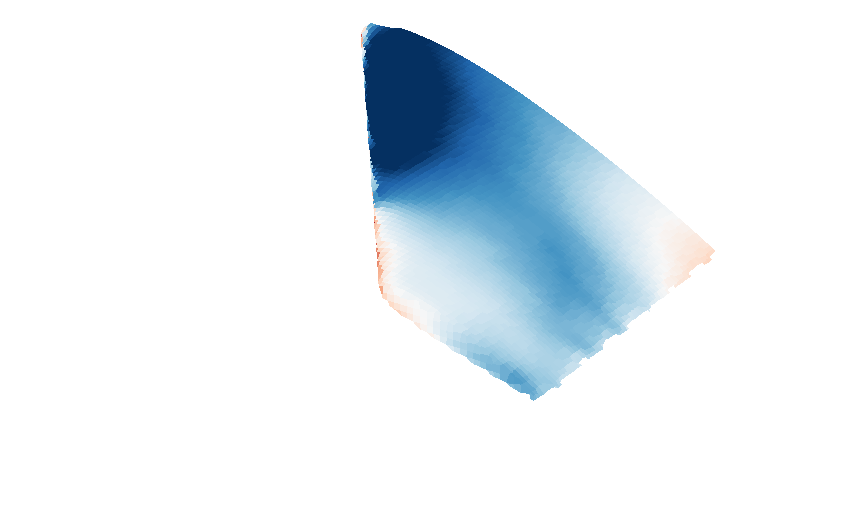}\\[-0.2em]
			\textbf{(g)} &
			\includegraphics[width=0.23\linewidth]{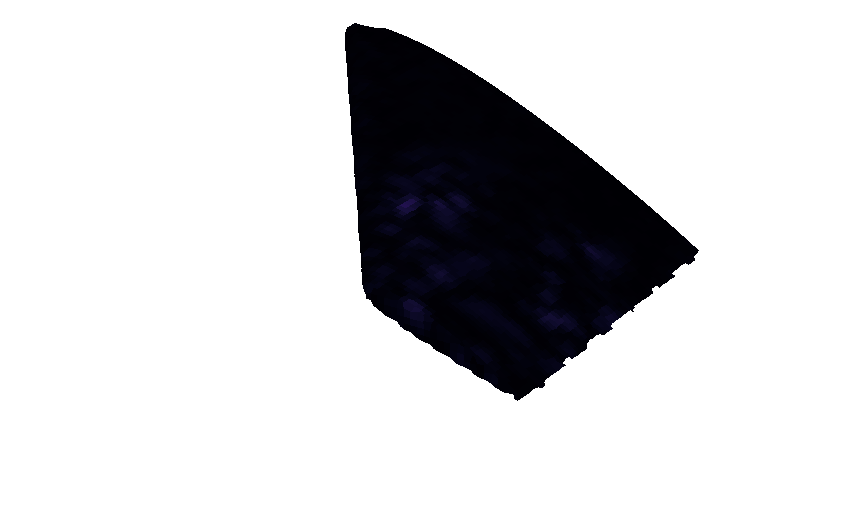} &
			\includegraphics[width=0.23\linewidth]{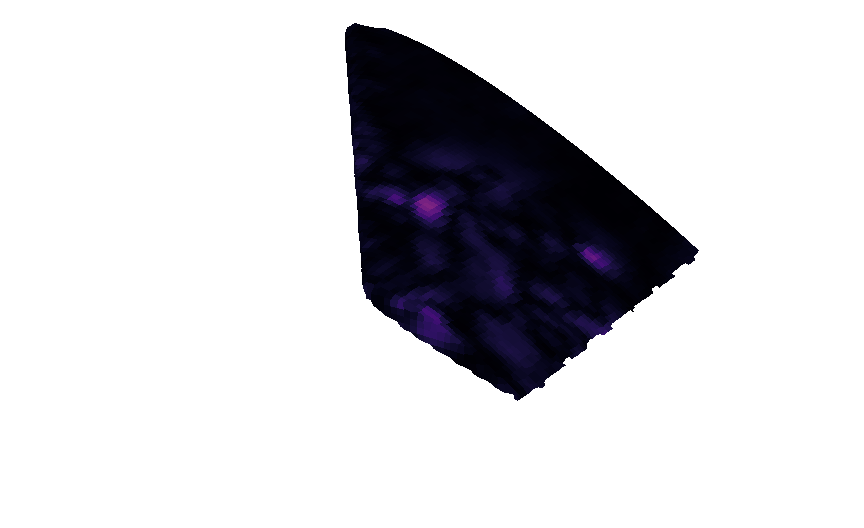} &
			\includegraphics[width=0.23\linewidth]{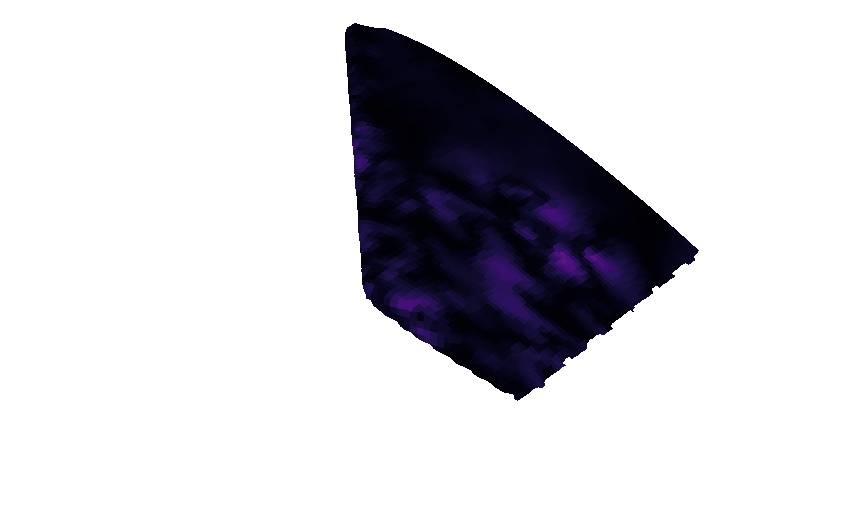} &
			\includegraphics[width=0.23\linewidth]{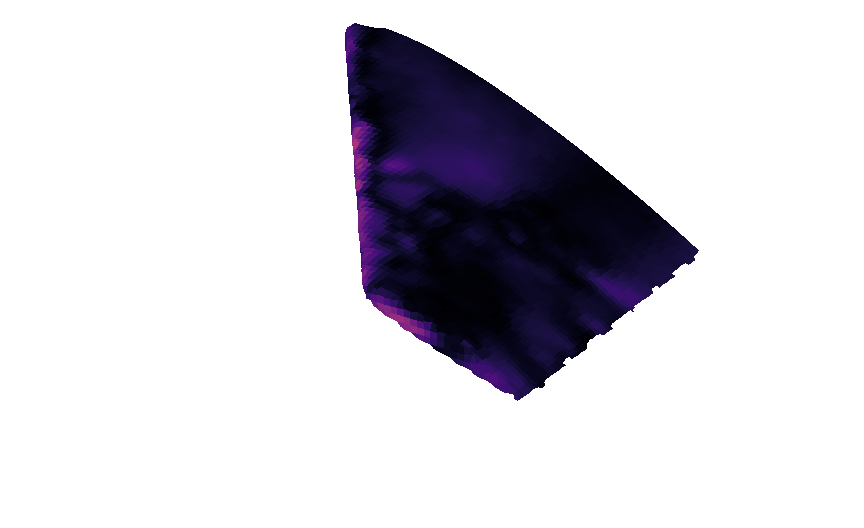}
		\end{tabular}
		\vspace{0.1em}
		\includegraphics[width=0.96\linewidth]{figures/ponwing_surface_centered_colorbars.png}
		\caption{\texttt{pOnWing} surface-centered pressure comparison for simulation \(11\). Columns are rollout steps \(t=1,5,10,50\); rows show (a) centered ground truth, (b) Base prediction, (c) Base absolute error, (d) Chebyshev BSP prediction, (e) Chebyshev BSP absolute error, (f) GLEAM prediction, and (g) GLEAM absolute error.}
		\label{fig:ponwing_surface_centered_sim11}
	\end{figure}
	
	\begin{figure}[!htbp]
		\centering
		\setlength{\tabcolsep}{1pt}
		\renewcommand{\arraystretch}{0.72}
		\scriptsize
		\begin{tabular}{@{}>{\raggedleft\arraybackslash}p{0.035\linewidth}cccc@{}}
			& \textbf{\(t=1\)} & \textbf{\(t=5\)} & \textbf{\(t=10\)} & \textbf{\(t=50\)}\\[-0.1em]
			\textbf{(a)} &
			\includegraphics[width=0.23\linewidth]{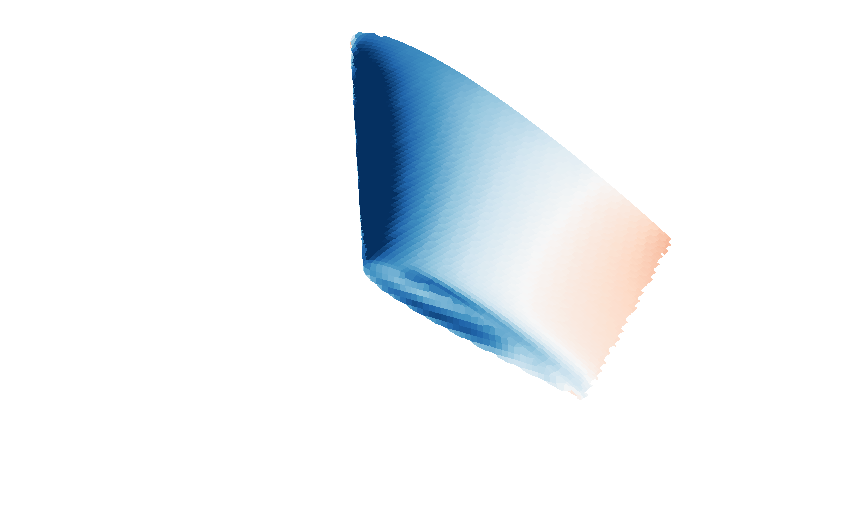} &
			\includegraphics[width=0.23\linewidth]{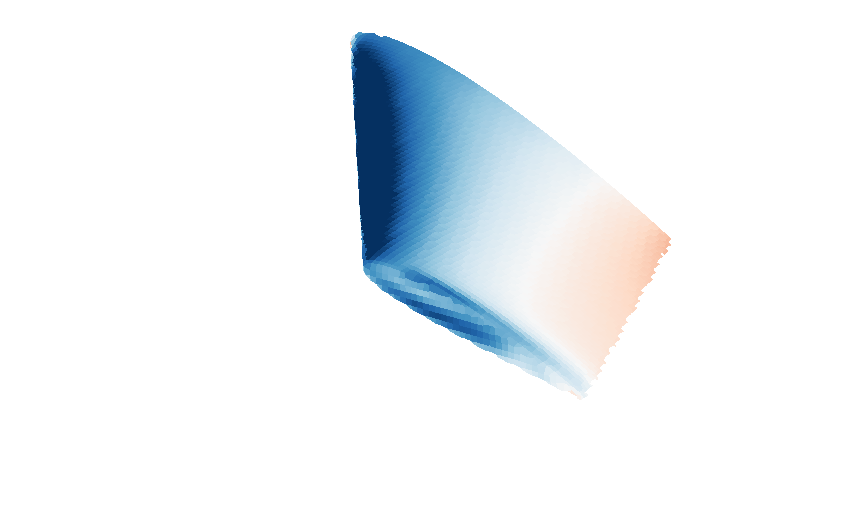} &
			\includegraphics[width=0.23\linewidth]{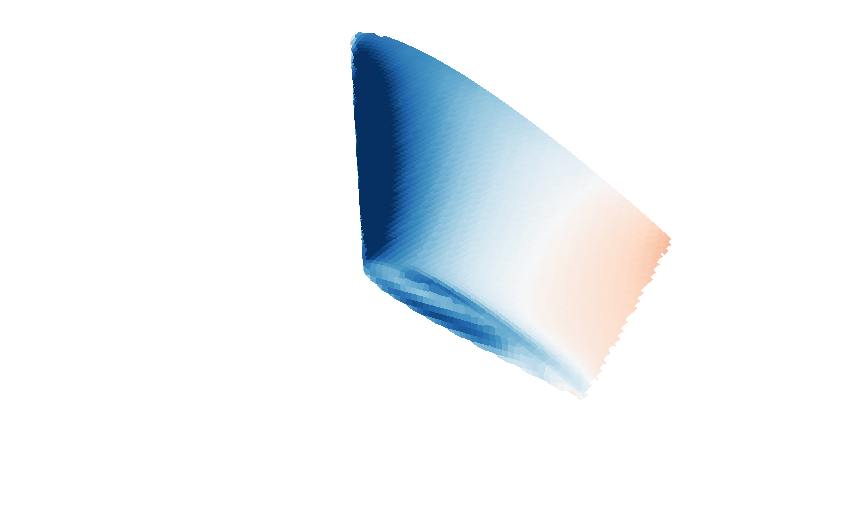} &
			\includegraphics[width=0.23\linewidth]{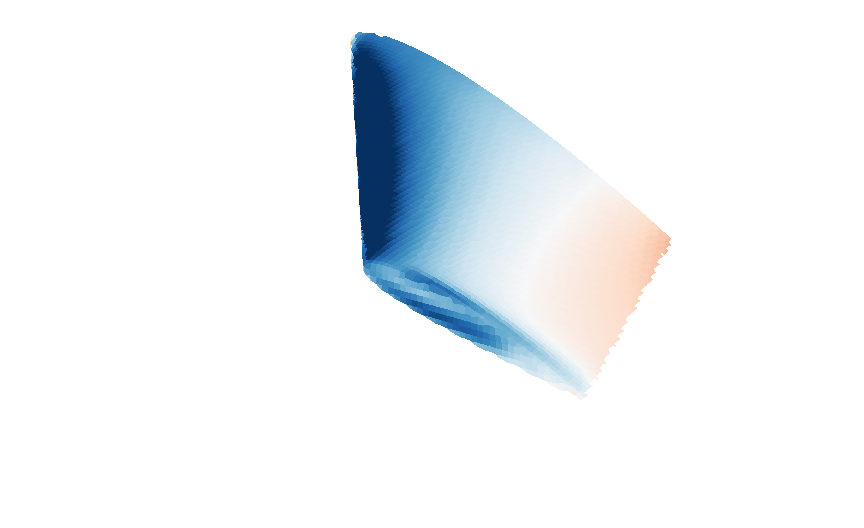}\\[-0.2em]
			\textbf{(b)} &
			\includegraphics[width=0.23\linewidth]{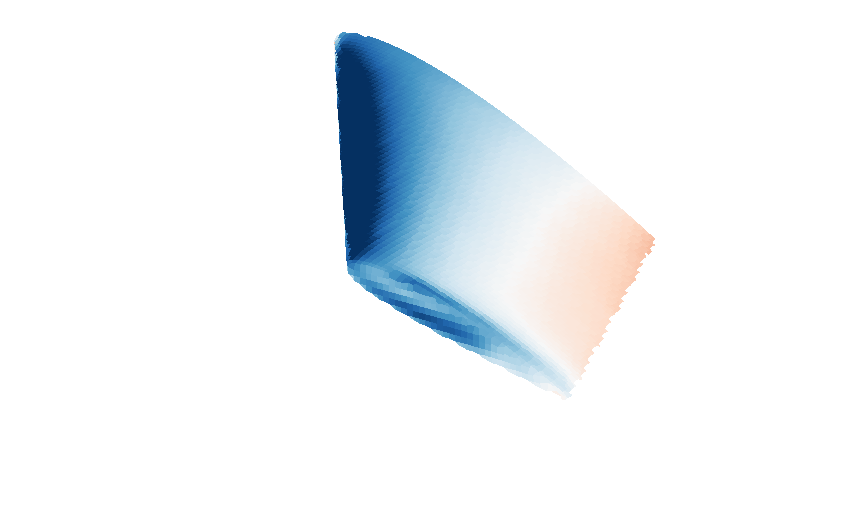} &
			\includegraphics[width=0.23\linewidth]{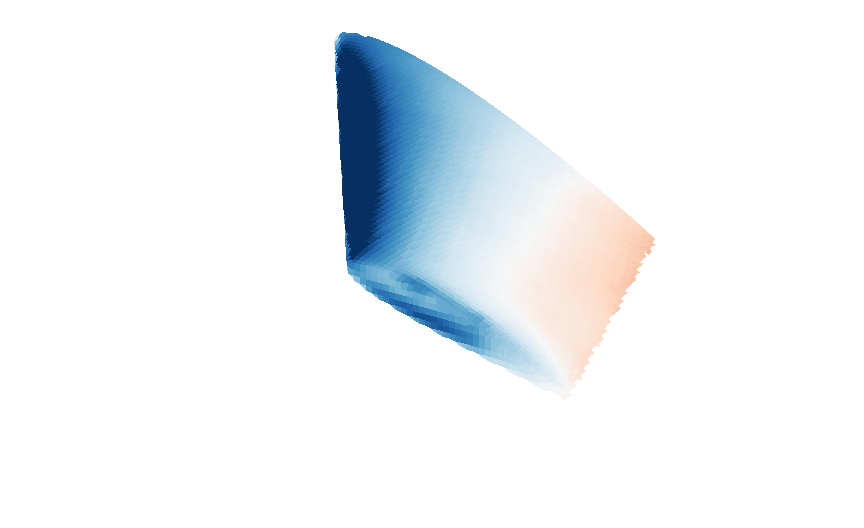} &
			\includegraphics[width=0.23\linewidth]{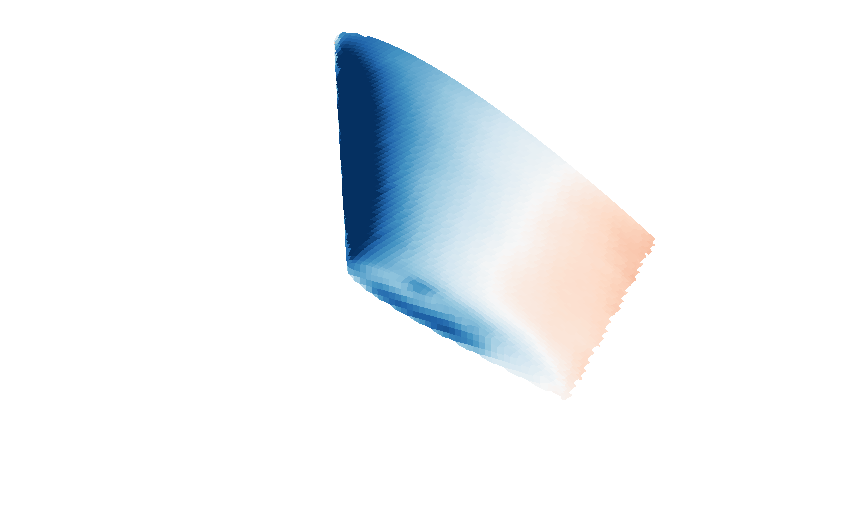} &
			\includegraphics[width=0.23\linewidth]{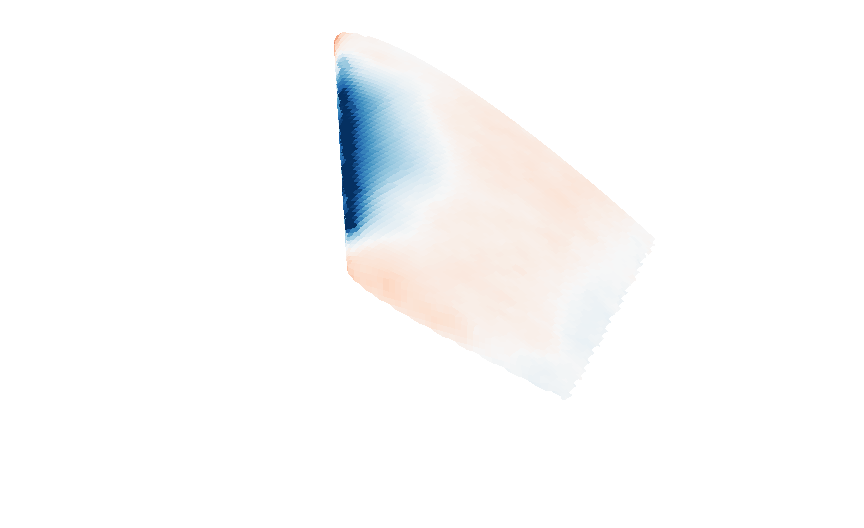}\\[-0.2em]
			\textbf{(c)} &
			\includegraphics[width=0.23\linewidth]{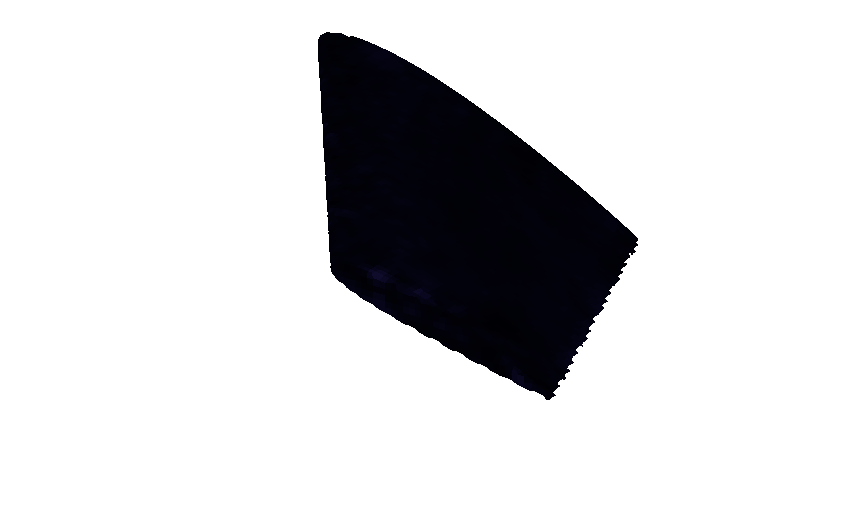} &
			\includegraphics[width=0.23\linewidth]{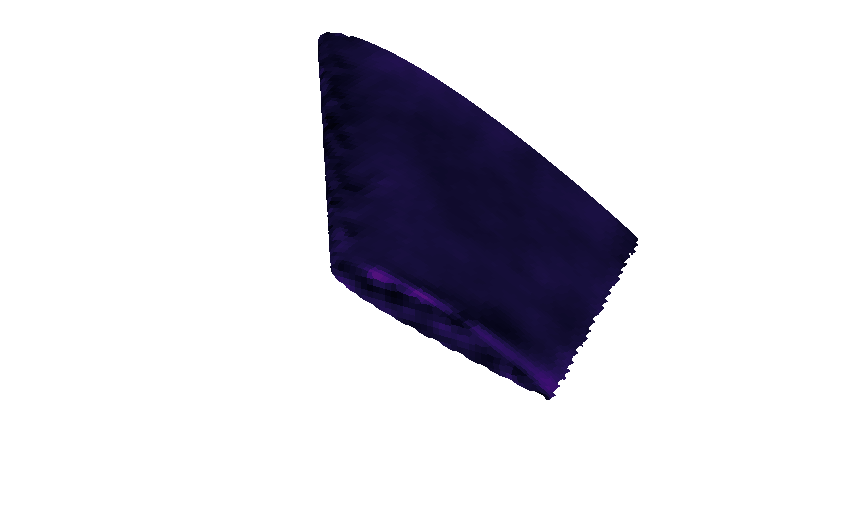} &
			\includegraphics[width=0.23\linewidth]{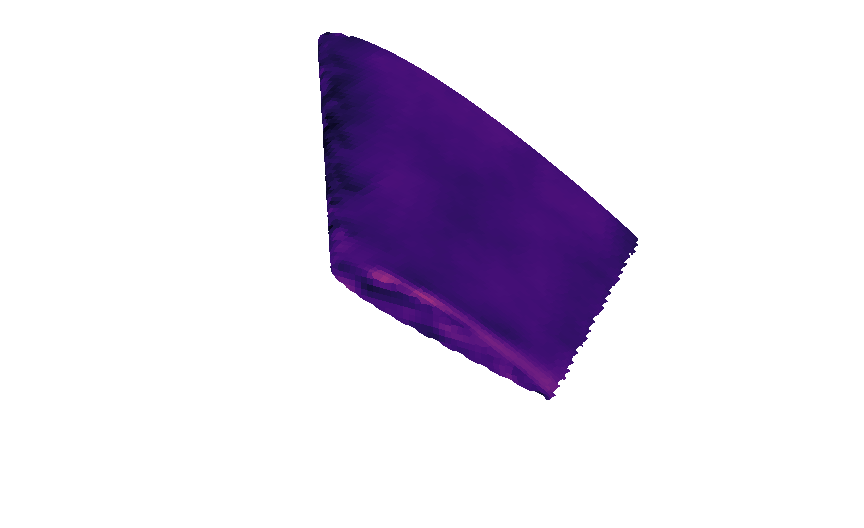} &
			\includegraphics[width=0.23\linewidth]{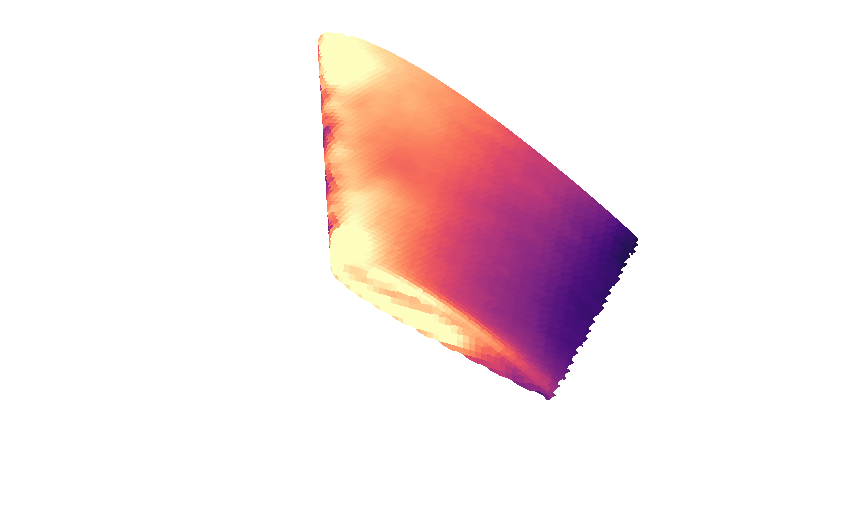}\\[-0.2em]
			\textbf{(d)} &
			\includegraphics[width=0.23\linewidth]{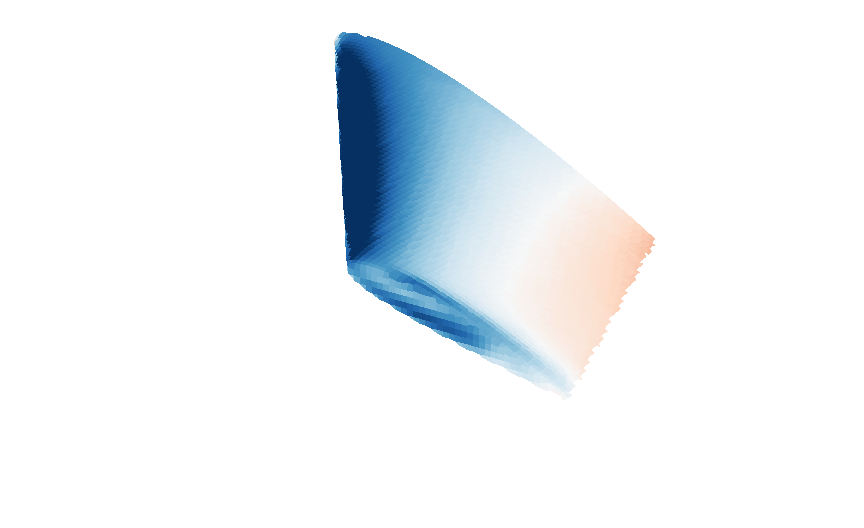} &
			\includegraphics[width=0.23\linewidth]{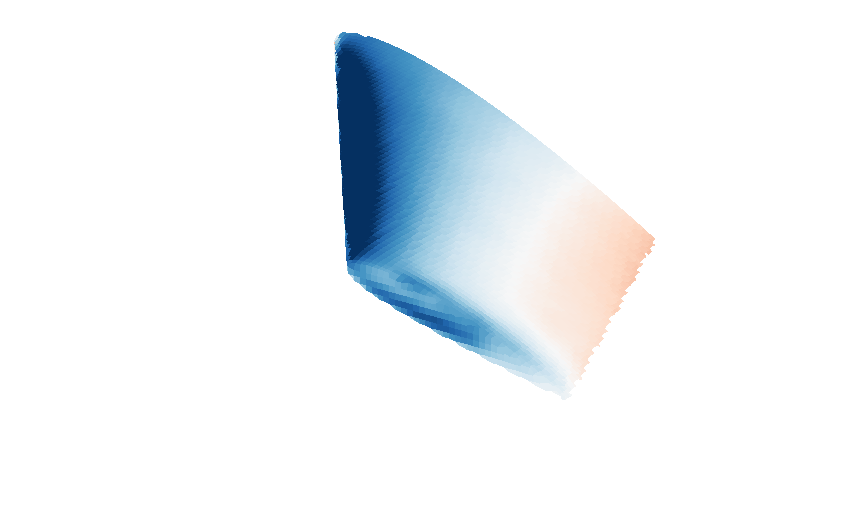} &
			\includegraphics[width=0.23\linewidth]{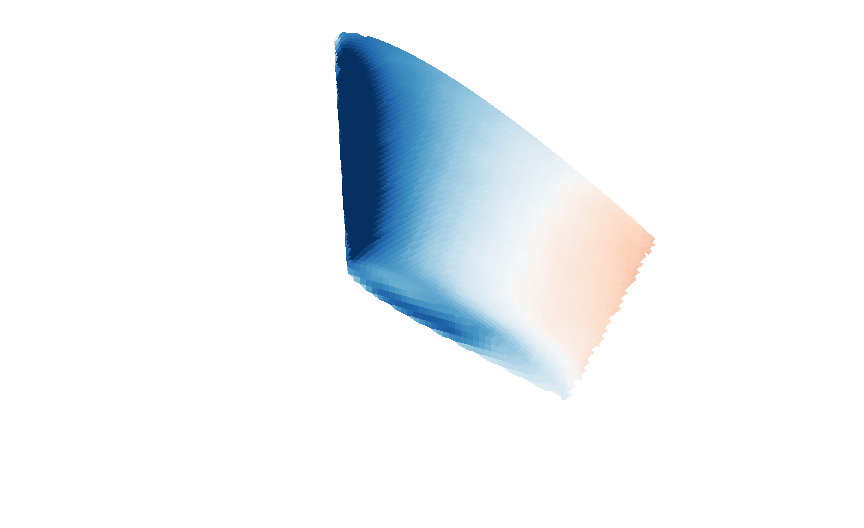} &
			\includegraphics[width=0.23\linewidth]{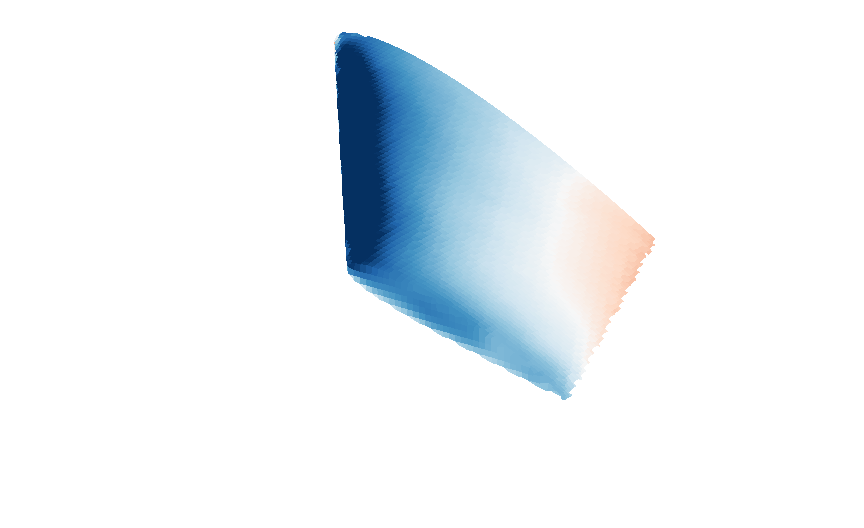}\\[-0.2em]
			\textbf{(e)} &
			\includegraphics[width=0.23\linewidth]{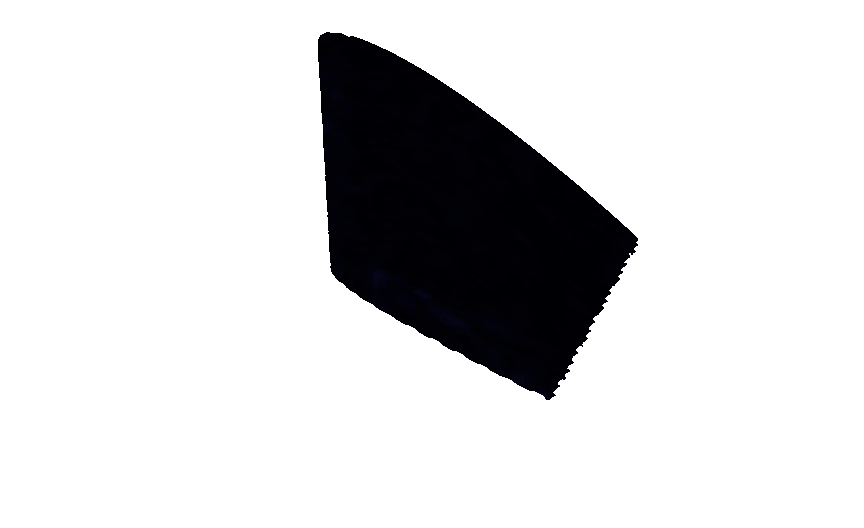} &
			\includegraphics[width=0.23\linewidth]{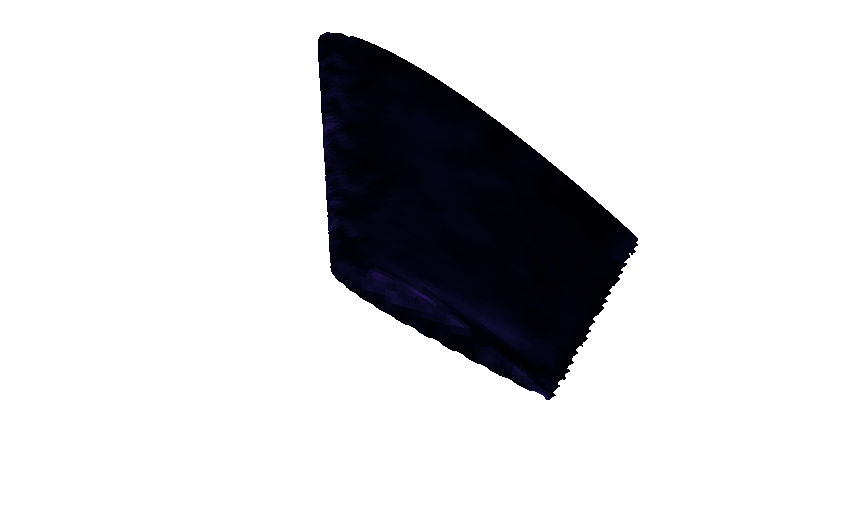} &
			\includegraphics[width=0.23\linewidth]{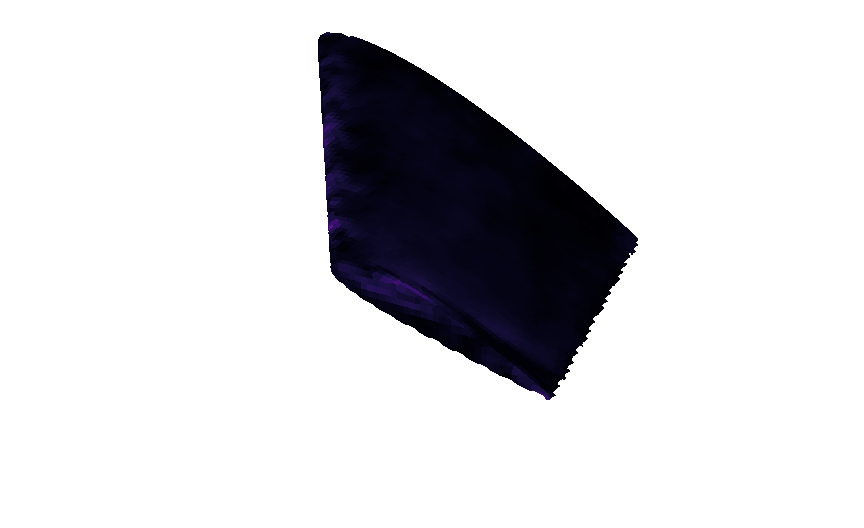} &
			\includegraphics[width=0.23\linewidth]{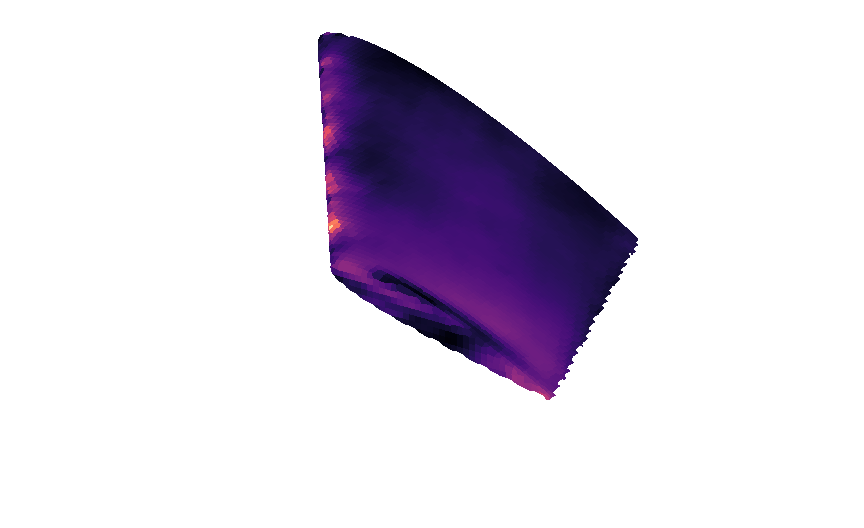}\\[-0.2em]
			\textbf{(f)} &
			\includegraphics[width=0.23\linewidth]{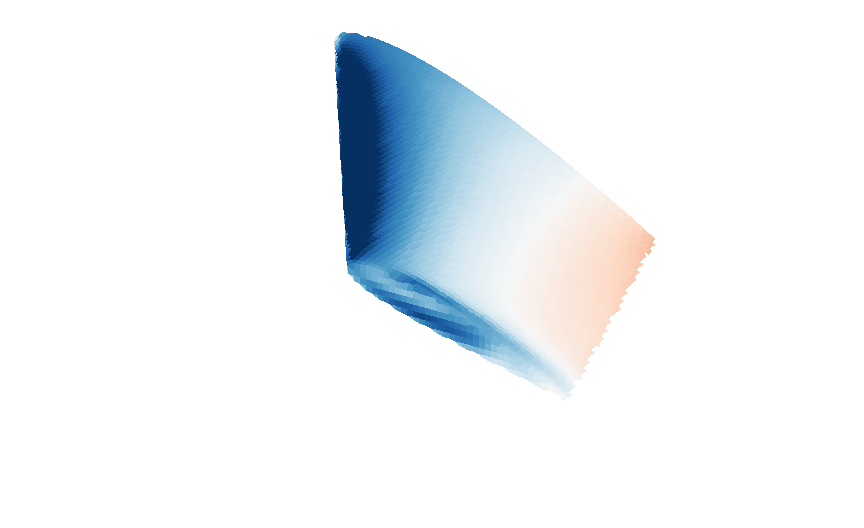} &
			\includegraphics[width=0.23\linewidth]{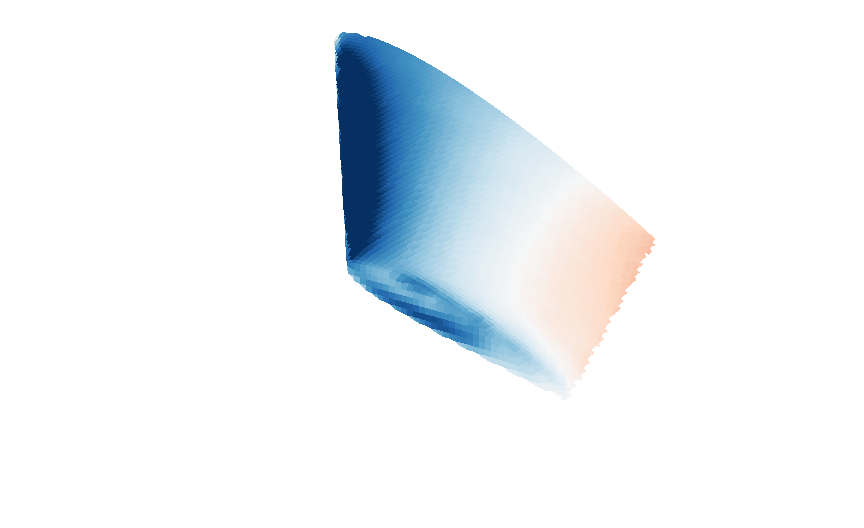} &
			\includegraphics[width=0.23\linewidth]{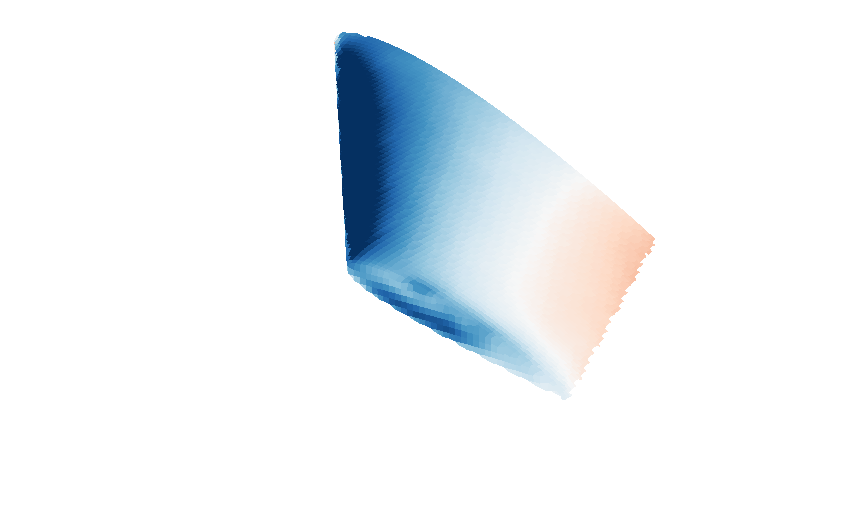} &
			\includegraphics[width=0.23\linewidth]{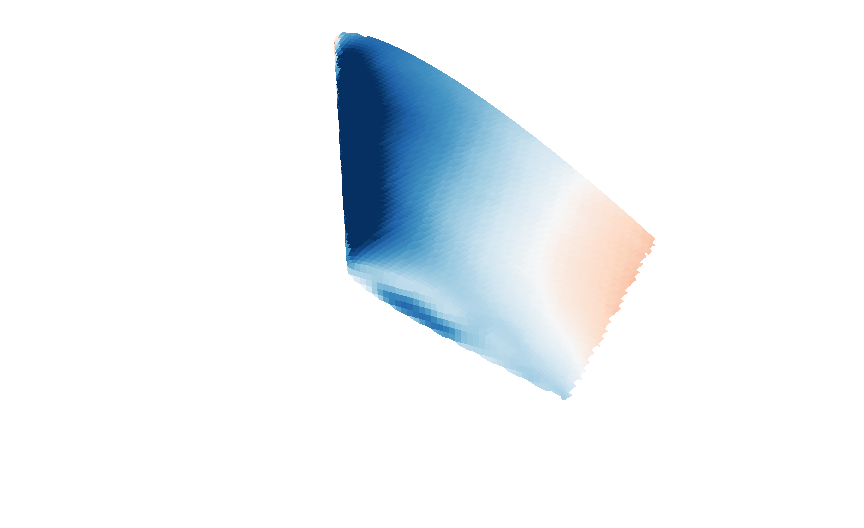}\\[-0.2em]
			\textbf{(g)} &
			\includegraphics[width=0.23\linewidth]{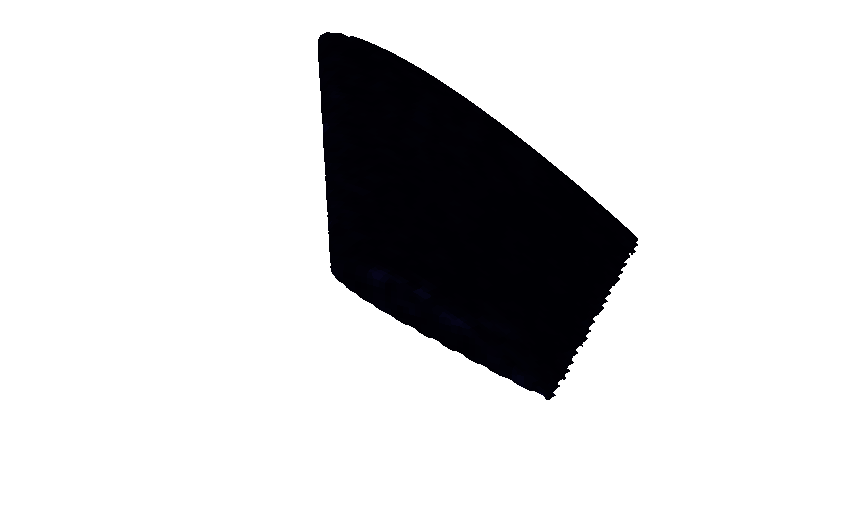} &
			\includegraphics[width=0.23\linewidth]{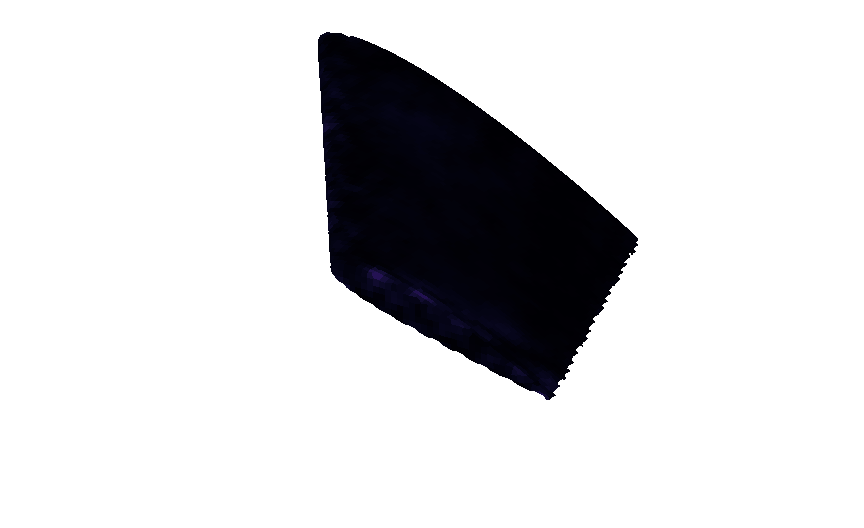} &
			\includegraphics[width=0.23\linewidth]{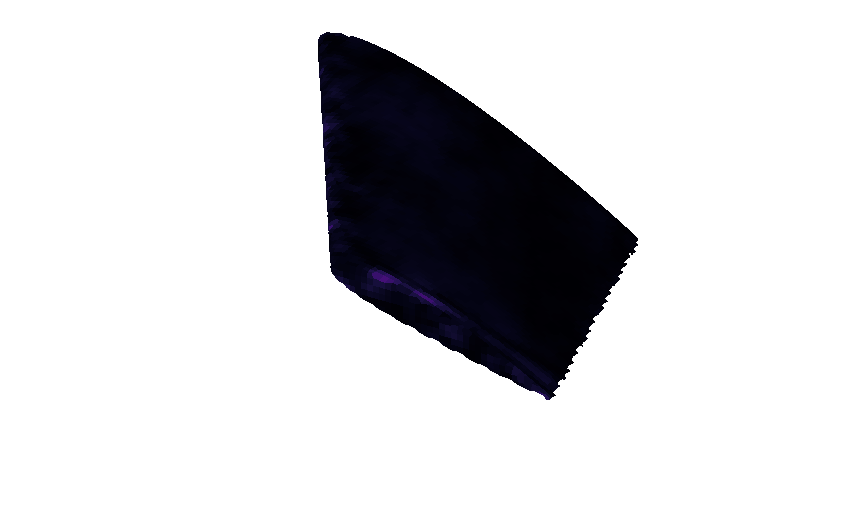} &
			\includegraphics[width=0.23\linewidth]{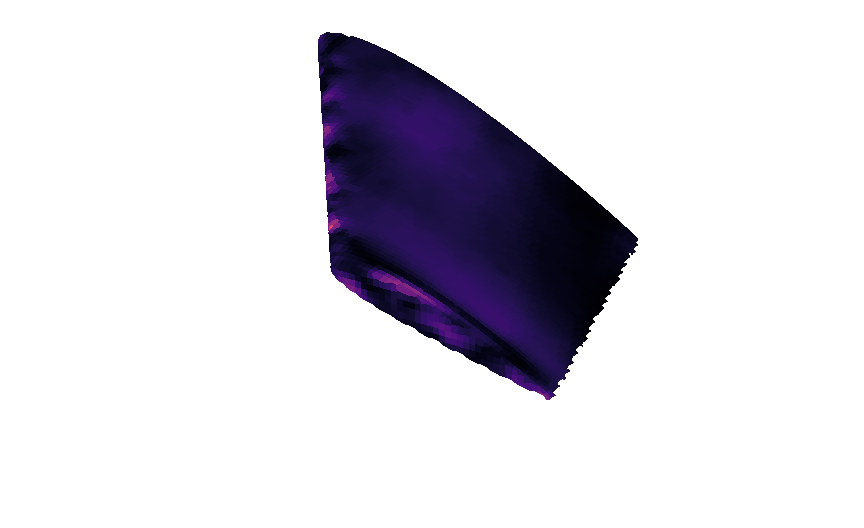}
		\end{tabular}
		\vspace{0.1em}
		\includegraphics[width=0.96\linewidth]{figures/ponwing_surface_centered_colorbars.png}
		\caption{\texttt{pOnWing} surface-centered pressure comparison for simulation \(14\). Columns are rollout steps \(t=1,5,10,50\); rows show (a) centered ground truth, (b) Base prediction, (c) Base absolute error, (d) Chebyshev BSP prediction, (e) Chebyshev BSP absolute error, (f) GLEAM prediction, and (g) GLEAM absolute error.}
		\label{fig:ponwing_surface_centered_sim14}
	\end{figure}
	
	\begin{figure}[!htbp]
		\centering
		\includegraphics[width=0.96\linewidth]{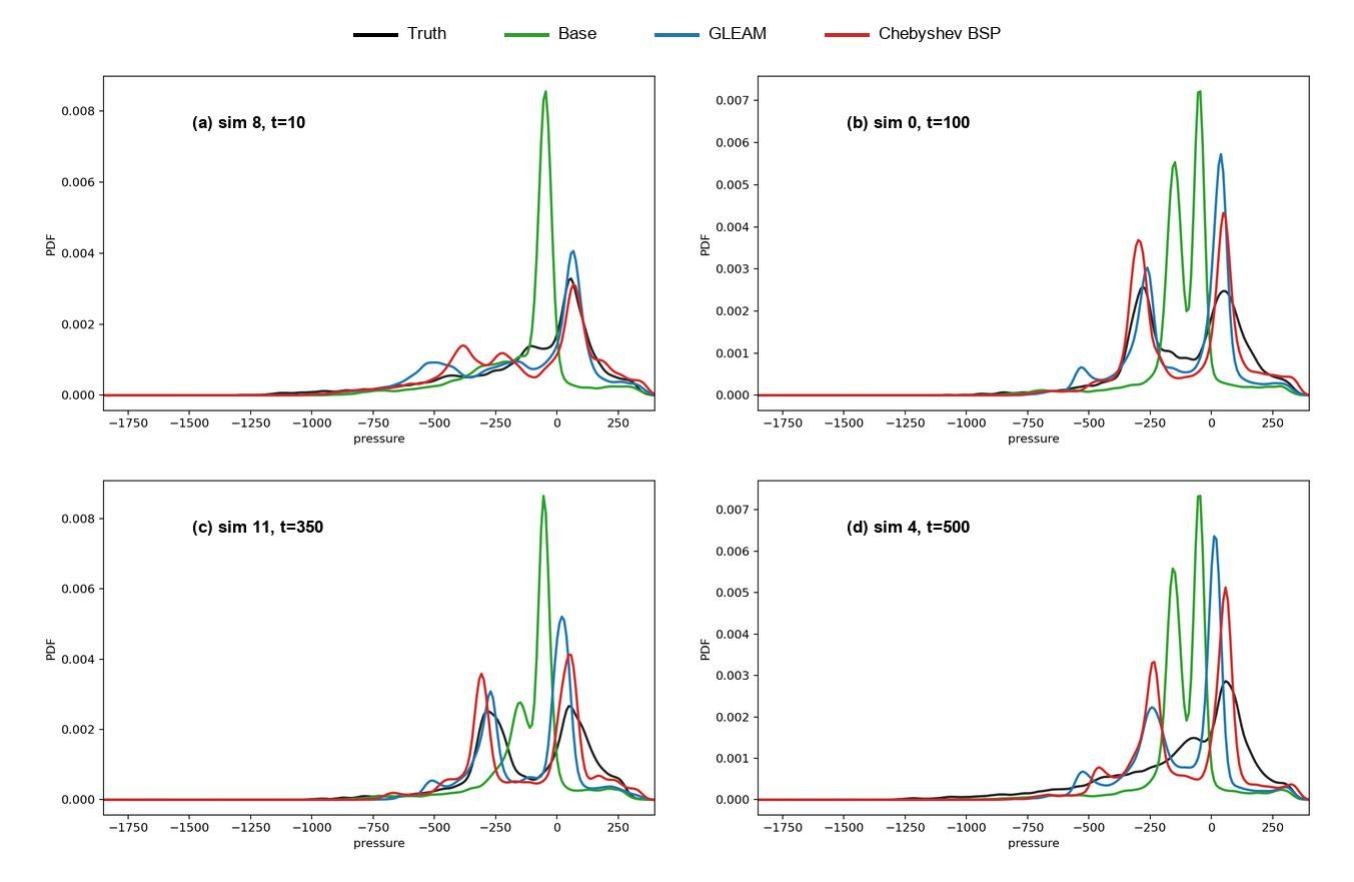}
		\caption{\texttt{pOnWing} pressure-PDF diagnostics at representative autoregressive rollout horizons. The panels compare the pressure distribution of the ground truth, Base model, GLEAM model, and Chebyshev BSP model at \(t=10,100,350,\) and \(500\) using different test simulations. The spectral variants better preserve the broad pressure distribution, while the Base model develops sharper displaced peaks.}
		\label{fig:ponwing_pressure_pdf_app}
	\end{figure}
	
	\section{\texorpdfstring{Comparison with Existing Stabilization and Spectral Baselines}{Comparison with Existing Stabilization and Spectral Baselines}}
	\label{app:bfs-similar-methods}
	To place the BFS spectral results against existing alternatives, we compare the Base model, PyNUFFT, Pushforward, Chebyshev BSP, and GLEAM on the same held-out BFS test sequence. The comparison keeps the autoregressive evaluation protocol fixed and separates the role of the auxiliary training signal. PyNUFFT applies nonuniform Fourier-style spectral supervision on the irregular point set using the PyNUFFT setting \citep{lin2017pynufft}, testing whether an off-grid Euclidean Fourier representation is sufficient for this unstructured BFS rollout. Pushforward provides the rollout-exposure reference following \citet{brandstetter2022message}, without adding graph-spectral energy alignment. Chebyshev BSP provides the final-level graph-spectral baseline, while GLEAM provides the hierarchy-aware graph-spectral variant. Because the BFS spectral runs already use a short rollout curriculum, the comparison should be read as a comparison among related trained models and auxiliary objectives, rather than as a strict exposure-versus-no-exposure ablation. We report complementary diagnostics: rollout RMSE, exact graph-spectral errors, velocity-magnitude PDFs, and late-horizon spatial error maps. All curves are computed from autoregressive rollouts using EMA weights. We also include a late-rollout EAGLE field comparison between Chebyshev BSP and Pushforward training to show how the rollout-exposure and graph-spectral objectives differ in full-field velocity and pressure variables.
	
	\begin{table}[!htbp]
		\centering
		\caption{Comparison of different methods wrt to the spectral methods to improve autoregressive forecast. RMSE is computed from the denormalized autoregressive velocity rollout. Spectral RMSE and spectral relative \(L_1\) are computed from \(24\) exact normalized-Laplacian graph bands and are used only as evaluation diagnostics.}
		\label{tab:bfs_similar_methods}
		\scriptsize
		\setlength{\tabcolsep}{3.2pt}
		\resizebox{\linewidth}{!}{%
			\begin{tabular}{lccccc}
				\toprule
				Method & Mean RMSE & Final RMSE & Mean spectral RMSE & Final spectral RMSE & Mean spectral rel. \(L_1\) \\
				\midrule
				Base & 14.061 & 22.729 & 20.497 & 31.789 & 0.512 \\
				PyNUFFT \(\lambda=0.01\) & 13.335 & 24.840 & 7.797 & \(\mathbf{3.410}\) & 0.195 \\
				Pushforward & 6.619 & 8.643 & 7.906 & 10.078 & 0.177 \\
				Chebyshev BSP \(\lambda=0.1\) & 7.468 & 12.333 & 10.010 & 20.562 & 0.226 \\
				GLEAM \(\lambda=10^{-4}\) & \(\mathbf{5.928}\) & \(\mathbf{8.469}\) & \(\mathbf{6.201}\) & 7.427 & \(\mathbf{0.145}\) \\
				\bottomrule
			\end{tabular}
		}
	\end{table}
	
	\begin{figure}[!htbp]
		\centering
		\begin{subfigure}[t]{0.49\linewidth}
			\centering
			\includegraphics[width=\linewidth]{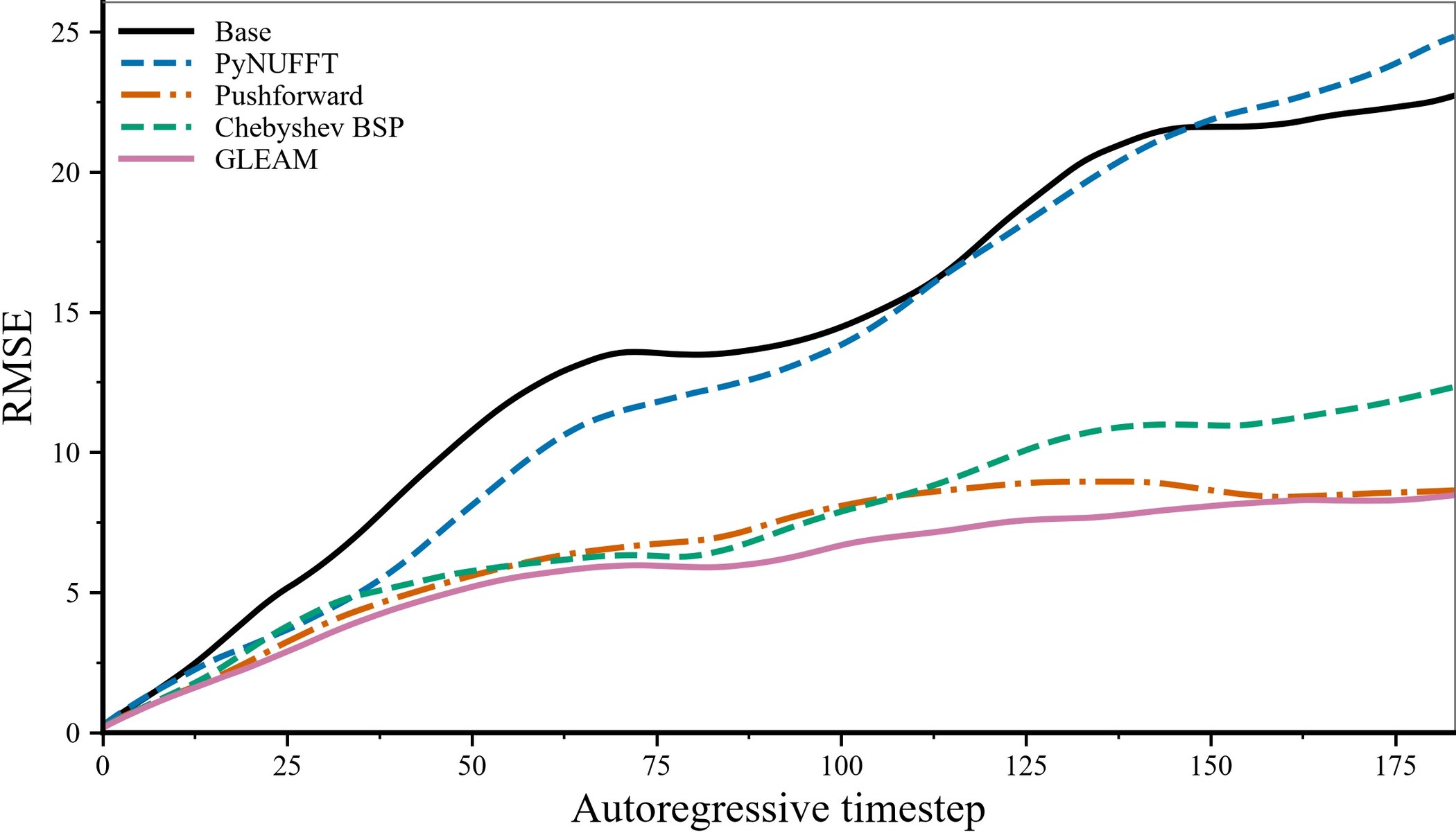}
			\caption{Velocity RMSE.}
			\label{fig:bfs_similar_methods_rmse}
		\end{subfigure}\hfill
		\begin{subfigure}[t]{0.49\linewidth}
			\centering
			\includegraphics[width=\linewidth]{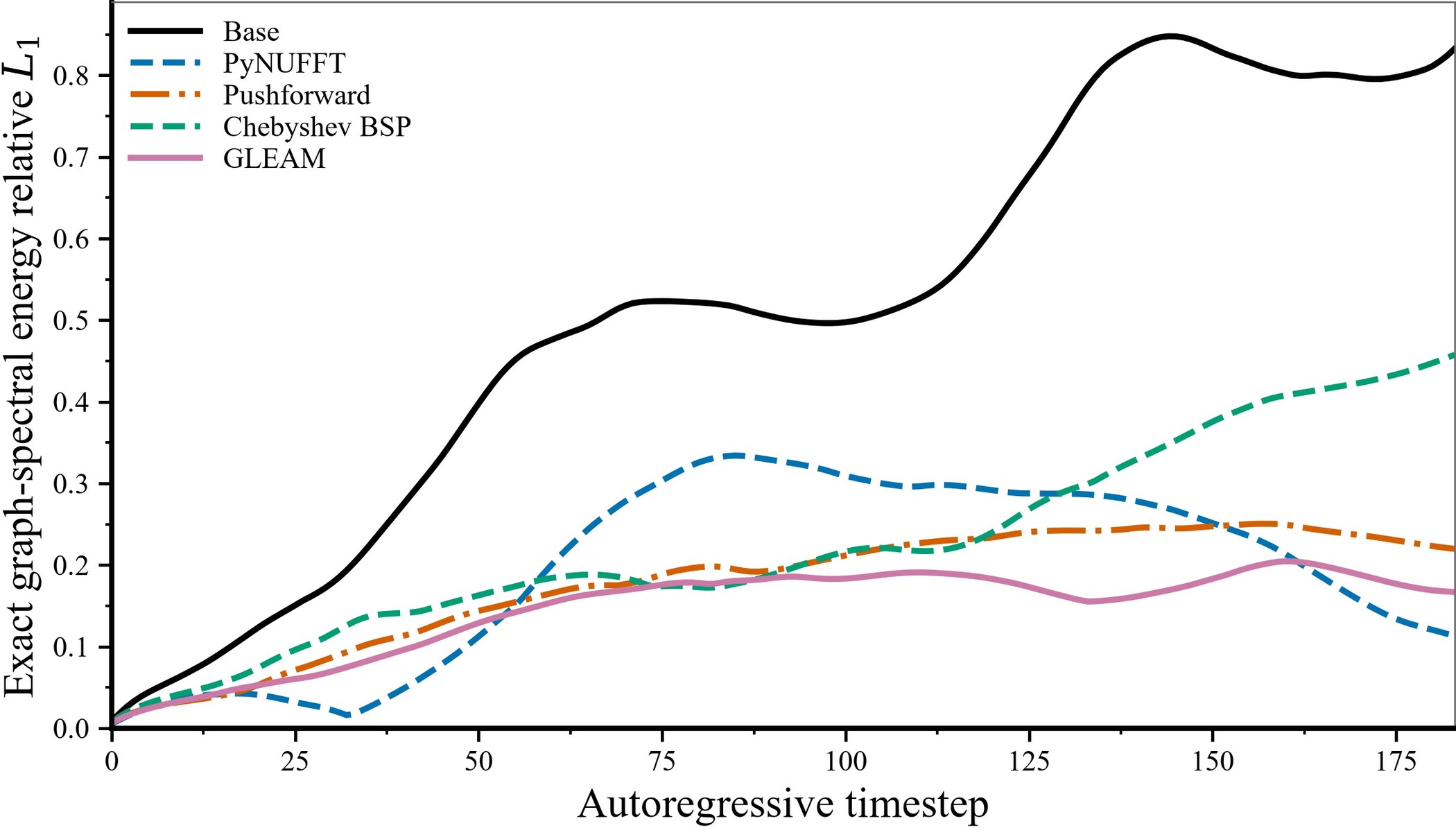}
			\caption{Graph-spectral relative \(L_1\).}
			\label{fig:bfs_similar_methods_spectral}
		\end{subfigure}
		\caption{BFS comparison with related trained baselines. Panel (a) reports autoregressive velocity RMSE, while panel (b) reports exact normalized-Laplacian spectral-energy relative \(L_1\) over the same rollout. Lower is better in both panels.}
		\label{fig:bfs_similar_methods_curves}
	\end{figure}
	
	\begin{figure}[!htbp]
		\centering
		\includegraphics[width=\linewidth]{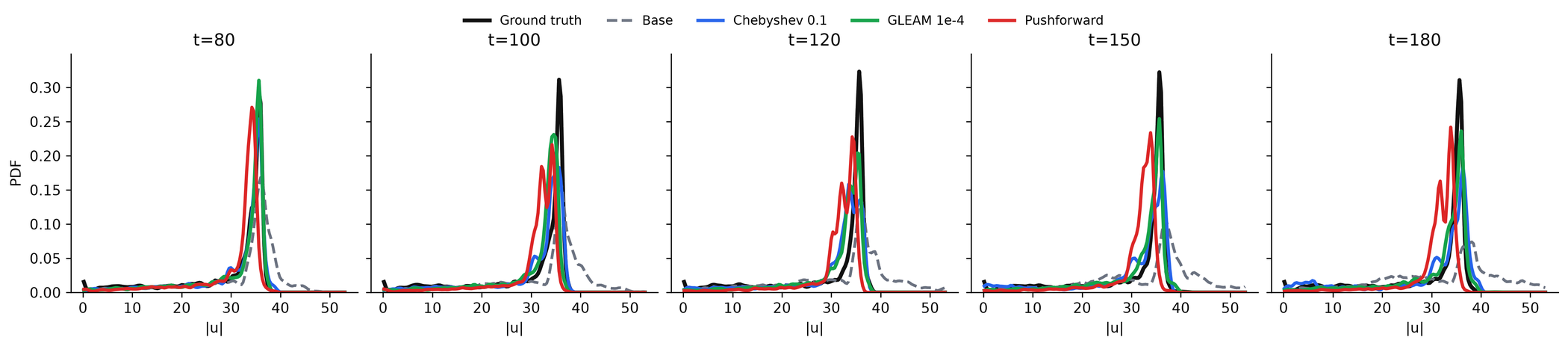}
		\caption{BFS velocity-magnitude probability-density comparison for the Base, Pushforward, Chebyshev BSP, and GLEAM rollouts at late autoregressive horizons. The panels compare the distribution of \(|\bm u|\) for the ground truth, Base model, Pushforward baseline, Chebyshev BSP with \(\lambda_{\mathrm{Cheb}}=0.1\), and GLEAM with \(\lambda_{\mathrm{GLEAM}}=10^{-4}\). The PDF diagnostic complements RMSE by showing whether each method preserves the distribution of speed values during rollout.}
		\label{fig:bfs_similar_methods_velocity_pdf}
	\end{figure}
	
	\begin{figure}[!htbp]
		\centering
		\includegraphics[width=0.96\linewidth]{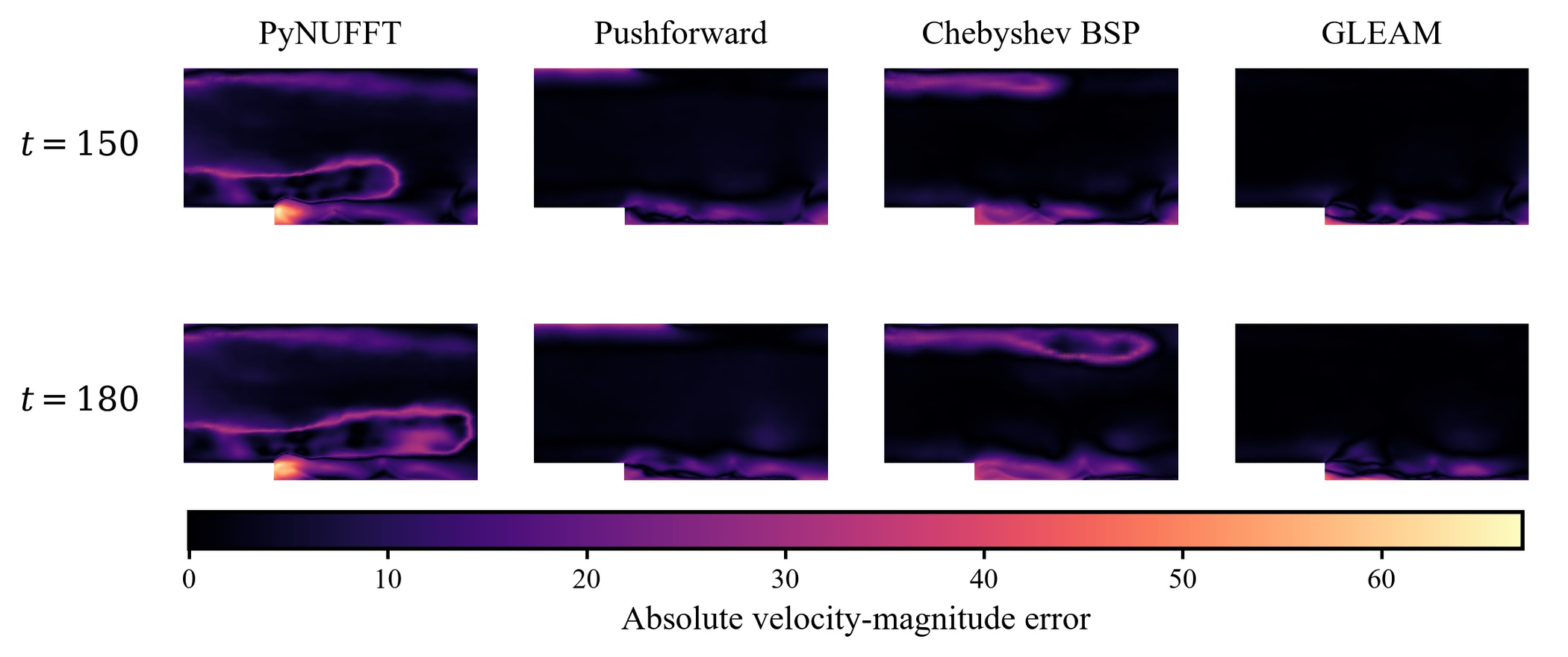}
		\caption{BFS absolute velocity-magnitude error maps at autoregressive horizons \(t=150\) and \(t=180\). Columns compare PyNUFFT, Pushforward, Chebyshev BSP, and GLEAM, with a shared error color scale across all panels.}
		\label{fig:bfs_similar_methods_error_maps}
	\end{figure}
	
	The BFS diagnostics above isolate the behavior of the trained models on one held-out sequence, while the EAGLE panels in Fig.~\ref{fig:eagle_high_t_chebyshev_pushforward} give a complementary field-level view of the same Chebyshev BSP versus Pushforward comparison. These late-time examples include both velocity magnitude and static pressure so that the comparison is not restricted to scalar rollout curves or velocity-only distributional diagnostics.
	
	\begin{figure}[!htbp]
		\centering
		\setlength{\tabcolsep}{1pt}
		\renewcommand{\arraystretch}{0.88}
		\scriptsize
		\begin{tabular}{@{}>{\raggedleft\arraybackslash}p{0.035\linewidth}cccc@{}}
			& \begin{tabular}[c]{@{}c@{}}\textbf{sample070}\\\(t=183\)\\\(|\bm u|\)\end{tabular}
			& \begin{tabular}[c]{@{}c@{}}\textbf{sample070}\\\(t=183\)\\\(p\)\end{tabular}
			& \begin{tabular}[c]{@{}c@{}}\textbf{sample014}\\\(t=194\)\\\(|\bm u|\)\end{tabular}
			& \begin{tabular}[c]{@{}c@{}}\textbf{sample014}\\\(t=194\)\\\(p\)\end{tabular}\\[-0.1em]
			\textbf{(a)} &
			\includegraphics[width=0.23\linewidth]{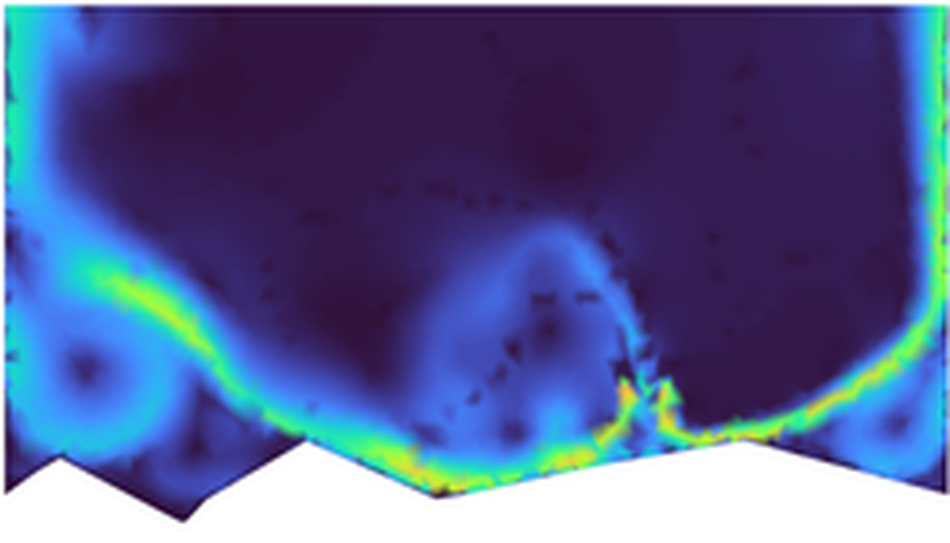} &
			\includegraphics[width=0.23\linewidth]{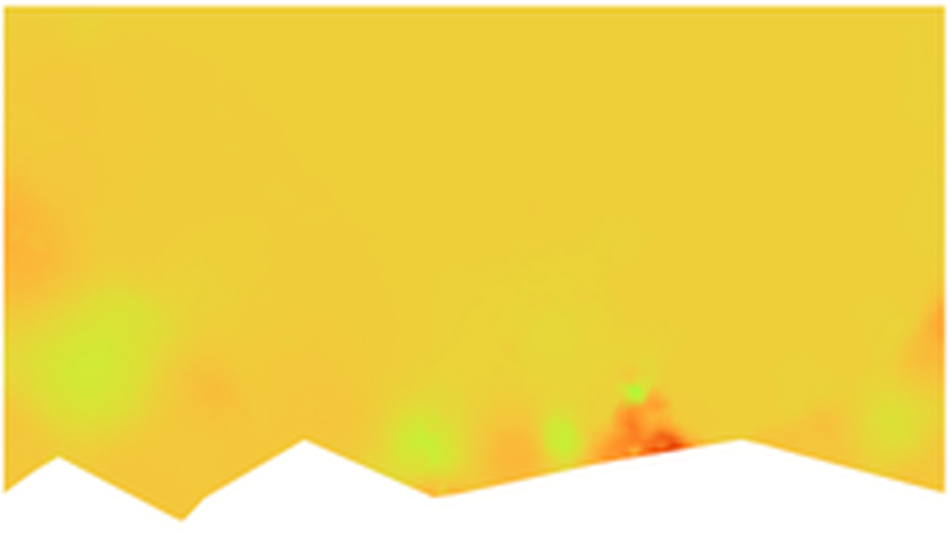} &
			\includegraphics[width=0.23\linewidth]{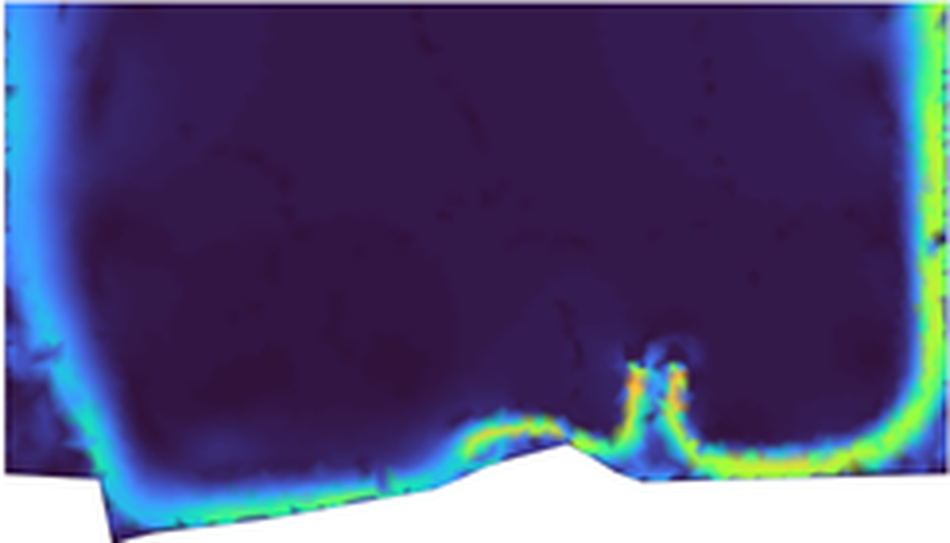} &
			\includegraphics[width=0.23\linewidth]{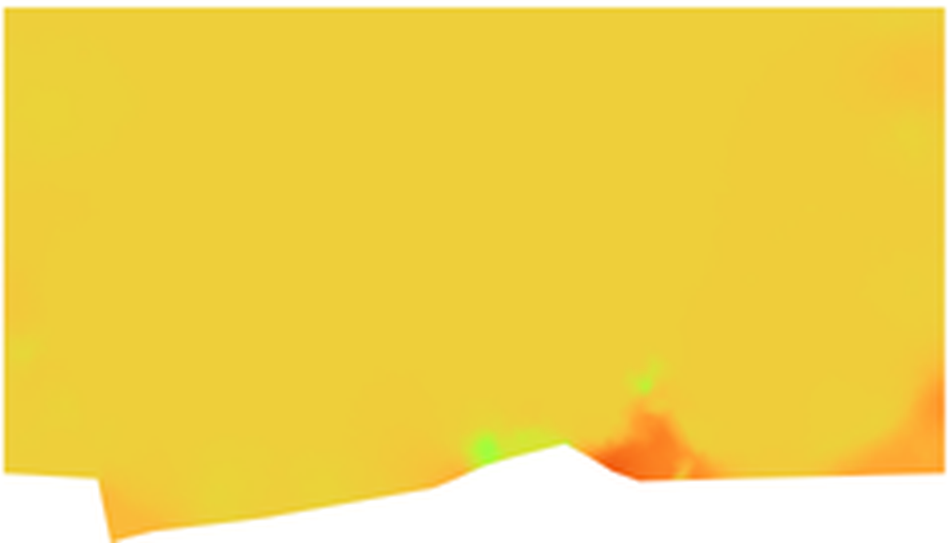}\\[-0.2em]
			\textbf{(b)} &
			\includegraphics[width=0.23\linewidth]{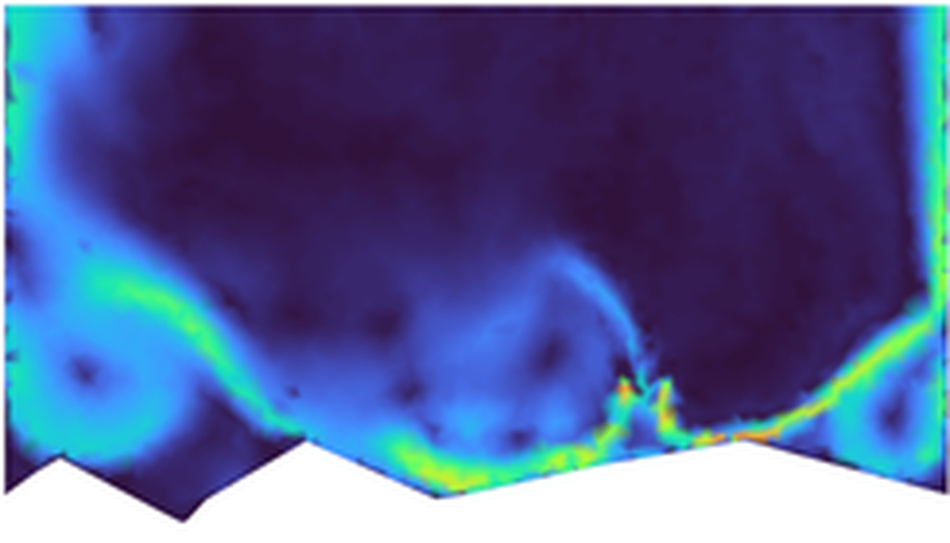} &
			\includegraphics[width=0.23\linewidth]{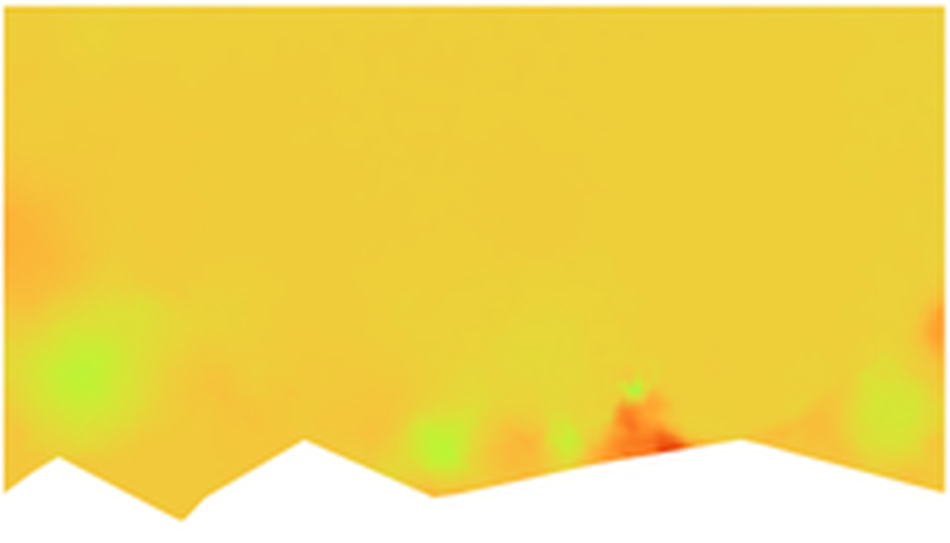} &
			\includegraphics[width=0.23\linewidth]{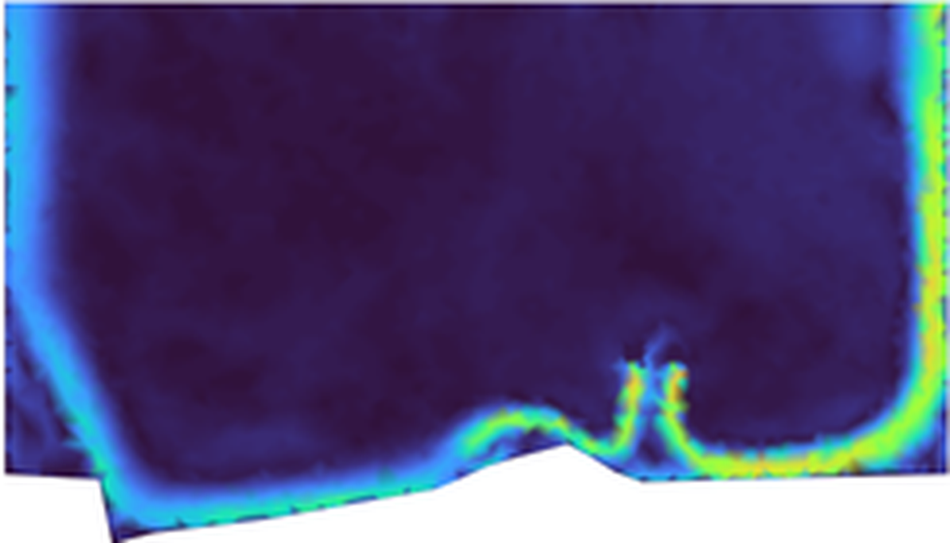} &
			\includegraphics[width=0.23\linewidth]{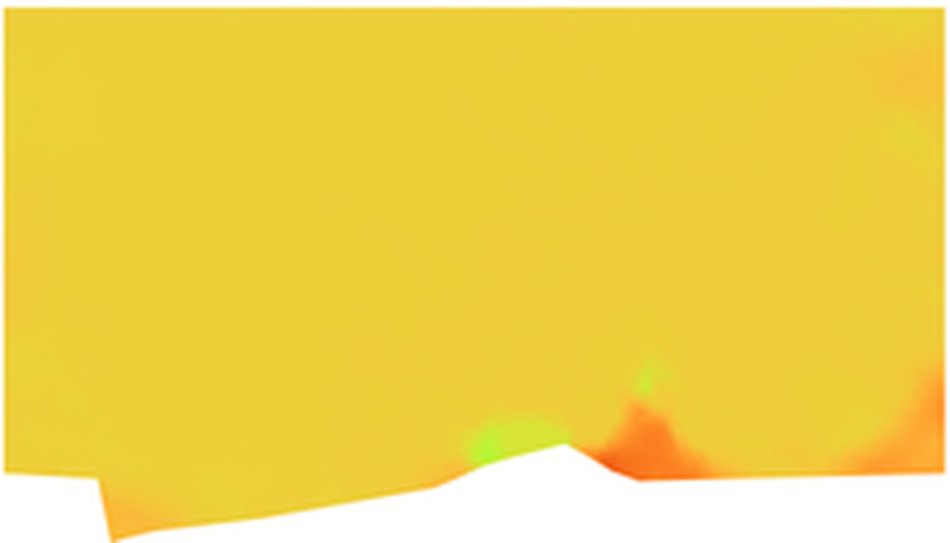}\\[-0.2em]
			\textbf{(c)} &
			\includegraphics[width=0.23\linewidth]{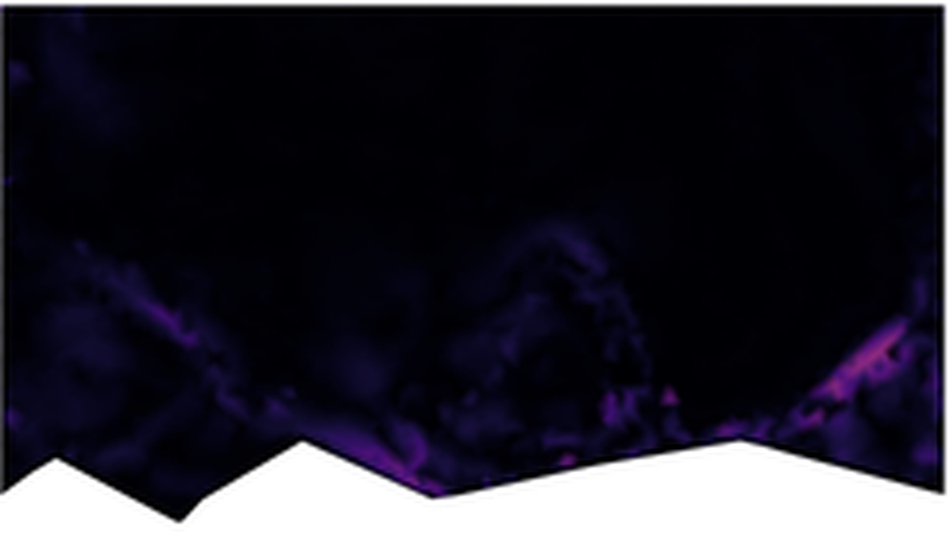} &
			\includegraphics[width=0.23\linewidth]{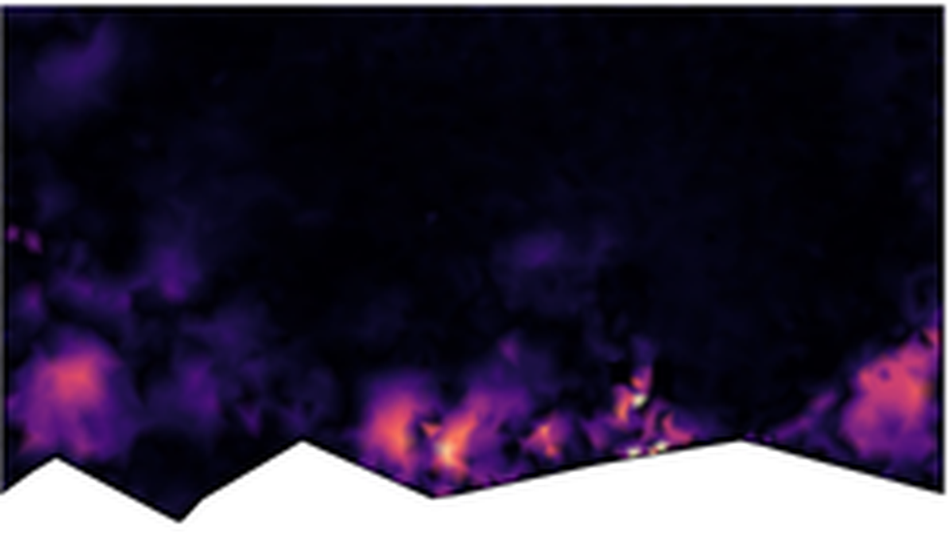} &
			\includegraphics[width=0.23\linewidth]{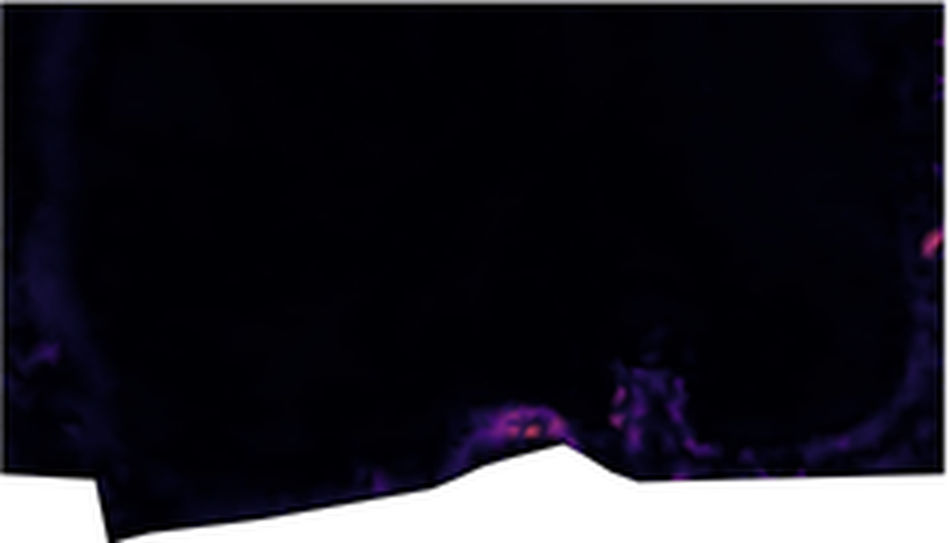} &
			\includegraphics[width=0.23\linewidth]{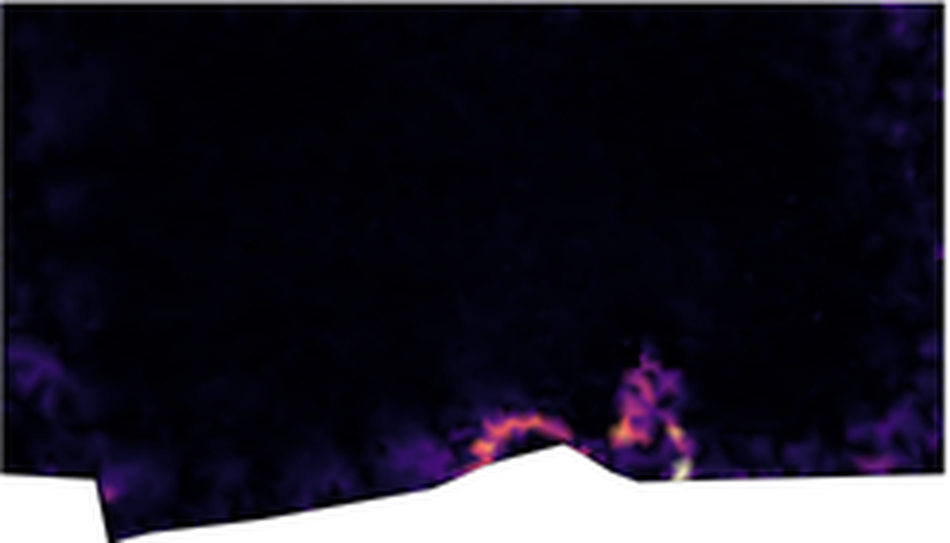}\\[-0.2em]
			\textbf{(d)} &
			\includegraphics[width=0.23\linewidth]{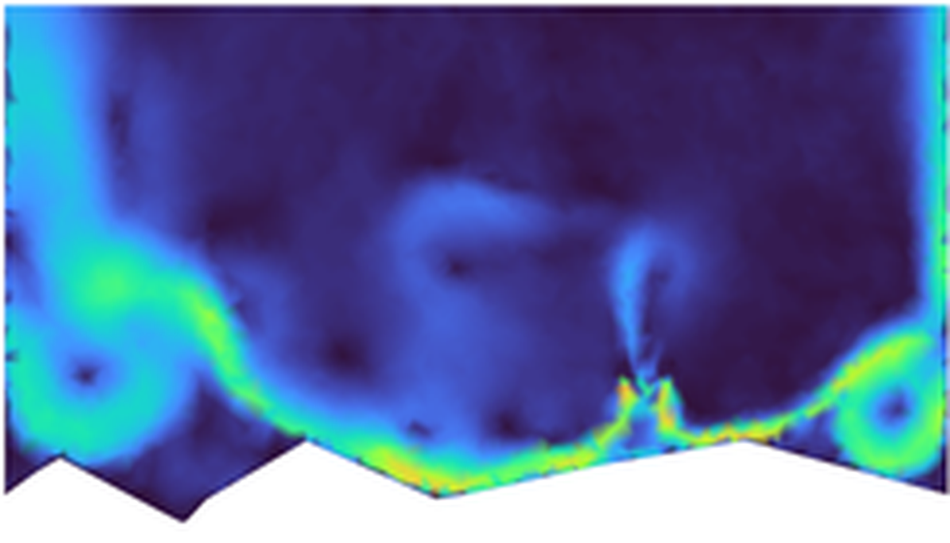} &
			\includegraphics[width=0.23\linewidth]{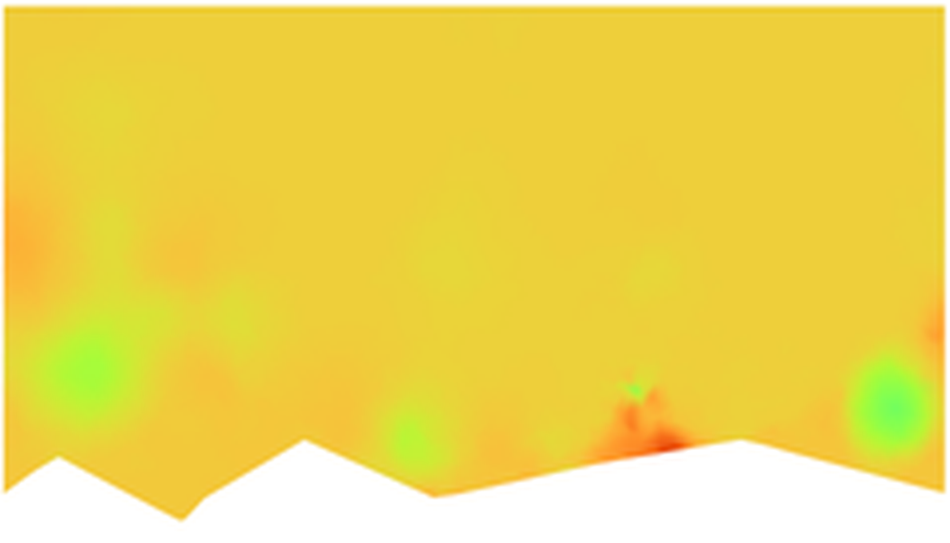} &
			\includegraphics[width=0.23\linewidth]{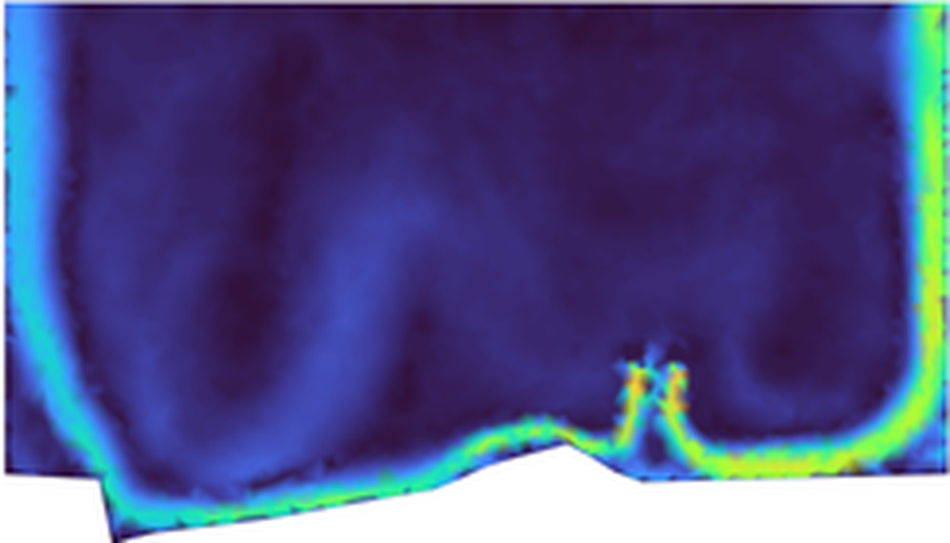} &
			\includegraphics[width=0.23\linewidth]{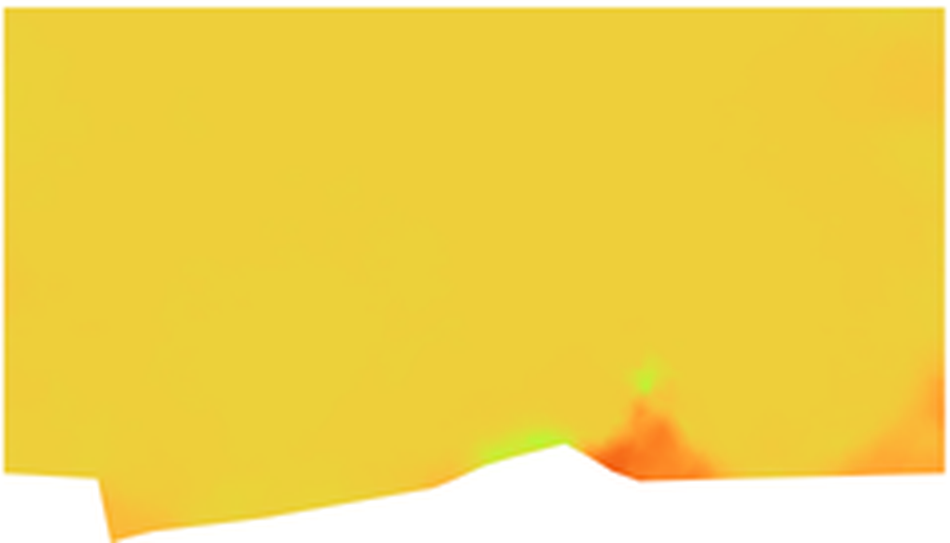}\\[-0.2em]
			\textbf{(e)} &
			\includegraphics[width=0.23\linewidth]{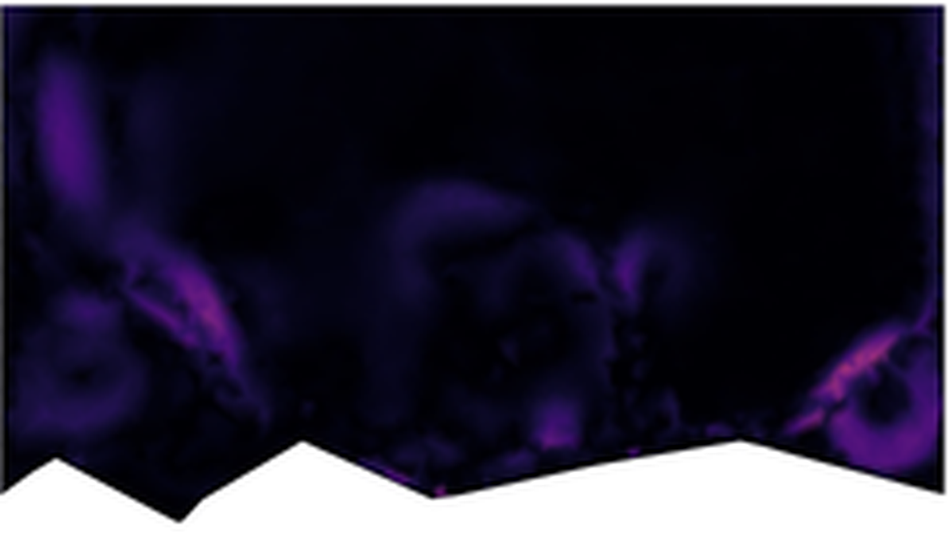} &
			\includegraphics[width=0.23\linewidth]{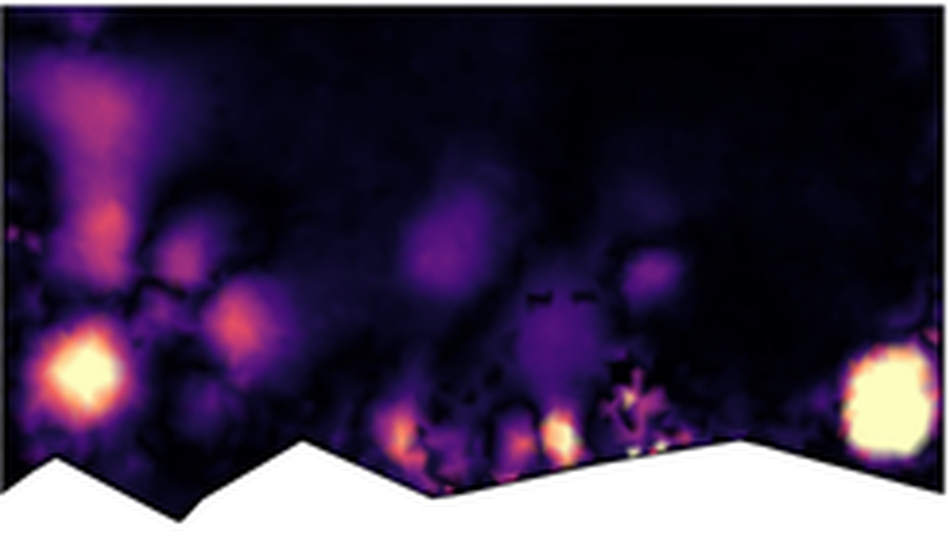} &
			\includegraphics[width=0.23\linewidth]{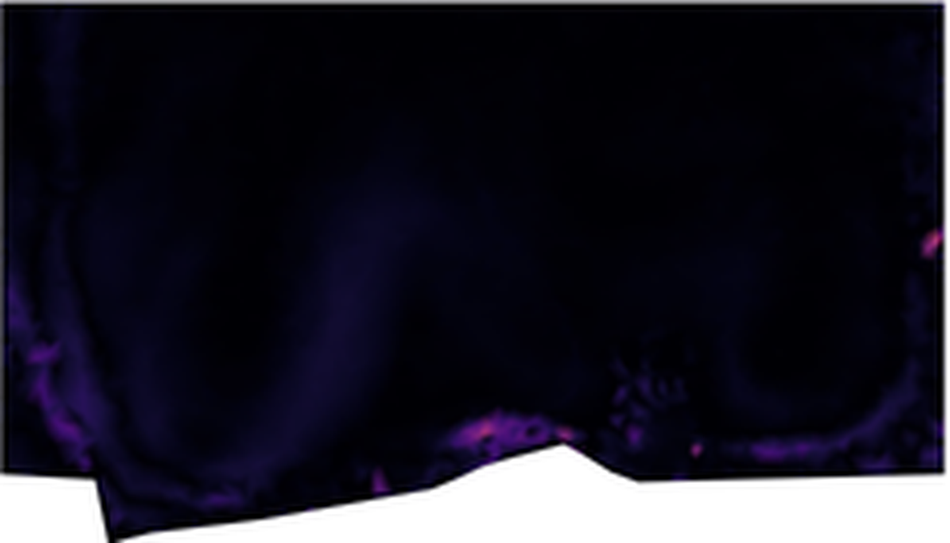} &
			\includegraphics[width=0.23\linewidth]{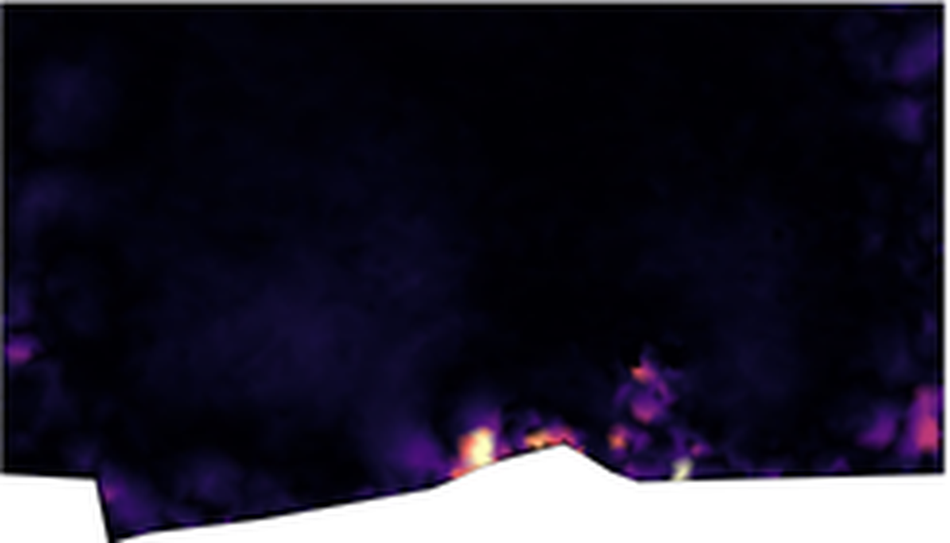}
		\end{tabular}
		\vspace{0.2em}
		\includegraphics[width=0.92\linewidth]{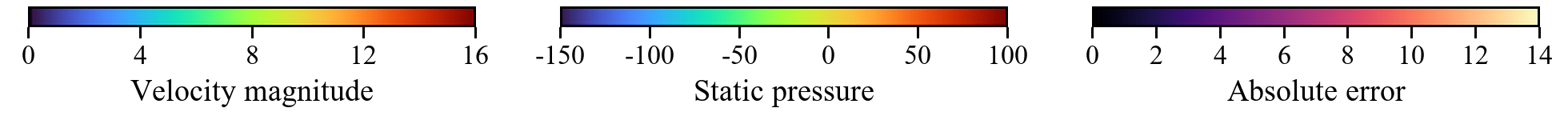}
		\caption{Late-rollout EAGLE visual comparison between Chebyshev BSP and Pushforward training for samples \(70\) and \(14\), included as a field-level comparison with existing rollout-exposure training. Columns are sample--field cases; rows show (a) ground truth, (b) Chebyshev BSP prediction, (c) Chebyshev BSP absolute error, (d) Pushforward prediction, and (e) Pushforward absolute error. The horizontal colorbars report the velocity-magnitude, static-pressure, and absolute-error scales.}
		\label{fig:eagle_high_t_chebyshev_pushforward}
	\end{figure}
	
	The comparison separates three effects. PyNUFFT substantially lowers the exact graph-spectral diagnostic relative to the Base model, and at the final horizon it has the smallest spectral RMSE in Table~\ref{tab:bfs_similar_methods}. However, its pointwise rollout degrades at long horizons: final RMSE is worse than the Base model, indicating that this off-grid Euclidean Fourier representation does not by itself control the autoregressive BFS forecast on the mesh. The standalone Pushforward baseline provides a rollout-exposure reference without graph-spectral alignment and gives a strong RMSE improvement over Base and PyNUFFT, reaching a final RMSE close to GLEAM. The PDF diagnostic in Fig.~\ref{fig:bfs_similar_methods_velocity_pdf} shows why this comparison should not be read from RMSE alone: Pushforward improves temporal robustness, but its velocity-magnitude distribution remains visibly displaced from the ground truth at several late horizons, whereas Chebyshev BSP and GLEAM more closely track the target distribution. GLEAM gives the lowest mean and final RMSE, the lowest mean exact graph-spectral RMSE, and the lowest mean spectral relative \(L_1\), while the late-horizon error maps in Fig.~\ref{fig:bfs_similar_methods_error_maps} show how the remaining spatial error is distributed relative to PyNUFFT, Pushforward, and Chebyshev BSP. The EAGLE fields in Fig.~\ref{fig:eagle_high_t_chebyshev_pushforward} add the corresponding full-field view: rollout exposure reduces some late-time error, while Chebyshev BSP more directly targets scale-resolved field structure. Thus, in this comparison, rollout exposure improves temporal robustness, while graph-intrinsic spectral constraints give a more balanced pointwise, distributional, and scale-resolved behavior across the reported diagnostics.
	
	\section{Additional EAGLE Qualitative Rollout Snapshots}
	\label{app:eagle-qualitative-snapshots}
	Figure~\ref{fig:eagle_appendix_qualitative} provides additional EAGLE sample-50 snapshots using the same visual convention as Fig.~\ref{fig:eagle_qualitative}. We include this sequence as a representative difficult case because the rollout enters the high-turbulence regime noted in the EAGLE benchmark: the initially organized shear layer develops into a more chaotic vortical field, and both autoregressive predictors accumulate visible error. The purpose of the figure is therefore not to claim that Chebyshev BSP removes the failure mode, but to show how the failure changes when graph-spectral supervision is added. With the shared absolute-error color scale, the Base error becomes broader and more diffuse as the rollout advances, especially around the evolving vortical region and lower boundary. Chebyshev BSP does not eliminate these late-time errors, but it keeps more of the coherent velocity structure visible and confines a larger portion of the error to localized shear-layer and boundary regions. The final snapshot remains challenging for both models, so the qualitative takeaway is partial mitigation of a difficult autoregressive case rather than complete recovery.
	
	\begin{figure}[!htbp]
		\centering
		\setlength{\tabcolsep}{1pt}
		\renewcommand{\arraystretch}{0.88}
		\scriptsize
		\begin{tabular}{@{}>{\raggedleft\arraybackslash}p{0.035\linewidth}ccccc@{}}
			& \textbf{\(t=5\)} & \textbf{\(t=50\)} & \textbf{\(t=100\)} & \textbf{\(t=180\)} & \textbf{\(t=260\)}\\[-0.1em]
			\textbf{(a)} &
			\includegraphics[width=0.185\linewidth]{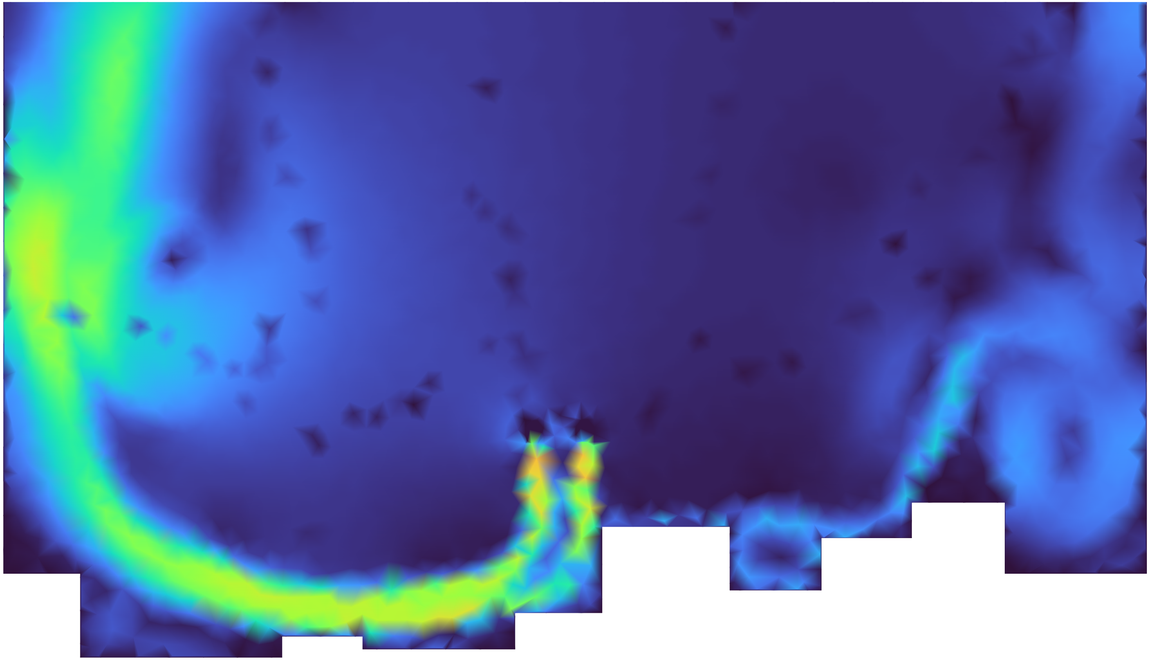} &
			\includegraphics[width=0.185\linewidth]{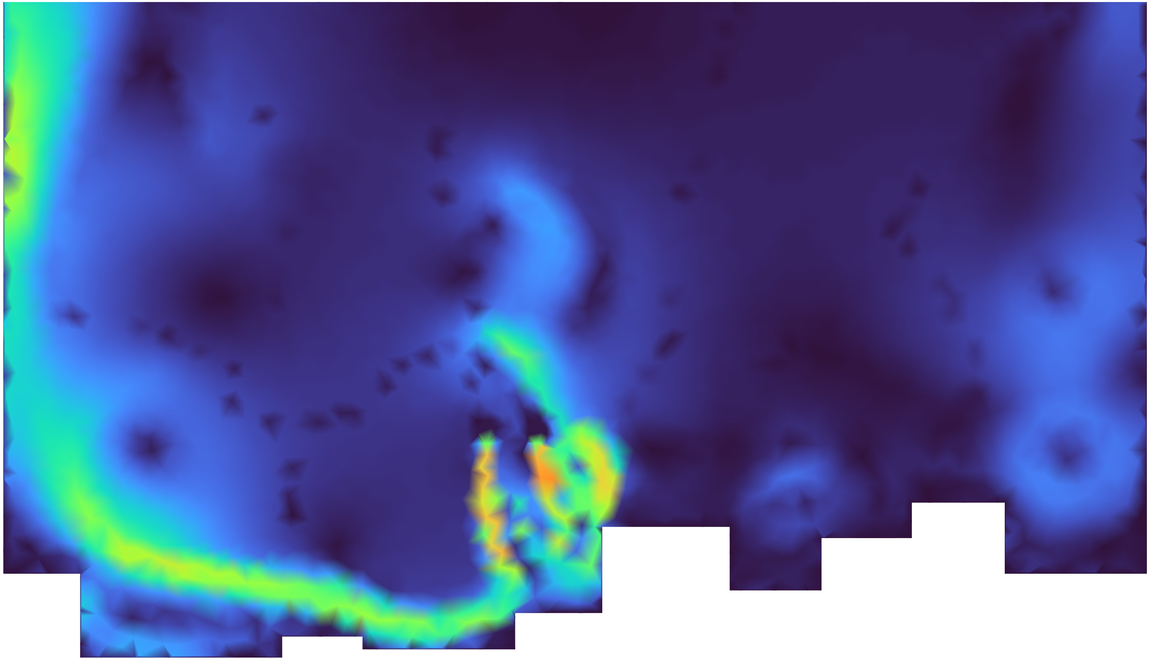} &
			\includegraphics[width=0.185\linewidth]{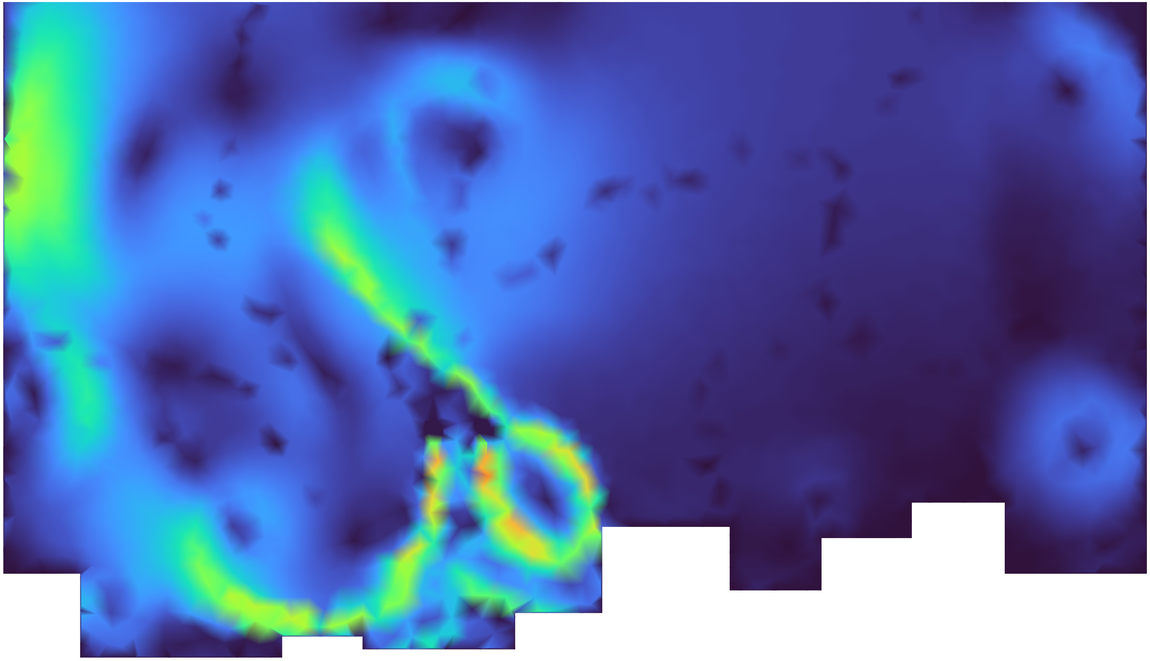} &
			\includegraphics[width=0.185\linewidth]{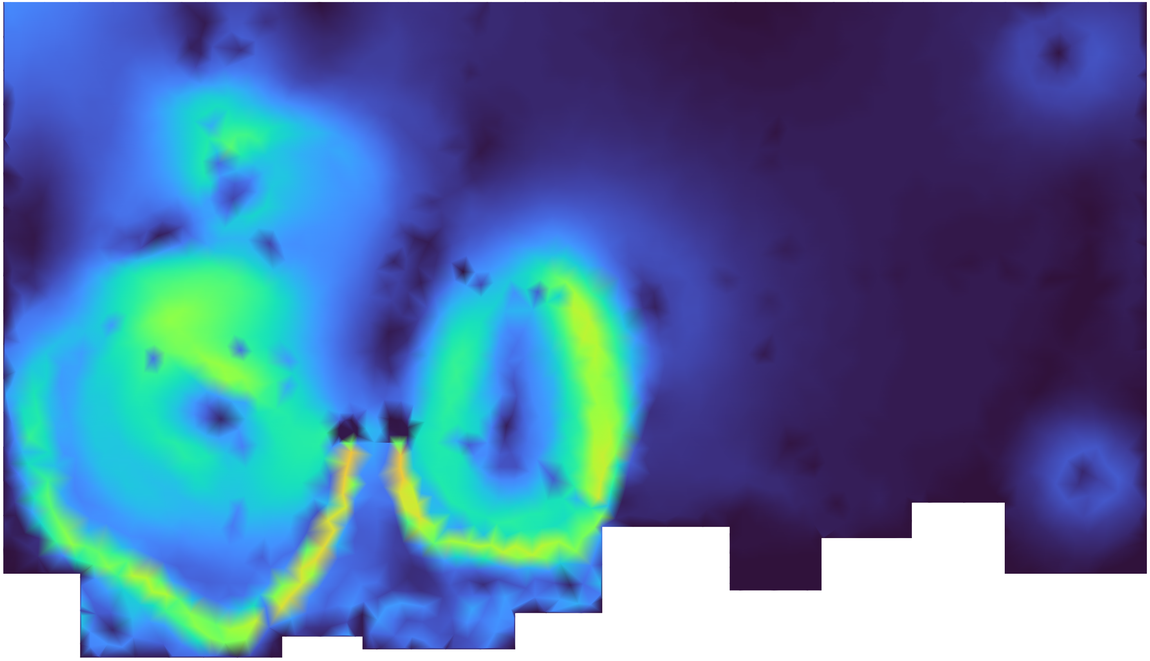} &
			\includegraphics[width=0.185\linewidth]{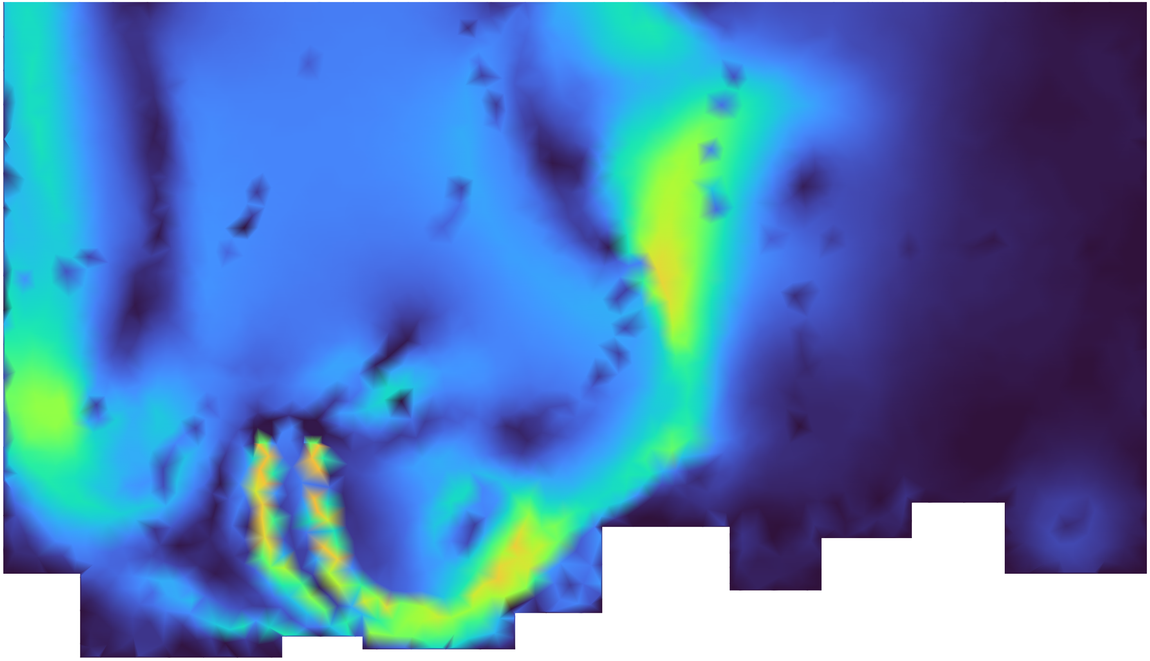}\\[-0.2em]
			\textbf{(b)} &
			\includegraphics[width=0.185\linewidth]{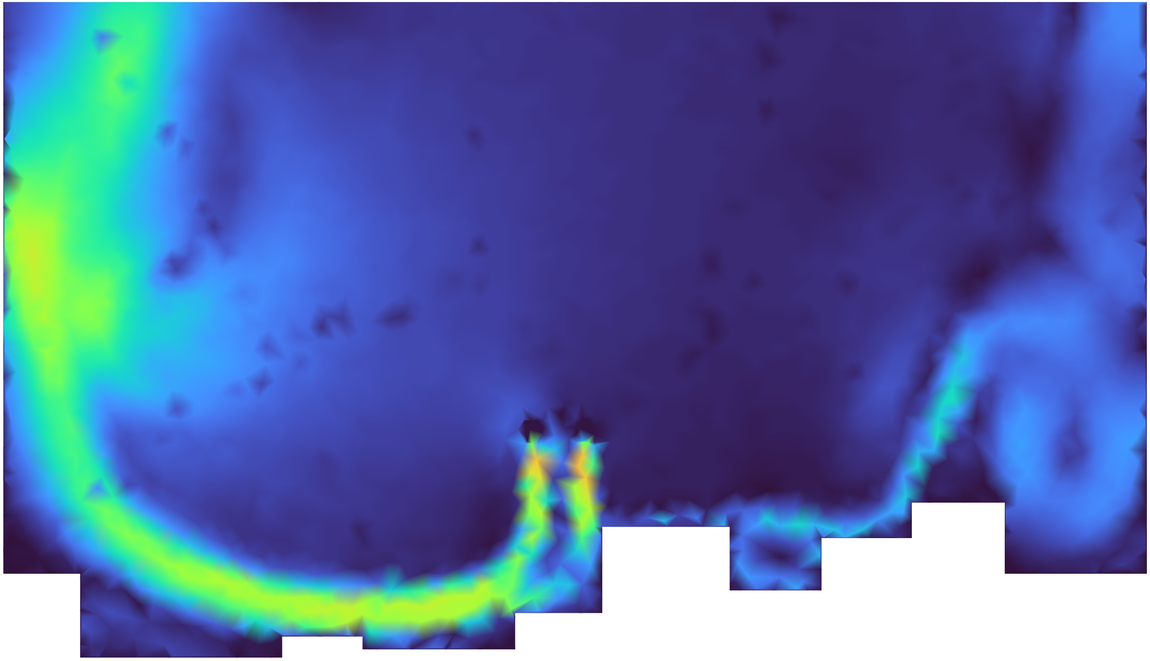} &
			\includegraphics[width=0.185\linewidth]{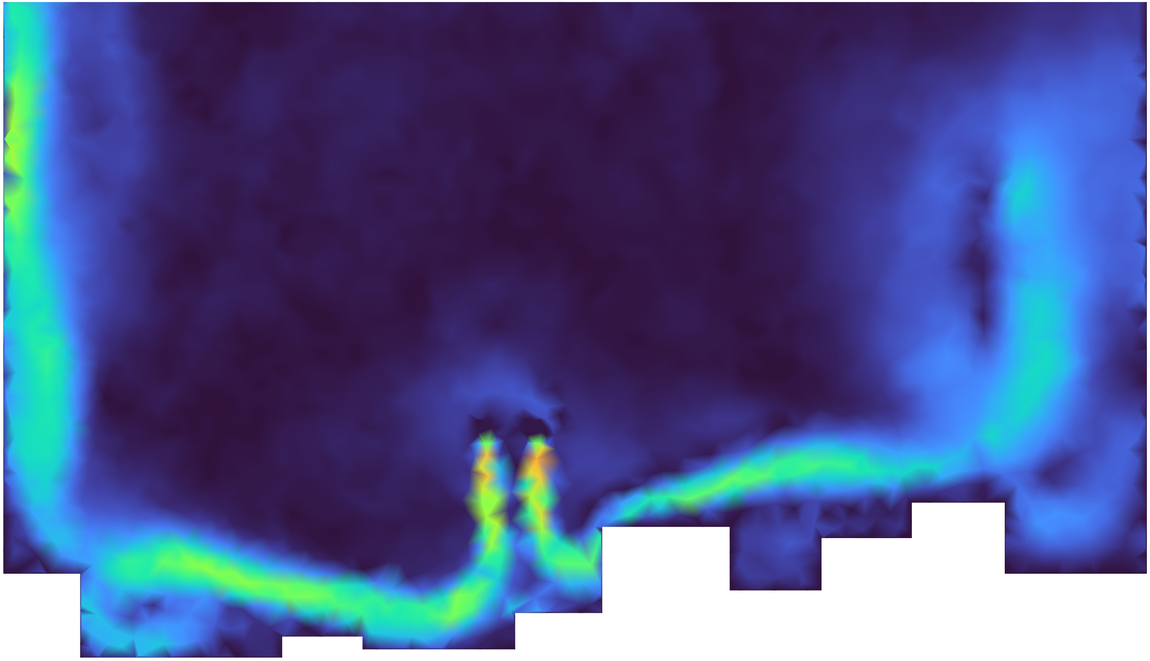} &
			\includegraphics[width=0.185\linewidth]{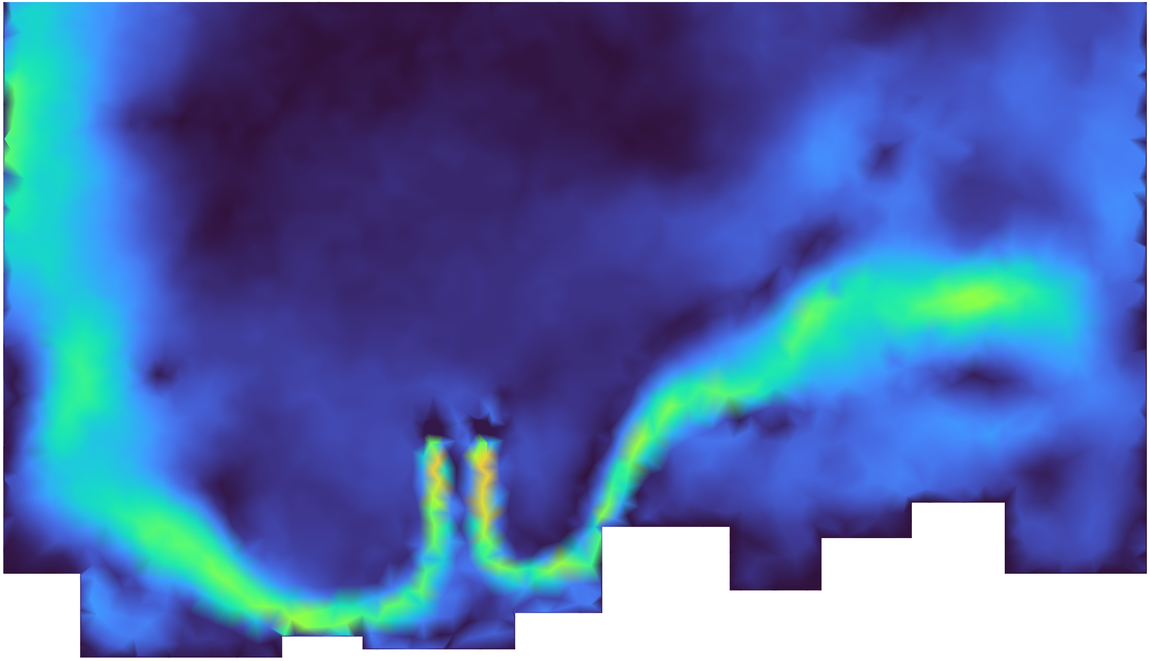} &
			\includegraphics[width=0.185\linewidth]{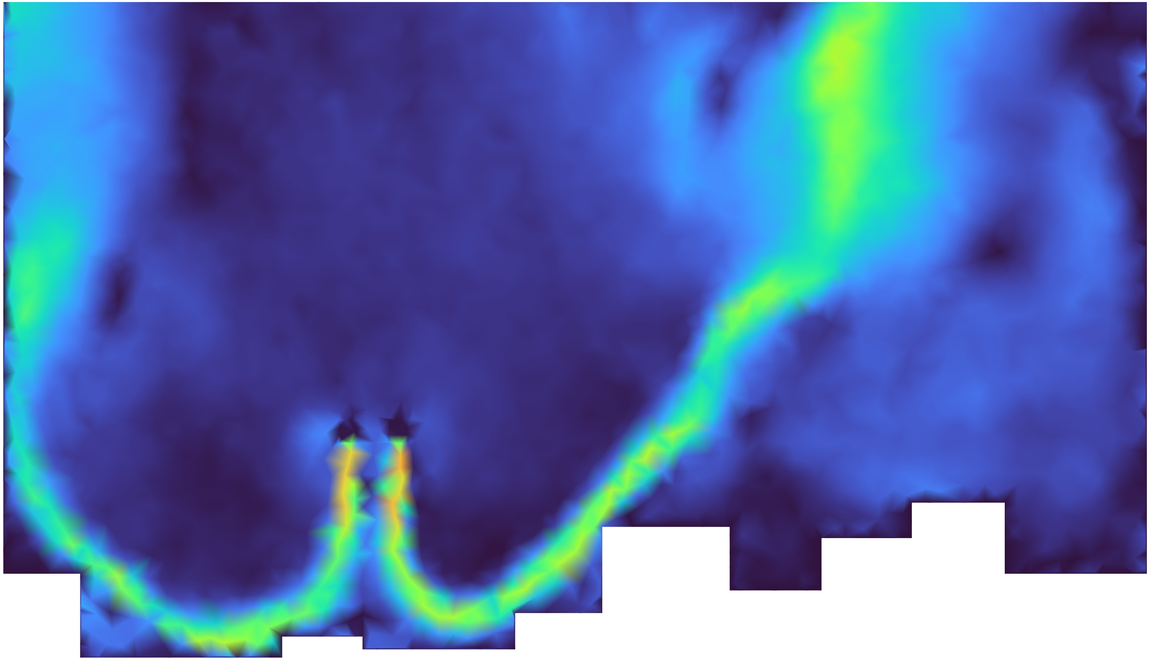} &
			\includegraphics[width=0.185\linewidth]{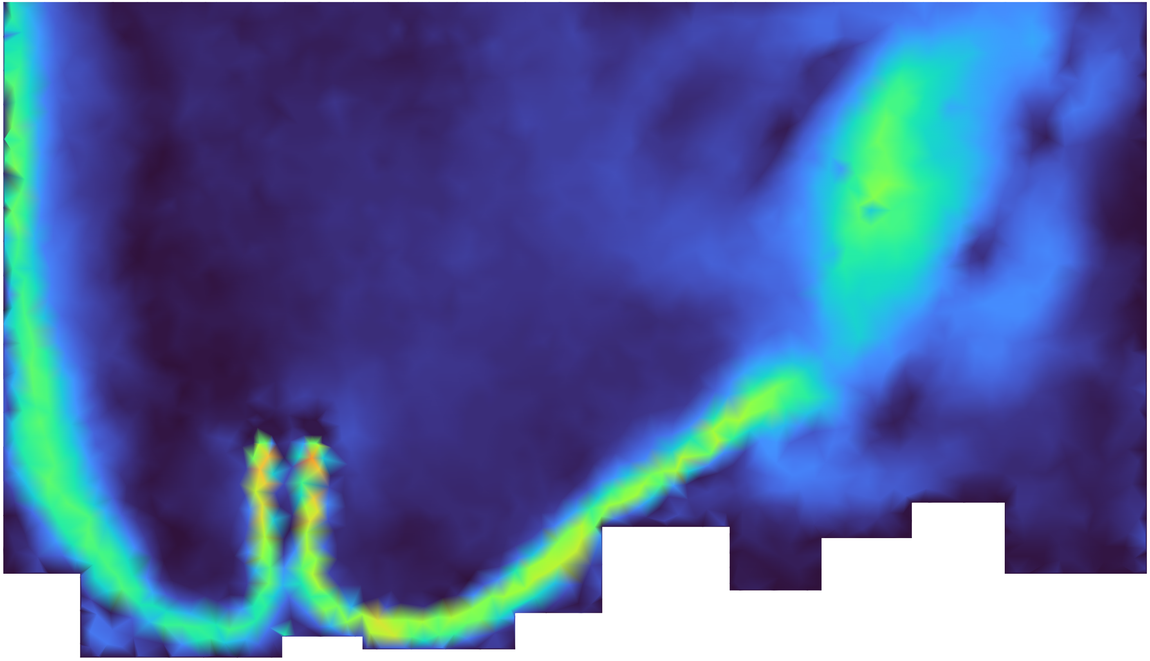}\\[-0.2em]
			\textbf{(c)} &
			\includegraphics[width=0.185\linewidth]{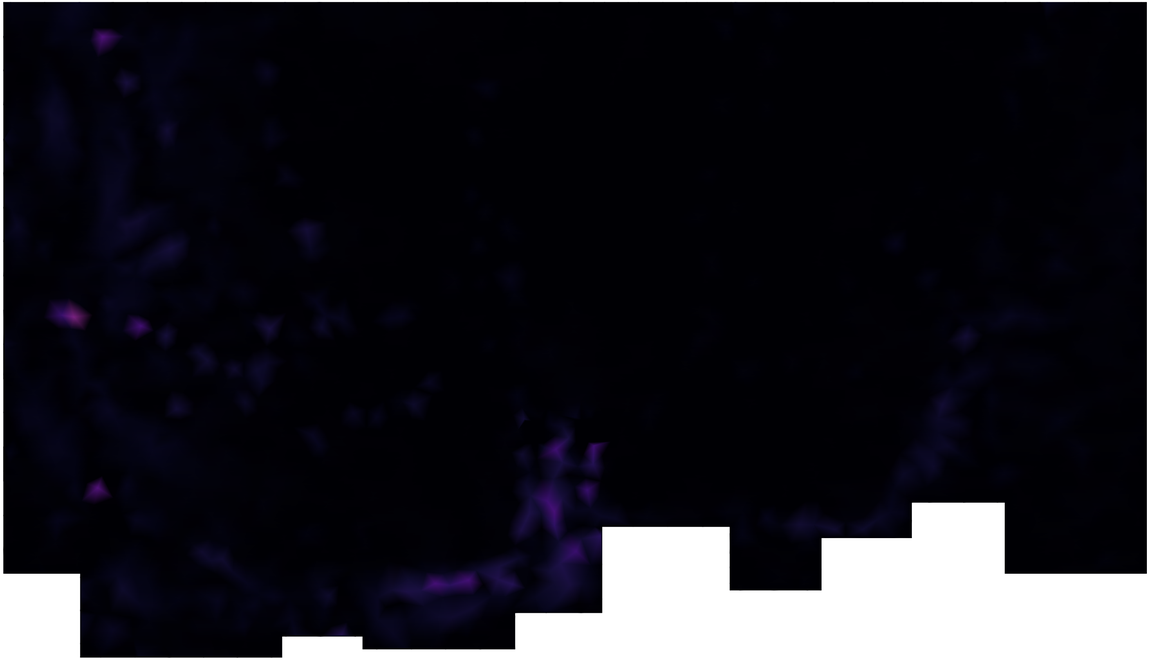} &
			\includegraphics[width=0.185\linewidth]{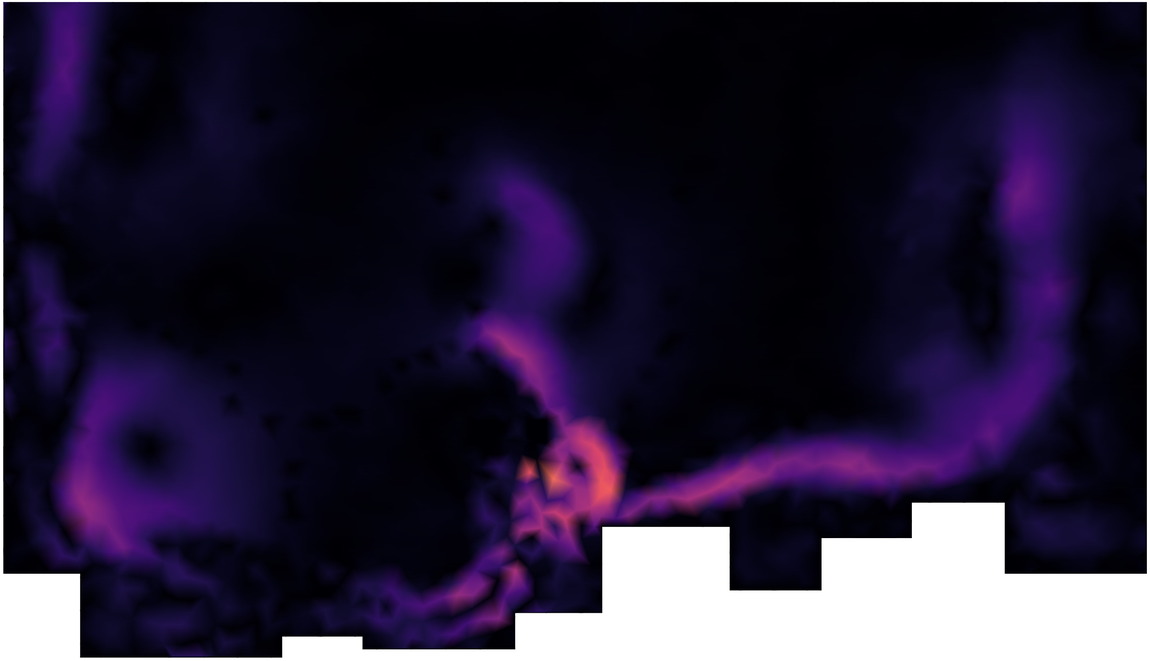} &
			\includegraphics[width=0.185\linewidth]{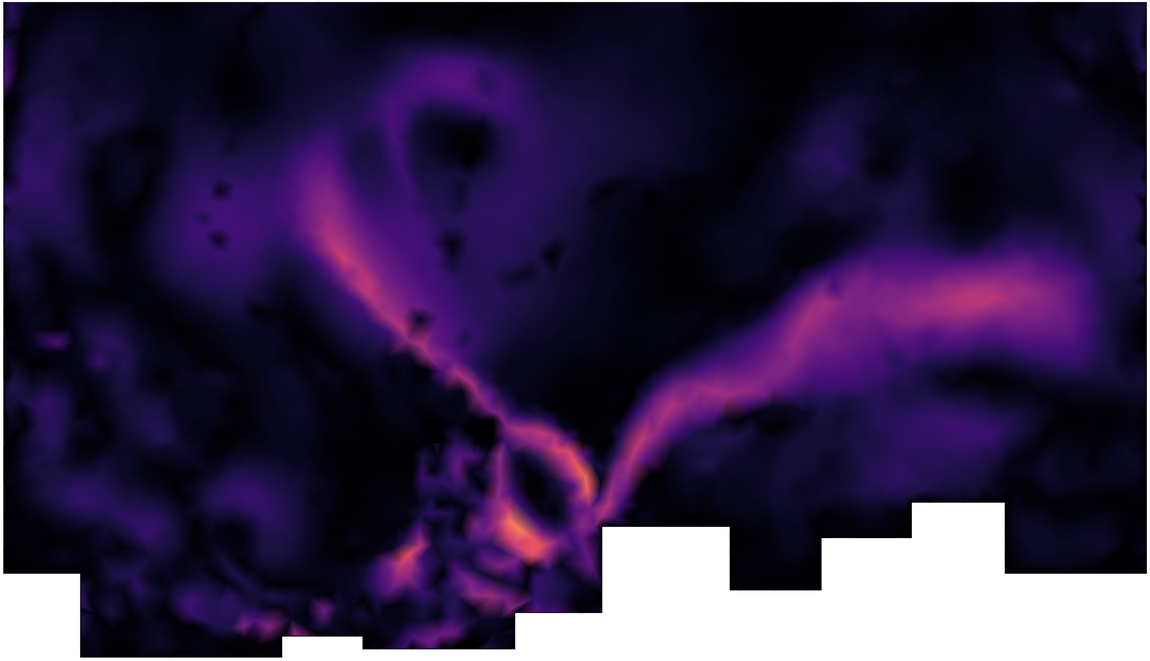} &
			\includegraphics[width=0.185\linewidth]{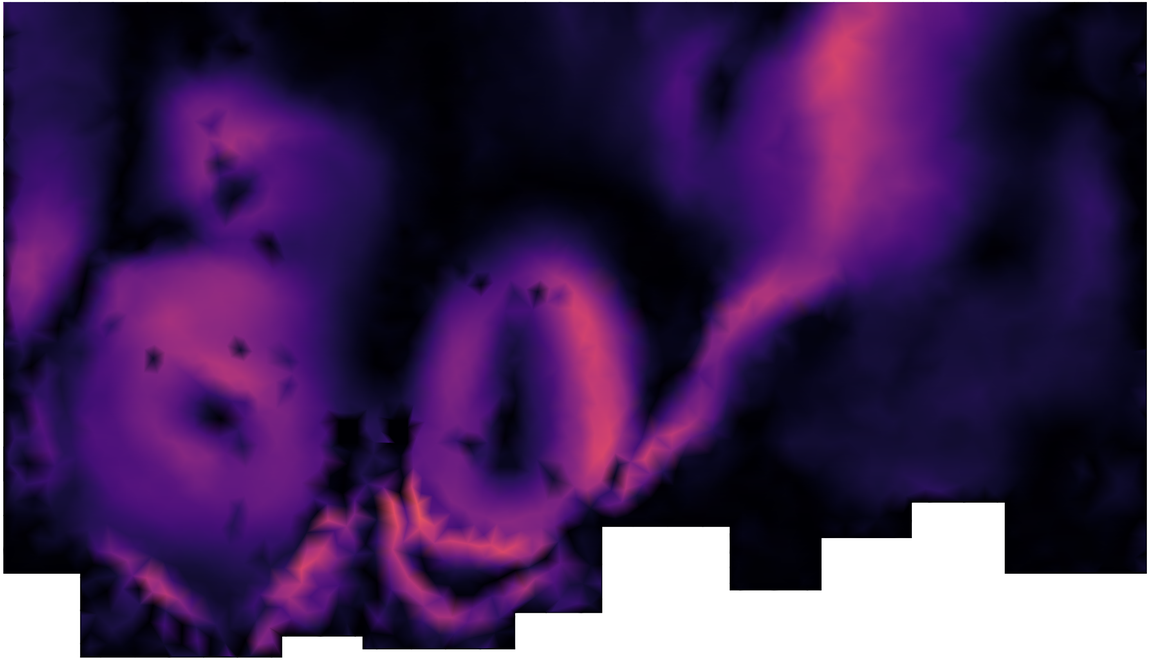} &
			\includegraphics[width=0.185\linewidth]{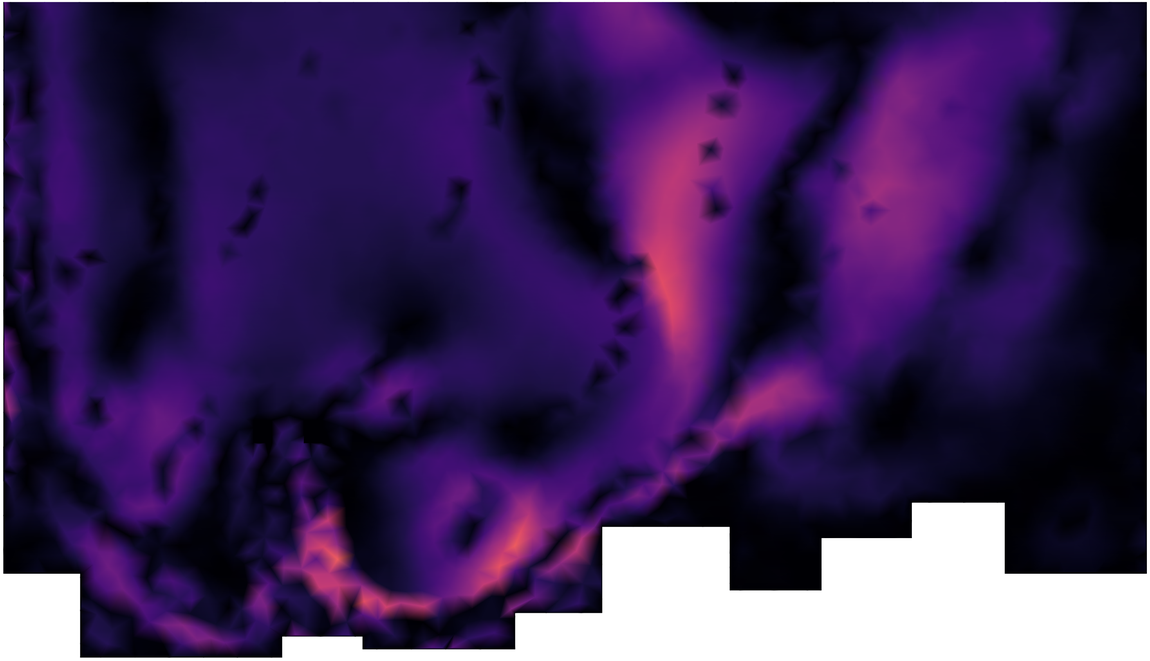}\\[-0.2em]
			\textbf{(d)} &
			\includegraphics[width=0.185\linewidth]{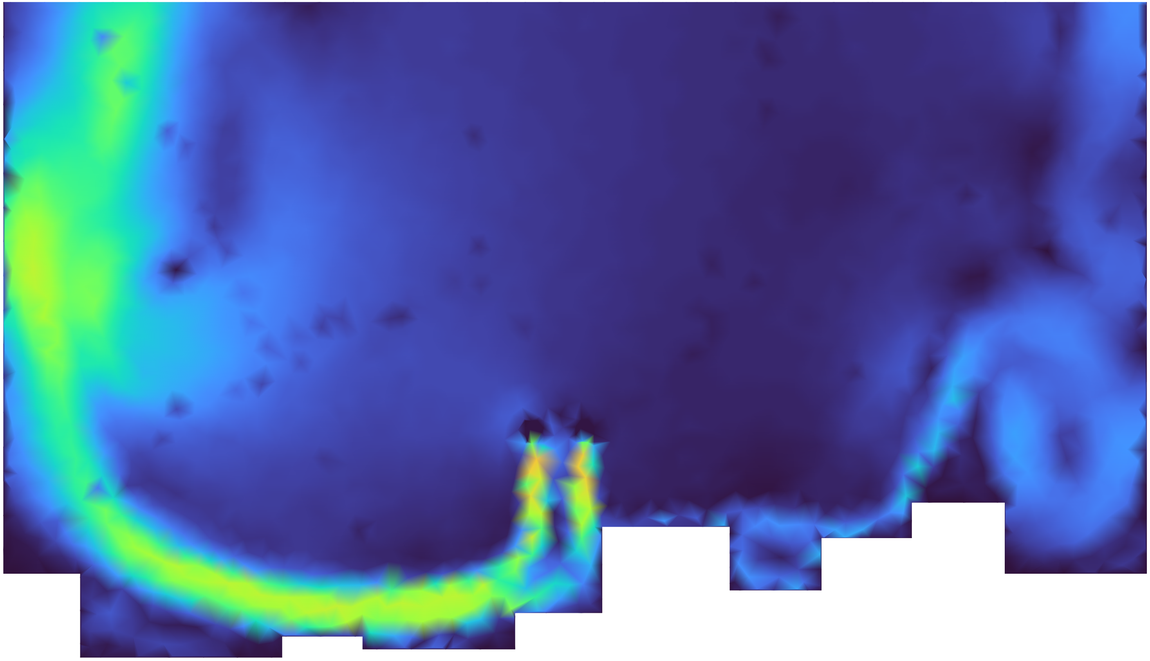} &
			\includegraphics[width=0.185\linewidth]{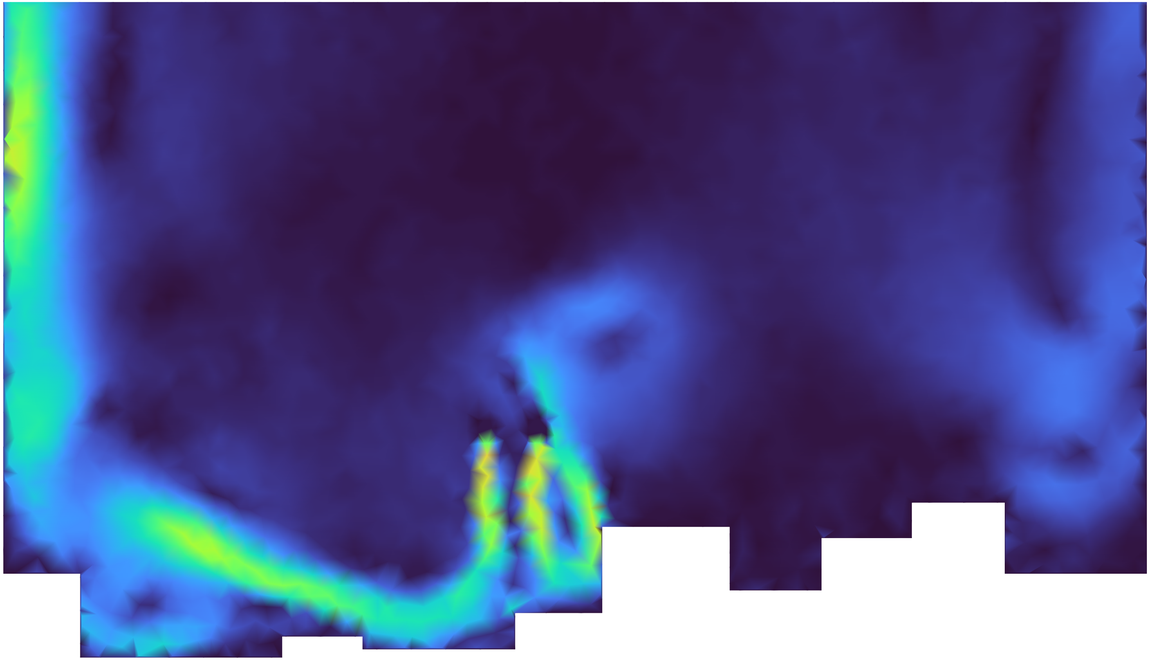} &
			\includegraphics[width=0.185\linewidth]{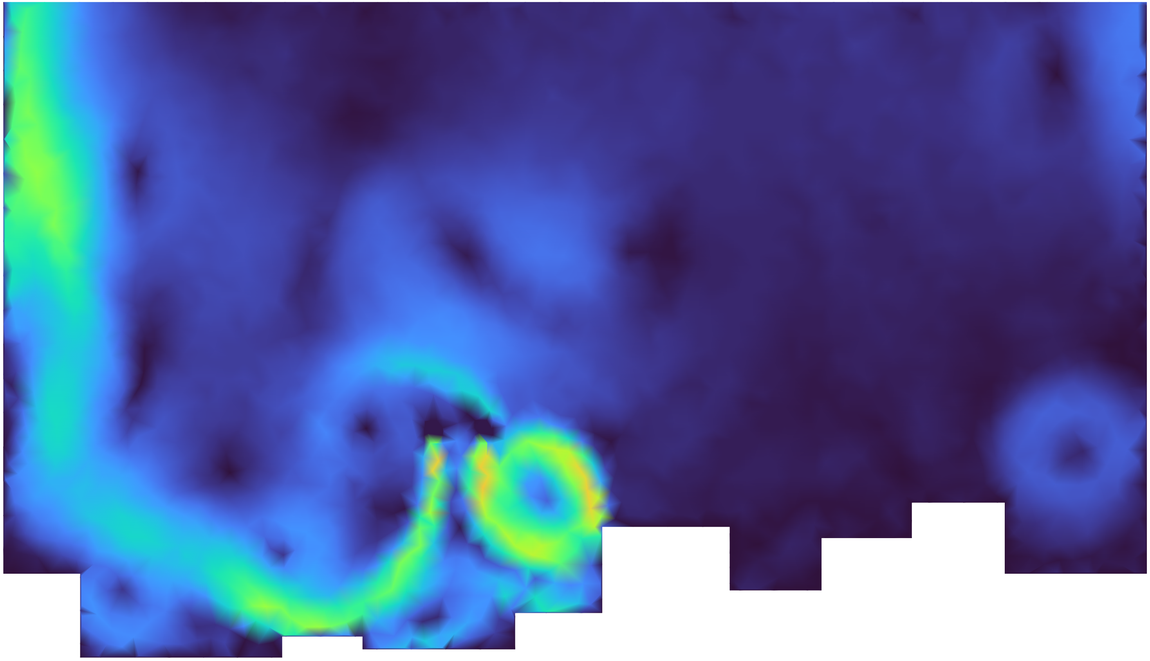} &
			\includegraphics[width=0.185\linewidth]{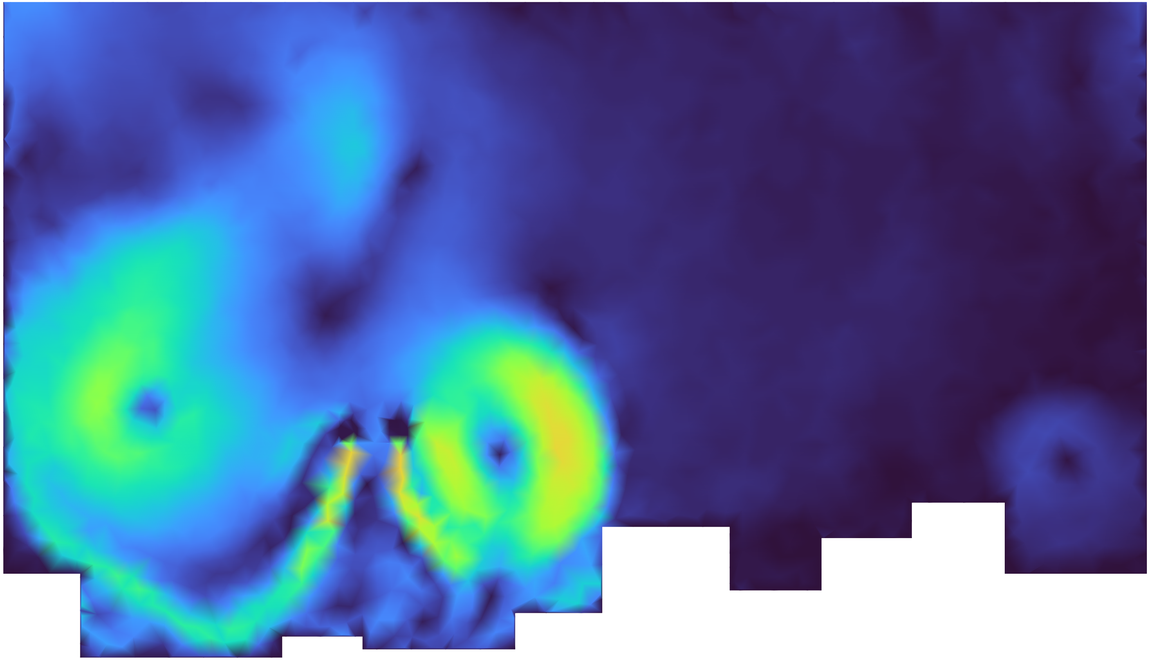} &
			\includegraphics[width=0.185\linewidth]{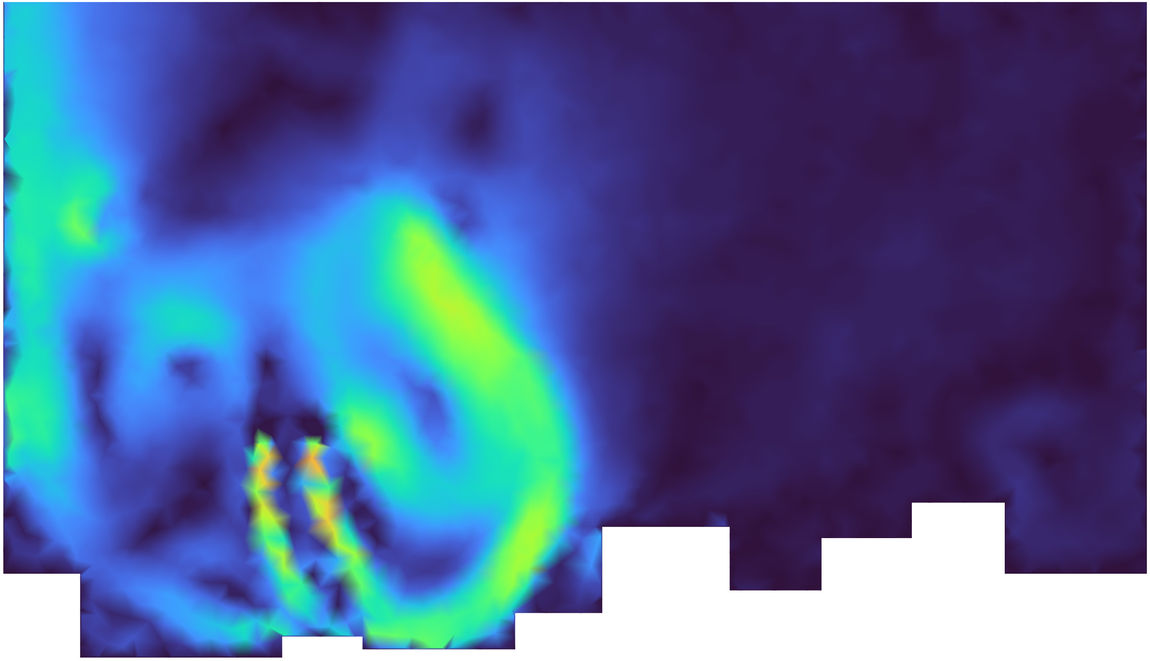}\\[-0.2em]
			\textbf{(e)} &
			\includegraphics[width=0.185\linewidth]{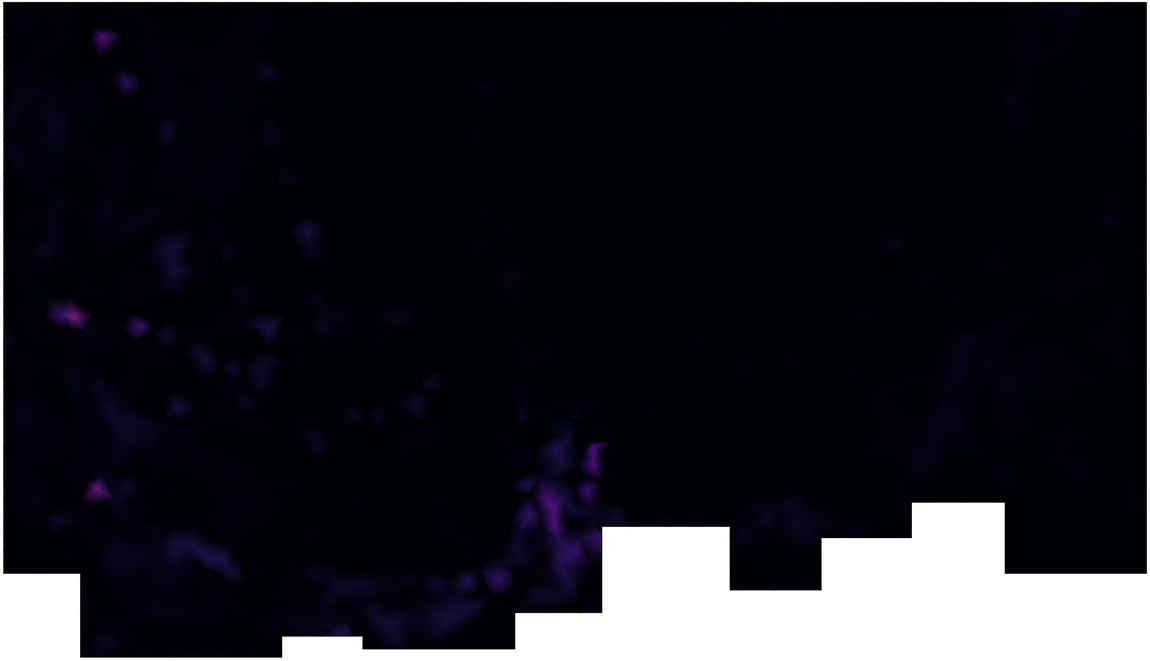} &
			\includegraphics[width=0.185\linewidth]{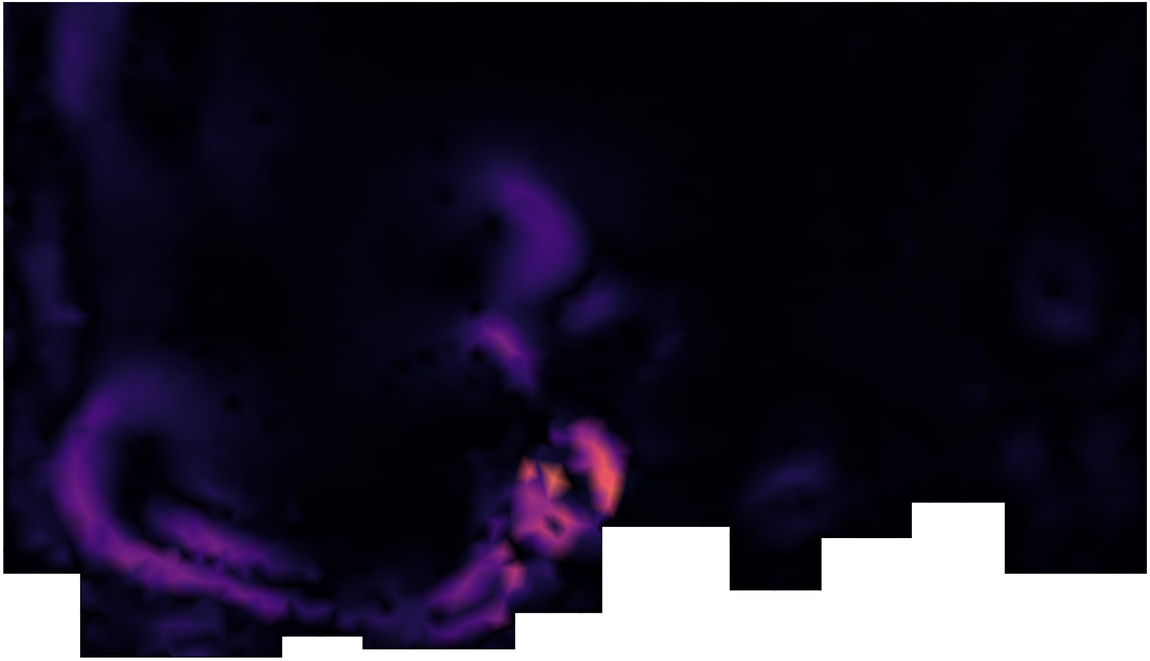} &
			\includegraphics[width=0.185\linewidth]{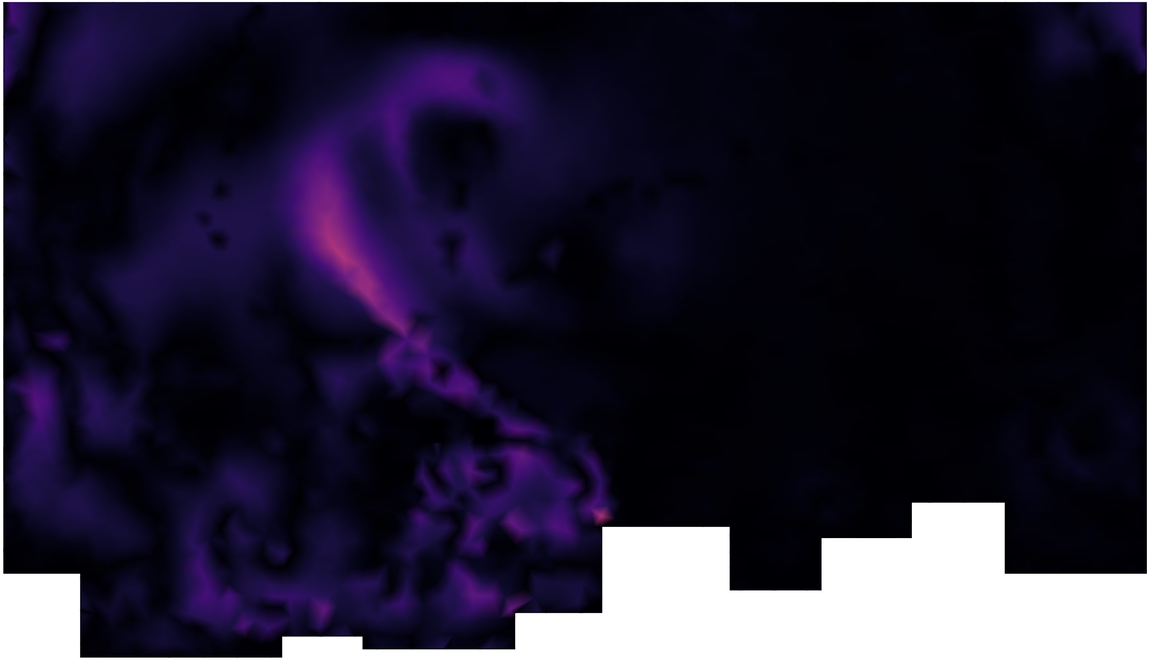} &
			\includegraphics[width=0.185\linewidth]{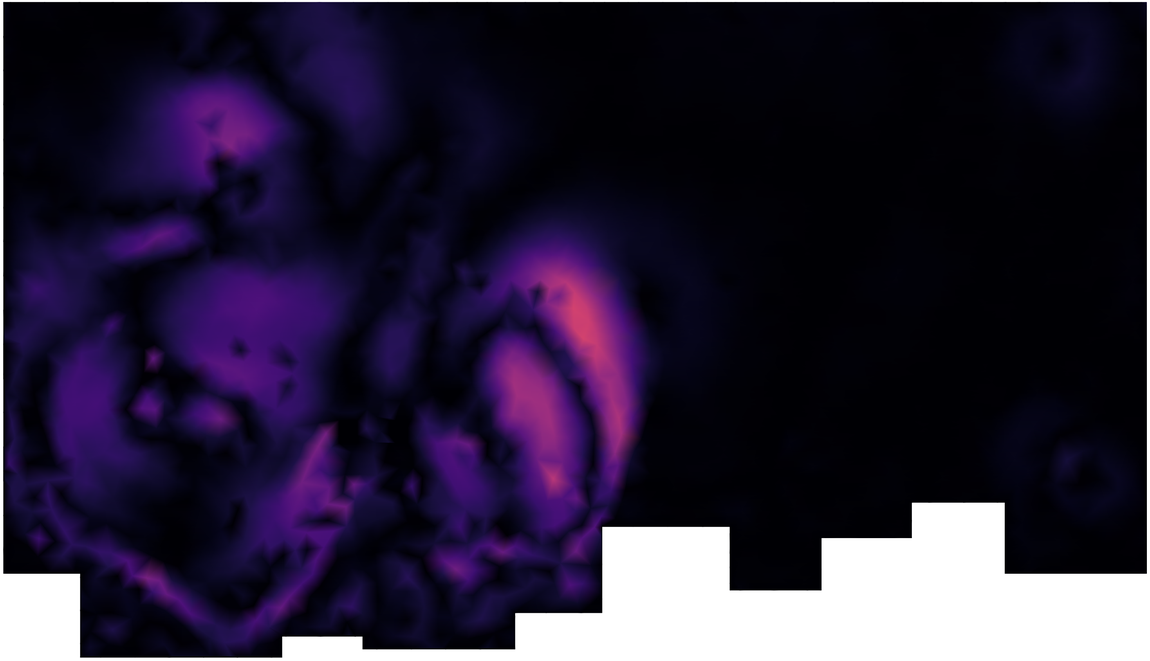} &
			\includegraphics[width=0.185\linewidth]{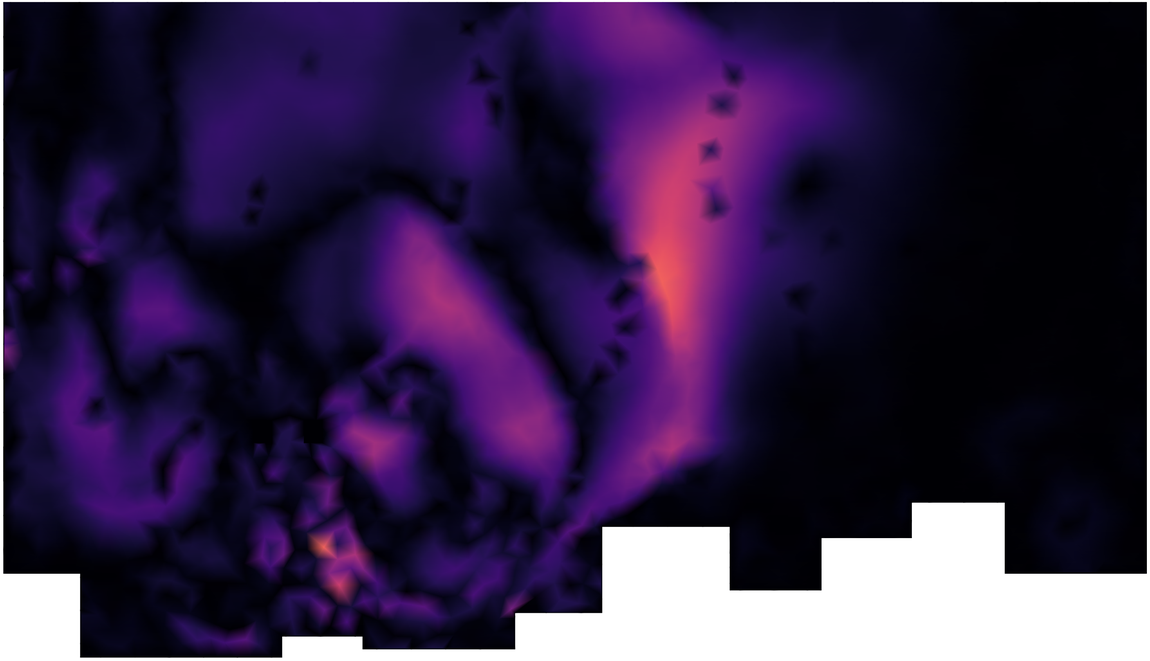}
		\end{tabular}
		\vspace{0.2em}
		
		\includegraphics[width=0.68\linewidth]{figures/eagle_qualitative_horizontal_colorbars.jpg}
		\caption{Additional EAGLE velocity-magnitude snapshots for sample 50 at \(t=5,50,100,180,\) and \(260\). Columns are rollout times; rows show (a) ground truth, (b) Base prediction, (c) Base absolute error, (d) Chebyshev prediction, and (e) Chebyshev absolute error. The horizontal colorbars report field value and absolute error.}
		\label{fig:eagle_appendix_qualitative}
	\end{figure}
	
	\section{Multilevel Graph Band-Power Loss}
	\label{app:multilevel-bsp}
	Many graph neural operators employ a hierarchy of coarsened graphs. If one enforces spectral consistency only on the finest graph, the predicted field may still exhibit inconsistent scale allocation across coarse levels used internally by the model. Let $\ell\in\{1,\dots,S\}$ denote graph hierarchy levels, with $\ell=1$ representing the finest graph. Let
	\begin{equation}
		\bm{u}^{(\ell)},\bm{v}^{(\ell)}
	\end{equation}
	be the predicted and target fields mapped to level $\ell$ via pooling or restriction operators. The levelwise spectral penalty is
	\begin{equation}
		\mathcal{L}_{\mathrm{spec}}^{(\ell)}=\mathcal{L}_{\mathrm{spec}}\paren{\bm{u}^{(\ell)},\bm{v}^{(\ell)}}.
	\end{equation}
	The multi-level loss is then
	\begin{equation}
		\mathcal{L}_{\mathrm{multi}}
		=\sum_{\ell=1}^{S}\alpha_\ell\mathcal{L}_{\mathrm{spec}}^{(\ell)},
		\qquad
		\alpha_\ell\ge 0,
		\qquad
		\sum_{\ell=1}^{S}\alpha_\ell=1.
		\label{eq:multi}
	\end{equation}
	This penalizes scale inconsistency across the operator hierarchy.
	
	\subsection{Why coarse-level supervision is useful}
	A mathematical way to view \eqref{eq:multi} is as a multiscale error seminorm. Let
	\[
	\bm{e}=\widehat{\bm{u}}-\bm{u}
	\]
	denote the fine-graph prediction error. Let \(P^{(\ell)}\) be the restriction or pooling map from the finest graph to level \(\ell\), and let \(S^{(\ell)}\) denote the linearized graph-spectral supervision operator on that level. For example, \(S^{(\ell)}\) may represent a low-rank Laplacian projection, a Chebyshev filter bank, or a local linearization of a band-energy penalty around the current prediction. A final-level spectral loss controls
	\begin{equation}
		\mathcal{L}_{\mathrm{fine}}(\bm{e})
		=
		\norm{S^{(1)}\bm{e}}_F^2,
	\end{equation}
	whereas multilevel supervision controls
	\begin{equation}
		\mathcal{L}_{\mathrm{hier}}(\bm{e})
		=
		\sum_{\ell\in\mathcal{H}}\alpha_\ell
		\norm{S^{(\ell)}P^{(\ell)}\bm{e}}_F^2 .
		\label{eq:hier_error_metric}
	\end{equation}
	Equivalently, the quadratic form induced by \eqref{eq:hier_error_metric} is
	\begin{equation}
		\mathcal{L}_{\mathrm{hier}}(\bm{e})
		=
		\left\langle
		\bm{e},
		Q_{\mathrm{hier}}\bm{e}
		\right\rangle,
		\qquad
		Q_{\mathrm{hier}}
		=
		\sum_{\ell\in\mathcal{H}}
		\alpha_\ell
		\paren{P^{(\ell)}}^\top
		\paren{S^{(\ell)}}^\top
		S^{(\ell)}
		P^{(\ell)} .
		\label{eq:hier_psd_metric}
	\end{equation}
	Since each summand in \(Q_{\mathrm{hier}}\) is positive semidefinite for \(\alpha_\ell\ge 0\), the hierarchical loss adds a nonnegative penalty on any error direction that remains visible after pooling and spectral filtering. In particular, an error component that is small locally but changes the coarse recirculation structure, pressure loading, or large-scale transport pattern satisfies \(P^{(\ell)}\bm{e}\ne 0\) for some coarse level and therefore receives an additional penalty. This is consistent with the classical multigrid view that smooth, low-frequency error components are represented and corrected on coarser grids \citep{brandt1977multigrid,briggs2000multigrid}. By contrast, a finest-level loss alone may mix this coarse error with local high-frequency discrepancies and does not explicitly separate the scale at which the error is represented.
	
	The same observation gives a useful optimization interpretation. Suppose a multilevel backbone maps fine features to a coarse representation \(\bm{h}^{(\ell)}\), processes that representation, and prolongates information back to the fine graph. If supervision is applied only to the final output, then gradients for parameters \(\theta_\ell\) acting at level \(\ell\) must pass through all subsequent prolongation and decoding operations:
	\begin{equation}
		\nabla_{\theta_\ell}\mathcal{L}_{\mathrm{fine}}
		=
		\left(
		\frac{\partial \widehat{\bm{u}}}{\partial \bm{h}^{(\ell)}}
		\frac{\partial \bm{h}^{(\ell)}}{\partial \theta_\ell}
		\right)^\top
		\nabla_{\widehat{\bm{u}}}\mathcal{L}_{\mathrm{fine}} .
		\label{eq:fine_grad_path}
	\end{equation}
	These gradients can be indirect or poorly conditioned for coarse graph modes. Adding the levelwise term in \eqref{eq:hier_error_metric} introduces an additional gradient at the resolution where the coarse error is expressed:
	\begin{equation}
		\nabla_{\theta_\ell}\mathcal{L}_{\mathrm{hier}}
		=
		\nabla_{\theta_\ell}\mathcal{L}_{\mathrm{fine}}
		+
		\alpha_\ell
		\left(
		\frac{\partial \widehat{\bm{u}}^{(\ell)}}{\partial \theta_\ell}
		\right)^\top
		\nabla_{\widehat{\bm{u}}^{(\ell)}}
		\norm{S^{(\ell)}
			\paren{\widehat{\bm{u}}^{(\ell)}-\bm{u}^{(\ell)}}}_F^2 .
		\label{eq:hier_direct_grad}
	\end{equation}
	Thus hierarchical supervision can improve the conditioning of the training signal for large-scale modes by supplying gradients at the graph levels where those modes are represented. This argument is also aligned with the deep-supervision literature, where auxiliary losses at intermediate representations are used to shorten gradient paths and improve gradient propagation \citep{lee2015deeply,elinas2022deep}. This does not constitute a general convergence guarantee, but it explains why coarse-level penalties are well matched to multilevel graph backbones: they regularize the intermediate resolutions through which long-range information is transported, rather than relying only on a final-level correction.
	
	\subsection{Lifted coarse-projector view of GLEAM}
	\label{app:gleam-resistance}
	GLEAM can remain competitive with Chebyshev BSP because its coarse low-rank bases are structured spectral summaries rather than arbitrary compressions. Following the spectral-embedding viewpoint of \citet{wang2020graspel}, the leading nontrivial Laplacian modes define a regularized embedding
	\begin{equation}
		U_r
		=
		\left[
		\frac{\psi_2}{\sqrt{\lambda_2+\tau}},
		\dots,
		\frac{\psi_{r+1}}{\sqrt{\lambda_{r+1}+\tau}}
		\right],
		\qquad \tau>0,
		\label{eq:gleam_appendix_embedding}
	\end{equation}
	for which pairwise graph-embedding distances satisfy
	\begin{equation}
		\norm{U_r^\top(e_i-e_j)}_2^2
		=
		\sum_{k=2}^{r+1}
		\frac{\paren{\psi_k(i)-\psi_k(j)}^2}{\lambda_k+\tau}.
		\label{eq:gleam_appendix_resistance_trunc}
	\end{equation}
	As \(r\) increases and \(\tau\to 0\), \eqref{eq:gleam_appendix_resistance_trunc} approaches the effective-resistance distance
	\begin{equation}
		(e_i-e_j)^\top L^+(e_i-e_j)
		=
		\sum_{k=2}^{N}
		\frac{\paren{\psi_k(i)-\psi_k(j)}^2}{\lambda_k}.
		\label{eq:gleam_appendix_resistance}
	\end{equation}
	Thus the retained low-rank basis encodes dominant large-scale graph geometry and is sensitive to coherent, long-range structures rather than merely being a cheaper coordinate system.
	
	This observation connects GLEAM to Chebyshev BSP through the quadratic forms used to measure band energy. Let \(L_f\) denote the finest graph Laplacian and let \(\Pi_m\) be the exact fine-graph projector onto band \(m\). Chebyshev BSP approximates this band energy by a polynomial filter \(H_m^{(K,Q)}\):
	\begin{equation}
		E_m^{\mathrm{Cheb}}(\bm u)
		=
		\frac{1}{2}\norm{H_m^{(K,Q)}\bm u}_2^2
		=
		\frac{1}{2}\bm u^\top
		\paren{H_m^{(K,Q)}}^\top H_m^{(K,Q)}
		\bm u
		\approx
		\frac{1}{2}\bm u^\top\Pi_m\bm u .
		\label{eq:gleam_cheb_quad_view}
	\end{equation}
	GLEAM instead evaluates spectral energy after mapping the field to graph level \(\ell\). Let \(P^{(\ell)}\) be the pooling map from the finest graph to level \(\ell\), and let \(B_m^{(\ell)}\) collect the retained coarse spectral vectors in band \(m\), optionally with the regularized eigenvalue weighting in \eqref{eq:gleam_appendix_embedding}. Then
	\begin{equation}
		E_m^{\mathrm{GLEAM},(\ell)}(\bm u)
		=
		\frac{1}{2}
		\norm{\paren{B_m^{(\ell)}}^\top P^{(\ell)}\bm u}_2^2
		=
		\frac{1}{2}\bm u^\top
		\widetilde{G}_m^{(\ell)}
		\bm u,
		\label{eq:gleam_quad_energy}
	\end{equation}
	where the lifted coarse spectral metric is
	\begin{equation}
		\widetilde{G}_m^{(\ell)}
		=
		\paren{P^{(\ell)}}^\top
		B_m^{(\ell)}
		\paren{B_m^{(\ell)}}^\top
		P^{(\ell)} .
		\label{eq:gleam_lifted_metric}
	\end{equation}
	Therefore, GLEAM does not ignore the fine field; it measures the fine-field energy against coarse spectral functions lifted through the same hierarchy used by the forecasting backbone.
	
	If the graph hierarchy preserves the relevant low-frequency subspaces, then \(\widetilde{G}_m^{(\ell)}\) approximates the fine-band energy operator on the dynamically important part of the solution space. In particular, if
	\begin{equation}
		\norm{\Pi_m-\widetilde{G}_m^{(\ell)}}_2\le \delta_m,
		\label{eq:gleam_projector_error_assumption}
	\end{equation}
	then for any field \(\bm u\),
	\begin{equation}
		\left|
		\frac{1}{2}\bm u^\top\Pi_m\bm u
		-
		E_m^{\mathrm{GLEAM},(\ell)}(\bm u)
		\right|
		\le
		\frac{\delta_m}{2}\norm{\bm u}_2^2 .
		\label{eq:gleam_energy_bound}
	\end{equation}
	For a prediction--target pair \((\widehat{\bm u},\bm u)\), the corresponding band-energy mismatch error is bounded by
	\begin{equation}
		\begin{aligned}
			&
			\left|
			\left[
			\frac{1}{2}\widehat{\bm u}^{\top}\Pi_m\widehat{\bm u}
			-\frac{1}{2}\bm u^{\top}\Pi_m\bm u
			\right]
			-
			\left[
			E_m^{\mathrm{GLEAM},(\ell)}(\widehat{\bm u})
			-E_m^{\mathrm{GLEAM},(\ell)}(\bm u)
			\right]
			\right|  \\
			&\hspace{3em}
			\le
			\frac{\delta_m}{2}
			\left(
			\norm{\widehat{\bm u}}_2^2+\norm{\bm u}_2^2
			\right).
		\end{aligned}
		\label{eq:gleam_mismatch_bound}
	\end{equation}
	Thus, when the hierarchy preserves the low- and mid-frequency eigenspaces that dominate rollout error, GLEAM provides an approximation to part of the band-energy information targeted by fine-level graph BSP.
	
	This explains why GLEAM can provide performance comparable to Chebyshev BSP in practice. Chebyshev BSP approximates fine-level graph-frequency projectors on the supervised graph. GLEAM sacrifices some fine-band resolution, but it supervises the coarse spectral variables through which multilevel graph backbones propagate long-range information. The Fiedler-guided pairwise contrast term further constrains how field contrasts are arranged across spectrally separated graph regions, helping preserve coherent structures such as recirculation zones, shear layers, and pressure-loading patterns. For flows whose autoregressive errors are dominated by coherent low- and mid-frequency drift, the approximation error introduced by GLEAM's retained-spectrum binning can therefore be comparable to the polynomial approximation error in Chebyshev BSP, while the supervision is better aligned with the multilevel architecture.
	
	\subsection{Why corrector quality matters in GLEAM}
	\label{app:gleam-corrector-quality}
	The lifted-projector view also explains why a coarse-to-fine corrector can help GLEAM. Let \(U_\star^{(\ell)}\) denote the exact retained embedding that would be obtained by solving the level-\(\ell\) Laplacian eigenproblem, and let \(\widetilde U^{(\ell)}\) denote the embedding obtained by prolongating from a coarser graph. The prolongated embedding can be written as
	\begin{equation}
		\widetilde U^{(\ell)}
		=
		U_\star^{(\ell)}+E_U^{(\ell)},
		\label{eq:gleam_corrector_embedding_error}
	\end{equation}
	where \(E_U^{(\ell)}\) is the spectral-transfer error introduced by interpolation, smoothing, and the mismatch between coarse and fine graph geometry. A corrector of the form used in \eqref{eq:general_corrector} produces
	\begin{equation}
		\widetilde U_+^{(\ell)}
		=
		\widetilde U^{(\ell)}+\Gamma R,
		\qquad
		R=\mathrm{NbrMean}(\widetilde U^{(\ell)})-\widetilde U^{(\ell)} .
		\label{eq:gleam_corrected_embedding}
	\end{equation}
	Here \(\Gamma\) denotes the rowwise correction operator induced by the coefficients \(\Gamma_i\) in \eqref{eq:general_corrector}. Dropping the level superscript for readability, the corrected embedding error is \(E_U^{+}=E_U+\Gamma R\). In the scalar case \(\Gamma=\gamma I\), the correction reduces the embedding error whenever
	\begin{equation}
		\norm{E_U+\gamma R}_F^2 < \norm{E_U}_F^2,
		\qquad
		\text{equivalently}\qquad
		2\gamma\langle E_U,R\rangle_F+\gamma^2\norm{R}_F^2<0 .
		\label{eq:gleam_corrector_descent}
	\end{equation}
	Therefore, if the residual direction has a component toward the true correction, \(\langle E_U,R\rangle_F<0\), there exists a positive correction strength that improves the transferred basis. If the residual is uninformative, a learned scalar coefficient can shrink toward zero and recover the no-corrector case.
	
	The quality of this corrected basis affects the GLEAM measurements. Let \(B_{m,\star}^{(\ell)}\) denote the exact retained basis for band \(m\) at level \(\ell\), and let \(\widetilde B_{m,+}^{(\ell)}\) be the corresponding corrected basis. The error in the lifted band-energy operator can be bounded by the subspace discrepancy
	\begin{equation}
		\Delta_{m,+}^{(\ell)}
		=
		\widetilde B_{m,+}^{(\ell)}
		\paren{\widetilde B_{m,+}^{(\ell)}}^\top
		-
		B_{m,\star}^{(\ell)}
		\paren{B_{m,\star}^{(\ell)}}^\top .
		\label{eq:gleam_corrector_projector_error}
	\end{equation}
	For any level field \(\bm u^{(\ell)}\),
	\begin{equation}
		\left|
		\frac{1}{2}\norm{\paren{\widetilde B_{m,+}^{(\ell)}}^\top\bm u^{(\ell)}}_2^2
		-
		\frac{1}{2}\norm{\paren{B_{m,\star}^{(\ell)}}^\top\bm u^{(\ell)}}_2^2
		\right|
		\le
		\frac{1}{2}
		\norm{\Delta_{m,+}^{(\ell)}}_2
		\norm{\bm u^{(\ell)}}_2^2 .
		\label{eq:gleam_corrector_energy_bound}
	\end{equation}
	Improving the corrector therefore reduces the bias in the retained band-energy term whenever it reduces the projector error \(\norm{\Delta_{m,+}^{(\ell)}}_2\). The same basis quality also affects the Fiedler-pair term, because pair weights depend on retained-embedding separations such as \(\norm{U_p-U_q}_2^2\). A better corrected embedding gives more reliable graph-separated pair weights and makes the pairwise contrast term act on the intended long-range structures. In this sense, the corrector does not add new spectral information beyond GLEAM's retained subspace; instead, it improves how consistently the low-rank spectral coordinates are transferred across the graph hierarchy.

	\section{Computational Complexity of Spectral Variants}
	\label{app:complexity}
	\label{app:computational-considerations}
	\subsection{Full Complexity Details}
	Let \(G_\ell=(V_\ell,\mathcal{E}_\ell)\) denote the graph at hierarchy level \(\ell\), with \(N_\ell=|V_\ell|\) nodes and \(|\mathcal{E}_\ell|\) edges. We use \(\ell=1\) for the finest graph, so \(G_1=G\), \(N_1=N\), and \(|\mathcal{E}_1|=|\mathcal{E}|\). We count one sparse Laplacian application to a \(C\)-channel field as \(\mathcal{O}(C|\mathcal{E}_\ell|)\). Constants associated with batching, GPU kernels, and feature dimensions of the forecasting backbone are omitted because the comparison here concerns only the auxiliary graph-spectral losses.
	
	\paragraph{Exact Graph BSP.}
	For a single graph, exact Graph BSP requires the eigendecomposition
	\[
	L=\Phi\Lambda\Phi^\top.
	\]
	With a dense eigensolver, the setup cost is \(\mathcal{O}(N^3)\) and storage of the complete eigenbasis is \(\mathcal{O}(N^2)\). Once \(\Phi\) is available, projecting a \(C\)-channel field into graph Fourier coordinates costs
	\[
	\mathcal{O}(N^2C),
	\]
	and band aggregation costs \(\mathcal{O}(NC)\). If exact BSP were applied independently to every level of a hierarchy, the corresponding cost would scale as
	\begin{equation}
		\sum_{\ell=1}^{S}
		\mathcal{O}\!\left(N_\ell^3 + N_\ell^2 C\right),
		\label{eq:appendix_exact_cost}
	\end{equation}
	ignoring the storage cost \(\sum_{\ell=1}^{S}\mathcal{O}(N_\ell^2)\). This is why exact Graph BSP is mainly used as a reference formulation or on fixed graphs where the eigenbasis can be computed once and reused.
	
	\paragraph{Chebyshev BSP.}
	Chebyshev BSP avoids eigendecomposition by replacing \(\Pi_m\) with the polynomial filter \(H_m^{(K,Q)}\) in \eqref{eq:cheb}. Computing the coefficients in \eqref{eq:cheb_coeffs} costs
	\[
	\mathcal{O}(MKQ),
	\]
	which is a one-time setup cost for fixed \(M\), \(K\), \(Q\), and window choice. The quadrature count \(Q\) does not multiply the per-step sparse graph filtering cost after the coefficients have been formed.
	
	A direct implementation that evaluates each of the \(M\) filters independently requires \(K\) sparse Laplacian applications per band and therefore costs
	\begin{equation}
		\mathcal{O}(MKC|\mathcal{E}|)
		\label{eq:appendix_cheb_independent_cost}
	\end{equation}
	per loss evaluation on the finest graph. If the Chebyshev recurrence vectors \(T_j(\tilde L)X\), \(j=0,\dots,K\), are shared across bands, the sparse-multiply part is computed once, giving
	\begin{equation}
		\mathcal{O}(KC|\mathcal{E}| + MKCN),
		\label{eq:appendix_cheb_shared_cost}
	\end{equation}
	where the second term accounts for combining the stored recurrence responses with the band coefficients. In practice, the dominant term depends on graph sparsity, implementation details, and whether recurrence responses are stored or streamed.
	
	For full hierarchical Chebyshev supervision, the same operation is applied on each supervised graph level. If the hierarchy has \(S\) levels and the same \(M\), \(K\), and \(C\) are used at each level, evaluating all bands independently costs
	\begin{equation}
		\sum_{\ell=1}^{S}
		\mathcal{O}\!\left(MKC|\mathcal{E}_\ell|\right).
		\label{eq:appendix_cheb_hier_independent_cost}
	\end{equation}
	If the Chebyshev recurrence responses are shared across bands separately at each level, the corresponding hierarchy-wide cost is
	\begin{equation}
		\sum_{\ell=1}^{S}
		\mathcal{O}\!\left(KC|\mathcal{E}_\ell| + MKCN_\ell\right).
		\label{eq:appendix_cheb_hier_shared_cost}
	\end{equation}
	The coefficient setup remains \(\mathcal{O}(MKQ)\) when the same normalized-Laplacian spectral windows are reused across levels. If each level uses its own window family or level-specific spectral endpoint, the setup cost becomes \(\sum_{\ell=1}^{S}\mathcal{O}(MKQ)\), but this is still a one-time coefficient cost rather than a per-rollout filtering cost. These expressions show why a fully hierarchical Chebyshev BSP loss is more expensive than a single final-level Chebyshev penalty: it avoids eigendecomposition, but repeats degree-\(K\) graph filtering over every supervised graph in the hierarchy.
	
	\paragraph{GLEAM low-rank hierarchy.}
	GLEAM uses a retained spectral embedding of rank \(r\) on selected hierarchy levels. Let \(\mathcal{H}\subseteq\{1,\dots,S\}\) denote the levels on which an eigensolve or cached embedding is used. In the implementation used for the BFS experiments, the low-rank basis is computed with a sparse partial symmetric eigensolver, using an ARPACK/Lanczos \texttt{eigsh}-type solve for only the smallest \(k_\ell=r_\ell+1\) Laplacian eigenpairs rather than a full dense eigendecomposition. The extra eigenpair accounts for the trivial constant mode, which is discarded before forming the nontrivial embedding. Let \(\operatorname{nnz}(L_\ell)\) denote the number of nonzero entries in the level-\(\ell\) Laplacian, with \(\operatorname{nnz}(L_\ell)\approx |\mathcal{E}_\ell|+N_\ell\) for a sparse graph Laplacian. Writing \(n_{\mathrm{cv},\ell}\) for the Lanczos search subspace size and \(q_{\mathrm{eig},\ell}\) for the number of sparse matrix-vector iterations, the approximate setup cost at level \(\ell\) is
	\begin{equation}
		\mathcal{O}\!\left(
		q_{\mathrm{eig},\ell}\operatorname{nnz}(L_\ell)
		+q_{\mathrm{eig},\ell}N_\ell n_{\mathrm{cv},\ell}
		+N_\ell n_{\mathrm{cv},\ell}k_\ell
		\right),
		\label{eq:appendix_gleam_setup_level}
	\end{equation}
	where the first term comes from sparse Laplacian products, the second from Lanczos-basis orthogonalization, and the last from extracting the requested \(k_\ell\)-dimensional invariant subspace. Since \(k_\ell\ll N_\ell\) and \(n_{\mathrm{cv},\ell}\) is usually chosen as a small multiple of \(k_\ell\), this replaces the dense \(\mathcal{O}(N_\ell^3)\) eigensolve and \(\mathcal{O}(N_\ell^2)\) eigenbasis storage by a sparse low-rank setup with memory
	\begin{equation}
		\mathcal{O}\!\left(\operatorname{nnz}(L_\ell)+N_\ell n_{\mathrm{cv},\ell}\right).
		\label{eq:appendix_gleam_sparse_memory}
	\end{equation}
	Across selected levels, the setup cost is
	\begin{equation}
		\sum_{\ell\in\mathcal{H}}
		\mathcal{O}\!\left(
		q_{\mathrm{eig},\ell}\operatorname{nnz}(L_\ell)
		+q_{\mathrm{eig},\ell}N_\ell n_{\mathrm{cv},\ell}
		+N_\ell n_{\mathrm{cv},\ell}k_\ell
		\right).
		\label{eq:appendix_gleam_setup}
	\end{equation}
	When the graph hierarchy is fixed, this setup can be cached and amortized across training. When a coarsest-level embedding is prolongated to finer levels, the sparse eigensolve term is paid only on the selected coarse or anchor levels, while prolongation and smoothing require approximately \(\mathcal{O}(N_\ell r)\) work per transferred level. If a dense fallback is used instead of the sparse partial eigensolver, the corresponding setup reverts to \(\mathcal{O}(N_\ell^3)\) time and \(\mathcal{O}(N_\ell^2)\) memory at that level.
	
	Once the embeddings are available, projecting a \(C\)-channel field onto a rank-\(r\) embedding at level \(\ell\) costs
	\begin{equation}
		\mathcal{O}(N_\ell r C).
		\label{eq:appendix_gleam_projection}
	\end{equation}
	Retained-band energy aggregation costs \(\mathcal{O}(rC)\), which is lower order when \(r\ll N_\ell\). If \(P_\ell\) Fiedler-guided pairs are sampled at level \(\ell\), the pairwise contrast term costs
	\begin{equation}
		\mathcal{O}(P_\ell C).
		\label{eq:appendix_gleam_pair_cost}
	\end{equation}
	The per-step GLEAM cost over supervised levels \(\mathcal{H}_{\mathrm{sup}}\) is therefore
	\begin{equation}
		\sum_{\ell\in\mathcal{H}_{\mathrm{sup}}}
		\mathcal{O}\!\left(N_\ell r C + rC + P_\ell C\right).
		\label{eq:appendix_gleam_online}
	\end{equation}
	For fixed small \(r\) and modest \(P_\ell\), this is lower cost than applying full-band Chebyshev filters at every level.
	
	\paragraph{Practical implication.}
	Neither approximation dominates across all regimes. Full hierarchical Chebyshev BSP applies eigenfree band-power matching across graph levels, but its cost grows with repeated sparse filtering over supervised levels and bands. GLEAM is preferable when multilevel supervision is needed at lower online cost, replacing repeated full-band filtering with cached low-rank projections and sampled pairwise geometry terms.
	
	\PrintCredit
	
	\section*{Data availability}
	
	Data will be made available on request.
	
	\section*{Code availability}
	
	The code implementation is available at \href{https://github.com/ISCLPurdue/Graph_BSP.git}{github.com/ISCLPurdue/Graph\_BSP}.
	
	\begin{coi}
		\section*{Declaration of competing interest}
		The authors declare that they have no known competing financial
		interests or personal relationships that could have appeared to
		influence the work reported in this paper.
	\end{coi}
	
	\begin{ack}
		\section*{Acknowledgements}
		We acknowledge support from ARO cooperative agreement W911NF-25-2-0183 and an ARO ECP award from the Program `Modeling of Complex Systems' (PM - Dr. Rob Martin). This research used resources of the Argonne Leadership Computing Facility, which is a U.S. Department of Energy Office of Science User Facility operated under Contract DE-AC02-06CH11357 and Purdue Rosen Center for Advanced Computing (RCAC) resources. We also acknowledge the helpful guidance of  Dibyajyoti Chakraborty and Hojin Kim for the Backwards Facing Step experiments.
	\end{ack}

	\bibliographystyle{plainnat}
	\bibliography{gleam}
	
\end{document}